\documentclass{article}
\usepackage[preprint]{neurips_2026}

\usepackage[utf8]{inputenc} 
\usepackage[T1]{fontenc}    
\usepackage{hyperref}       
\usepackage{url}            
\usepackage{booktabs}       
\usepackage{amsfonts}       
\usepackage{nicefrac}       
\usepackage{microtype}      
\usepackage{xcolor}         

\usepackage{enumitem}
\usepackage{comment}
\usepackage{placeins}
\usepackage{amsmath}
\usepackage{amsthm}
\usepackage{amssymb}
\usepackage{pgfplots}
\usepackage{pgfplotstable}
\usepackage{graphicx}
\usepackage{graphviz}
\usepackage{subcaption}
\usetikzlibrary{math}
\usetikzlibrary{shapes,arrows}
\usetikzlibrary{arrows.meta,positioning,calc,backgrounds,fit,decorations.pathreplacing}
\usetikzlibrary{backgrounds,fit}
\tikzset{>={Latex[width=2.5mm,length=2.5mm]}}
\usetikzlibrary{positioning}
\tikzstyle{block}=[draw opacity=0.7,line width=1.4cm]
\tikzset{arrow_e/.style = {->,> = latex'}}
\newtheorem{thm}{Theorem}[section]
\newtheorem{prop}{Proposition}[section]
\newtheorem{rem}{Remark}[section]

\newtheorem{lem}{Lemma}[section]
\newtheorem{defn}{Definition}

\newtheorem{assump}{Assumption}

\newcommand{\boxend}{\hfill \ensuremath{\Box}}

\usepackage{algorithm}
\usepackage{algpseudocode}
\usepackage{booktabs}
\usepackage{multirow}

\usepackage{tikz}
\usetikzlibrary{shapes.geometric, arrows.meta, positioning, fit}

\tikzstyle{doc} = [rectangle, draw, text centered, minimum height=1.2em, minimum width=2.5em, fill=white]
\tikzstyle{db} = [cylinder, draw, shape border rotate=90, aspect=0.25, minimum height=2em, minimum width=1.5em, fill=blue!10]
\tikzstyle{llm} = [draw, thick, regular polygon, regular polygon sides=6, minimum size=1.5em, fill=orange!20]
\tikzstyle{arrow} = [thick, ->, >=Stealth]
\tikzstyle{block} = [rectangle, draw, fill=yellow!10, text width=6em, text centered, minimum height=3em]

\usepackage[most]{tcolorbox}

\newtcolorbox{insightbox}{
enhanced,
colback=gray!15,
colframe=white,
boxrule=0pt,
arc=0mm,
left=10pt,
right=10pt,
top=8pt,
bottom=8pt,
borderline west={3pt}{0pt}{blue!60!black}
}

\title{Chain-based Adaptive Reconfiguration Over Lattices for Hallucination Reduction}

\author{
  Joan Vendrell Gallart \\
  Department of Mechanical and Aerospace\\
  University of California Irvine\\
  Irvine, CA 92617-4322, USA \\
  \texttt{jvendrel@uci.edu} \\
  \And
  Solmaz Kia \\
  Department of Mechanical and Aerospace \\
  University of California Irvine\\
  Irvine, CA 92617-4322, USA \\
  \texttt{solmaz@uci.edu} \\
  \AND
  Russell Bent \\
  T-5 Los Alamos National Laboratory \\
  Los Alamos, NM 88220, USA \\
  \texttt{rbent@lanl.gov} \\
  \And
  Michael Grosskopf \\
  CAI-4 Los Alamos National Laboratory \\
  Los Alamos, NM 88220, USA \\
  \texttt{ mikegros@lanl.gov} \\
}

\begin{document}

\maketitle

\begin{abstract}
  We introduce \texttt{CAROL} (\textbf{\underline{C}}hain-based \textbf{\underline{A}}daptive \textbf{\underline{R}}econfiguration \textbf{\underline{O}}ver \textbf{\underline{L}}attices), a probabilistic framework for test-time hallucination reduction in large language models. Rather than relying on token-level uncertainty, \texttt{CAROL} defines a semantic uncertainty measure based on the consistency between generated responses and a trusted context, inducing a string-submodular objective over a lattice of textual sequences. This formulation enables hallucination mitigation to be cast as a Markov chain accept–reject process with provable convergence and near-optimality guarantees, allowing the model to iteratively refine outputs toward semantic consistency. By operating at the level of meaning, \texttt{CAROL} unifies hallucination detection and mitigation within a single framework. Empirical results on question answering and multi-agent reasoning benchmarks show that \texttt{CAROL} significantly reduces hallucinations and improves reliability and interpretability compared to likelihood-based and retrieval-augmented baselines, while maintaining competitive computational efficiency.
\end{abstract}

\section{Introduction}
\label{sec::introduction}

Large Language Models (LLMs) are increasingly deployed within high-stakes decision-making pipelines, spanning clinical reasoning, scientific discovery, and autonomous cyber--physical systems~\cite{FZ-GW-ZH-LS-HC-YW-NL-PY:23}. Despite their impressive fluency, these models exhibit a persistent and consequential failure mode: they often generate responses that are linguistically well-formed yet semantically unsupported by their underlying knowledge, a phenomenon broadly referred to as \emph{hallucination}~\cite{EJ-SH:24}. In agentic and multi-step
reasoning settings, such errors do not remain localized; instead, they propagate and compound across reasoning chains, undermining reliability precisely where correctness is most critical.

Existing approaches to hallucination mitigation largely fall into two paradigms: \emph{prevention} and \emph{detection}. Prevention strategies intervene during generation, either through training-time modifications~\cite{HL-SW-YZ-YD-JL:24, PS-FX:22} or via test-time augmentation mechanisms such as Retrieval-Augmented Generation (\texttt{RAG})~\cite{PL-EP-AP-FP-VK-NG-HK-ML-WY-TR-SR-DK:20, JH-SB-IH-AK:24}, which enrich prompts with external context. While \texttt{RAG} improves factual grounding in many cases, adding more context does not guarantee semantic consistency. Empirically, prompt expansion introduces substantial overheads, including increased latency and cost~\cite{JV-MG-RB:26}, long-context degradation~\cite{YD-MT-SR-SR-SB-AG-AW-RS-EH-HP:25}, and uneven
attention allocation~\cite{NL-KL-JH-AP-MB-FP-PL:23}, and can itself destabilize generation, particularly for smaller or resource-constrained models.

Detection methods instead attempt to flag hallucinations after generation, typically by estimating uncertainty via token-level entropy or related likelihood-based proxies~\cite{AG-FB-MK:24}. However, such measures treat tokens independently and remain fundamentally syntactic: they fail to account for semantic relationships between generated content and trusted knowledge sources (Table~\ref{tab:equivalence}). Recent work on \emph{semantic entropy}~\cite{SF-JK-LK-YG:24} advances this direction by measuring variability across groups of meaning-equivalent tokens (sentences). Yet this improvement comes at a cost:
semantic entropy requires multiple forward passes of the same prompt and conflates consistency with correctness. A systematically wrong model may produce stable yet incorrect outputs under repeated sampling~\cite{MS-YC-YT-AS:23}. Repetition, by itself, does not confer truth.

\begin{table}[t]
\centering
\caption{{\footnotesize Illustration of semantic, syntactic, and lexical equivalence. Foundation models implicitly favor lexical equivalence, which often fails to capture semantic alignment.}}
\label{tab:equivalence}
\footnotesize
\begin{tabular}{p{4cm} p{5cm} c c c}
\toprule
\textbf{Sentence A} & \textbf{Sentence B} &
\textbf{Lexical} & \textbf{Syntactic} & \textbf{Semantic} \\
\midrule
\multirow{3}{*}{%
\begin{tabular}[t]{@{}l@{}}
2028 Olympics are held \\
in Los Angeles.
\end{tabular}}
& 2028 Olympics are NOT held in Los Angeles. & $\approx$ & \checkmark & $\times$ \\
& 2028 Olympics are held in Brisbane. & $\approx$ & \checkmark & $\times$ \\
& LA hosts the 2028 Olympics. & $\times$ & $\times$ & \checkmark \\
\bottomrule
\end{tabular}
\end{table}

Critically, both prevention and detection address semantic correctness only through indirect, heuristic signals. They encourage consistency but do not \emph{enforce} alignment with trusted knowledge, nor do they impose explicit structure on how meaning should accumulate over a generated response. As a result, inference-time guarantees remain elusive. This limitation reflects a deeper issue: hallucinations are fundamentally \emph{sequence-level semantic failures}, yet most mitigation strategies operate
at the level of tokens, likelihoods, or prompt manipulation. What is missing is a principled way to reason about \emph{how meaning evolves over a generated sequence}, and how each newly generated unit contributes, positively or negatively, to semantic consistency with known facts.

\paragraph{Research question and guiding hypothesis.}
Can hallucination be treated not as a statistical artifact of token generation, but as a measurable semantic deviation from trusted knowledge, and can this reframing yield a principled inference-time framework with theoretical guarantees, without modifying the model or requiring repeated sampling? We investigate whether grounding generation in a logically structured axiom set $\Gamma$ and casting hallucination mitigation as mutual information maximization over a vocabulary lattice, novel in this context, admits closed-form sub-optimality bounds and provable convergence at inference time.

A key question underlying this perspective is how to formally define “truth” in generated language. Under Tarskian semantics \cite{IN:04}, truth is not intrinsic but defined relative to a reference model. Consequently, any hallucination mitigation method must rely—explicitly or implicitly—on an external grounding source. Existing approaches incorporate such references indirectly, without enforcing consistency. In contrast, we make this dependence explicit and algorithmically actionable through a trusted axiom set $\Gamma$, treating hallucination as a deviation from a reference model of truth.

Our central hypothesis is that this view motivates a fundamental shift: from \emph{token assembly} to \emph{string assembly}. Rather than appending symbols from a vocabulary $\mathbb{V}$ and evaluating reliability retroactively via distributional proxies, the generative unit is elevated to a semantically coherent $\ell$-gram, whose entailment relationship with $\Gamma$ is enforced \emph{online}. Crucially, $\Gamma$ is not used to augment the prompt, but to define the semantic reference against which each generated unit is evaluated.

This reframing naturally induces a diminishing-returns structure: once a sequence is well-aligned with $\Gamma$, adding redundant or weakly related content yields progressively smaller semantic gains and may even introduce errors. This property is precisely captured by \emph{string submodularity}, a generalization of classical set submodularity that respects generation order and operates over the prefix structure of sequences~\cite{ZZ-EC-AP-WM:13}. Modeling semantic alignment as a string-submodular mutual information objective over the vocabulary lattice yields three key advantages: (i) a theoretically grounded notion of semantic consistency, (ii) tractable inference-time optimization with provable guarantees, and (iii) a unified mechanism in which hallucination detection and mitigation emerge from the same accept--reject signal.

Building on this structure, we introduce \texttt{CAROL}
(\textbf{\underline{C}}hain-based
\textbf{\underline{A}}daptive
\textbf{\underline{R}}econfiguration
\textbf{\underline{O}}ver
\textbf{\underline{L}}attices), a probabilistic test-time refinement algorithm that models $\ell$-gram generation as a Markov chain over the vocabulary lattice, guided by a string-submodular semantic objective. \texttt{CAROL} requires neither repeated sampling nor access to internal model probabilities, enabling black-box deployment with convergence and near-optimality guarantees. Our main contributions are:
\begin{itemize}
\item \textbf{(C1) One-shot semantic consistency metric.} We introduce an inference-time hallucination score computable in a single forward pass, based on exemplar clustering of $\ell$-grams relative to $\Gamma$~\cite{AK-DG:14}. On \textbf{FEVER}~\cite{JT-AV-CC-AM:18}, it achieves $59.5\%$ detection accuracy, outperforming token-level entropy by $10$ points and matching \textsf{HalluciNot}~\cite{BP-AL-PJ-PA:25} at $\nicefrac{1}{6}$ the GPU cost.

\item \textbf{(C2) String-submodular objective with near-optimality guarantees.}
We show that the axiom-grounded consistency score induces a string-submodular mutual information function $\mathrm{I}(\mathsf{S};\Gamma)$ over the vocabulary lattice $\mathcal{L}(\mathbb{V}^*)$, yielding closed-form sub-optimality bounds.

\item \textbf{(C3) \texttt{CAROL} algorithm with provable convergence.} Using probabilistic submodular sampling~\cite{AG-HH-AK:15} and path coupling~\cite{RB-MD:97}, we bound the mixing time of \texttt{CAROL} to sub-$\mathcal{O}(n \log n)$, establishing the first provable convergence result for inference-time hallucination control in black-box LLMs. Across \textbf{TruthfulQA}~\cite{SL-JH-OE:21}, \textbf{HaluEval}~\cite{JL-XC-WZ-JN-JW:23}, and \textbf{HotPotQA}~\cite{ZY-PQ-SZ-YB-WC-RS-CM:18}, \texttt{CAROL} reduces hallucination by up to $40\%$ relative to \texttt{RAG}, at comparable latency and lower API cost.
\end{itemize}

\paragraph{Paper outline.}
Section~\ref{sec::preliminaries} reviews background on string submodularity. Section~\ref{sec::statement} formalizes the problem setting. Section~\ref{sec::method} presents \texttt{CAROL} and Section~\ref{sec::analysis} its theoretical properties. Section~\ref{sec::numerical} reports empirical results, and Section~\ref{sec::conclusion} concludes with future directions.

\section{Preliminaries}
\label{sec::preliminaries}

This section introduces the discrete optimization structures that underlies our framework. All concepts are independent of language models and hallucination and serve solely to formalize the search space and objective class instantiated in Section~\ref{sec::statement}.

Let $\mathbb{V}$ denote a ground set of atomic elements (e.g., tokens, phrases, or semantic units) and let $\mathbb{V}^*$ denote the set of all finite sequences over $\mathbb{V}$. A sequence $\mathsf{S} = (s_1,\ldots,s_\ell)\in\mathbb{V}^*$ has length
$|\mathsf{S}|=\ell$. For two sequences $\mathsf{S}^1,\mathsf{S}^2\in\mathbb{V}^*$, we denote concatenation by $\mathsf{S}^1\oplus\mathsf{S}^2$ and write $\mathsf{S}^1\preceq\mathsf{S}^2$ if $\mathsf{S}^1$ is a prefix of $\mathsf{S}^2$. Let us define a group of $\kappa$ sequences as $\mathbb{S}=(\mathsf{S}^1,\cdots,\mathsf{S}^\kappa)$. 

In many sequential decision problems, not all sequences in $\mathbb{V}^*$ are admissible. Feasibility often depends on the entire prefix history, encoding constraints such as coherence or structural consistency. We model such path-dependent constraints by a family $\mathcal{I}\subseteq\mathbb{V}^*$ of \emph{feasible} sequences. A convenient abstraction for prefix-closed feasibility systems is provided by greedoids.

\begin{defn}[Greedoid~\cite{BK-LL:81}]
A pair $(\mathbb{V},\mathcal{I})$ is a \emph{greedoid} if: \textbf{(G1)} (\emph{Non-emptiness}) $\emptyset \in \mathcal{I}$; \textbf{(G2)} (\emph{Accessibility}) for any nonempty $\mathsf{S}\in\mathcal{I}$, there exists $s\in\mathsf{S}$ such that $\mathsf{S}\ominus\{s\}\in\mathcal{I}$;
  \textbf{(G3)} (\emph{Exchange}) for any $\mathsf{S}_1,\mathsf{S}_2\in
    \mathcal{I}$ with $|\mathsf{S}_1|<|\mathsf{S}_2|$, there exists
    $s\in\mathsf{S}_2\ominus\mathsf{S}_1$ such that
    $\mathsf{S}_1\oplus s\in\mathcal{I}$.
\boxend
\end{defn}

\begin{defn}[Basis of a Greedoid]
\label{def:basis}
A \emph{basis} is a maximal feasible sequence $\mathsf{S}\in\mathcal{I}$ such that no extension $\mathsf{S}\oplus s$ with $s\in\mathbb{V}$ remains feasible. The set of all bases is denoted $\mathcal{B}(\mathcal{I})$.
\boxend
\end{defn}

\paragraph{String submodularity.}
Let $f:\mathbb{V}^*\to\mathbb{R}_{\ge 0}$ be an objective defined over sequences. The function $f$ is \emph{string submodular} if for any $\mathsf{S}^1\preceq\mathsf{S}^2$ and any $s\in\mathbb{V}$ not appearing in $\mathsf{S}^2$,
\begin{equation}
f(\mathsf{S}^1\oplus s)-f(\mathsf{S}^1)
\;\ge\;
f(\mathsf{S}^2\oplus s)-f(\mathsf{S}^2).
\end{equation}
The function is \emph{normalized} if $f(\emptyset)=0$ and
\emph{prefix-monotone} if $f(\mathsf{S}^2)\ge f(\mathsf{S}^1)$ whenever $\mathsf{S}^1\preceq\mathsf{S}^2$.

String submodularity captures diminishing returns over ordered prefixes and generalizes classical set submodularity to sequential domains.

\section{Problem Statement}
\label{sec::statement}

We now specialize the abstract structures of Section~\ref{sec::preliminaries} to the problem of inference-time hallucination reduction. Let $\mathsf{q}$ be a query forming a prompt $\rho$ and let $\Gamma$ be a set of trusted contextual
statements. Given a language model $F(\rho;\theta)$ pretrained over a dataset $\mathcal{D}$, with no assumptions on $\Gamma \cap \mathcal{D}$, the generated response is a finite sequence of $\ell-$grams $\mathbb{S}=(\mathsf{S}^i)_{i=1}^\kappa$ such that $\mathsf{S}^i=(s^i_1,\ldots,s^i_\ell)\in\mathbb{V}^*$.

\begin{assump}[Reliable Context]
\label{assump:context}
The context $\Gamma$ consists of trusted statements, modeled as a set of axioms $\mathcal{A}$ that serve as the reference for semantic evaluation (see Appendix~\ref{app:axioms}).
\end{assump}

\begin{assump}[Black-box Generation]
\label{assump:no_information}
The language model $F(\rho;\theta)$ is treated as a black box: only the generated outputs $\mathsf{S}$ are observable, and no internal probabilities are accessible.
\end{assump}

Hallucination is assessed at the level of complete, semantically coherent responses rather than individual tokens. Accordingly, not all sequences in $\mathbb{V}^*$ are considered valid candidates. We define a feasibility family $\mathcal{I}\subseteq\mathbb{V}^*$ encoding path-dependent semantic coherence constraints over partial generations and assume $(\mathbb{V},\mathcal{I})$ forms a greedoid, see Proposition~\ref{prop:llm_greedoid}. Each feasible prefix represents a partially coherent response, while bases $\mathcal{B}(\mathcal{I})$ correspond to maximally extended responses whose semantic content can be evaluated holistically. Our goal is to select, among all admissible complete responses, one that is most semantically consistent with the trusted context:
\begin{equation}
\label{eqn:general_main_problem}
\min_{\mathbb{S}\in\mathcal{B}(\mathcal{I})}\ \mathbb{H}(\mathbb{S}\mid\Gamma),
\end{equation}
where $\mathbb{H}(\mathbb{S}\mid\Gamma)$ quantifies semantic inconsistency between the response and the axioms in $\Gamma$. Under Assumptions~\ref{assump:context}--\ref{assump:no_information},
$\mathbb{H}$ is computable in a single forward pass and does not rely on repeated sampling or likelihood information. The construction of $\mathbb{H}$ proceeds independently of $\mathcal{I}$ and serves purely as the optimization objective.

\paragraph{Step 1: Entailment as semantic alignment.}
We adopt entailment as a proxy measure for semantic alignment between generated responses and the trusted context. Following Tarski’s model-theoretic notion of truth, a statement $\varphi$ is entailed by $\Gamma$ if no interpretation makes every element of $\Gamma$ true while $\varphi$ false, denoted $\Gamma\vDash_{\mathcal{FS}}\varphi$. We define a continuous relaxation in embedding space by representing statements as embeddings $\varphi=\phi(\mathsf{S})$, and computing alignment as
\begin{equation}
\mathcal{E}(\mathsf{S}^1,\mathsf{S}^2)=
\frac{\langle \varphi^1,\varphi^2\rangle}{\|\varphi^2\|},
\label{eqn::entailment}
\end{equation}
which serves as an asymmetric similarity score capturing directionality rather than exact logical implication.

\paragraph{Step 2: Exemplar clustering and semantic entropy.}
A response is considered non-hallucinated if its semantic units are mutually non-redundant and consistently aligned with $\Gamma$. To operationalize this, we cluster the context $\Gamma$ around the generated $\ell$-grams using the entailment operator via exemplar-based clustering~\cite{ID-ML-AG-PW-JV-SR:22}. Each $x\in\Gamma$ is assigned to the nearest $\ell$-gram:
\begin{equation}
\label{eqn:clustering}
d(\mathbb{S},x)=\min_{\mathsf{S}^i\in\mathbb{S}}\mathcal{E}(\mathsf{S}^i,x),
\end{equation}
with cluster assignment probabilities
\begin{equation}
\label{eqn:cluster_prob}
p(x|\mathsf{S}^i)=
\frac{\exp(-\beta\,\mathcal{E}(\mathsf{S}^i,x))}
{\sum_{i\in c}\exp(-\beta\,\mathcal{E}(\mathsf{S}^i,x))},
\end{equation}
where $c=\{1,\cdots,\kappa\}$ is the amount of generated $\ell-$grams (medoids) and $\beta$ is a distribution hyperparameter. The resulting partition induces a semantic entropy over the response, where low entropy signals concentration around the trusted context and high entropy signals divergence.

\subsection{Hallucination reduction as string-submodular maximization}

Notice that by the definition of $\mathbb{H}$ as the deviation of $\mathbb{S}$ with respect to $\Gamma$ (Step 1), to minimize~\eqref{eqn:general_main_problem} is equivalent to maximizing semantic alignment with $\Gamma$. We therefore reformulate the problem as
\begin{equation}
\label{eqn:main_problem}
\max_{\mathbb{S}\in\mathcal{I}} \ \mathsf{I}(\mathbb{S};\Gamma)
\quad \text{s.t.} \quad |\mathsf{S}|\le\ell,~\forall\mathsf{S}\in\mathbb{S}
\end{equation}
where $\mathsf{I}(\mathbb{S};\Gamma)=\mathsf{SE}(\Gamma)-\mathsf{SE}(\Gamma|\mathbb{S})$ denotes the clustering-induced surrogate mutual information between the response and the axiom set, $\mathsf{SE}(\Gamma)=\sum_{x\in\Gamma}p(x)\log p(x)$ and $\mathsf{SE}(\Gamma,\mathbb{S}) =  -\sum_{c\in\mathcal{C}}\big[\sum_{x\in\Gamma|x\in c}p(x|\mathsf{S}^c)\big]\log \big[\sum_{x\in\Gamma|x\in c}p(x|\mathsf{S}^c)\big]$ stands for the semantic entropy of the clustering of $\Gamma$ over $\mathbb{S}$. With certain abuse of notation we use $x\in c$ to describe element $x$ in the cluster with medoid $c$. 

As the response grows, each additional semantic unit tends to overlap with existing content or restate information already aligned with $\Gamma$, yielding diminishing marginal semantic gain. In Appendix~\ref{app:analysis}, we show that this induces string submodularity of $\mathsf{I}(\cdot;\Gamma)$ over the feasible sequence family. This structure penalizes redundancy and enables principled, near-optimal inference-time optimization aligned with the sequential nature of LLM generation.

\section{Chain-based Adaptive Reconfiguration Over Lattices Algorithm}
\label{sec::method}

As discussed in Section~\ref{sec::introduction}, we define reliability through the coverage that a generated response $\mathsf{S}$ achieves over a set of foundational axioms $\Gamma$. Leveraging the submodularity of~\eqref{eqn:main_problem}, see Appendix~\ref{app:analysis}, we build upon probabilistic submodular models~\cite{AG-HH-AK:15} to introduce a hallucination-aware rejection mechanism based on a sequential greedy scheme (Algorithm~\ref{alg:algorithm}). Our approach conceptualizes the generation process as a search over a combinatorial lattice $\mathcal{L}(\mathbb{V})$ (Appendix~\ref{app:lattice}). In this space, a large language model is abstracted as a policy $\pi_\theta$ that induces transitions between states by appending tokens $v \in \mathbb{V}$. This naturally aligns with the standard Markov Decision Process (MDP) formulation $\mathcal{M} = (\mathcal{S}, \mathcal{A}, \pi_\theta, R)$, where states correspond to partial sequences, actions to vocabulary elements, and $\pi_\theta$ defines the next-token distribution.

Existing approaches optimize $\pi_\theta$ directly, evaluating generation quality at each step. In contrast, we decouple search from optimization. Specifically, we reinterpret $\pi_\theta$ as an \emph{auxiliary proposal distribution} $q$, whose role is to explore the lattice $\mathcal{L}(\mathbb{V})$. Starting from $\mathcal{S}_0 = \mathsf{q}$, the proposal iteratively expands the state until a semantically meaningful unit $\mathsf{S}^i = \mathcal{S}_t \ominus \mathcal{S}_{t-1}$ is formed (lines 5--8). Each candidate unit $\mathsf{S}^i$ is then evaluated by clustering the context $\Gamma$ against it (Algorithm~\ref{alg:clustering}), enabling the computation of its marginal contribution in terms of mutual information (line 9). The current state $\mathcal{S}_t$ thus accumulates the query $\mathsf{q}$ and all accepted semantic units up to iteration $t$. The resulting marginal gain defines a probabilistic submodular objective, which induces a target distribution $p$ (line 10). This distribution governs an accept--reject step (lines 11--16), determining whether the candidate $\mathsf{S}^i$ is incorporated into the state, see Remark~\ref{rem:gibbs}. In this way, generation is guided not by likelihood alone, but by incremental semantic consistency with respect to $\Gamma$, effectively maximizing~\eqref{eqn:main_problem}.

\begin{algorithm}[t]
    \caption{{\small Chain-based Adaptive Reconfiguration Over Lattices}}
    \label{alg:algorithm}{\footnotesize
    \begin{algorithmic}[1]
        \State \textbf{Input:} Ground set vocabulary $\mathbb{V}$, submodular distribution $p(\cdot)\propto\exp(\beta F(\cdot)))$, LLM $q(\cdot)$, maximum allowed iterations $T_{\max}$, initial query $\mathsf{q}$ and context $\Gamma$.
        \State $\mathcal{S}_0\leftarrow\mathsf{q}$
        \For{$t=1$\textbf{  to  }$\min(T_{\max},\tau)$}
        \State $\mathsf{S}^t \leftarrow \mathsf{S}^{t-1}$
        \While{$|\mathsf{S}^t|\leq\ell$}\Comment{Get a meaningful response}
           \State $v\leftarrow q(\mathbb{V},\mathsf{S}^t)$
           \State $\mathsf{S}^t \leftarrow \mathsf{S}^t\oplus v$
        \EndWhile
        \State $\Delta_F(\mathsf{S}^t|\mathcal{S}_t,\Gamma)\leftarrow I(\mathcal{S}_t\oplus\mathsf{S}^t;\Gamma,\mathcal{S}_t):=\mathsf{SE}(\Gamma|\mathcal{S}_t) - \mathsf{SE}(\Gamma|\mathcal{S}_t\oplus\mathsf{S}^t)$ \Comment{Mutual Information~\eqref{eqn:main_problem}}
        \State $p_{add}\leftarrow \exp(\beta\Delta_F(v|\mathcal{S}_t))/(1+\exp(\Delta_F(v|\mathcal{S}_t)))$
        \State $z\leftarrow\text{Uniform}([0,1])$
        \If{$z\leq p_{add}$} \Comment{Accept or Reject}
        \State $\mathcal{S}_{t+1}\leftarrow\mathcal{S}_t\oplus\{\mathsf{S}^t\}$
        \Else
        \State $\mathcal{S}_{t+1}\leftarrow\mathcal{S}_t\ominus\{\mathsf{S}^t\}$
        \EndIf
        \State Update $q(\cdot)$ \Comment{Fine-tunning or Chain-of-Thoughts}
        \EndFor
        \State \textbf{Return:} response $\mathbb{S} = \mathcal{S}_t\ominus\mathsf{q}$
    \end{algorithmic}
    }
\end{algorithm}

After \texttt{CAROL} accepts or rejects the generated string, the pipeline either continuously to generate a new string from the current state $\mathcal{S}_t$ or step-backs into previous step $\mathcal{S}_{t-1}$ and generates again from there \cite{JV-AK-SK:25}. Under this case, in order to avoid leading to the same response, the process must be tuned, see Figure \ref{fig:agentic_pipeline}. This update step can be performed in two distinct directions based on the MDP formulation: First, supposing access to the model, a fine-tuning operation can be performed to update the model policy, see Appendix~\ref{app:fine_tuning}. Alternatively, under Assumption~\ref{assump:no_information}, if there is no access to model's policy modification or the computational resources available suppose a limiting constraint, a Chain-of-Thoughts reasoning \cite{SH-YG-HL-TL-XS-XW-SX-HM-AS-QG-ZW-ZH:24} can be used to update the state $\mathcal{S}_t$ into a new state $\widetilde{\mathcal{S}}_t$ in order to shift models response towards optimality (e.g. $\mathcal{S}_t = \{$ \textit{"Which is California's capital?", "Los Angeles is the capital of California"}$\}$ and $\widetilde{\mathcal{S}}_t = \{$ \textit{"Which is California's capital?", "Los Angeles is the capital of California", "Revise the previous statement as it is not contrasted."$\}$}).

\begin{figure}[t] 
    \centering
    \begin{tikzpicture}[scale=0.3,
                        >=Latex,
                        font=\scriptsize,
                        box/.style={rectangle, rounded corners=6pt, draw, very thick, align=center, inner sep=6pt, minimum width=28mm, minimum height=8mm, fill=#1!12},
                        dbox/.style={rectangle, rounded corners=6pt, draw, very thick, align=center, inner sep=6pt, minimum width=40mm, minimum height=14mm, fill=#1!12},
                        page/.style={rectangle, draw, very thick, minimum width=18mm, minimum height=12mm, fill=red!5},
                        tinybox/.style={rectangle, draw, very thick, minimum width=14mm, minimum height=8mm, align=center, fill=#1!10},
                        carol/.style={rectangle, draw, very thick, minimum width=8mm, minimum height=18mm, fill=black!80},
                        arrow/.style={-Latex, very thick},
                        thinarrow/.style={-Latex, thick},
                        note/.style={font=\footnotesize, inner sep=1pt, fill=white, align=center}
                    ]
                    
        \node[box=gray] (docs) at (-6, 1.9)
        {\includegraphics[width=6mm]{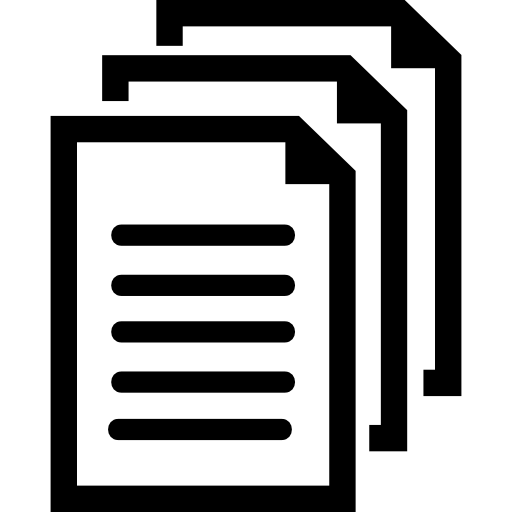}\hspace{2mm} \raisebox{0.5\height}{\begin{tabular}{@{}l@{}}
        Task-specific\\
        documents
        \end{tabular}}};
        
        \node[box=gray, minimum height=4mm] (query) [below=2mm of docs]
          {\includegraphics[width=5.5mm]{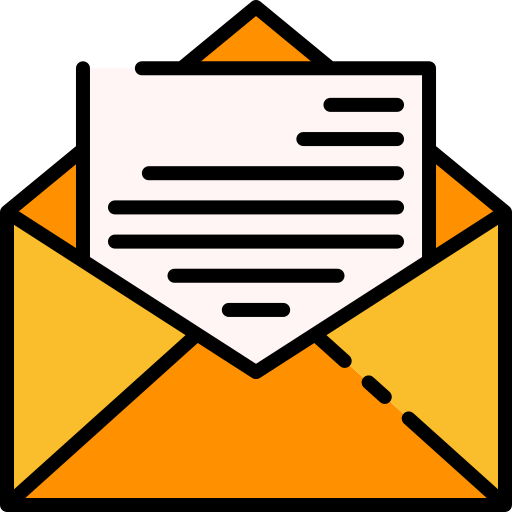} \hspace{2mm} \raisebox{0.7\height}{\begin{tabular}{@{}l@{}}
        Query $\mathsf{q}$
        \end{tabular}}};
        
        
        \coordinate (encjoin) at ($(docs.east)!0.5!(query.east) + (1.0cm,0)$);
        \node[circle, fill=black, inner sep=0.8pt] (encdot) at (encjoin) {};
        
        \node[box=green, minimum width=20mm] (embed) [right=15mm of encjoin] {Embedding\\model};
        
        \draw[arrow] (docs.east) -- (encdot);
        \draw[arrow] (query.east) -- (encdot);
        
        \draw[arrow] (encdot) -- node[midway, above=2pt, note]{encode} (embed.west);
        
        \node[dbox=cyan,minimum width=20mm] (vdb) [right=22mm of embed]
          {\includegraphics[width=7mm]{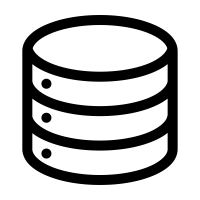}\\[-1mm]Vector\\database};
        \draw[arrow] (embed) -- node[above=1pt, note] {$\phi(\rho)$} (vdb);
        
        \node[dbox=yellow, minimum width=30mm, minimum height=18mm] (bundle) [below=4mm of vdb]
          {Query $\mathsf{q}$\\[2mm] ---\\ \includegraphics[width=7mm]{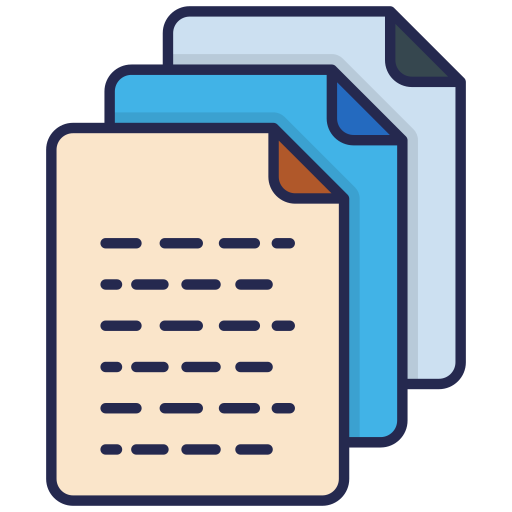}\hspace{2mm}\includegraphics[width=7mm]{Figures/model_representation/letter.png}\\[1mm]
           Context $\Gamma$};
        \draw[arrow] (vdb.south) -- node[right=0mm, note, xshift=-25mm] {Similarity search} (bundle.north);
        
        \node[box=white, minimum width=34mm, minimum height=16mm, draw=red!65, fill=red!7] (llm) [left=15mm of bundle]
          {\includegraphics[width=9mm]{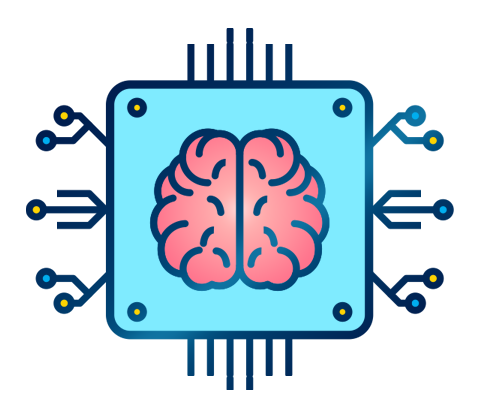}\\[-1mm]Large Language\\Model};
        
        \draw[arrow] (bundle.west) -- node[above=1pt, note] {$\mathsf{q}$} (llm.east);
        
        \node[tinybox=gray] (lgrama) [left=6mm of llm] {$\ell$-gram};
        \node[carol] (carol) [left=6mm of lgrama] {};
        \node[page]  (resp)  [left=6mm of carol] {response $\mathsf{S}$};
        
        \draw[arrow] (llm.west) -- (lgrama.east);
        \draw[arrow] (lgrama.west) -- (carol.east);
        \draw[arrow] (carol.west) -- (resp.east);
        
        \node[font=\scriptsize\bfseries, text=white, rotate=90] at (carol) {\texttt{CAROL}};
        
        \coordinate (tap) at ($(llm.north)+(0,10mm)$);
        \coordinate (midup) at ($(resp.north)!0!(llm.north)+(0,15mm)$);
        
        \draw[thinarrow] (resp.north) -- (midup)
    node[above=2pt, note]{update} -- (tap) -- (llm.north);

        \coordinate (tap2) at ($(llm.north)+(-100mm,-60mm)$);  
        \coordinate (tap3) at ($(llm.north)+(-100mm,-42mm)$); 
        \coordinate (bundleTop) at ($(bundle.west)+(0,-30mm)$);
        \draw[thinarrow](bundleTop) -- node[midway, below=1.8pt, note]{Context} (tap2) -- (tap3);
    \end{tikzpicture}
    \caption{\small Illustration of the proposed pipeline. The system retrieves additional context to generate factual information and encodes it. Then it performs a similarity search to generate a set of context axioms $\Gamma$ which is used later by \texttt{CAROL} to accept or reject the generated response, providing updating feedback to the model.}
    \label{fig:agentic_pipeline}
\end{figure}

As shown in Section~\ref{sec::numerical}, Algorithm~\ref{alg:algorithm} is sensitive to the sampling temperature $\beta$, or equivalently, to the acceptance threshold $z$. This parameter controls a fundamental trade-off between reliability and generative flexibility, as discussed in Appendix~\ref{app:limitations_broader_impact}. Lower values enforce stricter adherence to the axioms $\Gamma$, favoring conservative and highly consistent outputs, while higher values allow greater exploration, enabling the generation of novel but potentially less grounded content. Equally important is the role of the context $\Gamma$ and the clustering mechanism, see Appendix~\ref{app:axioms}. The performance of \texttt{CAROL} is inherently tied to the quality and coverage of the underlying axioms: the method constrains generation to remain within the semantic support induced by $\Gamma$. Consequently, richer and more reliable context sets directly expand the space of valid, non-hallucinated responses. In this sense, \texttt{CAROL} does not increase nor decrease model capacity, but rather reallocates it toward faithfully exploiting the available trusted knowledge.

\section{Theoretical Analysis}
\label{sec::analysis}

As discussed in Section~\ref{sec::introduction}, the strength of the proposed framework lies in the submodular structure of the objective, which enables closed-form sub-optimality guarantees. The generation process induces a Markov chain over the solution space, where the LLM serves as a proposal distribution $q(\cdot)$ and the mutual information defines an acceptance--rejection distribution $p(\cdot)$. This formulation allows us to characterize the algorithm through mixing-time analysis (Definition~\ref{defn::mixing_times}), with the stationary distribution concentrating on high-value responses $\mathbb{S}^\star$ for a given query $\mathsf{q}$ over the vocabulary space.

\begin{defn}
    \label{defn::mixing_times}
    The mixing time $t_\text{mix}(\epsilon)$ of a Markov Chain is the maximum time until the Markov Chain is $\epsilon$-close to its steady state distribution such that
    $$t_\text{mix}(\epsilon) = \min\big\{ t\geq 0: \max_{x\in\mathcal{S}}\big[\max_{\mathcal{A}\subseteq\mathcal{S}}|\text{Pr}(X_t\in\mathcal{A}|X_0=x) - \pi(\mathcal{A})| \big] \leq \epsilon \big\}.~~~~~~~~~~~~~~~~~~~~~~~~~~~~~~~~~~~~~~~~\boxend$$
\end{defn}

Let us denote the marginal gain on the mutual information of adding a string $\mathsf{S}^i$ to state $\mathcal{S}_t$ as
$$\Delta_F(\mathsf{S}^t|\mathcal{S}_t,\Gamma) :=  I(\mathcal{S}_t\oplus\mathsf{S}^t;\Gamma,\mathcal{S}_t)=\mathsf{SE}(\Gamma|\mathcal{S}_t) - \mathsf{SE}(\Gamma|\mathcal{S}_t\oplus\mathsf{S}^t), $$
and the following curvature quantity,
$$\gamma_{\mathsf{I},\beta} := \max_{\mathcal{S}_t\in\mathcal{L},\mathsf{S}^j\in\mathbb{V}^*} \sum_{\mathsf{S}^i\in\mathbb{V}^*} \tanh{\Big( \frac{\beta}{2}\big|\Delta_F(\mathsf{S}^i|\mathcal{S}_t,\Gamma) - \Delta_F(\mathsf{S}^i|\mathcal{S}_t\oplus\mathsf{S}^j,\Gamma) \big| \Big)}$$
representing the maximum possible mutual information among all new possible states. Then, the following Theorem~\ref{thm:fast_mixing} bounds the mixing time for the proposed Algorithm~\ref{alg:algorithm}.

\begin{thm}
  \label{thm:fast_mixing}
    For any function $\mathsf{I}:\mathbb{V}^*\rightarrow\mathbb{R}$, if $\gamma_{\mathsf{I},\beta} < 1$ for any iteration $t$, then the mixing time of Algorithm~\ref{alg:algorithm} is bounded as
    $$t_{mix}(\epsilon)\leq \frac{1}{q_{min}(\mathsf{S}^\ell) - q_{max}(\mathsf{S}^\ell) \bar{\gamma}_{\mathsf{I},\beta}}n(\log{(n)}+\log{(\frac{1}{\epsilon})}) = \tau,$$
    where $\bar{\gamma}_{\mathsf{I},\beta} = \max_t \gamma_{\mathsf{I},\beta}$, $n$ is the number of possible tokens in $\mathbb{V}$ and $q_{min}(\mathsf{S}^\ell)$ and $q_{max}(\mathsf{S}^\ell)$ are the probability of the auxiliary distribution $q$ to sample the least and most probable $\ell$-gram as a response.
\end{thm}

\begin{proof}
    This proof is conducted by \emph{path coupling}. See Appendix~\ref{app:analysis} for the full development. 
\end{proof}

Note that Theorem~\ref{thm:fast_mixing} establishes a mixing time faster than $\mathcal{O}(n\log n)$. The hallucination level of a generated response is quantified via the semantic entropy computed from the final exemplar-based clustering solution (cf. Section~\ref{sec::statement}). This provides not only a sub-optimality measure, but also a principled robustness metric indicating the degree to which a response can be trusted. It is also worth to comment that uncertainty on $\Gamma$ consistency does not break Theorem~\ref{thm:fast_mixing}, but acts like a perturbation term on the semantic objective, see Appendix~\ref{app:analysis}.

\subsection{Clustering Analysis}

One of the key novelties of this work is the use of an exemplar-based clustering method. Existing approaches to semantic entropy rely on classical clustering techniques, which are highly sensitive to the choice of clustering \emph{temperature}, often producing either overly coarse or overly fragmented partitions (see Figure~\ref{fig:temperature}). As shown in \cite{LK-YG-SF:23}, the resulting partition critically influences the computed semantic entropy. This sensitivity is further amplified when $\ell$-gram probabilities are derived from cluster sizes \cite{SF-JK-LK-YG:24}, where an increasing number of clusters can artificially inflate entropy, introducing a strong dependence on data volume for stability.

To overcome these limitations, we use the exemplar-based clustering procedure in Algorithm~\ref{alg:clustering}, which treats generated $\ell$-grams as medoids and therefore avoids temperature tuning. Cluster density, rather than cluster count, drives the semantic entropy: dense assignments indicate agreement with $\Gamma$ and yield low entropy, while sparse or diffuse partitions signal disagreement and higher hallucination risk, see Appendix~\ref{app:results}. Moreover, the submodularity-aware selection, as shown in Proposition~\ref{prop:greedy+clustering}, gives a near-optimal coverage guarantee, discouraging redundant medoids and promoting sparse representatives with broader coverage of $\Gamma$.

\begin{prop}[Sequential Greedy Algorithm on \eqref{eqn:main_problem}]
    \label{prop:greedy+clustering}
    Sequential Greedy Algorithm will obtain a near-optimal exemplar-based clustering such that, being $f(\cdot)$ the overall pairwise distance centroid-item and $\mathcal{S}_c$ the formed clustering over the space $(\mathbb{S},\Gamma)$, then  $f(\mathcal{S}_c)\geq(1-\text{e}^{-1}) f(\mathcal{S}_c^\star)$.
\end{prop}
\begin{proof}
    See Appendix~\ref{app:analysis} for the full development. 
\end{proof}

\begin{figure}[t]
    \centering
    \begin{minipage}{0.35\linewidth}
        \centering
        \includegraphics[width=\linewidth]{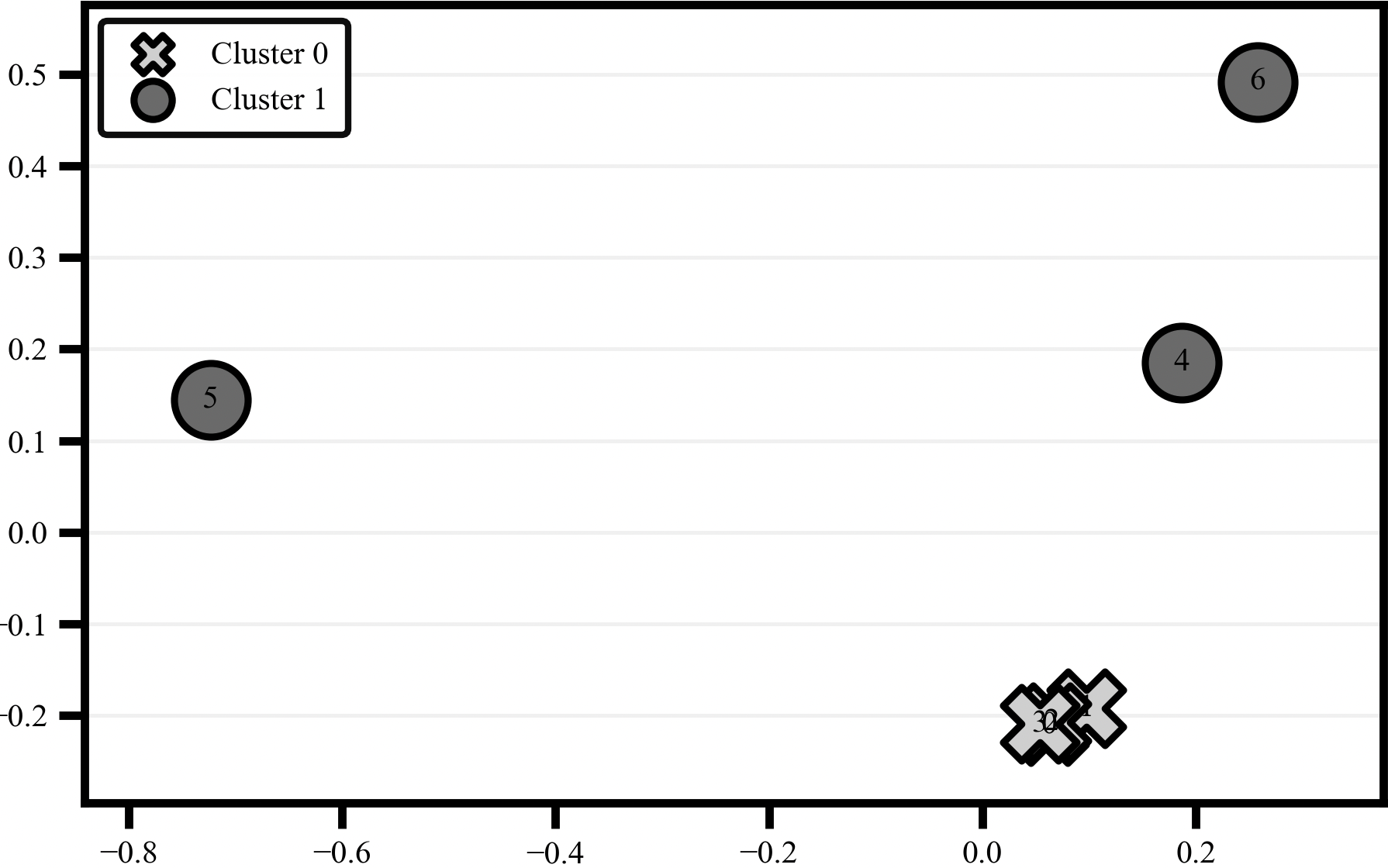}\\
        {\small (a) Clustering Temperature $0.2$}
    \end{minipage}\hfill
    \begin{minipage}{0.35\linewidth}
        \centering
        \includegraphics[width=\linewidth]{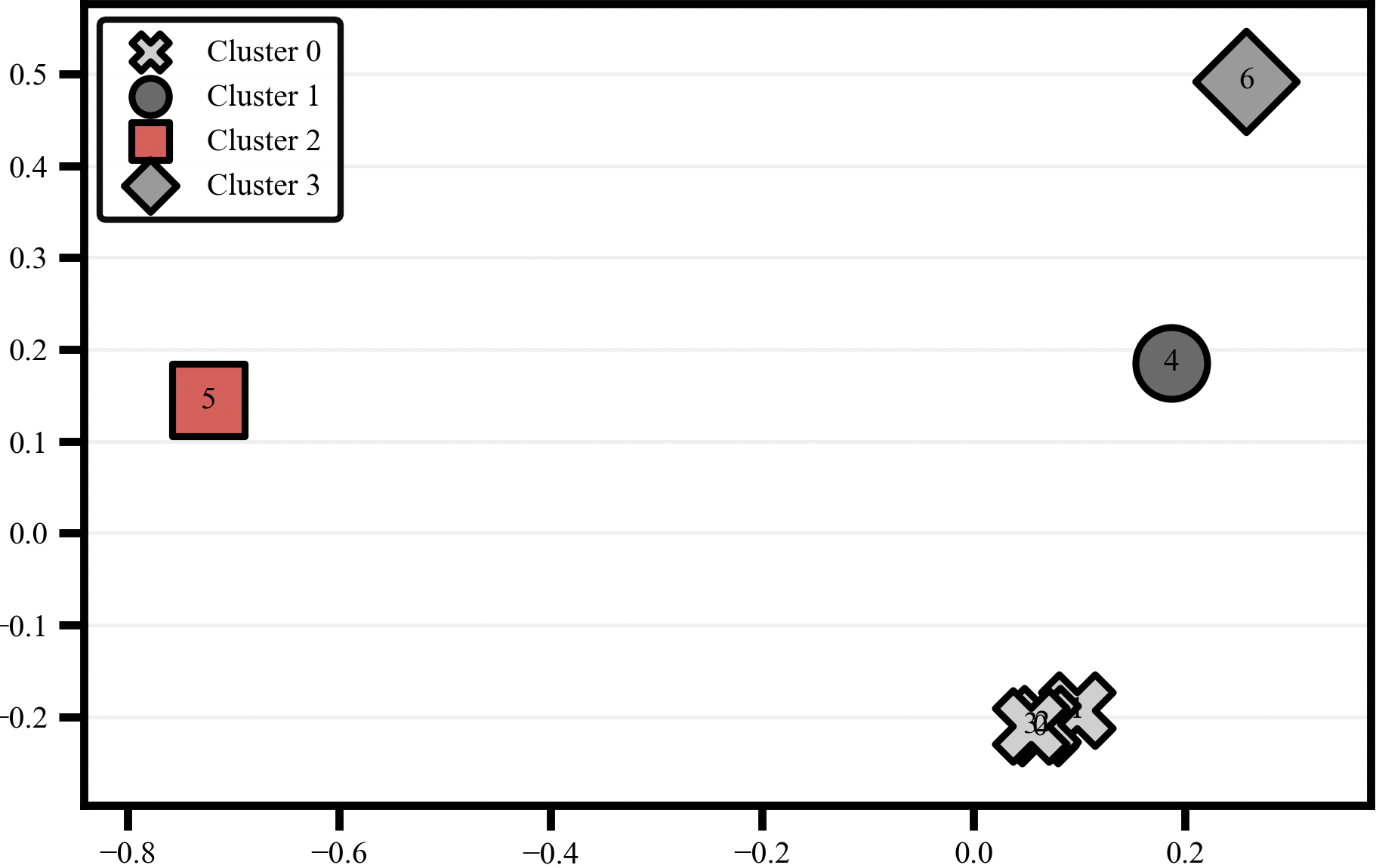}\\
        {\small (b) Clustering Temperature $0.6$}
    \end{minipage}\hfill
    \begin{minipage}{0.28\linewidth}
        \centering
        \includegraphics[width=0.93\linewidth]{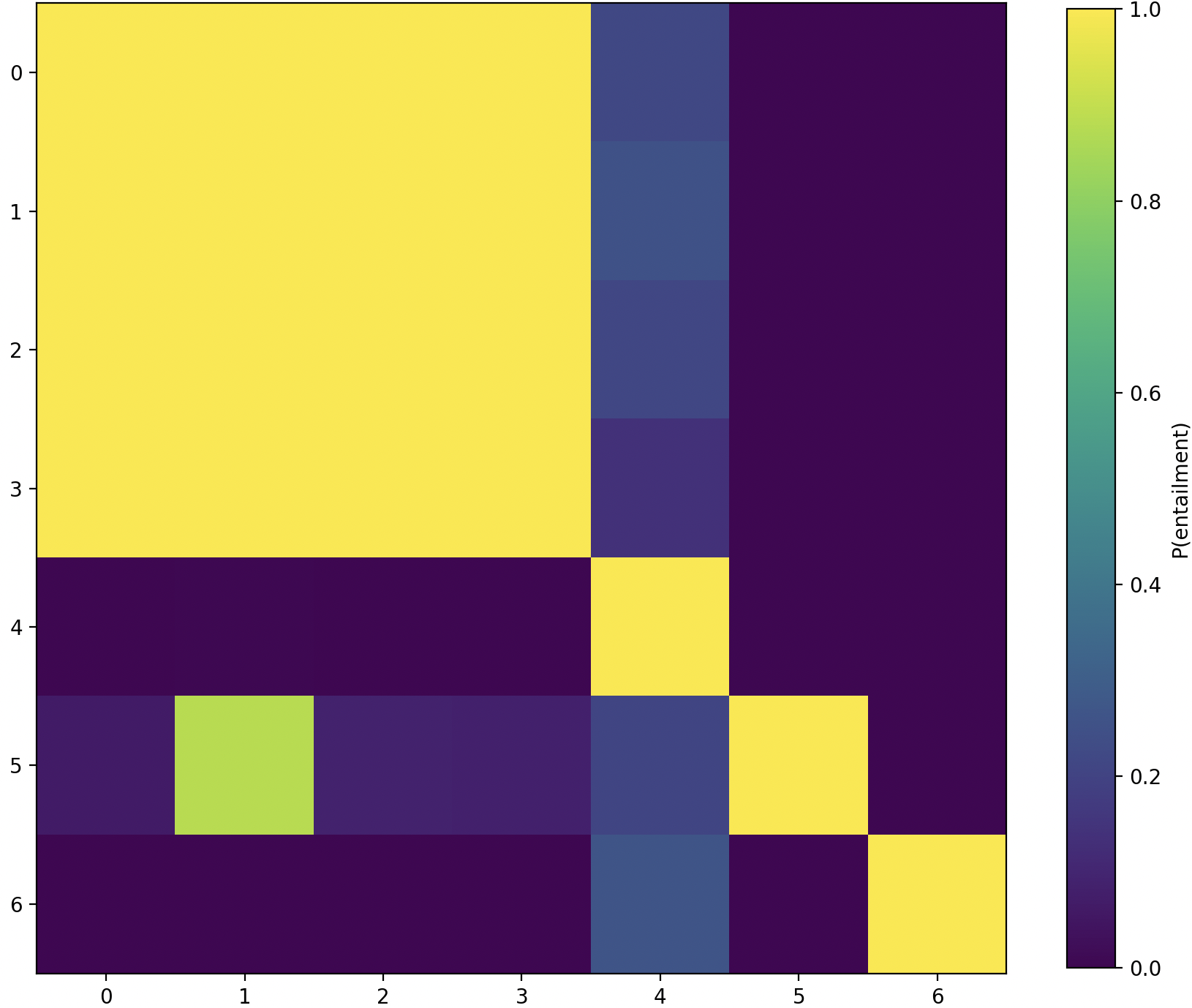}\\
        {\small (c) Confusion Matrix}
    \end{minipage}
    \caption{\small Example of clustering sensitivity to temperature. A total of $7$ sentences has been clustered: [\textit{"Paris is France's capital city.", "Paris is the capital of France.", "The capital of France is Paris.", "France is a country in Europe.",  "Paris is known for the Eiffel Tower.", "Berlin is France capital.", "Paris is in France and is the capital."}]. Note that in function of the clustering \emph{temperature} the clustering is more or less refined.}
    \label{fig:temperature}
\end{figure}

\begin{algorithm}
    \footnotesize
    \caption{\small Exemplar-based Clustering}
    \label{alg:clustering} \footnotesize
    \begin{algorithmic}[1]
        \State \textbf{Input:} Ground set of strings $\Gamma$, set of generated $\ell$-grams $\mathbb{S}=\{\mathsf{S}^1,\cdots,\mathsf{S}^\kappa\}$.
        \State $\mathcal{S}\leftarrow \{[i]:[\mathsf{S}^i]~\text{\textbf{for}}~i=1,\cdots,\kappa\}$
        \For{$x\in\Gamma$}
        \State $j \leftarrow \arg\min_{i\in\{1,\cdots,\kappa\}} \mathcal{E}(\mathsf{S}^i,x)$
        \State $\mathcal{S}_j \leftarrow \mathcal{S}_j \cup \{x\}$
        \EndFor
        \State \textbf{Return:} clustering $\mathcal{S}$
    \end{algorithmic}
\end{algorithm}

\section{Case Study}
\label{sec::numerical}

To evaluate the effectiveness of the proposed framework, we consider two complementary case studies. First, we assess the proposed semantic entropy-based hallucination metric \eqref{eqn:main_problem} on the \textbf{FEVER} dataset~\cite{JT-AV-CC-AM:18}, comparing it against the \textsf{HalluciNot} model~\cite{BP-AL-PJ-PA:25} and classical token-level entropy baselines. Second, we evaluate the full \texttt{CAROL} framework across three benchmark datasets: \textbf{TruthfulQA}~\cite{SL-JH-OE:21}, \textbf{HaluEval}~\cite{JL-XC-WZ-JN-JW:23}, and \textbf{HotPotQA}~\cite{ZY-PQ-SZ-YB-WC-RS-CM:18}. These datasets capture diverse hallucination settings, including fact fabrication under misleading prompts (TruthfulQA), hallucinations in summarization, question answering, and dialogue (HaluEval), and multi-hop reasoning with confabulation effects (HotPotQA; see Section~\ref{sec::introduction}).

We evaluate two language models, \emph{GPT-5-nano} and \emph{Llama-3.1-8B}, under Algorithm~\ref{alg:algorithm}, and compare their performance against standard \texttt{RAG}-based pipelines. Our evaluation considers not only hallucination reduction, but also the trade-off between response quality, latency, and computational cost. As described in Section~\ref{sec::method}, \texttt{CAROL} operates over $\ell$-gram semantic units derived from the context $\Gamma$. In practice, we implement this within an agentic framework using \texttt{URSA}~\cite{Grosskopf2025URSA}, deployed on an \emph{NVIDIA GeForce RTX 2080 Ti} GPU, see Appendix~\ref{app:ursa}. The system comprises three agents: \emph{Planner}, which generates the query and constructs $\Gamma$; \emph{Researcher}, which produces candidate responses (treated as $\ell$-grams); and \emph{Reasoner}, which aggregates accepted outputs across iterations.

\subsection{Hallucination Detection}

Prior work has demonstrated the advantages of semantic entropy over classical token-level uncertainty metrics~\cite{LK-YG-SF:23}. In this section, we show that the proposed metric not only outperforms token-level baselines, but also remains competitive with specialized hallucination detection methods such as \textsf{HalluciNot}. We evaluate performance in terms of detection accuracy, along with implementation metrics including execution time and GPU utilization. Accuracy is computed as the proportion of correctly classified instances (true positives and true negatives) over the total number of examples. Additional results and insights are provided in Appendix~\ref{app:results}, \ref{app:ursa} and \ref{app:datasets}.

\begin{table}[t]
\centering
\caption{\scriptsize Comparison between the proposed semantic entropy metric, classical token-level entropy metric and \textsf{HalluciNot} hallucination detector on \textbf{FEVER} dataset. Results are reported as mean $\pm$ standard deviation over multiple runs. We report detection accuracy, specificity, latency, and GPU allocation percentage. See Appendix~\ref{subsec:fever}. }
\label{tab:hallucination_comparison}
\scriptsize
\begin{tabular}{lcccc}
\toprule
\textbf{Method} & \textbf{Accuracy} $\uparrow$ & \textbf{Specificity} $\uparrow$ & \textbf{Latency (s/sample)} $\downarrow$ & \textbf{GPU Usage (\%)} $\downarrow$ \\
\midrule
Token-level Entropy & $0.4898 \pm 0.004$ & $0.4633 \pm 0.006$ & $\mathbf{0.189 \pm 0.002}$ & $\mathbf{3.765 \pm 0.010}$ \\
Semantic Entropy    & $\mathbf{0.5954 \pm 0.005}$ & $0.6477 \pm 0.007$ & $0.314 \pm 0.003$ & $3.835 \pm 0.012$ \\
\textsf{HalluciNot} & $0.3348 \pm 0.003$ & $\mathbf{0.9411 \pm 0.002}$ & $1.460 \pm 0.020$ & $24.75 \pm 0.150$ \\
\bottomrule
\end{tabular}
\end{table}

As reported in Table~\ref{tab:hallucination_comparison}, the proposed method identifies hallucinated statements more reliably than token-level entropy, which provides limited semantic discrimination. It also achieves performance comparable to \textsf{HalluciNot} with substantially lower computational overhead, making it suitable for resource-constrained settings. The resulting detection policy is more conservative, as the clustering-induced sparsity prioritizes contextual consistency and reduces false positives.

\subsection{Hallucination Mitigation}

Finally, we evaluate \texttt{CAROL}'s hallucination mitigation performance by comparing it against a standard \texttt{RAG} implementation (Table~\ref{tab:carol_vs_rag_hierarchy}). All considered datasets provide ground-truth optimal responses for each query. Accuracy is therefore computed by measuring the distance between the generated response $\mathbb{S}$ and the ground-truth $\mathbb{S}^\star$, as defined in Section~\ref{sec::preliminaries}. A response is deemed correct if it lies within an $\epsilon$-neighborhood of the ground truth, and incorrect otherwise. Overall accuracy is then computed as the proportion of optimal responses over the total number of evaluated samples.

We note that \textbf{TruthfulQA} offers a fine-grained evaluation setting, with multiple task categories and annotated examples of hallucinated, correct, and optimal responses, enabling detailed analysis of model behavior. Similarly, \textbf{HotPotQA} includes diverse question types and difficulty levels; in our experiments, we focus on \emph{bridge} and \emph{comparison} prompts to assess multi-hop reasoning performance. Additional results and breakdowns are provided in Appendix~\ref{app:results}.

\begin{table}[t]
\centering
\caption{\scriptsize \texttt{CAROL} vs.\ \texttt{RAG} comparison across datasets, models, and integrations. Results are reported as mean $\pm$ standard deviation over multiple runs. Lower hallucination, latency, and cost are better; higher accuracy is better.}
\label{tab:carol_vs_rag_hierarchy}
\tiny
\begin{tabular}{l | l | l c c c c}
\toprule
\textbf{Dataset} & \textbf{Model} & \textbf{Integration} &
\textbf{Hallucination} $\downarrow$ &
\textbf{Accuracy} $\uparrow$ &
\textbf{Latency (s)} $\downarrow$ &
\textbf{Cost (\$)} $\downarrow$ \\
\midrule

\multirow{4}{*}{TruthfulQA}
& \multirow{2}{*}{GPT\textendash5\textendash nano}
    & \texttt{RAG} & - & $0.356 \pm 0.006$ & $\mathbf{119.2 \pm 1.10}$ & $0.2933 \pm 0.004$ \\
    \cmidrule(lr){3-7}
&   & \texttt{CAROL} & $0.257 \pm 0.005$ & $\mathbf{0.556 \pm 0.007}$ & $164.8 \pm 1.35$ & $\mathbf{0.1463 \pm 0.003}$ \\
\cmidrule(lr){2-7}
& \multirow{2}{*}{Llama\textendash3.1\textendash8B}
    & \texttt{RAG} & - & $0.127 \pm 0.004$ & $\mathbf{120.9 \pm 1.25}$ & - \\
    \cmidrule(lr){3-7}
&   & \texttt{CAROL}& $0.348 \pm 0.006$ & $\mathbf{0.540 \pm 0.008}$ & $194.3 \pm 1.60$ & - \\
\midrule

\multirow{4}{*}{HaluEval}
& \multirow{2}{*}{GPT\textendash5\textendash nano}
    & \texttt{RAG} & - & $0.638 \pm 0.008$ & $\mathbf{0.517 \pm 0.007}$ & $0.0376 \pm 0.001$ \\
    \cmidrule(lr){3-7}
&   & \texttt{CAROL} & $0.138 \pm 0.004$ & $\mathbf{0.690 \pm 0.006}$ & $1.062 \pm 0.015$ & $\mathbf{0.0163 \pm 0.001}$ \\
\cmidrule(lr){2-7}
& \multirow{2}{*}{Llama\textendash3.1\textendash8B}
    & \texttt{RAG} & - & $0.574 \pm 0.007$ & $\mathbf{1.096 \pm 0.010}$ & - \\
    \cmidrule(lr){3-7}
&   & \texttt{CAROL} & $0.254 \pm 0.005$ & $\mathbf{0.616 \pm 0.008}$ & $1.630 \pm 0.020$ & - \\
\midrule

\multirow{4}{*}{HotPotQA}
& \multirow{2}{*}{GPT\textendash5\textendash nano}
    & \texttt{RAG} & - & $0.865 \pm 0.009$ & $\mathbf{1.518 \pm 0.012}$ & $0.1107 \pm 0.002$ \\
    \cmidrule(lr){3-7}
&   & \texttt{CAROL} & $0.591 \pm 0.007$ & $\mathbf{0.978 \pm 0.010}$ & $2.858 \pm 0.030$ & $\mathbf{0.0319 \pm 0.001}$ \\
\cmidrule(lr){2-7}
& \multirow{2}{*}{Llama\textendash3.1\textendash8B}
    & \texttt{RAG} & - & $0.809 \pm 0.008$ & $\mathbf{1.281 \pm 0.011}$ & - \\
    \cmidrule(lr){3-7}
&   & \texttt{CAROL} & $0.497 \pm 0.006$ & $\mathbf{0.962 \pm 0.009}$ & $2.383 \pm 0.025$ & - \\
\bottomrule
\end{tabular}
\end{table}

Table~\ref{tab:carol_vs_rag_hierarchy} shows that the proposed method outperforms standard \texttt{RAG} not only in accuracy but also in computational cost. This behavior stems from the design of \texttt{CAROL}, which can be interpreted as an adaptive form of \texttt{RAG}: additional information is incorporated only when the generated response deviates from the optimal trajectory. As a result, token usage is selectively regulated, leading to reduced computational overhead. Moreover, reduction in prompt size yields faster per-step computations, enabling the incorporation of the rewiring mechanism with limited impact on latency. Nevertheless, \texttt{CAROL} exhibits higher overall latency than standard \texttt{RAG}, primarily due to the clustering step, which becomes a bottleneck for large context sizes. Improving this component-through precomputed similarity structures over $\Gamma$, remains an important direction for future work, see Appendix~\ref{app:limitations_broader_impact}.

The performance gains are particularly pronounced for smaller models such as \emph{Llama-3.1-8B}. As noted in~\cite{JV-MG-RB:26}, such models are more sensitive to prompt length and can suffer from instability when handling large contexts. By regulating context usage, \texttt{CAROL} provides a more robust alternative in resource-constrained settings. Furthermore, results on \textbf{TruthfulQA} indicate that among non-hallucinated outputs, \texttt{CAROL} achieves a higher proportion of optimal responses. This supports the intuition behind the exemplar-based clustering introduced in Section~\ref{sec::statement}: beyond enforcing semantic consistency with $\Gamma$, the clustering mechanism promotes sparsity, reducing redundancy and improving the overall quality of generated responses. Additional details are provided in Appendix~\ref{app:results}.

\section{Conclusions}
\label{sec::conclusion}
\vspace{-0.2in}
This work introduced \texttt{CAROL}, a principled framework for hallucination reduction in large language models that operates entirely at test time in a black-box setting, requiring neither access to internal probabilities nor repeated model evaluations. Central to the approach is a single-pass semantic entropy measure grounded in logical entailment, enabling direct assessment of consistency with respect to a trusted context $\Gamma$. To support this, we propose an exemplar-based clustering mechanism that eliminates sensitivity to temperature-based partitioning and avoids iterative refinement. This formulation induces a string-submodular objective over a lattice of textual sequences, allowing hallucination mitigation to be cast as a probabilistic accept–reject process with provable convergence and near-optimality guarantees. Empirically, \texttt{CAROL} achieves more reliable and consistent outputs than entropy- and retrieval-based baselines at competitive cost. More broadly, it reframes hallucination as a sequence-level deviation from axiomatic knowledge, providing a scalable foundation for reliable and interpretable test-time generation in agentic AI systems.

\vskip 0.2in
\newpage
\bibliographystyle{plainnat}
\bibliography{sample}

\newpage
\appendix
\section{Extended results}
\label{app:results}

In this appendix we present the breakdown of the results for each dataset. By looking closer at the results some extra conclusions can be abstracted about the advantages that \texttt{CAROL} present over \texttt{RAG} method.

\subsection{FEVER Dataset}
\label{subsec:fever}

\begin{figure}[ht]
    \centering
    \begin{subcaptiongroup}
        \begin{subfigure}{0.3\textwidth}
            \centering
            \includegraphics[width=\linewidth]{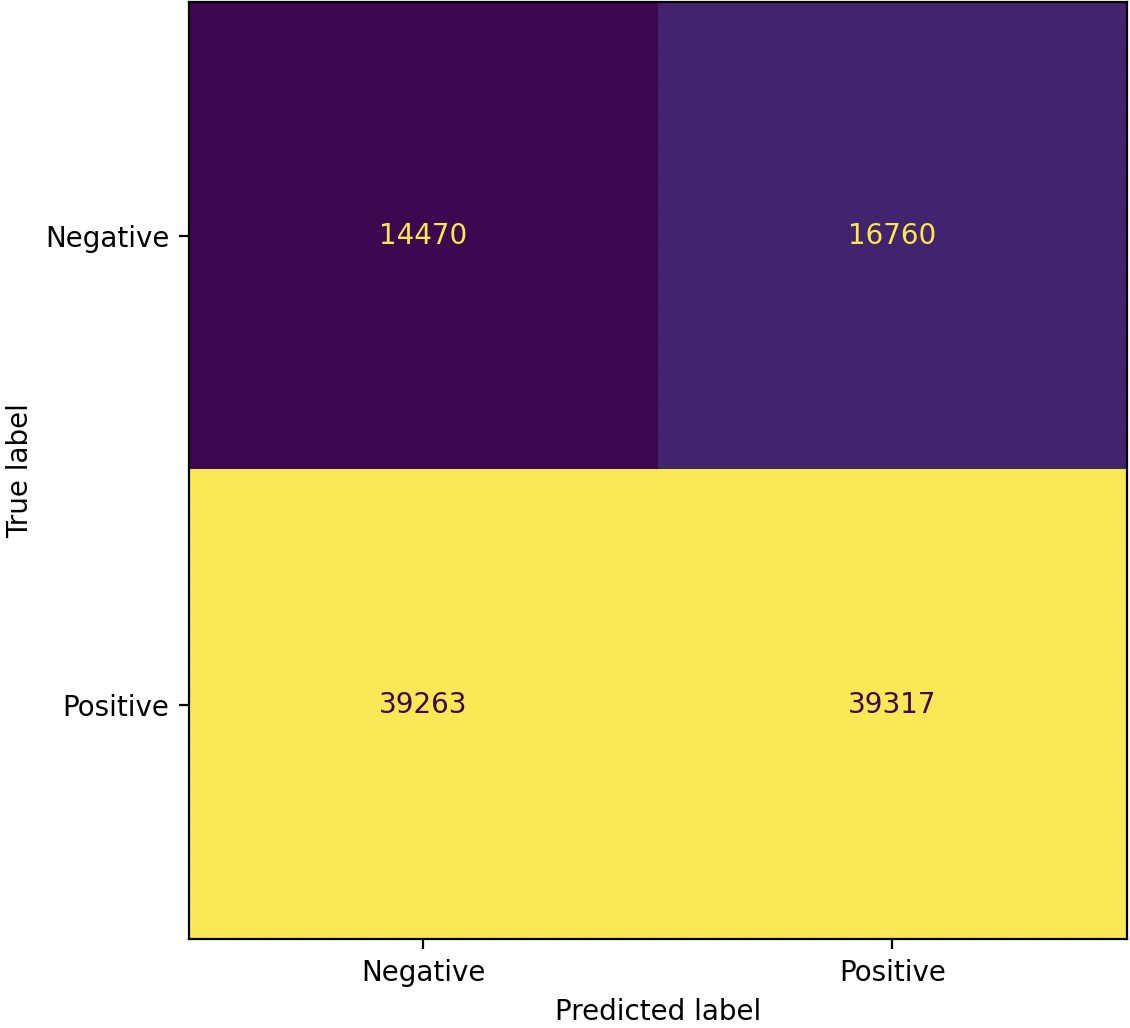}
            \caption{Token-level Entropy}
        \end{subfigure}
        \begin{subfigure}{0.3\textwidth}
            \centering
            \includegraphics[width=\linewidth]{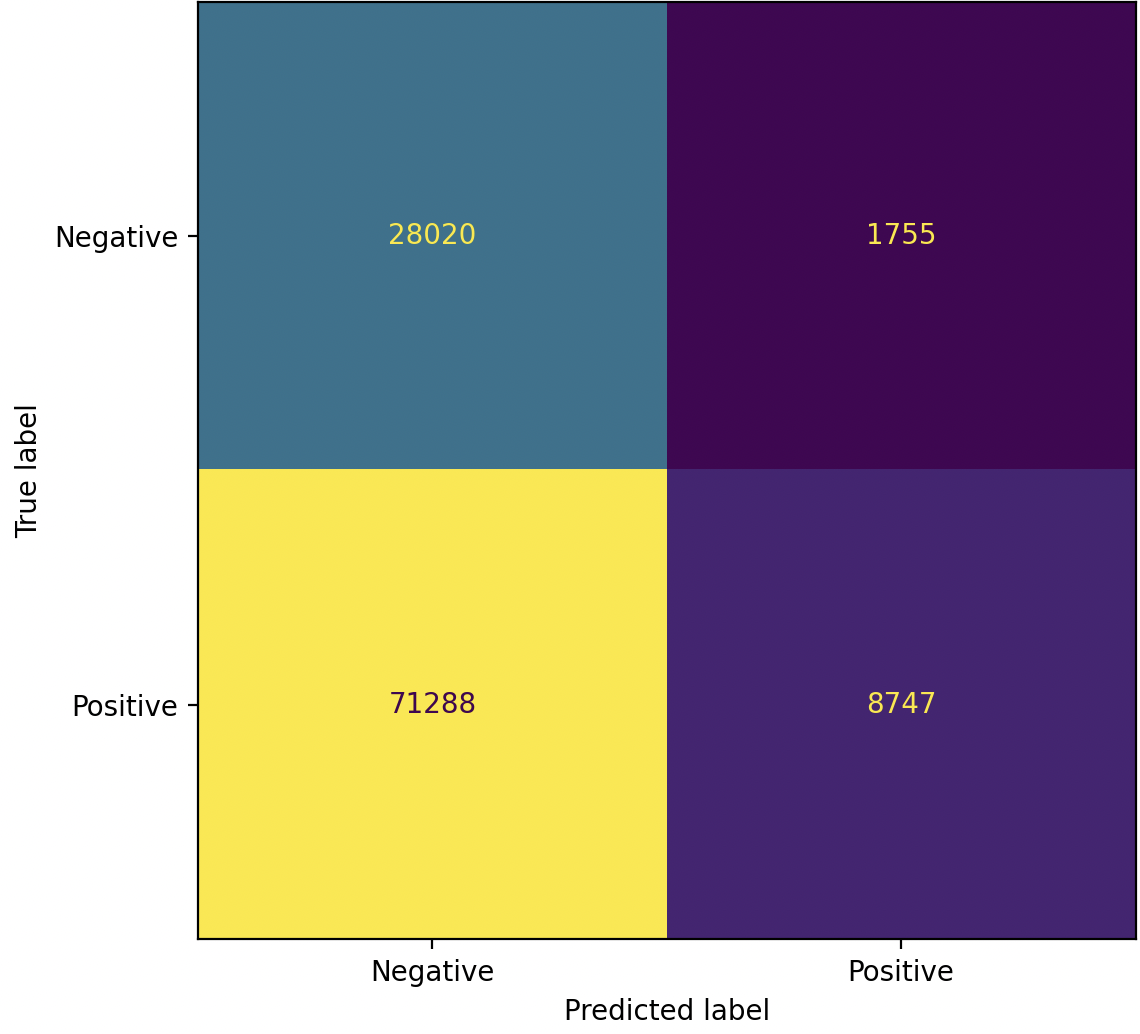}
            \caption{Semantic Entropy}
        \end{subfigure}
        \begin{subfigure}{0.3\textwidth}
            \centering
            \includegraphics[width=\linewidth]{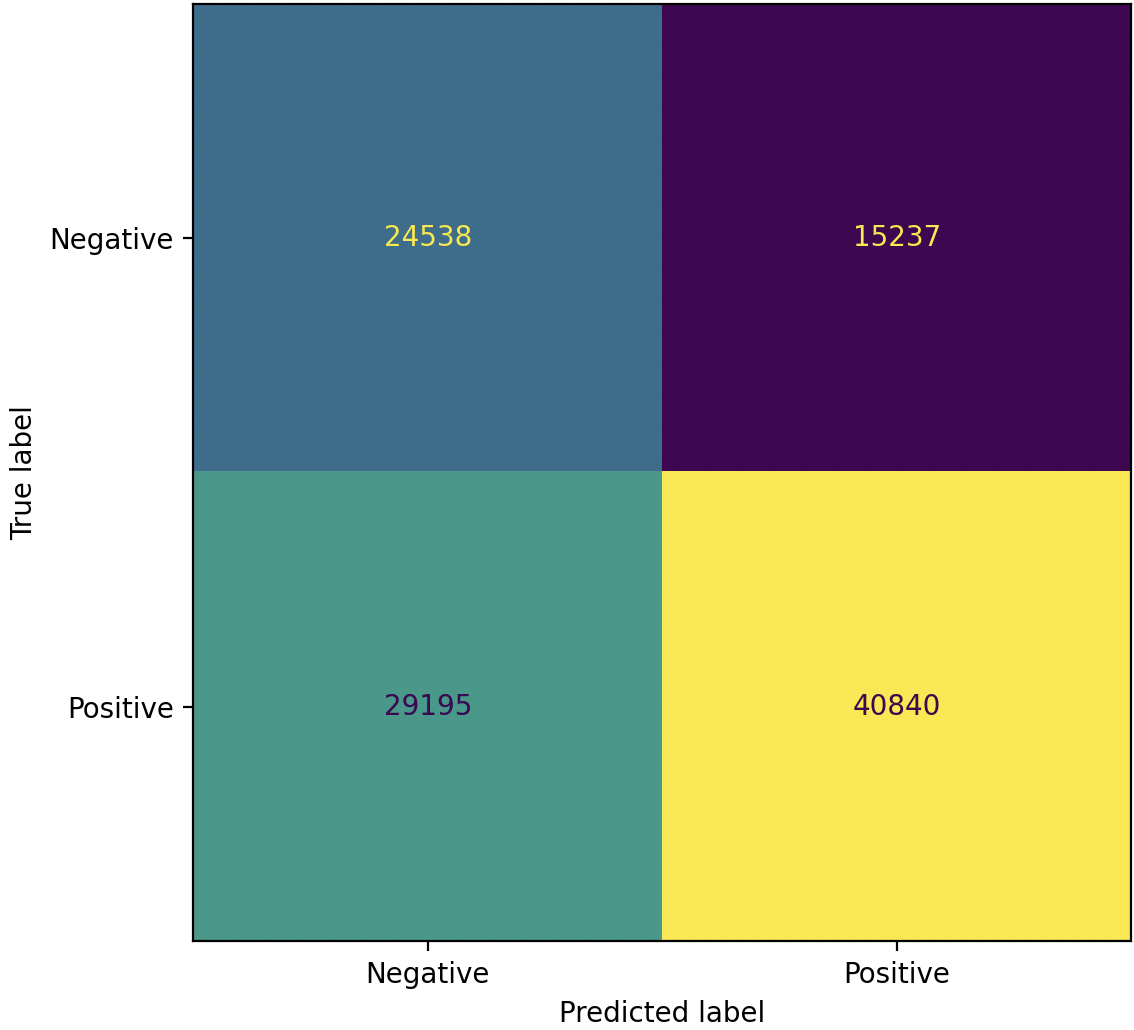}
            \caption{\textsf{HalluciNot}}
        \end{subfigure}
    \end{subcaptiongroup}
\caption{Confusion matrix for each hallucination metric on \textbf{FEVER} dataset.}
\end{figure}

\begin{figure}[ht]
    \centering
    \includegraphics[width=0.7\linewidth]{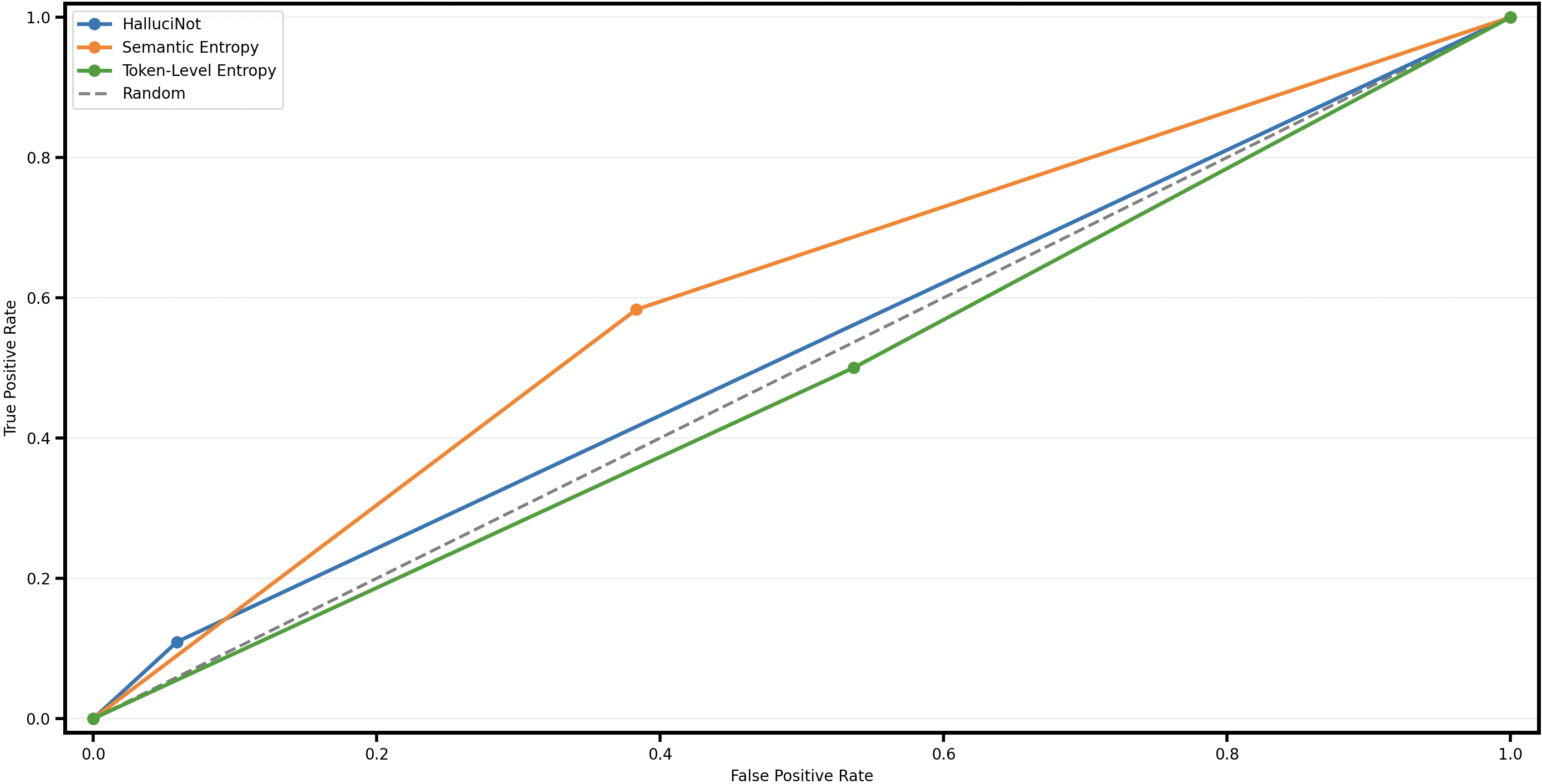}
    \caption{ROC curve for the hallucination metric on the \textbf{FEVER} dataset.}
    \label{fig:roc_fever}
\end{figure}

\begin{figure}[ht]
    \centering
    \includegraphics[width=0.7\linewidth]{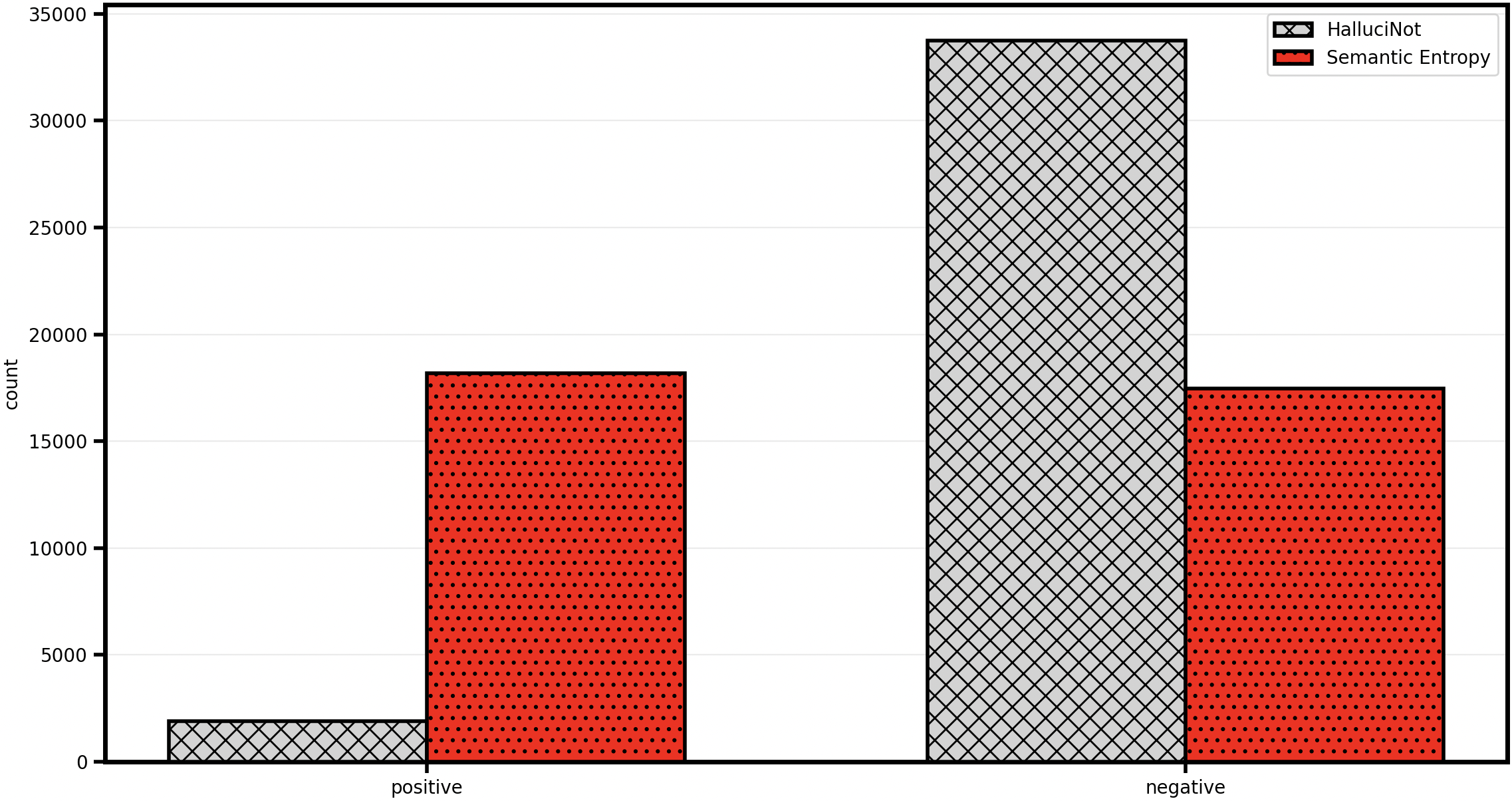}
    \caption{Test examples for the hallucination metric on the \textbf{FEVER} dataset.}
    \label{fig:results_fever}
\end{figure}

\begin{table}[t]
\centering
\caption{\scriptsize Comparison between the proposed semantic entropy metric, classical token-level entropy metric and \textsf{HalluciNot} hallucination detector on \textbf{FEVER}. Results are reported as mean $\pm$ standard deviation over multiple runs. We report extended performance metrics.}
\label{tab:hallucination_comparison}
\scriptsize
\begin{tabular}{lccc}
\toprule
\textbf{Metric} & \textbf{Token-level Entropy} & \textbf{Semantic Entropy} & \textbf{\textsf{HalluciNot}} \\
\midrule
Accuracy $\uparrow$    & $0.4898 \pm 0.004$ & $\mathbf{0.5954 \pm 0.005}$ & $0.3348 \pm 0.003$ \\
Precision $\uparrow$   & $0.7011 \pm 0.006$ & $0.7283 \pm 0.006$ & $\mathbf{0.8329 \pm 0.004}$ \\
Recall $\uparrow$      & $0.5003 \pm 0.005$ & $\mathbf{0.5831 \pm 0.006}$ & $0.1093 \pm 0.002$ \\
Specificity $\uparrow$ & $0.4633 \pm 0.006$ & $0.6169 \pm 0.007$ & $\mathbf{0.9411 \pm 0.002}$ \\
F1 $\uparrow$          & $0.5840 \pm 0.005$ & $\mathbf{0.6477 \pm 0.006}$ & $0.1932 \pm 0.003$ \\
FPR $\downarrow$       & $0.5367 \pm 0.006$ & $0.3831 \pm 0.005$ & $\mathbf{0.0589 \pm 0.002}$ \\
Latency (s/sample) $\downarrow$ & $\mathbf{0.1890 \pm 0.002}$ & $0.3140 \pm 0.003$ & $1.4600 \pm 0.020$ \\
Throughput (sample/s) $\uparrow$ & $\mathbf{5.2900 \pm 0.050}$ & $3.1800 \pm 0.040$ & $0.6800 \pm 0.010$ \\
GPU Usage (\%) $\downarrow$ & $\mathbf{3.7650 \pm 0.010}$ & $3.8350 \pm 0.012$ & $24.750 \pm 0.150$ \\
\bottomrule
\end{tabular}
\end{table}

\FloatBarrier


\subsection{TruthfulQA}
\label{subsec:truthful}

\begin{figure}[ht]
    \centering
    \begin{subcaptiongroup}
        \begin{subfigure}{0.3\textwidth}
            \centering
            \includegraphics[width=\linewidth]{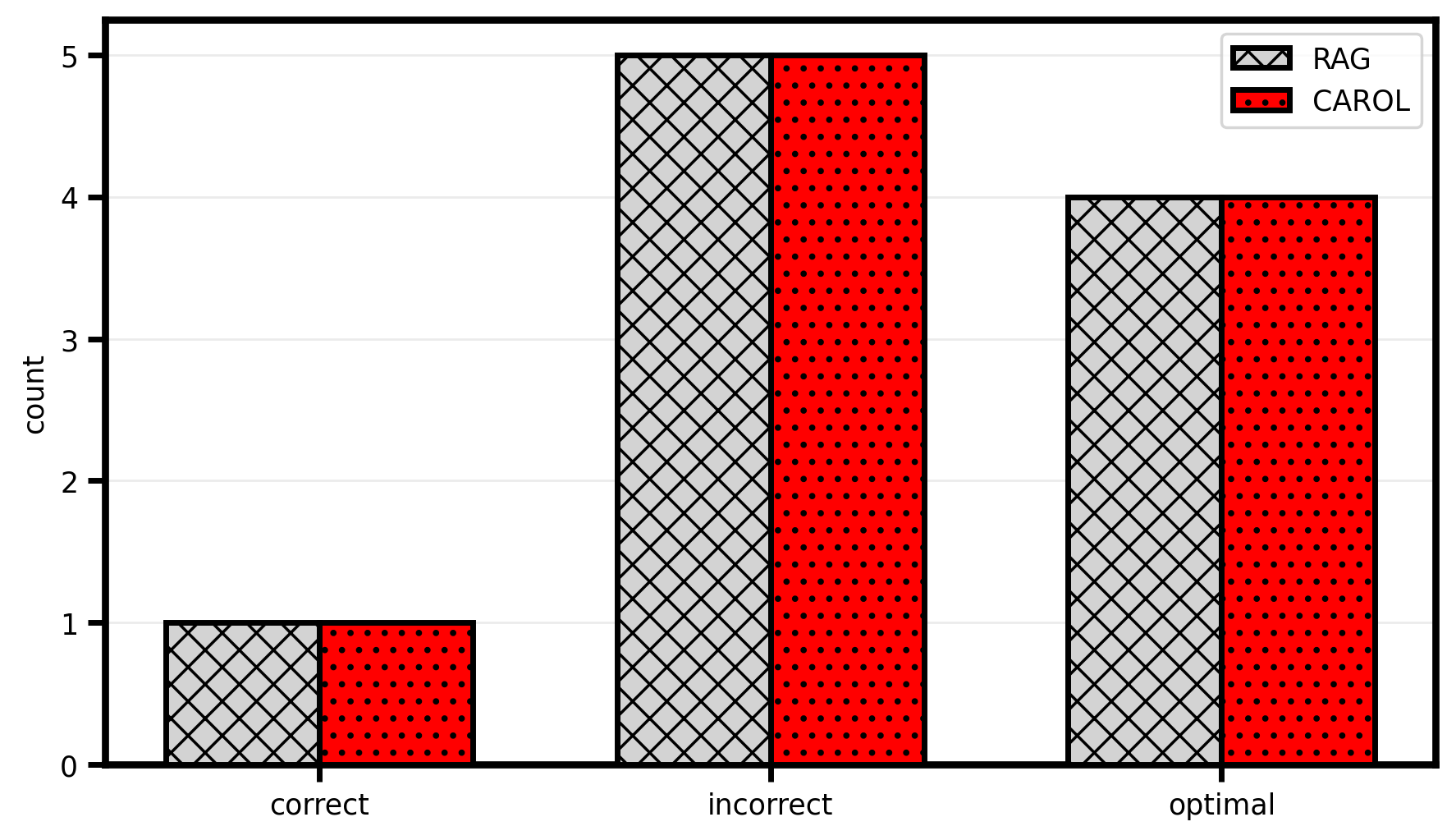}
            \caption{Advertising}
        \end{subfigure}
        \begin{subfigure}{0.3\textwidth}
            \centering
            \includegraphics[width=\linewidth]{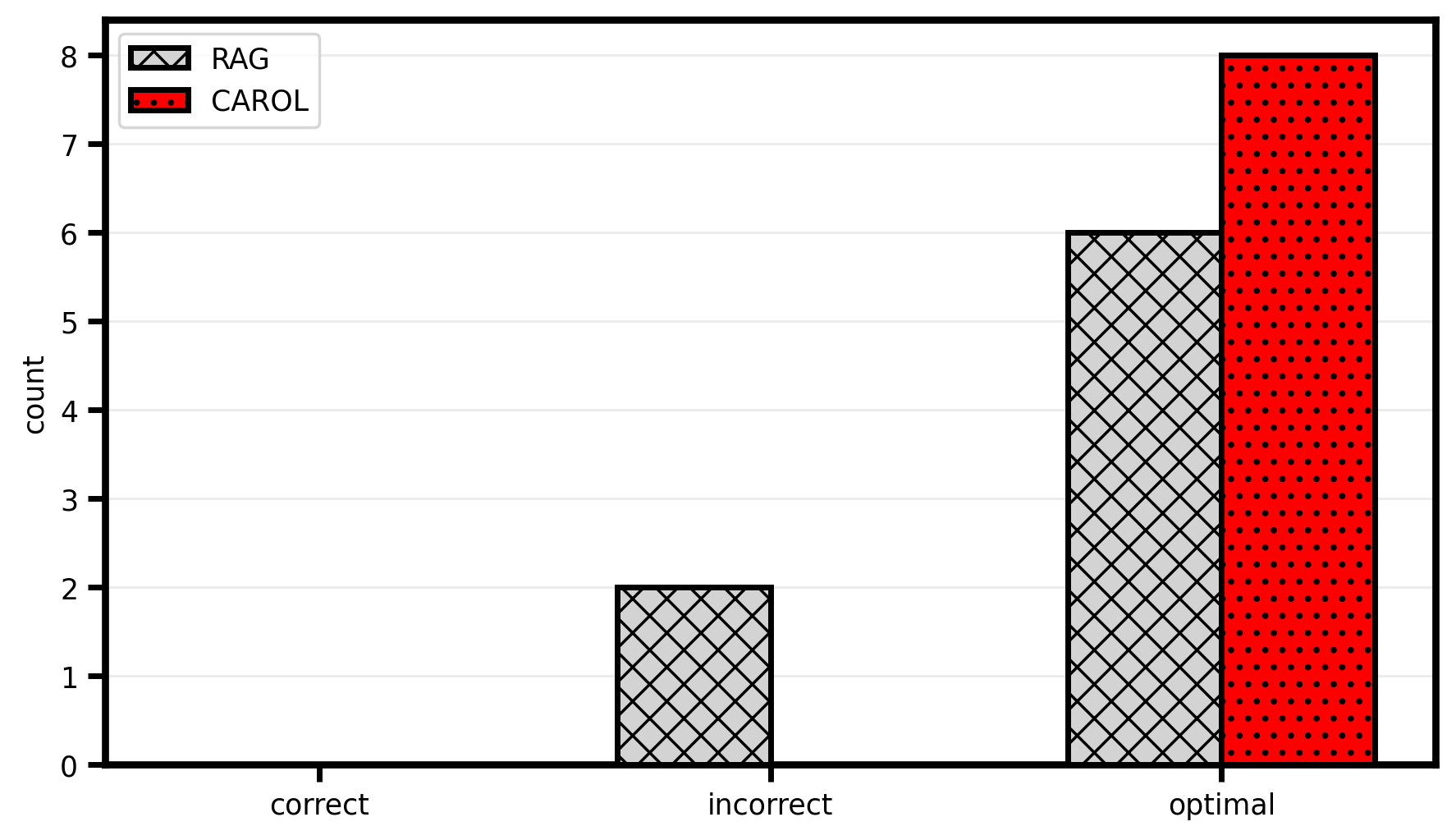}
            \caption{Confusion}
        \end{subfigure}
        \begin{subfigure}{0.3\textwidth}
            \centering
            \includegraphics[width=\linewidth]{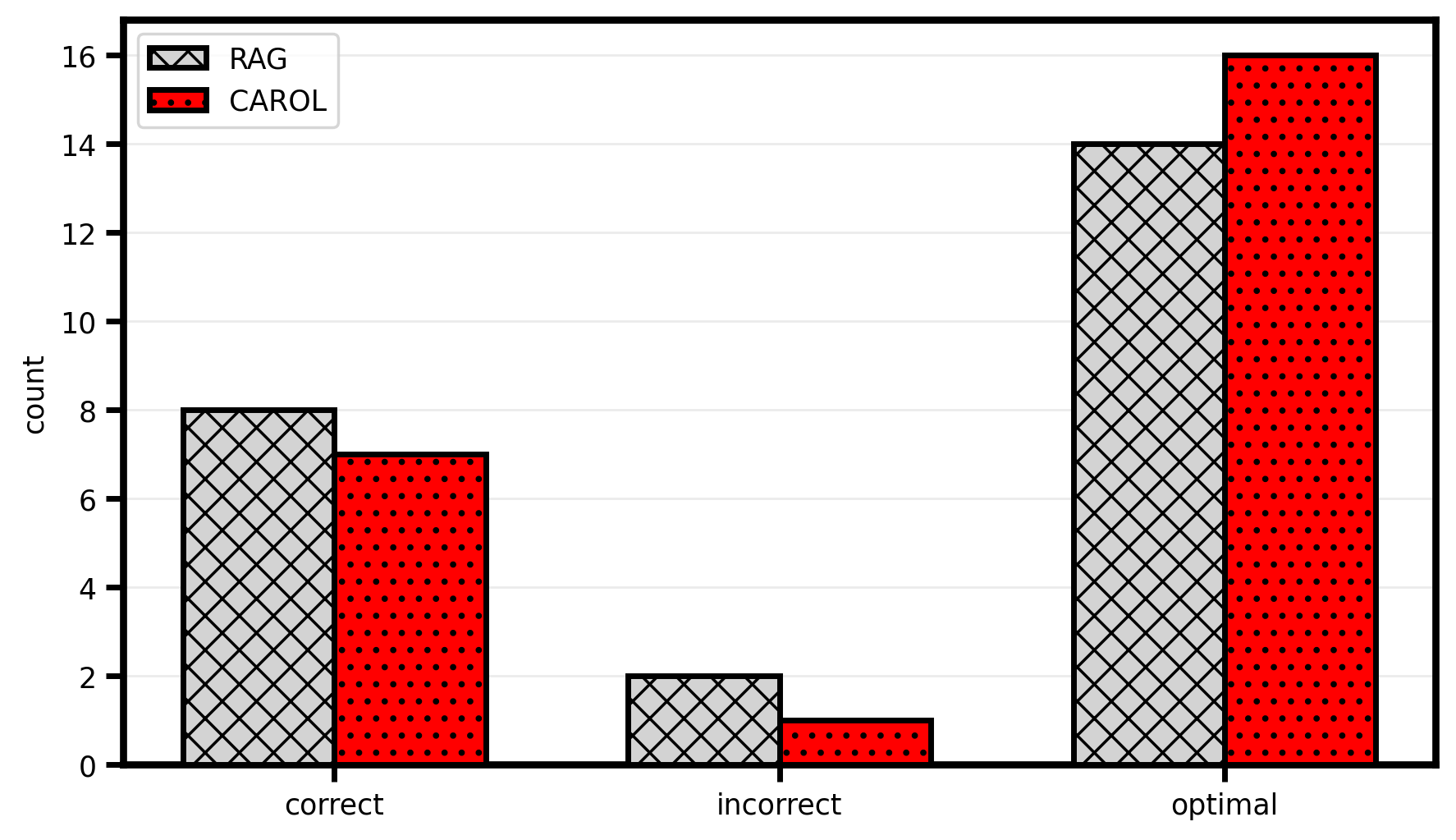}
            \caption{Conspiracies}
        \end{subfigure}
    \end{subcaptiongroup}
    \begin{subcaptiongroup}
        \begin{subfigure}{0.3\textwidth}
            \centering
            \includegraphics[width=\linewidth]{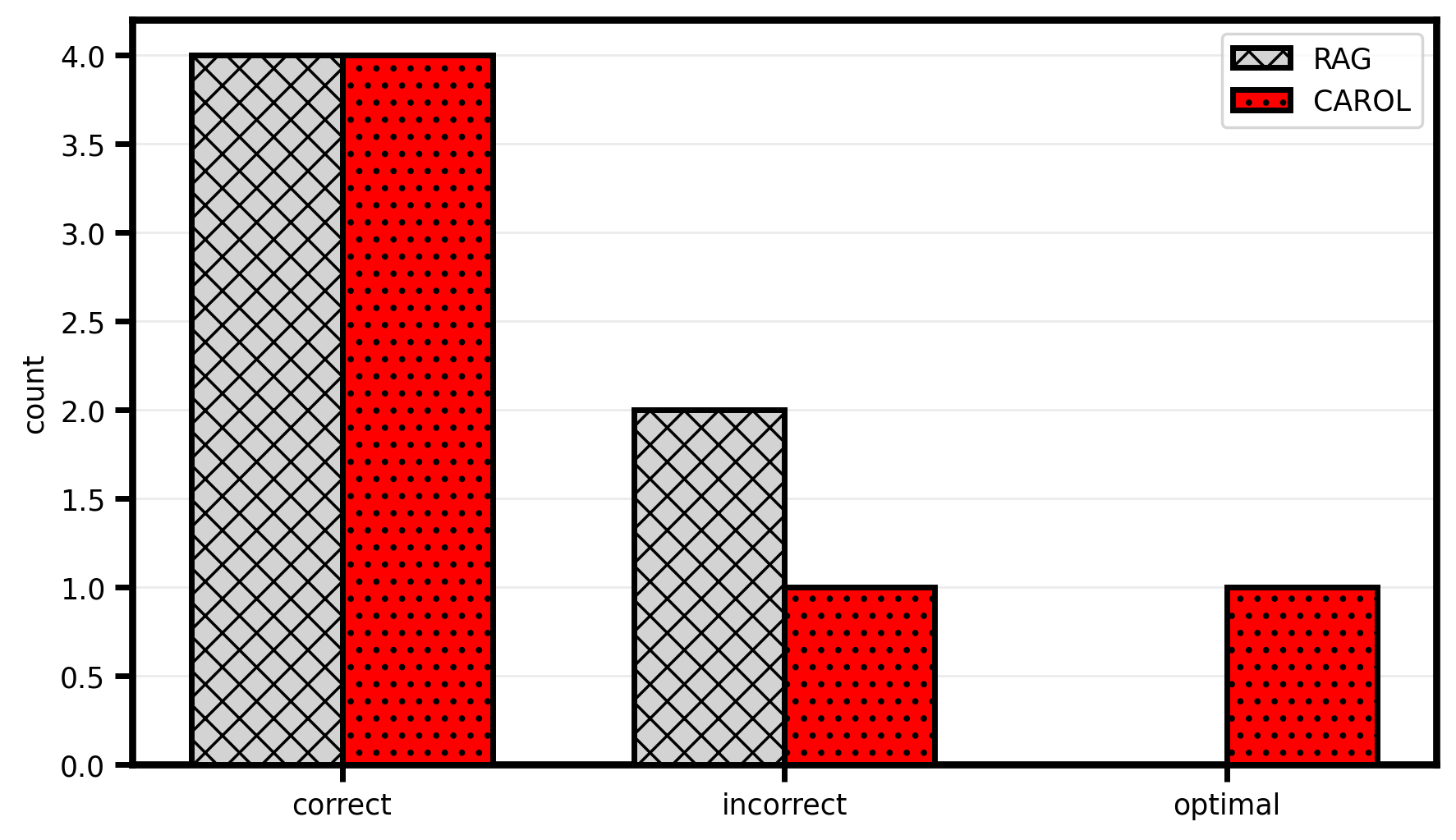}
            \caption{Mandela Effect}
        \end{subfigure}
        \begin{subfigure}{0.3\textwidth}
            \centering
            \includegraphics[width=\linewidth]{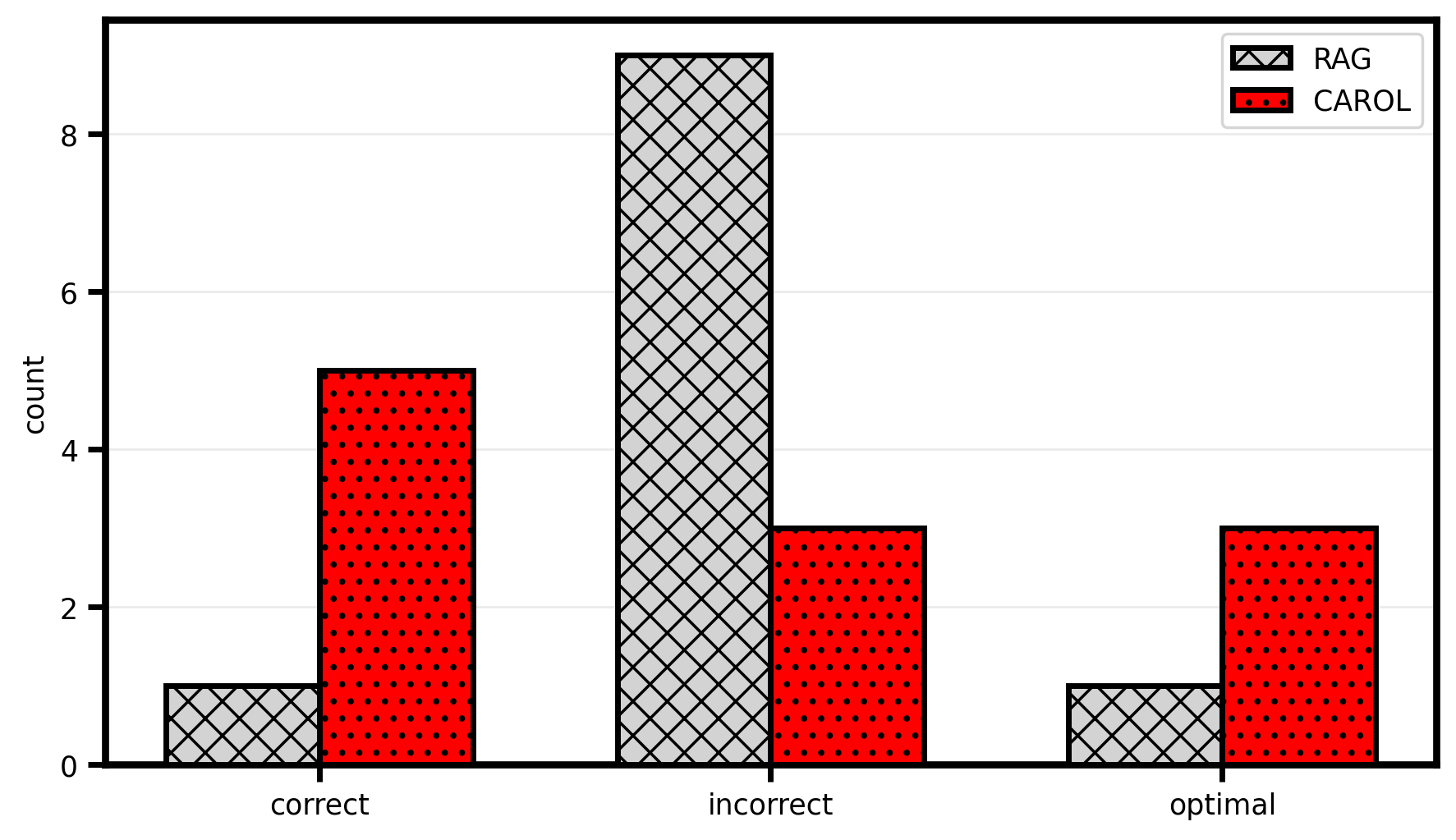}
            \caption{Distraction}
        \end{subfigure}
        \begin{subfigure}{0.3\textwidth}
            \centering
            \includegraphics[width=\linewidth]{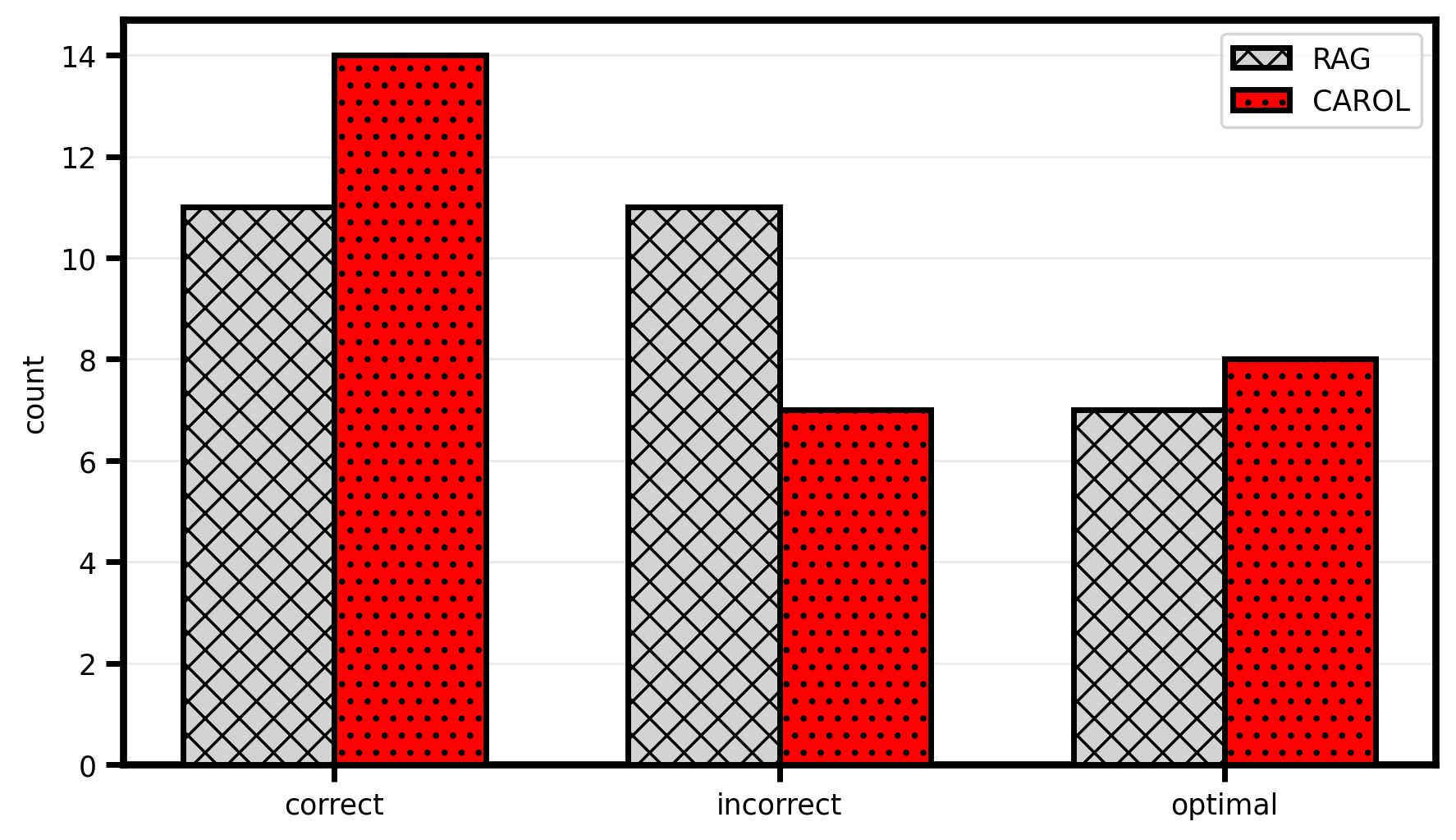}
            \caption{Economics}
        \end{subfigure}
    \end{subcaptiongroup}
    \begin{subcaptiongroup}
        \begin{subfigure}{0.3\textwidth}
            \centering
            \includegraphics[width=\linewidth]{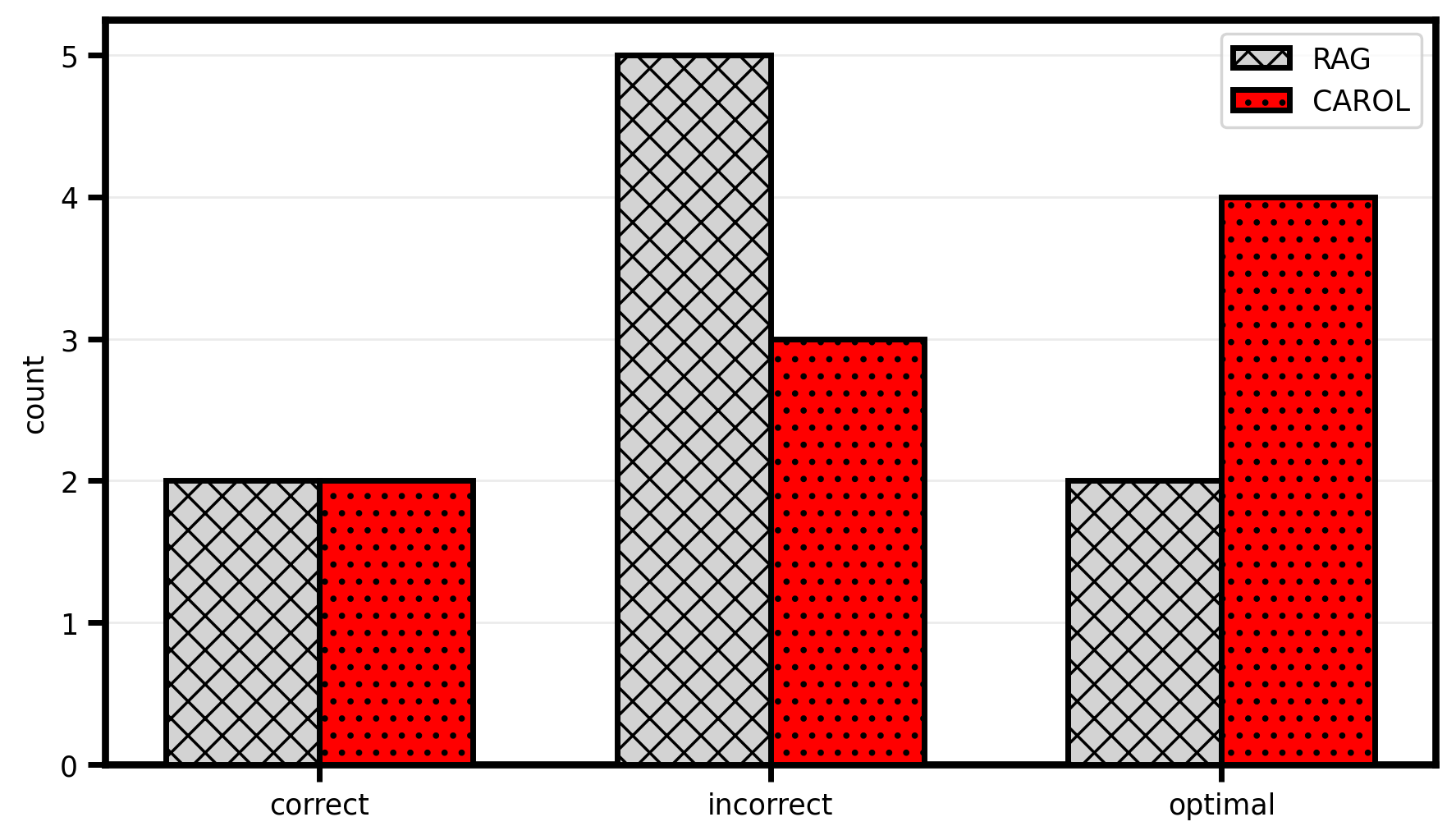}
            \caption{Education}
        \end{subfigure}
        \begin{subfigure}{0.3\textwidth}
            \centering
            \includegraphics[width=\linewidth]{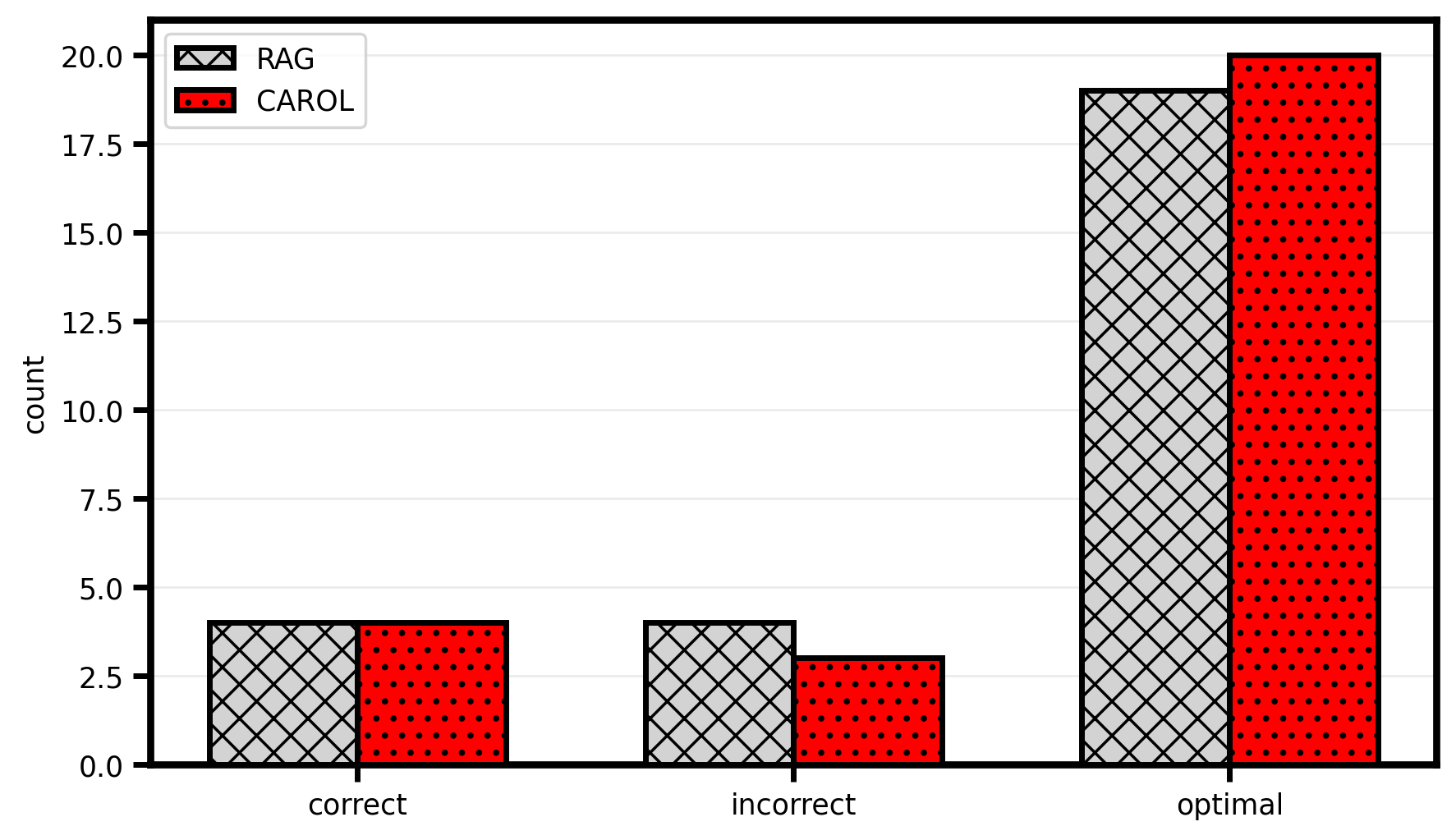}
            \caption{Fiction}
        \end{subfigure}
        \begin{subfigure}{0.3\textwidth}
            \centering
            \includegraphics[width=\linewidth]{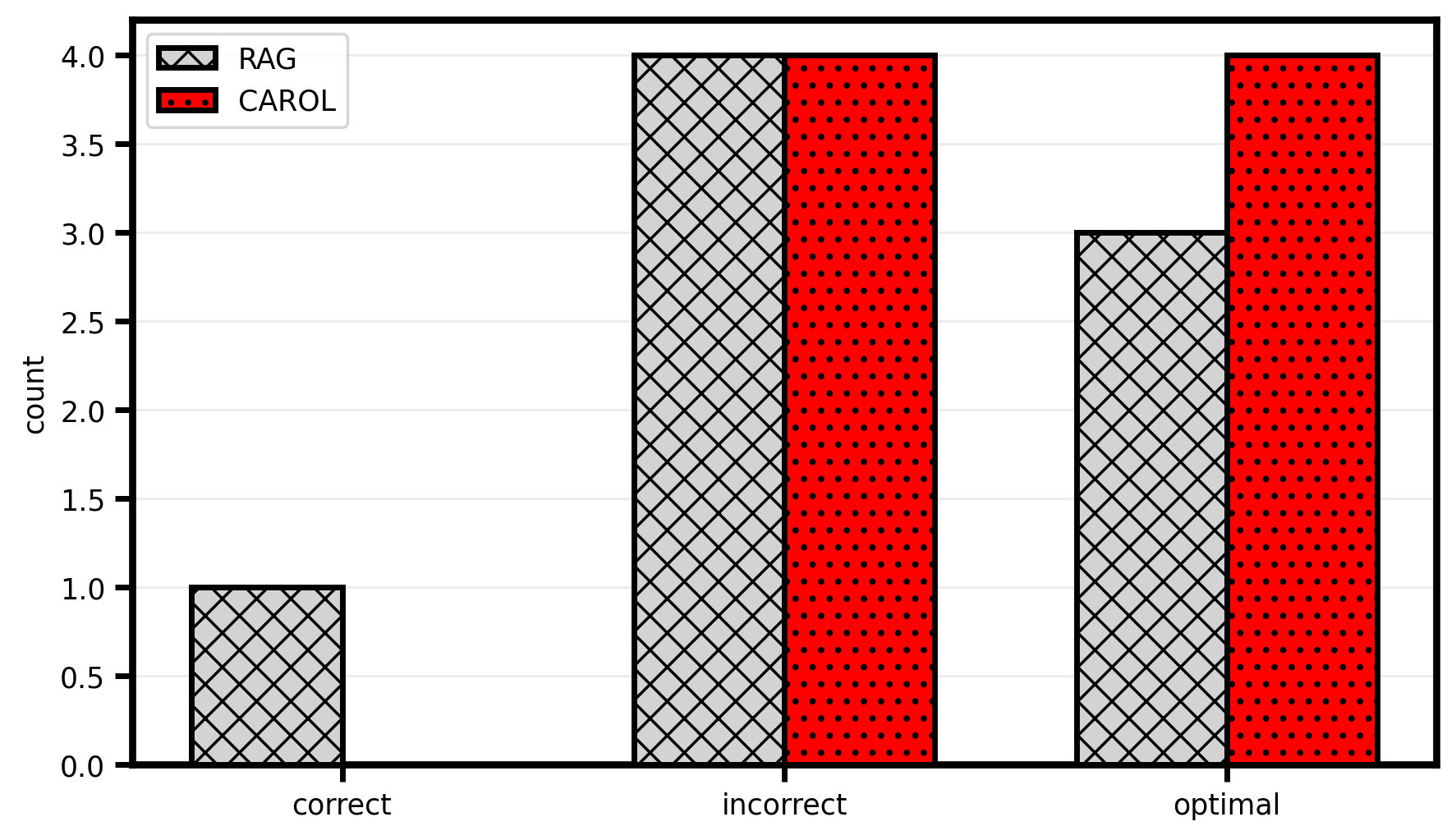}
            \caption{Finance}
        \end{subfigure}
    \end{subcaptiongroup}
    \begin{subcaptiongroup}
        \begin{subfigure}{0.3\textwidth}
            \centering
            \includegraphics[width=\linewidth]{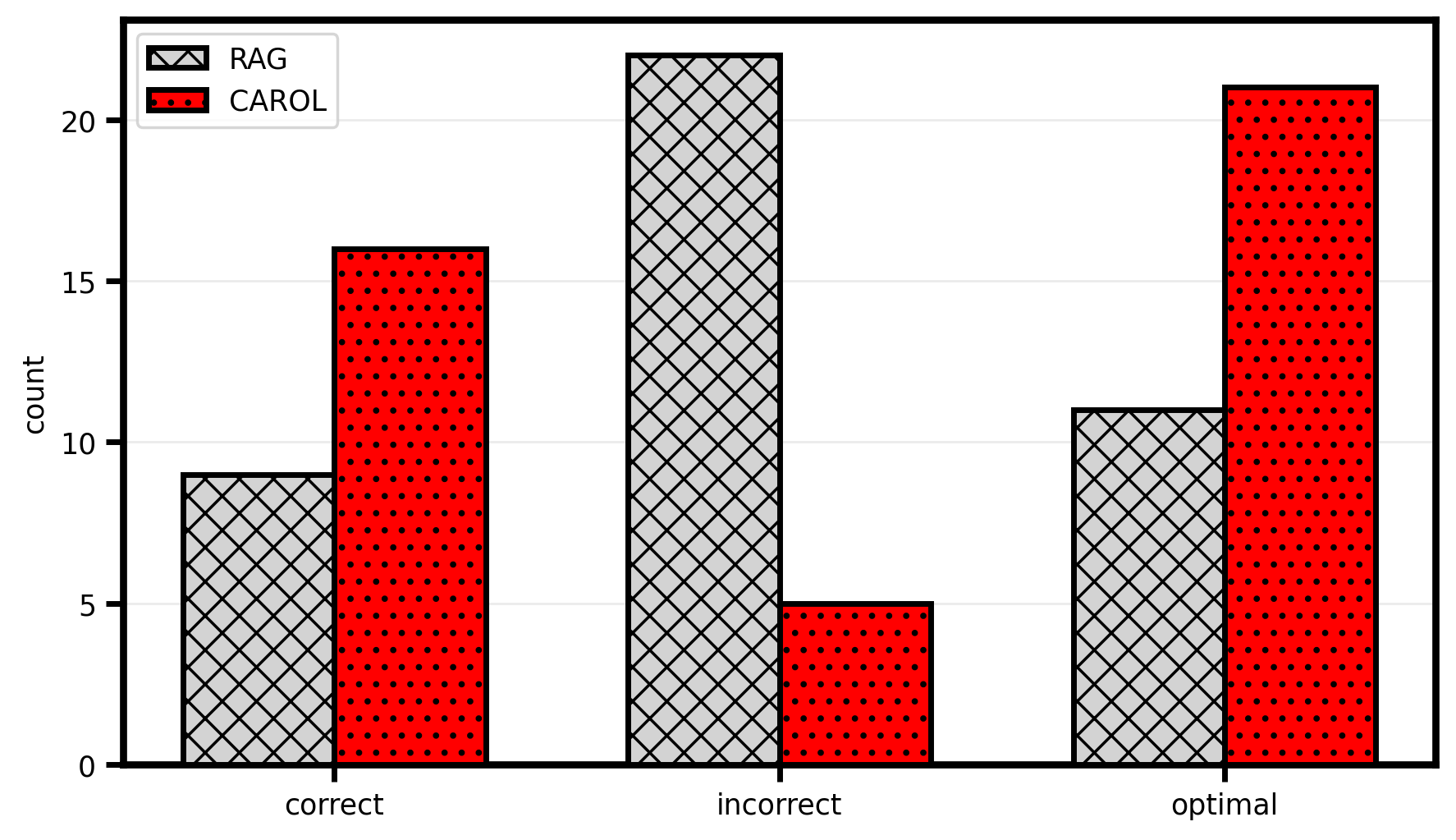}
            \caption{Health}
        \end{subfigure}
        \begin{subfigure}{0.3\textwidth}
            \centering
            \includegraphics[width=\linewidth]{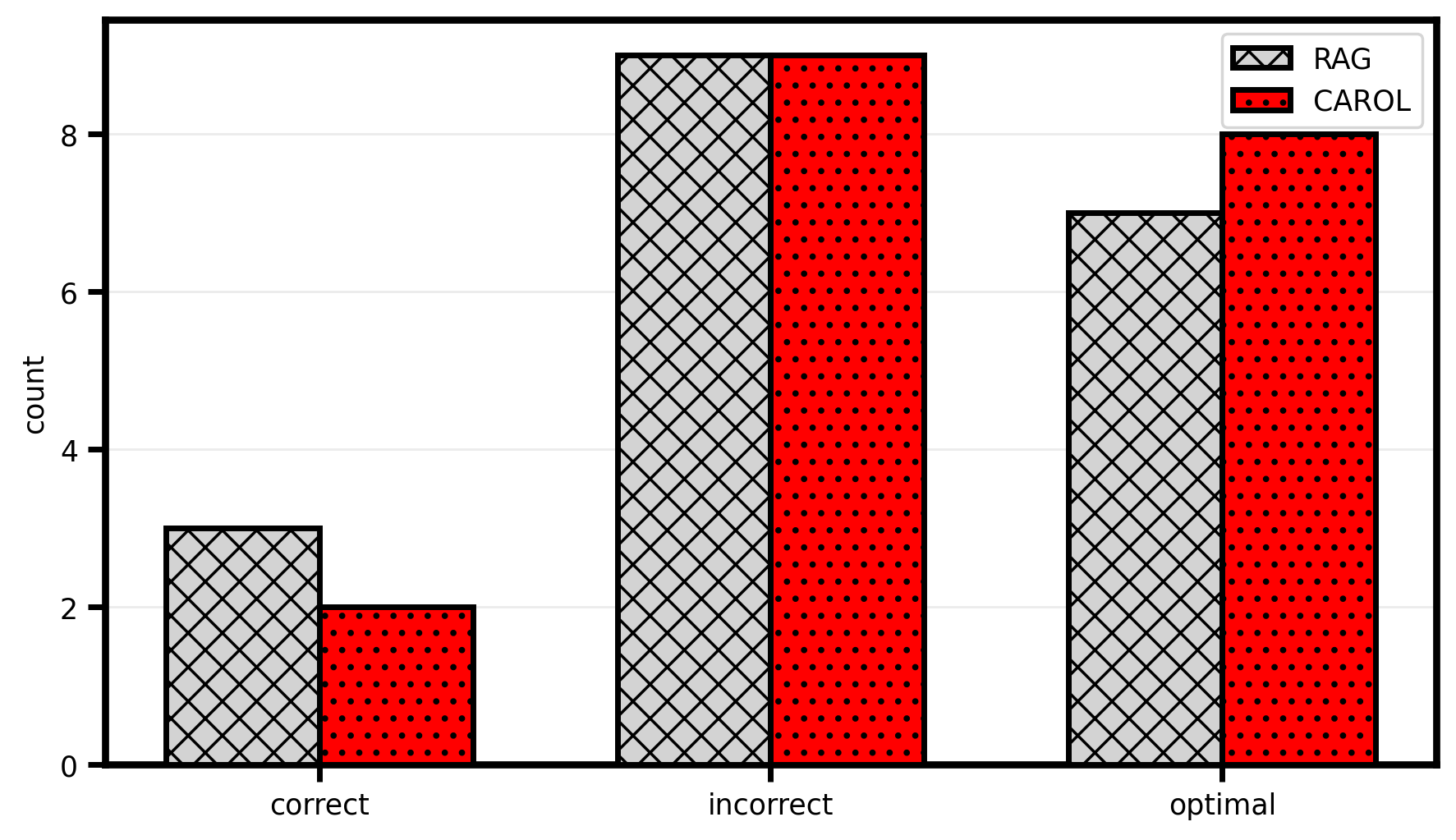}
            \caption{History}
        \end{subfigure}
        \begin{subfigure}{0.3\textwidth}
            \centering
            \includegraphics[width=\linewidth]{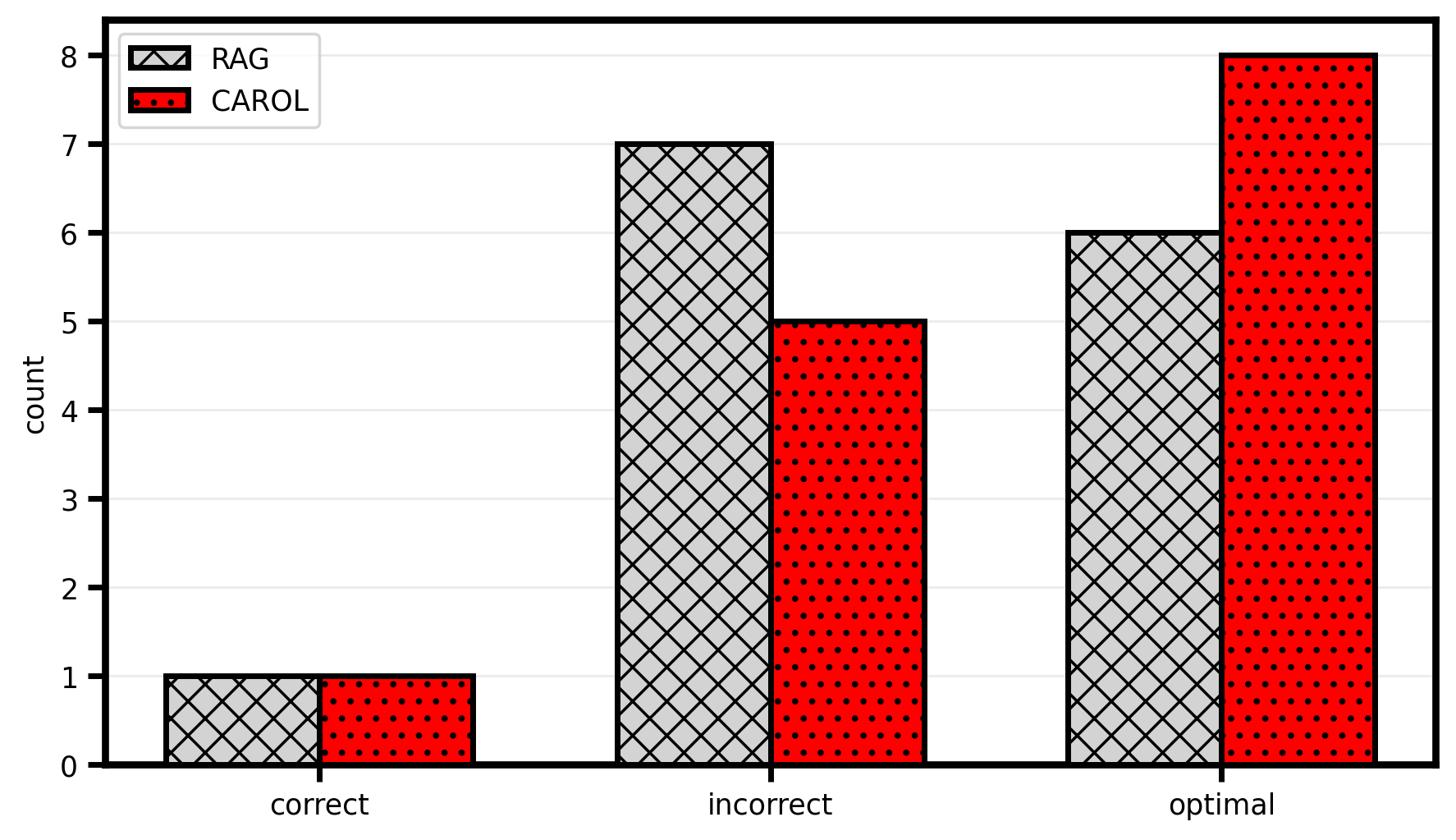}
            \caption{Indexical Error}
        \end{subfigure}
    \end{subcaptiongroup}
\caption{\textbf{TruthfulQA} on GPT-5-nano extended results in each category, part $1$.}
\end{figure}

\begin{figure}[ht]
    \centering
    \begin{subcaptiongroup}
        \begin{subfigure}{0.3\textwidth}
            \centering
            \includegraphics[width=\linewidth]{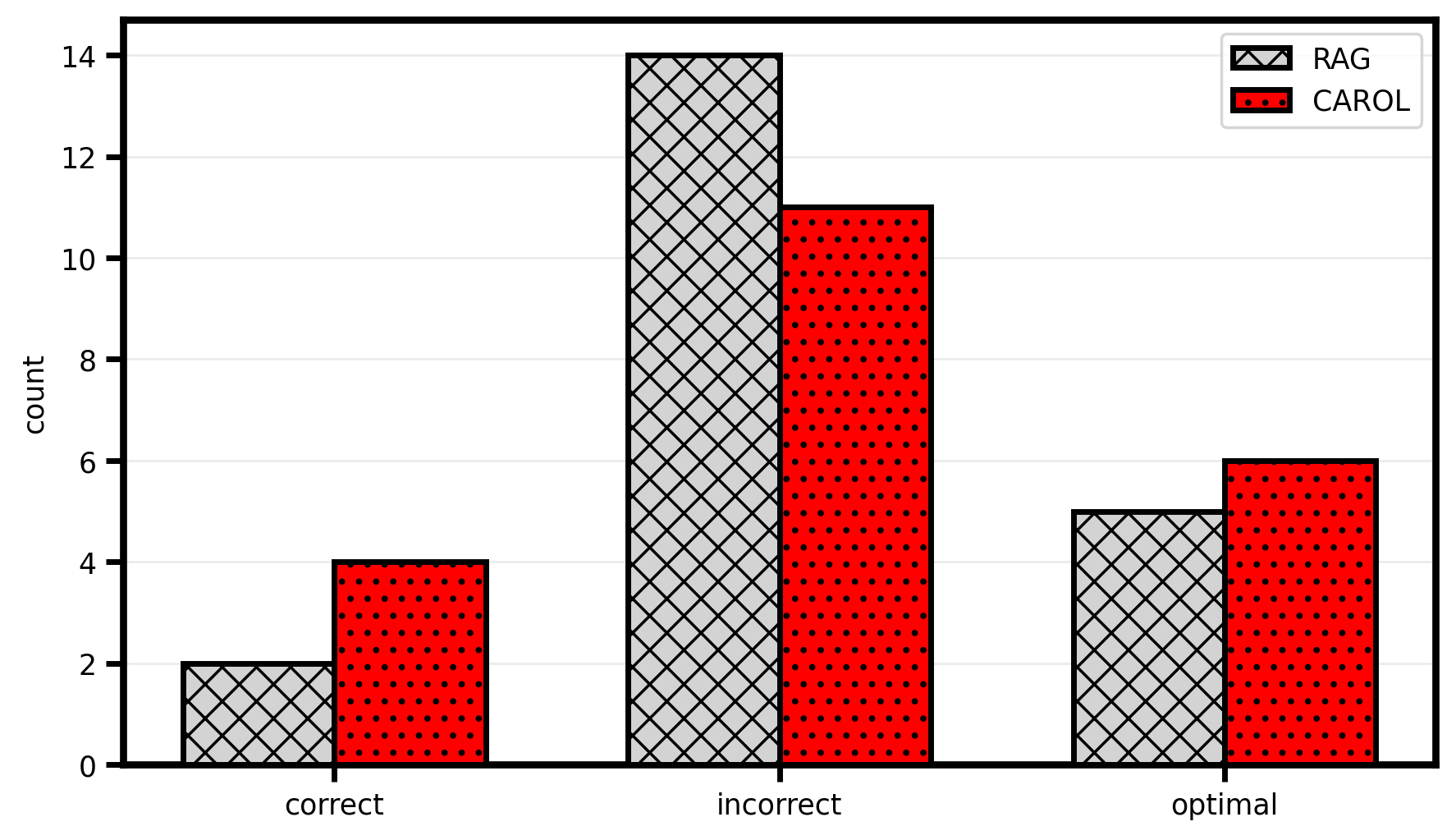}
            \caption{Language}
        \end{subfigure}
        \begin{subfigure}{0.3\textwidth}
            \centering
            \includegraphics[width=\linewidth]{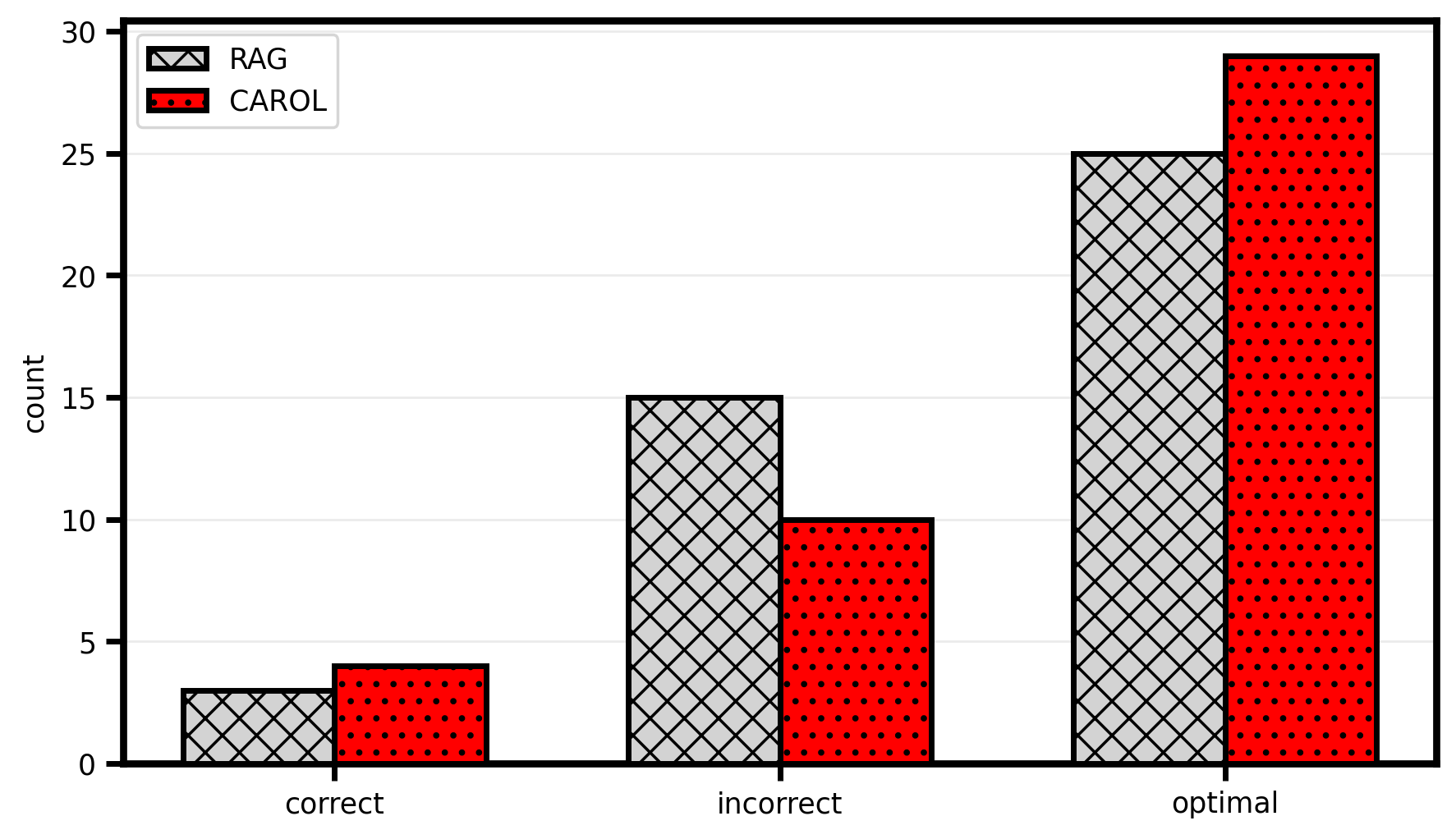}
            \caption{Law}
        \end{subfigure}
        \begin{subfigure}{0.3\textwidth}
            \centering
            \includegraphics[width=\linewidth]{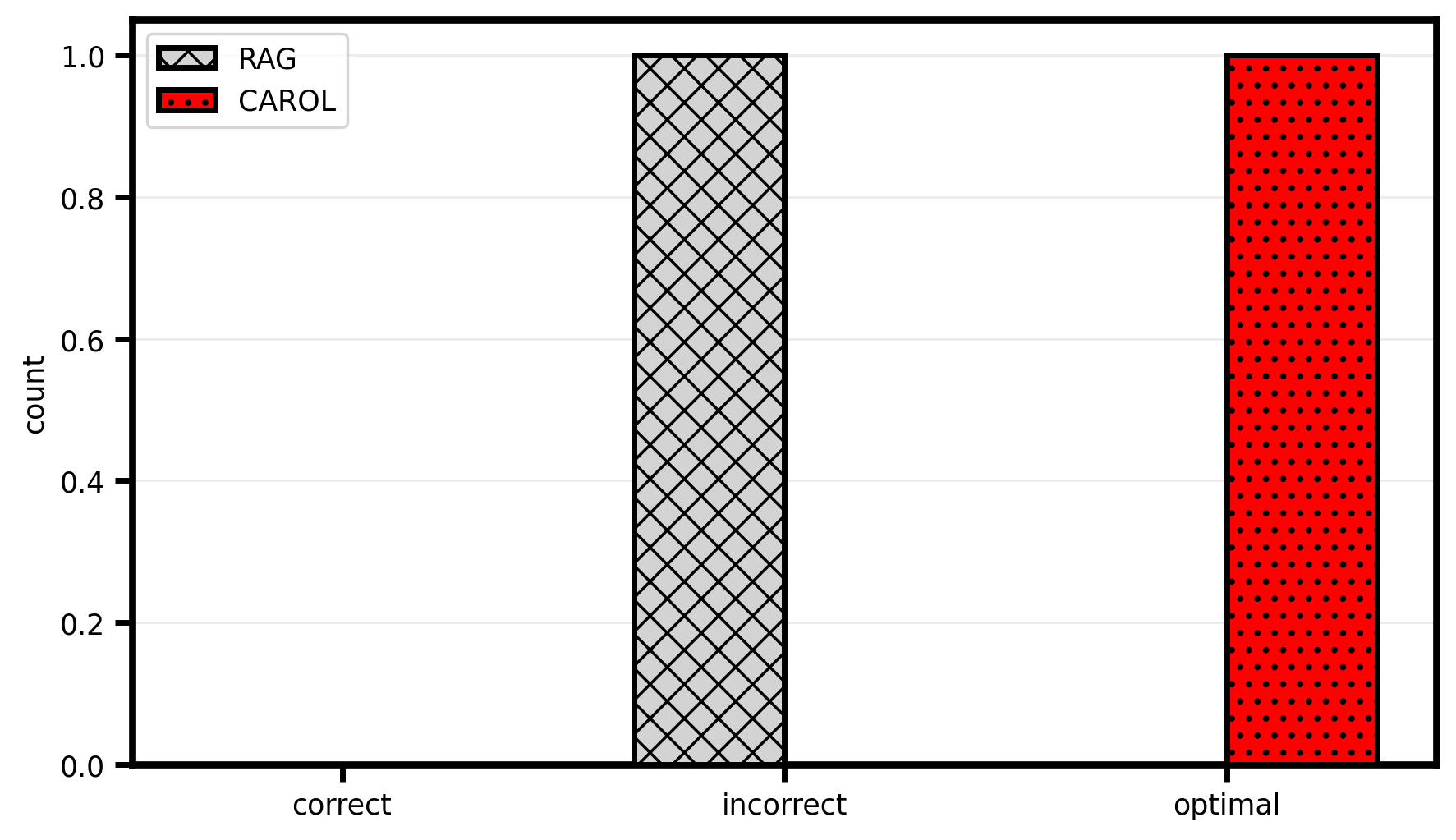}
            \caption{Logical Falsehood}
        \end{subfigure}
    \end{subcaptiongroup}
    \begin{subcaptiongroup}
        \begin{subfigure}{0.3\textwidth}
            \centering
            \includegraphics[width=\linewidth]{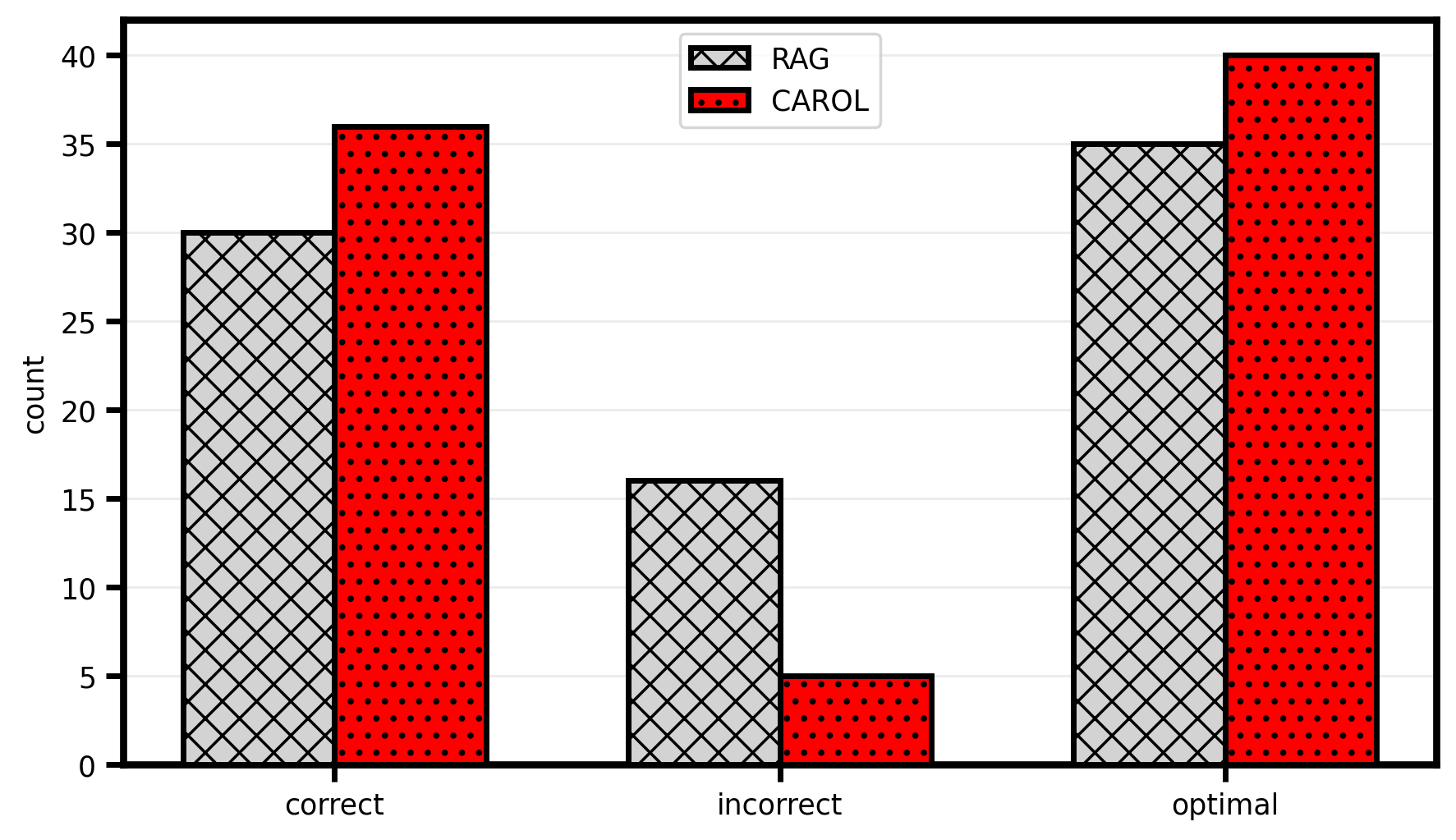}
            \caption{Misconceptions}
        \end{subfigure}
        \begin{subfigure}{0.3\textwidth}
            \centering
            \includegraphics[width=\linewidth]{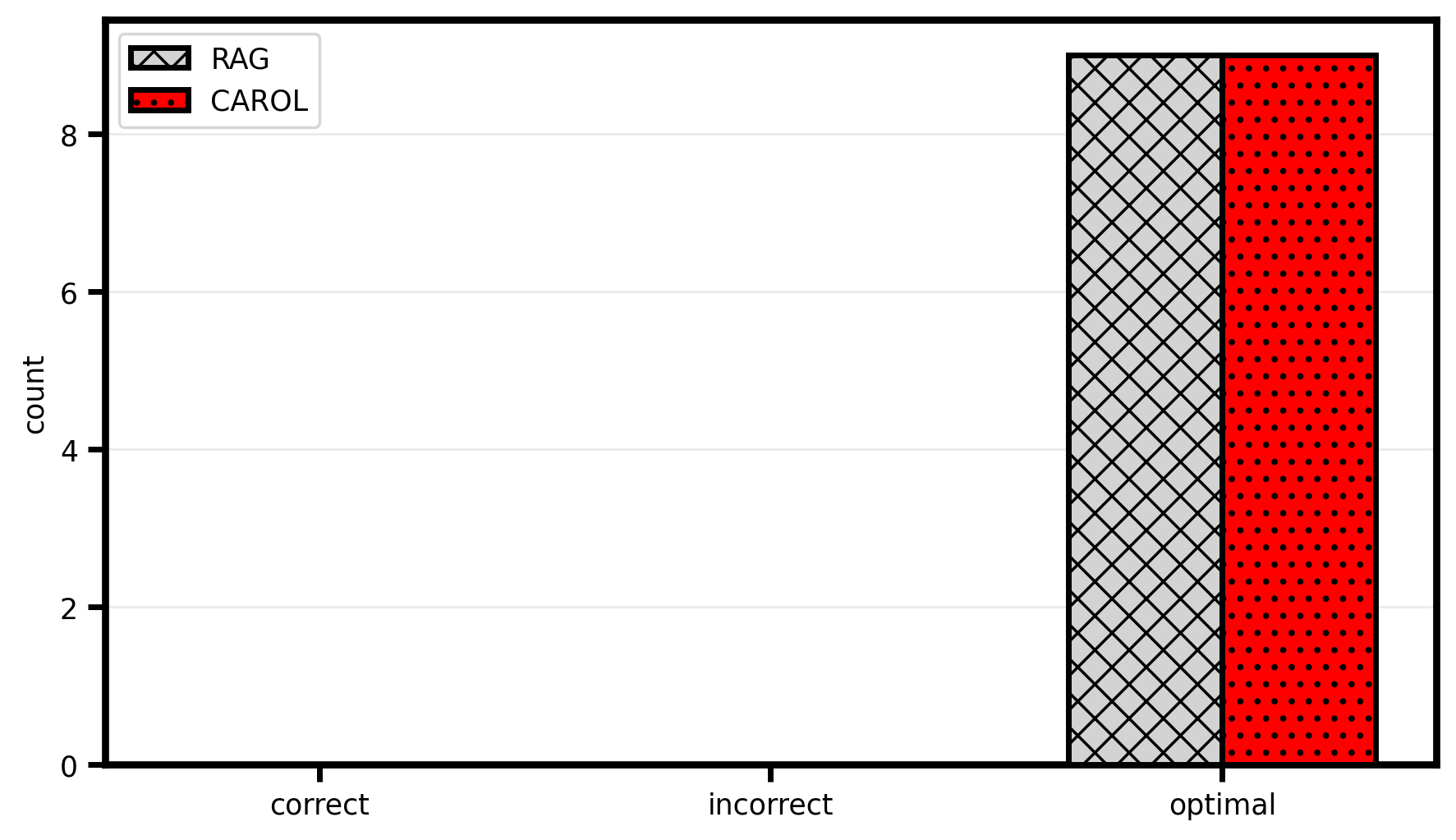}
            \caption{Misinformation}
        \end{subfigure}
        \begin{subfigure}{0.3\textwidth}
            \centering
            \includegraphics[width=\linewidth]{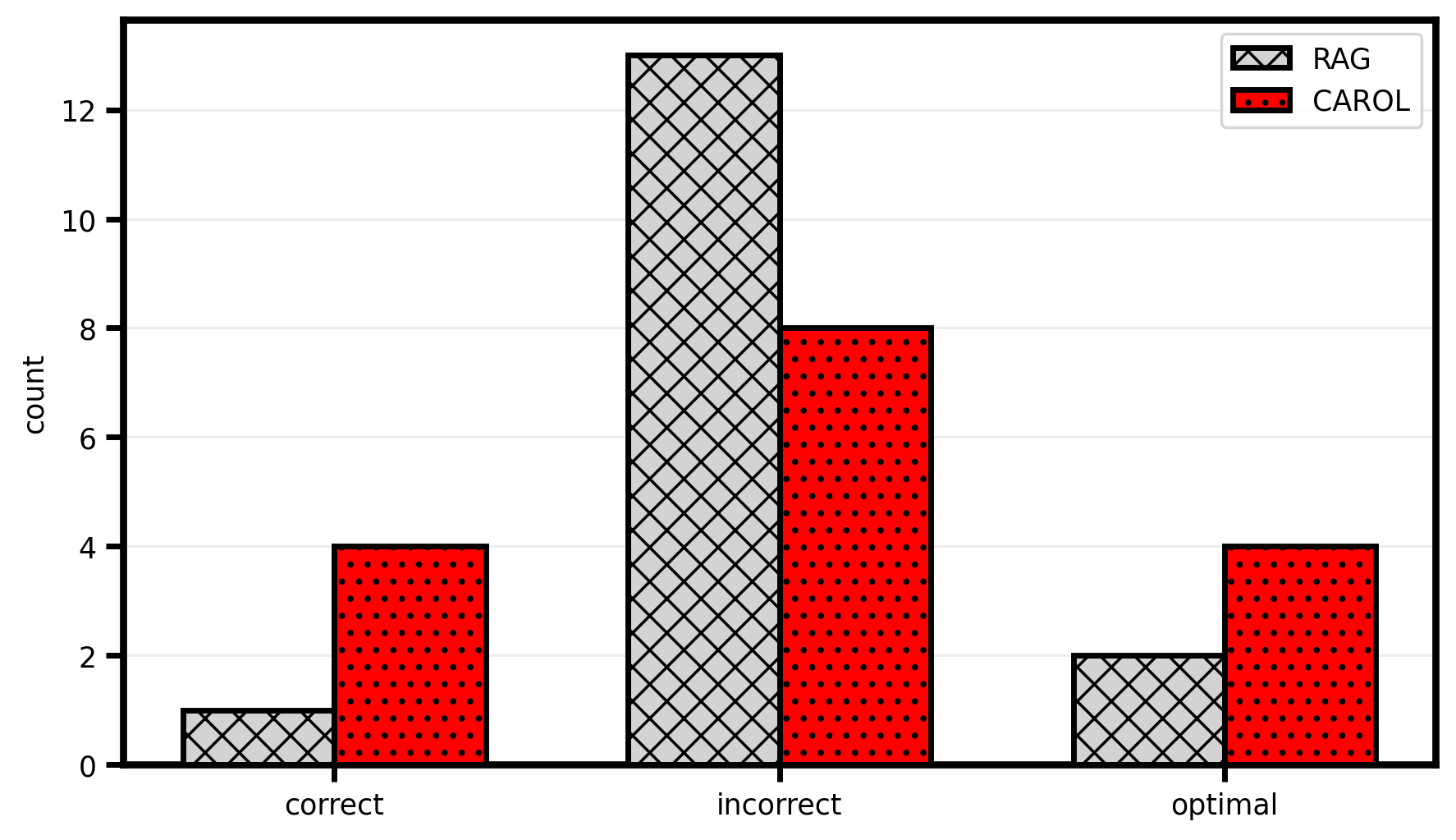}
            \caption{Misquotations}
        \end{subfigure}
    \end{subcaptiongroup}
    \begin{subcaptiongroup}
        \begin{subfigure}{0.3\textwidth}
            \centering
            \includegraphics[width=\linewidth]{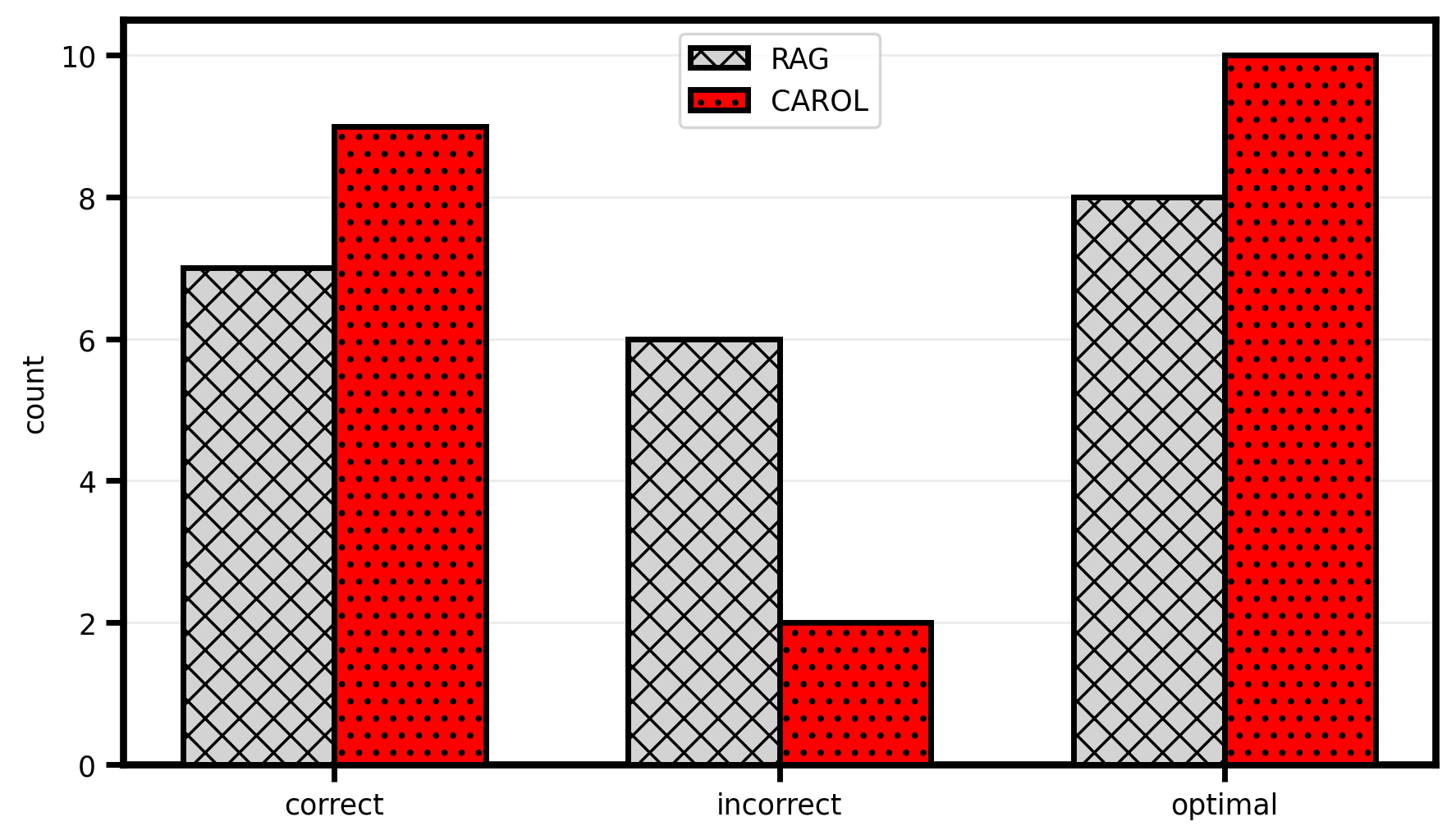}
            \caption{Myths}
        \end{subfigure}
        \begin{subfigure}{0.3\textwidth}
            \centering
            \includegraphics[width=\linewidth]{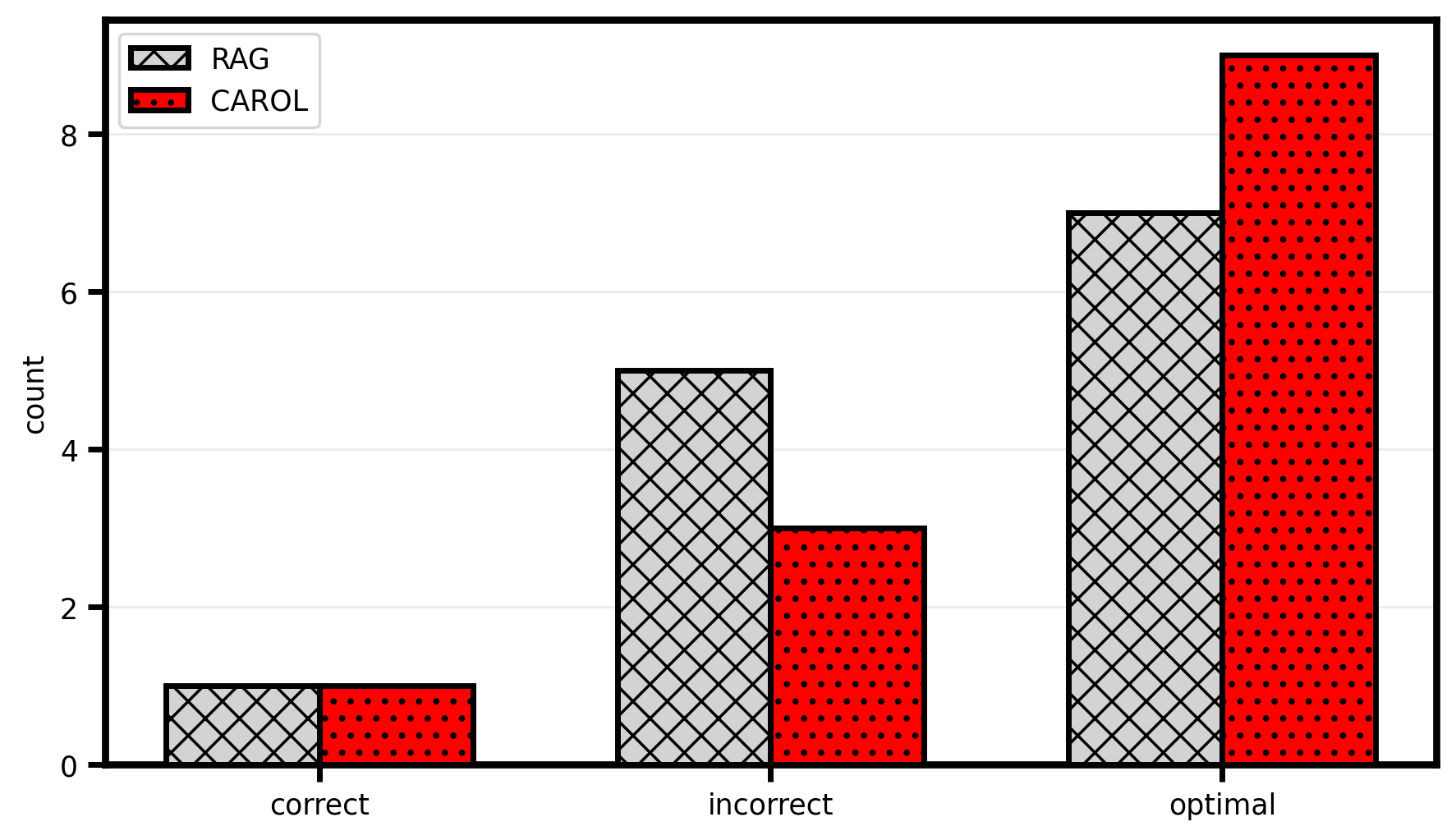}
            \caption{Nutrition}
        \end{subfigure}
        \begin{subfigure}{0.3\textwidth}
            \centering
            \includegraphics[width=\linewidth]{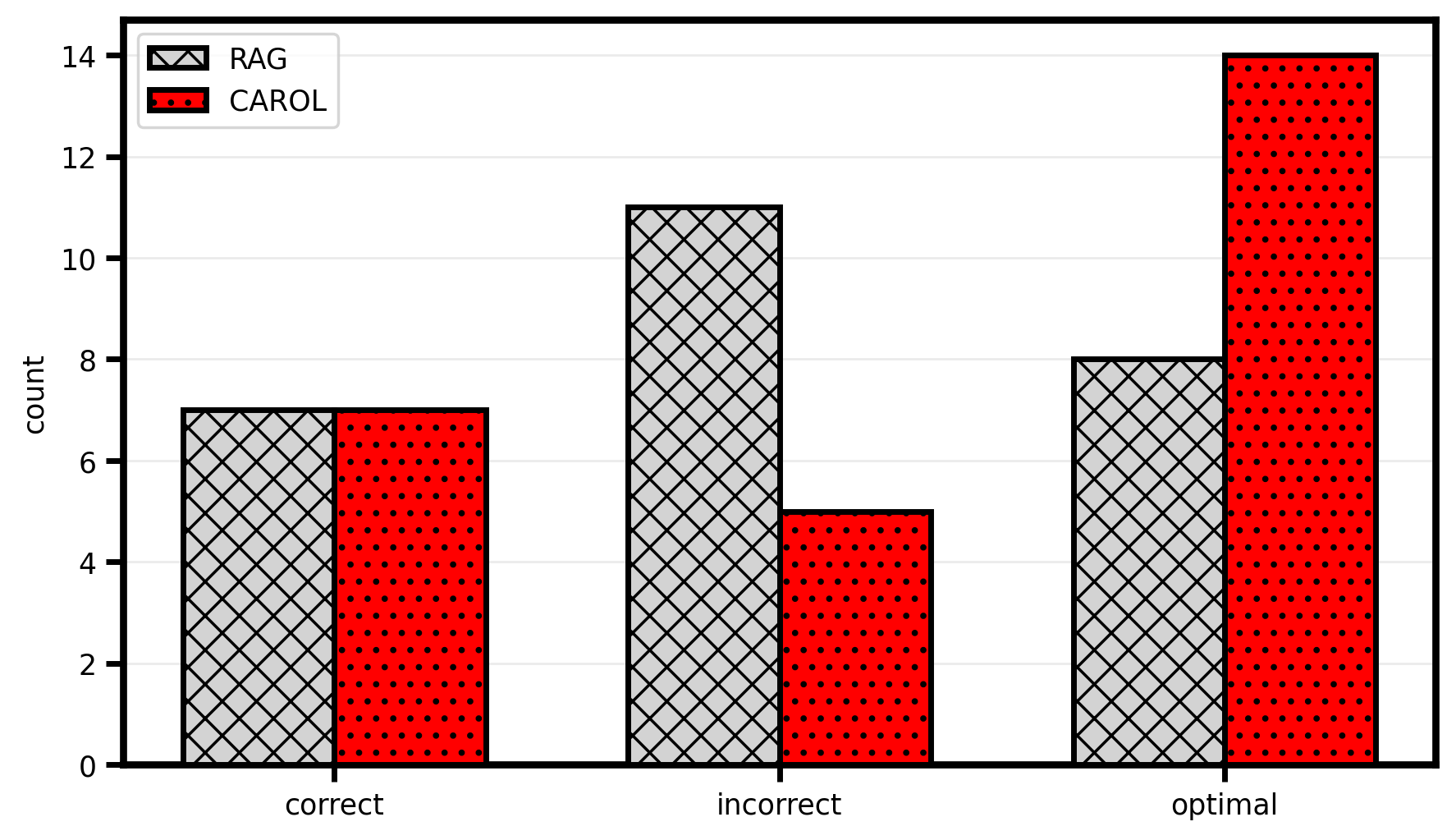}
            \caption{Paranormal}
        \end{subfigure}
    \end{subcaptiongroup}
    \begin{subcaptiongroup}
        \begin{subfigure}{0.3\textwidth}
            \centering
            \includegraphics[width=\linewidth]{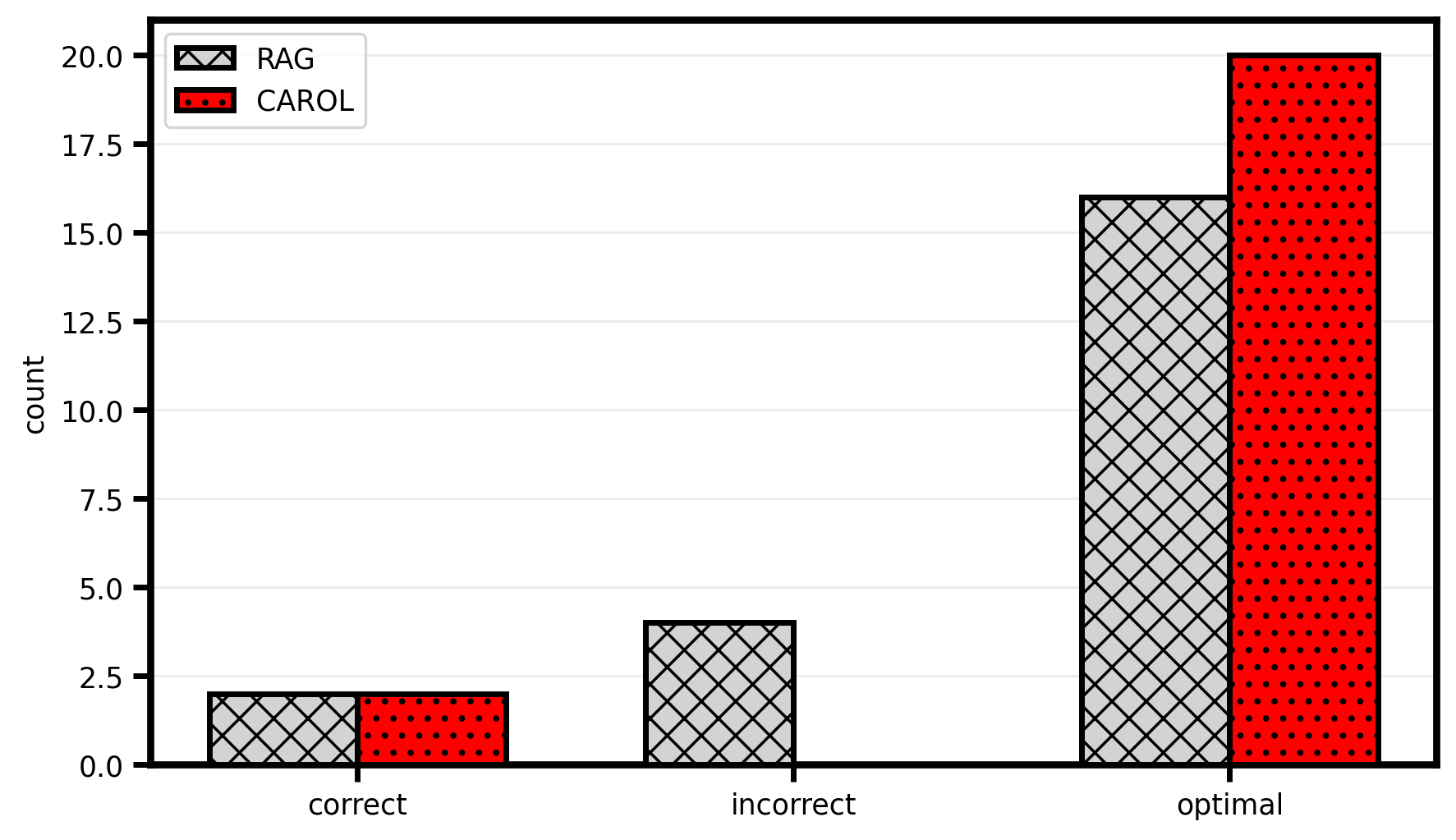}
            \caption{People}
        \end{subfigure}
        \begin{subfigure}{0.3\textwidth}
            \centering
            \includegraphics[width=\linewidth]{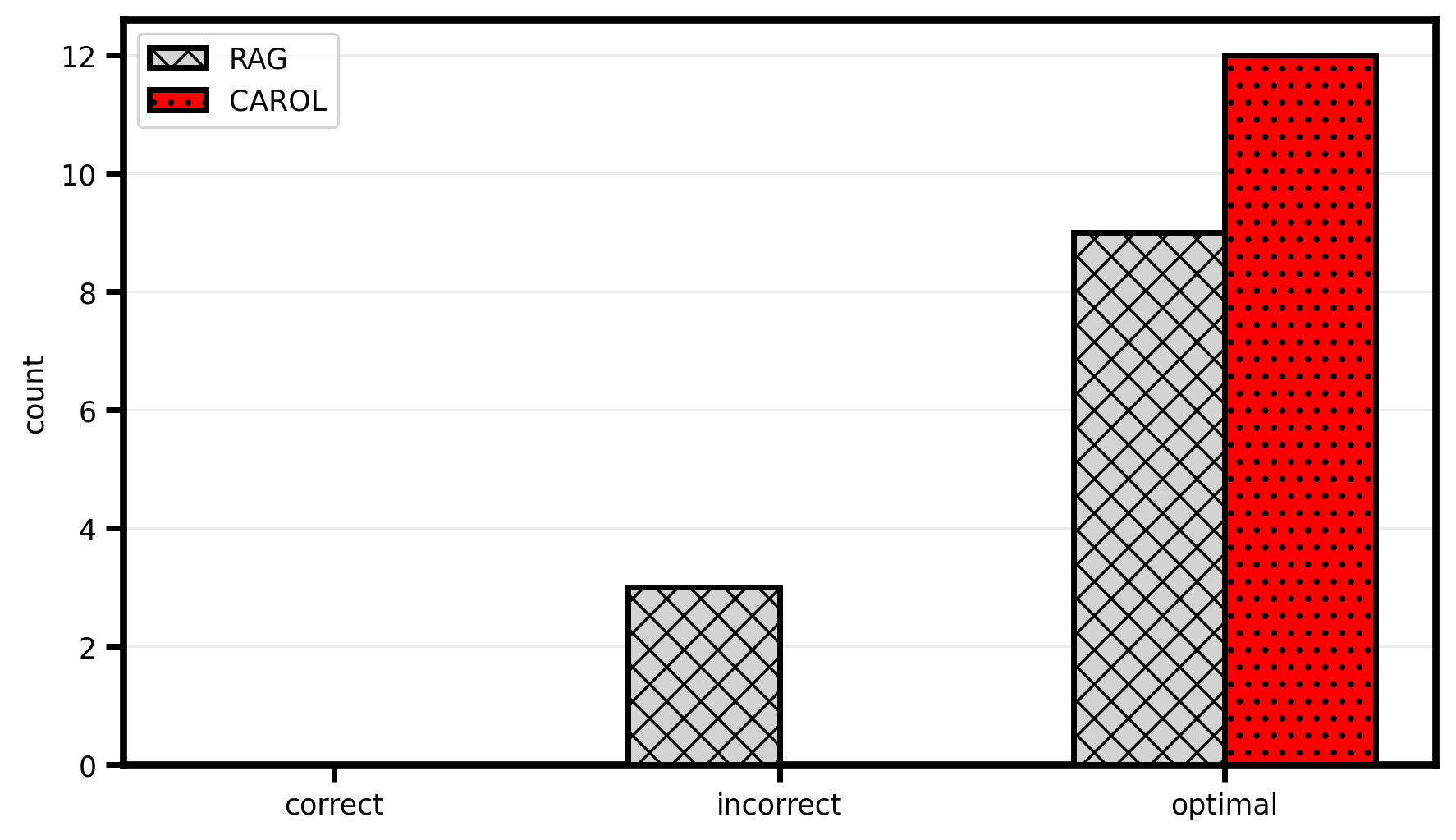}
            \caption{Places}
        \end{subfigure}
        \begin{subfigure}{0.3\textwidth}
            \centering
            \includegraphics[width=\linewidth]{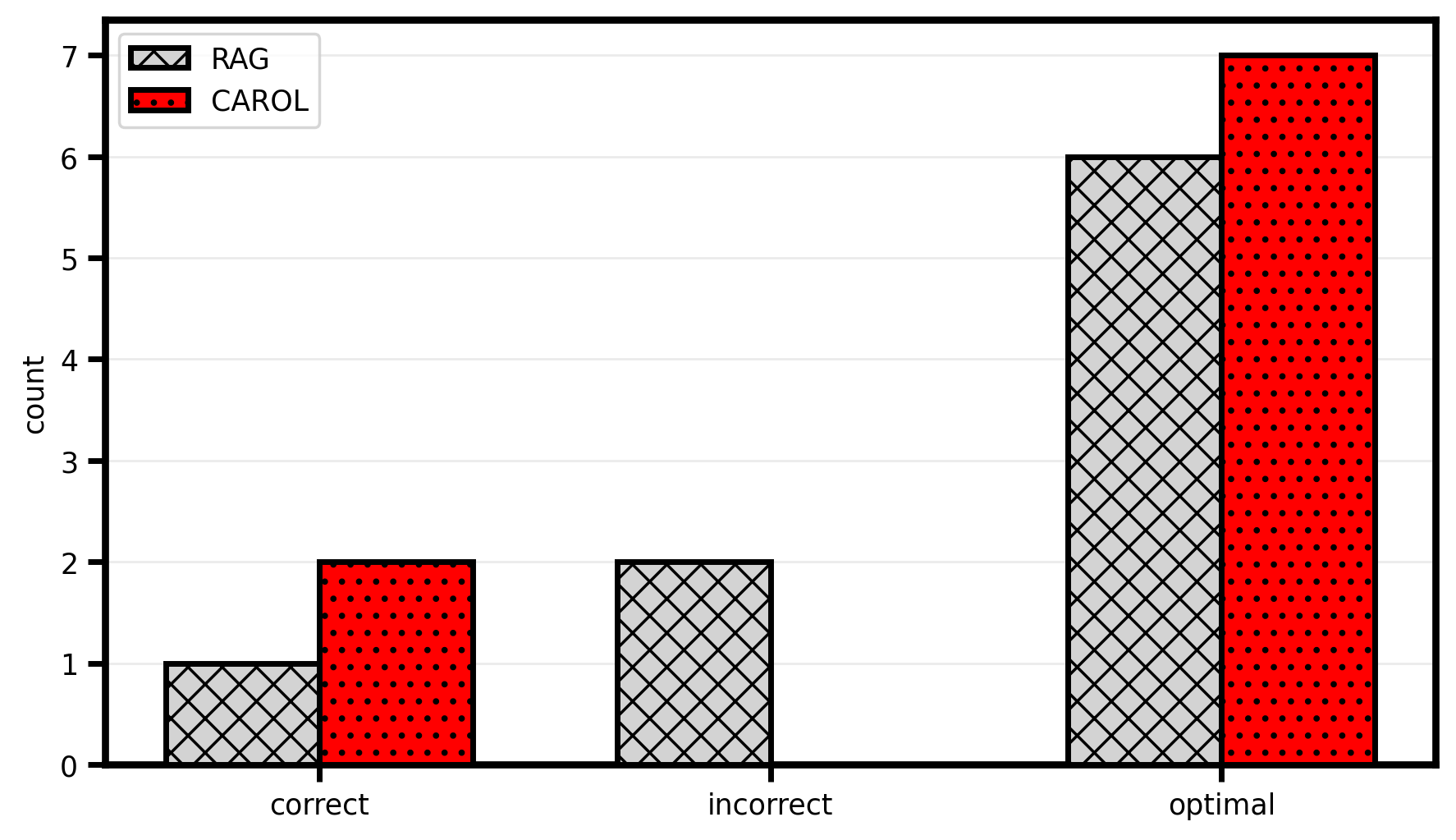}
            \caption{Politics}
        \end{subfigure}
    \end{subcaptiongroup}
\caption{\textbf{TruthfulQA} on GPT-5-nano extended results in each category, part $2$.}
\end{figure}

\begin{figure}[ht]
    \centering
    \begin{subcaptiongroup}
        \begin{subfigure}{0.3\textwidth}
            \centering
            \includegraphics[width=\linewidth]{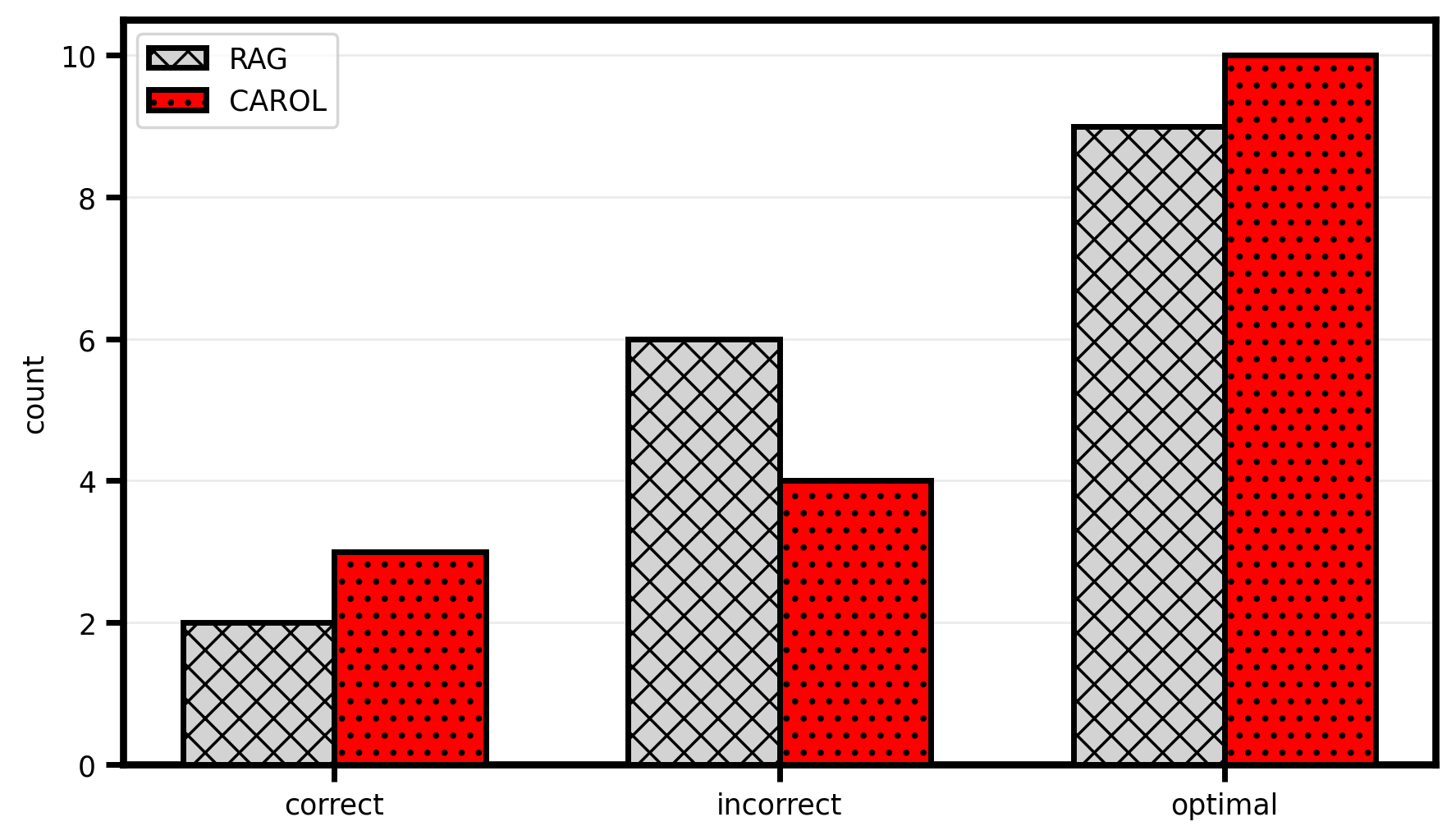}
            \caption{Proverbs}
        \end{subfigure}
        \begin{subfigure}{0.3\textwidth}
            \centering
            \includegraphics[width=\linewidth]{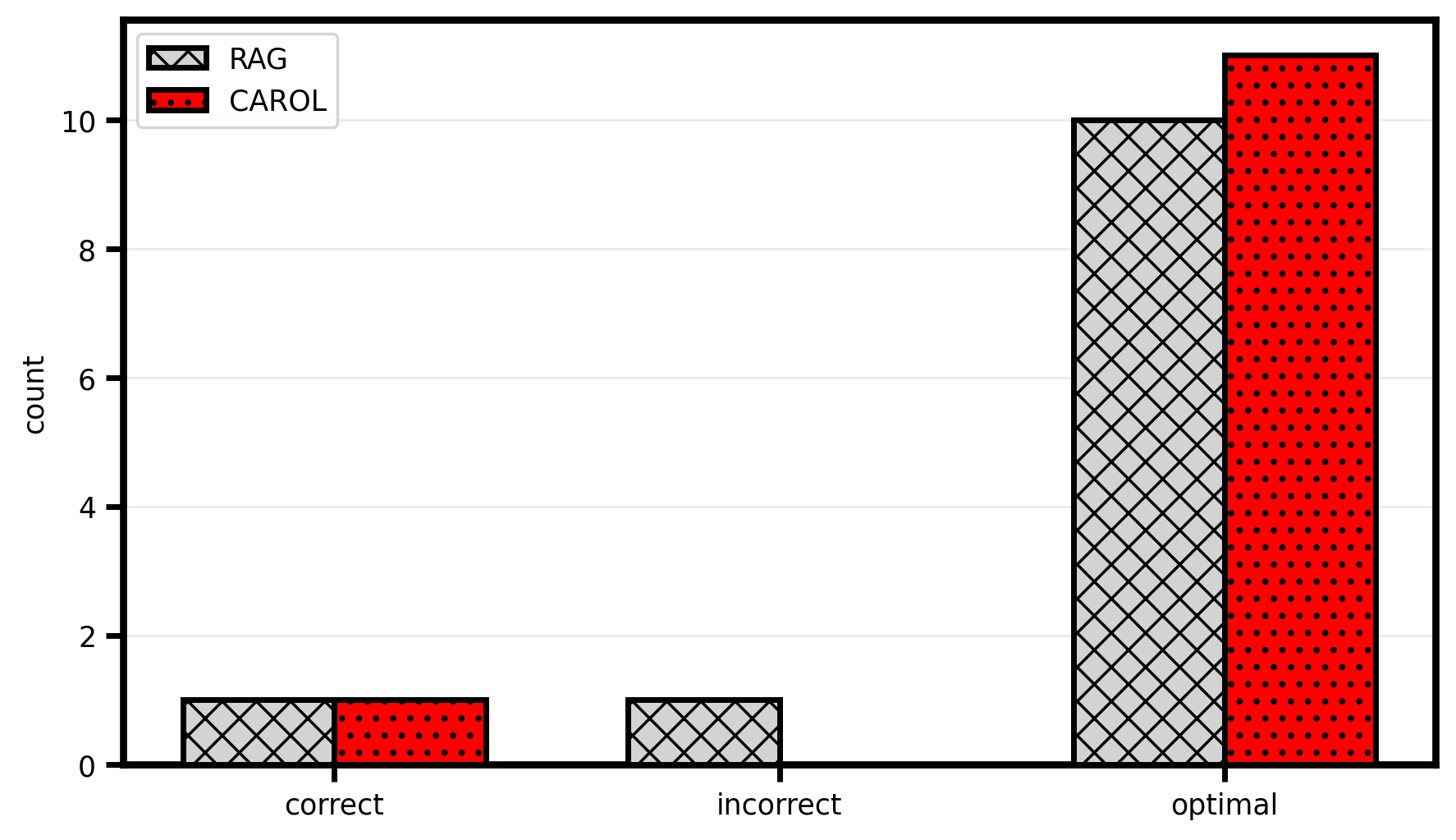}
            \caption{Psychology}
        \end{subfigure}
        \begin{subfigure}{0.3\textwidth}
            \centering
            \includegraphics[width=\linewidth]{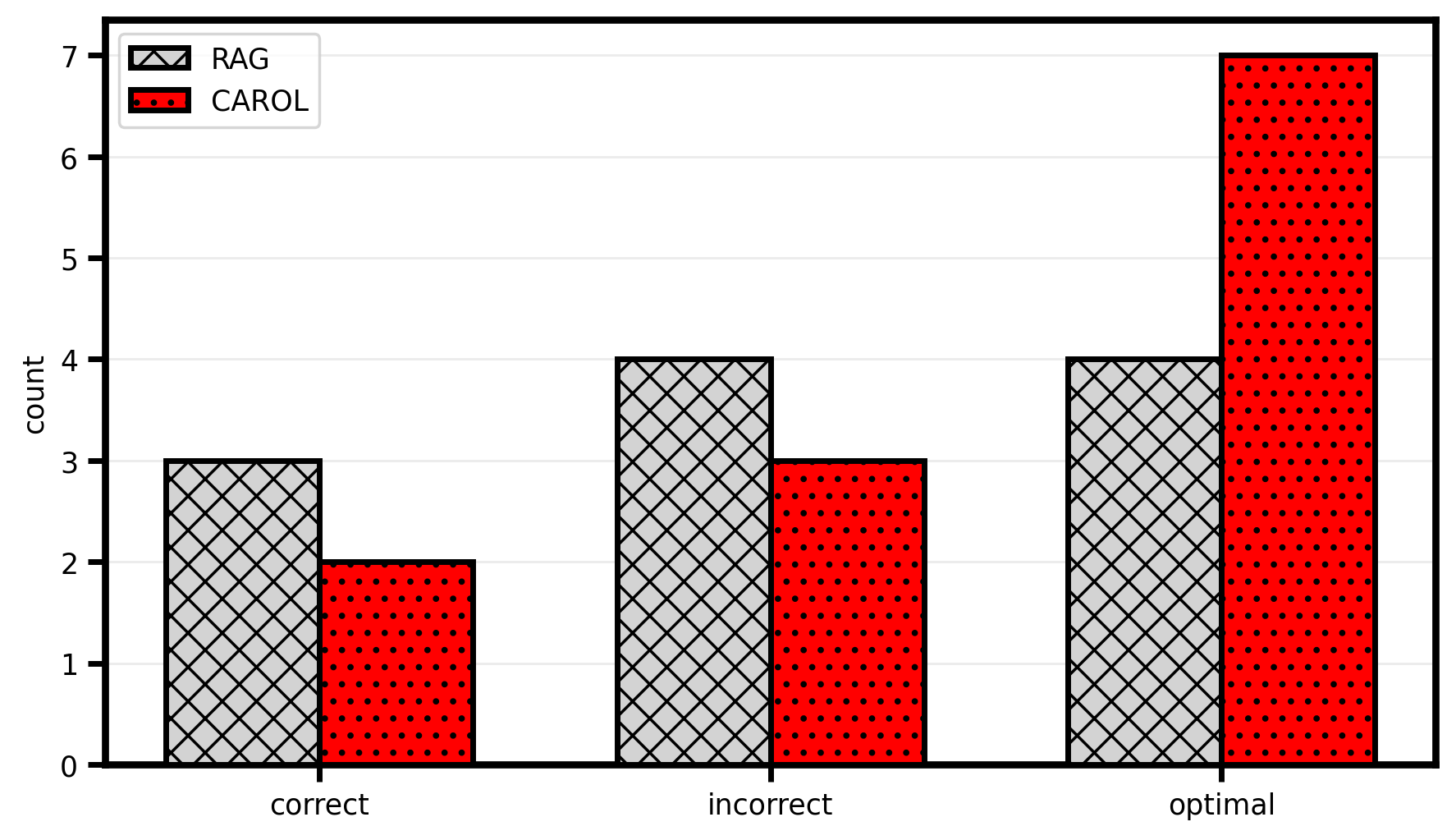}
            \caption{Religion}
        \end{subfigure}
    \end{subcaptiongroup}
    \begin{subcaptiongroup}
        \begin{subfigure}{0.3\textwidth}
            \centering
            \includegraphics[width=\linewidth]{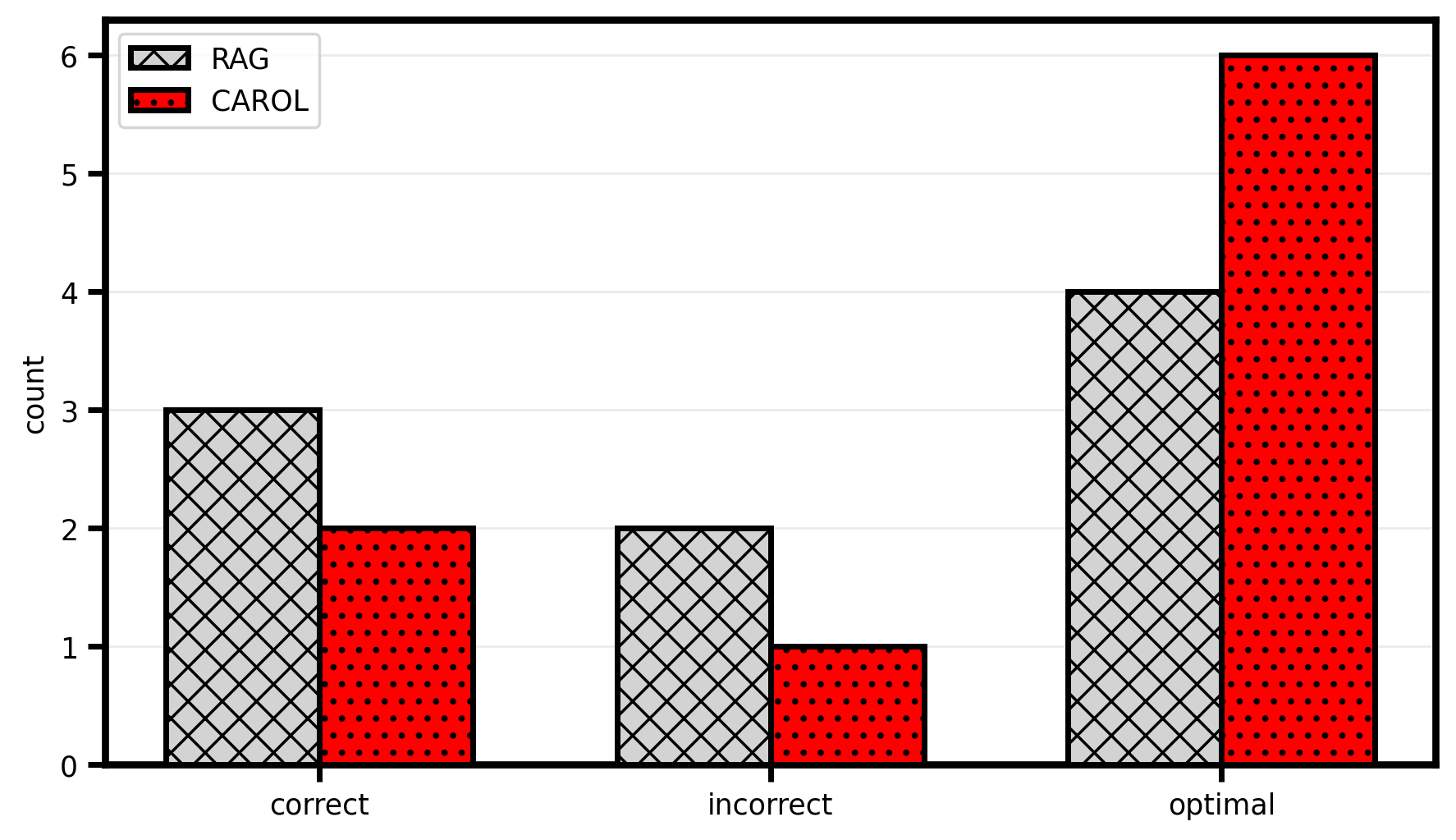}
            \caption{Science}
        \end{subfigure}
        \begin{subfigure}{0.3\textwidth}
            \centering
            \includegraphics[width=\linewidth]{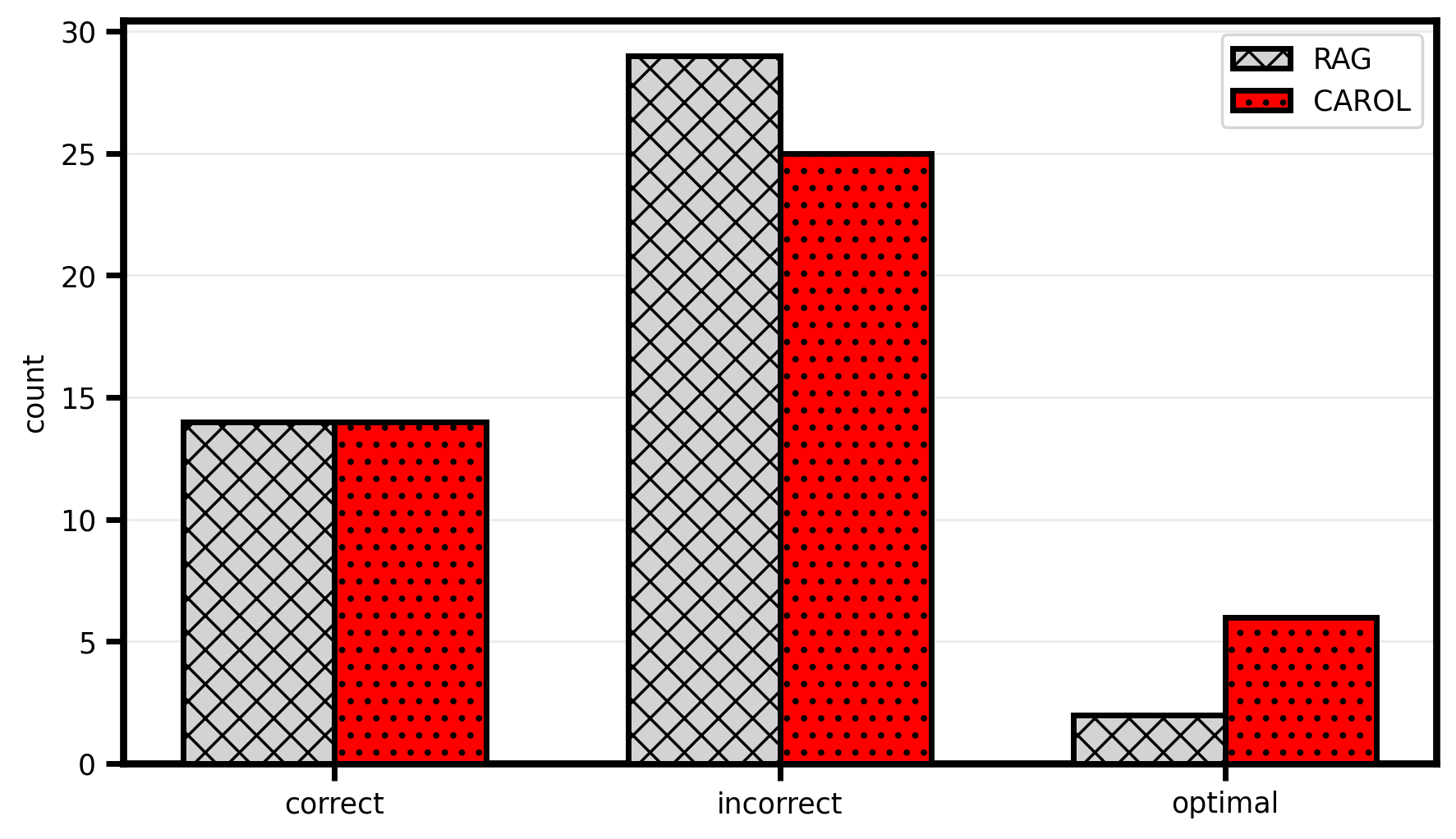}
            \caption{Sociology}
        \end{subfigure}
        \begin{subfigure}{0.3\textwidth}
            \centering
            \includegraphics[width=\linewidth]{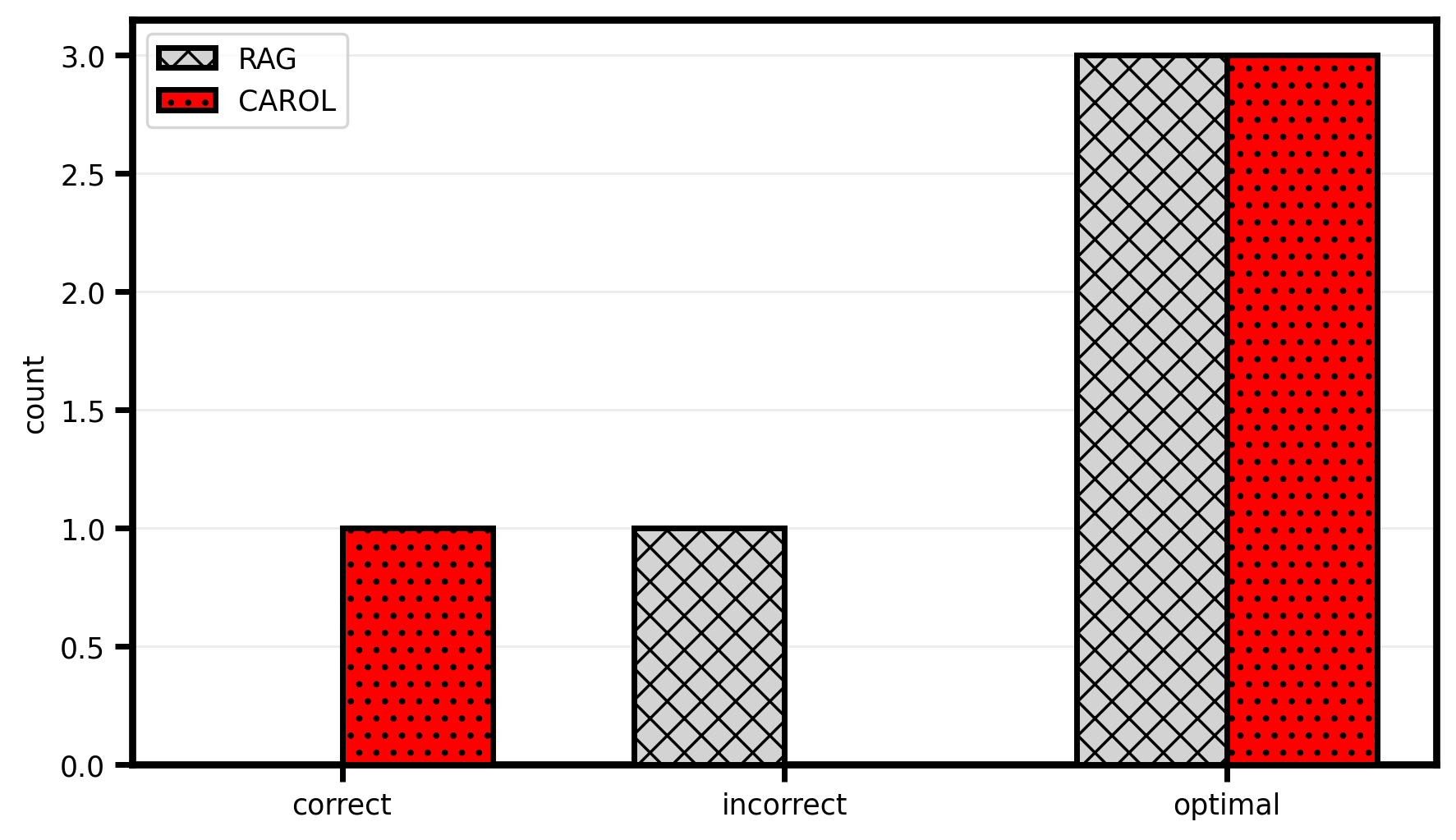}
            \caption{Statistics}
        \end{subfigure}
    \end{subcaptiongroup}
    \begin{subcaptiongroup}
        \begin{subfigure}{0.3\textwidth}
            \centering
            \includegraphics[width=\linewidth]{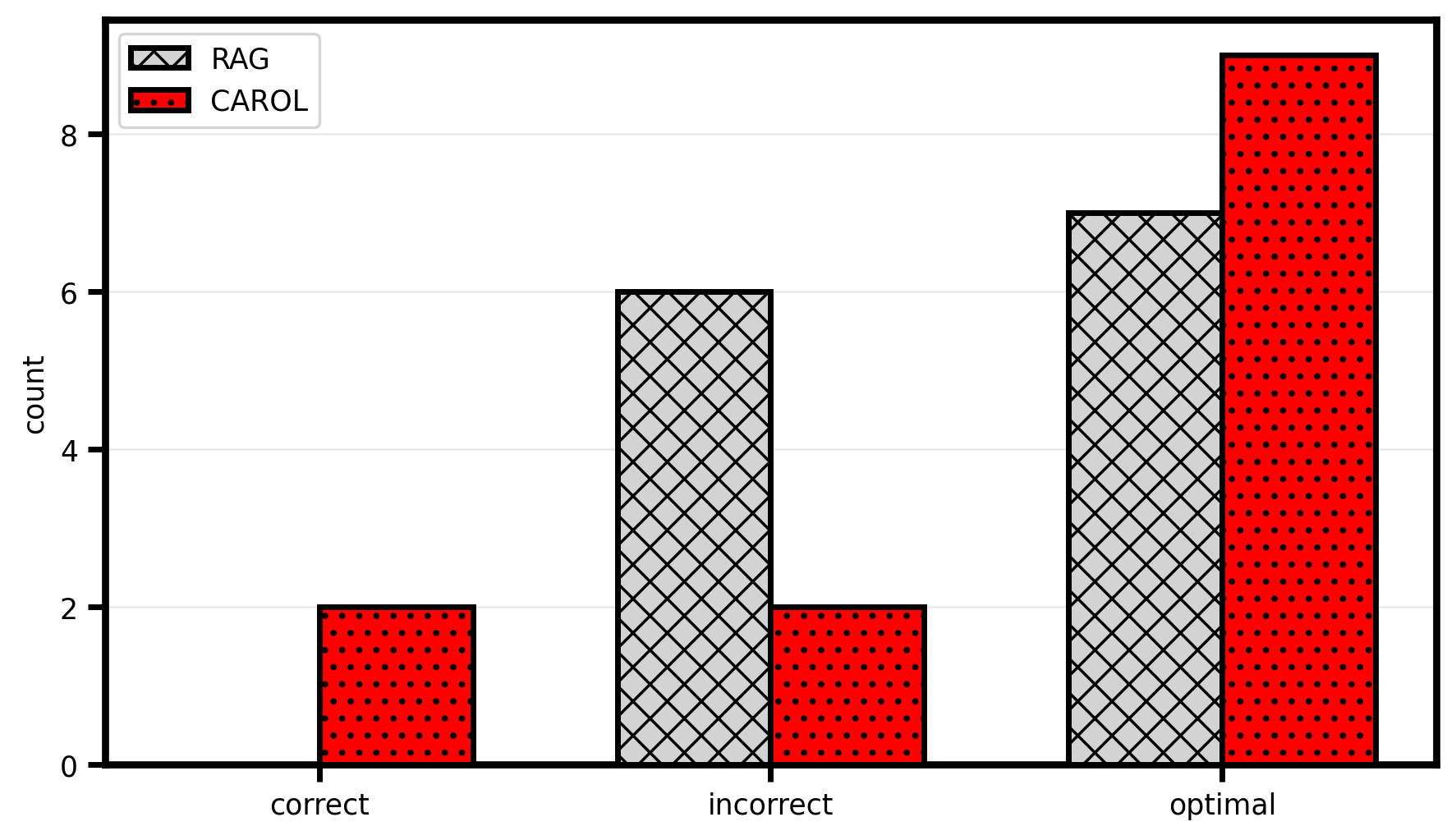}
            \caption{Stereotypes}
        \end{subfigure}
        \begin{subfigure}{0.3\textwidth}
            \centering
            \includegraphics[width=\linewidth]{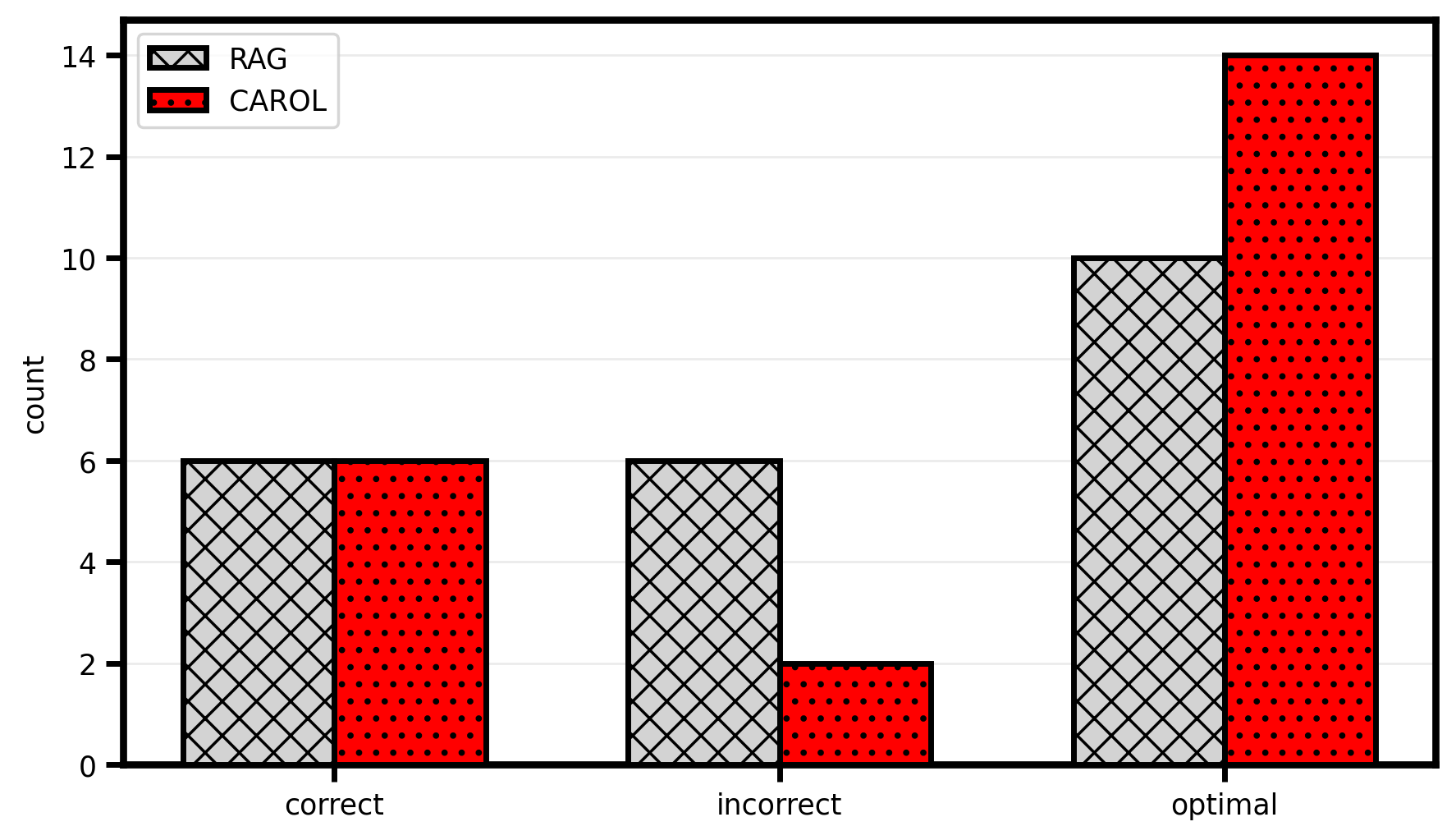}
            \caption{Superstition}
        \end{subfigure}
        \begin{subfigure}{0.3\textwidth}
            \centering
            \includegraphics[width=\linewidth]{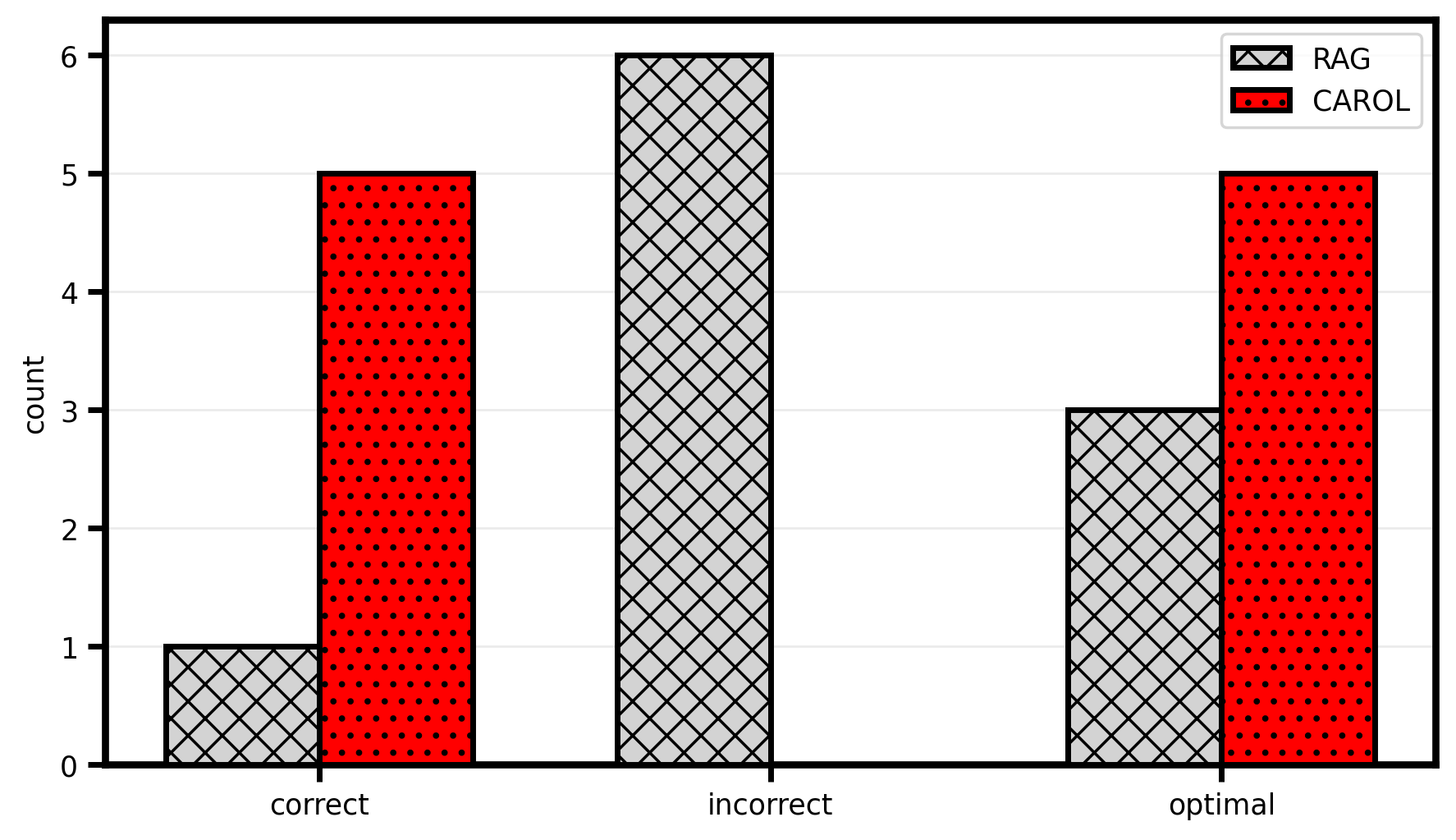}
            \caption{Subjective}
        \end{subfigure}
    \end{subcaptiongroup}
    \begin{subcaptiongroup}
        \begin{subfigure}{0.3\textwidth}
            \centering
            \includegraphics[width=\linewidth]{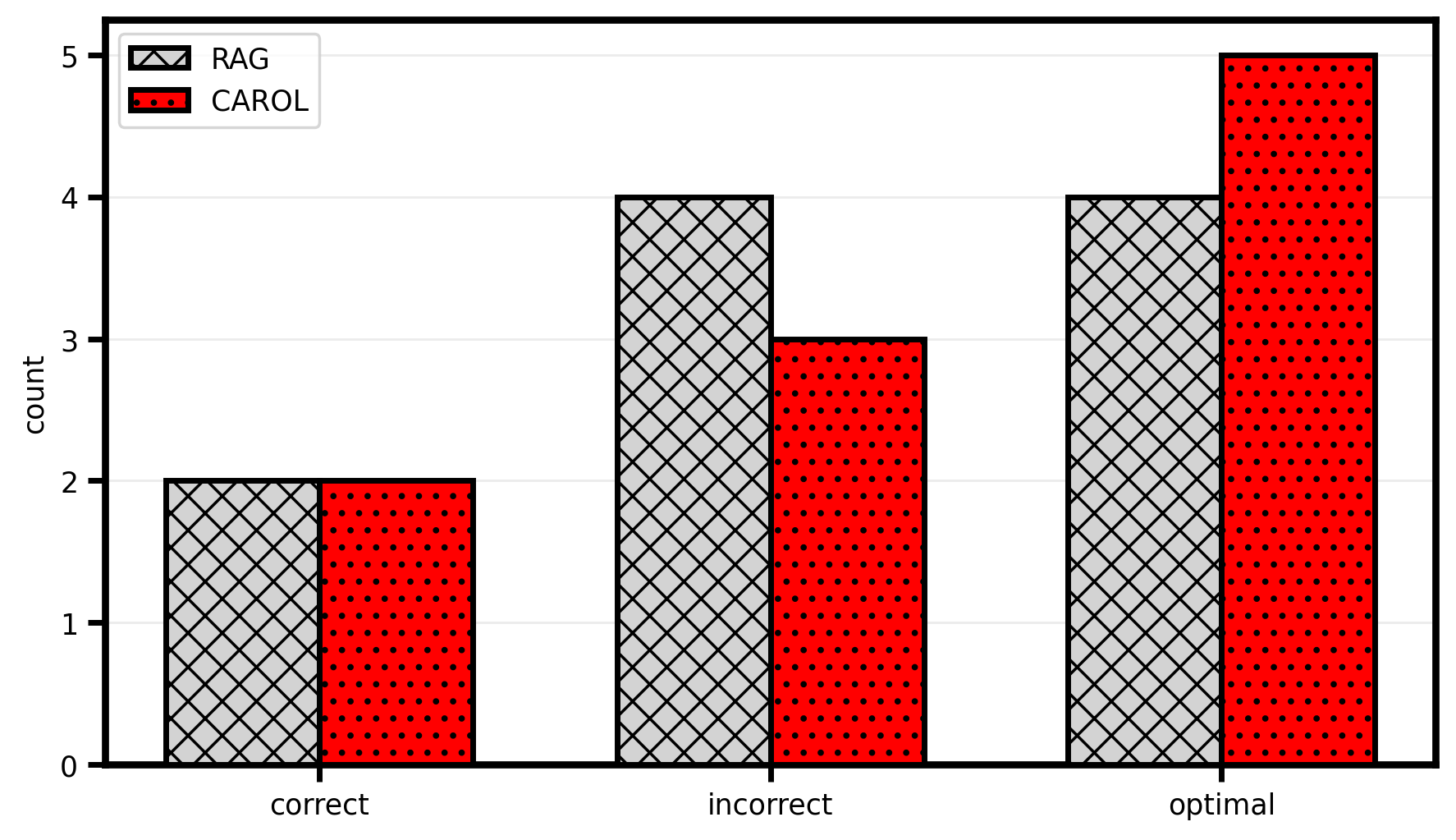}
            \caption{Weather}
        \end{subfigure}
    \end{subcaptiongroup}
\caption{\textbf{TruthfulQA} on GPT-5-nano extended results in each category, part $3$.}
\end{figure}

\begin{figure}[ht]
    \centering
    \begin{subcaptiongroup}
        \begin{subfigure}{0.45\textwidth}
            \centering
            \includegraphics[width=\linewidth]{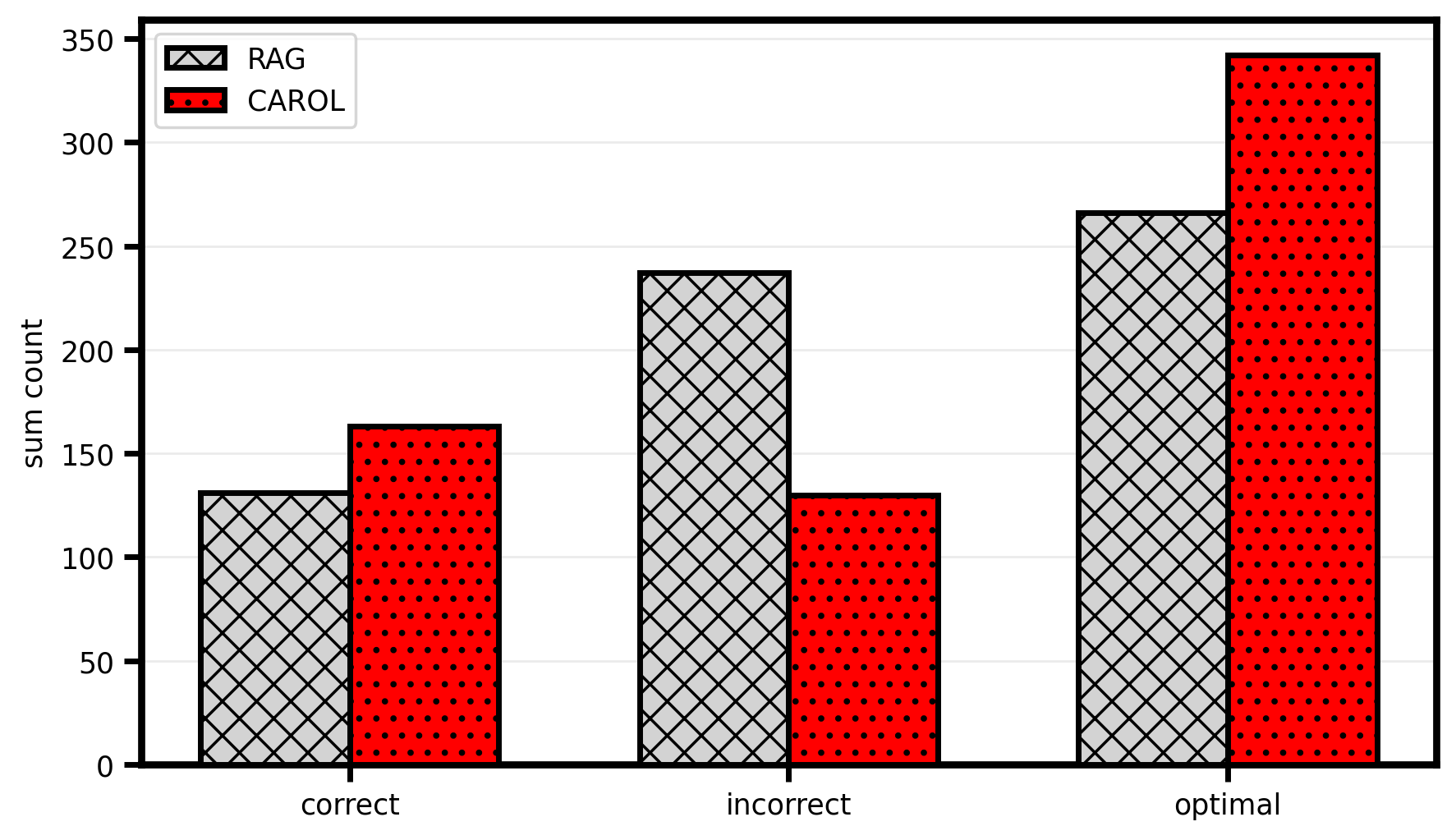}
            \caption{Results}
        \end{subfigure}
        \begin{subfigure}{0.45\textwidth}
            \centering
            \includegraphics[width=\linewidth]{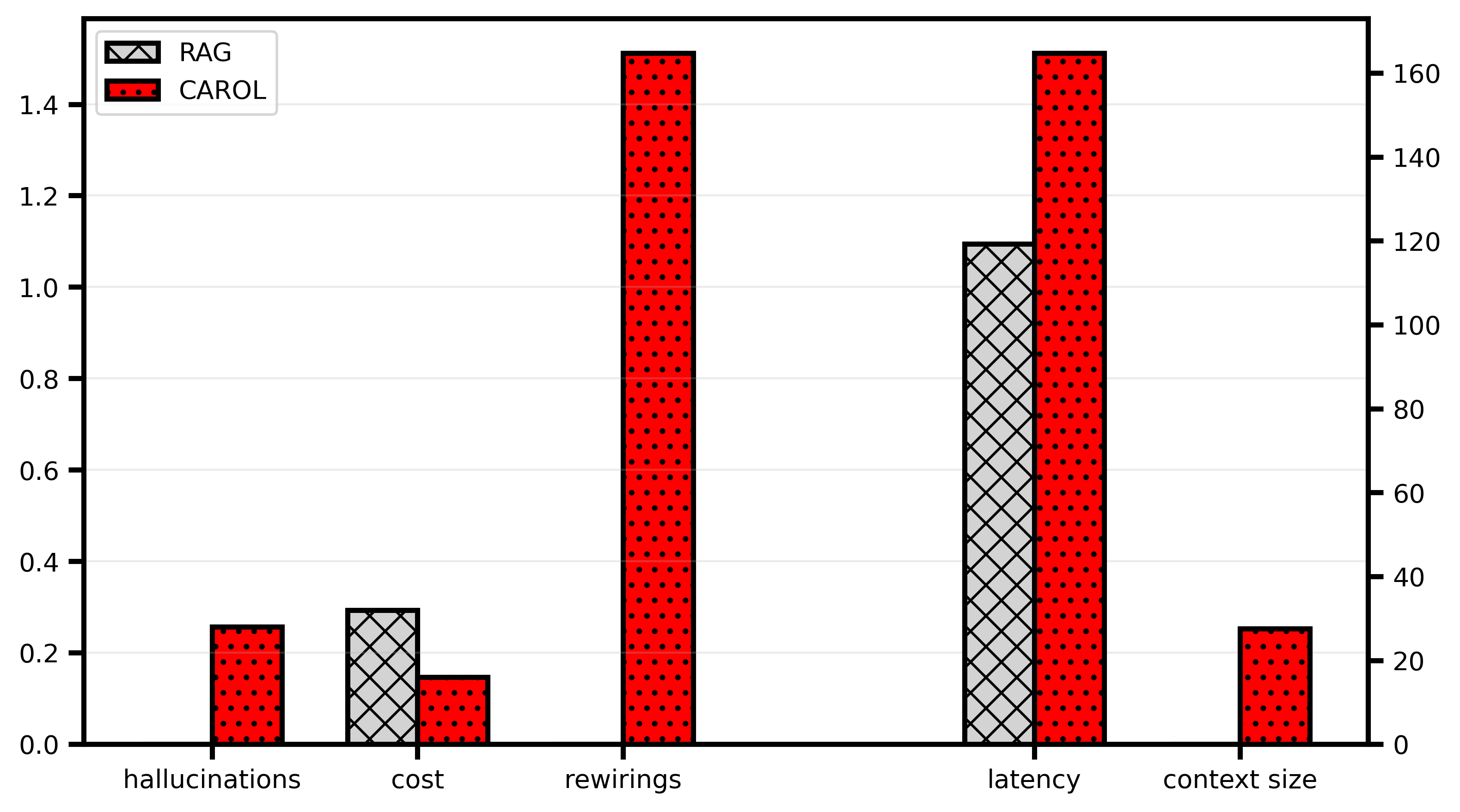}
            \caption{Execution Metrics}
        \end{subfigure}
    \end{subcaptiongroup}
\caption{\textbf{TruthfulQA} on GPT-5-nano overall results.}
\end{figure}


\begin{figure}[ht]
    \centering
    \begin{subcaptiongroup}
        \begin{subfigure}{0.3\textwidth}
            \centering
            \includegraphics[width=\linewidth]{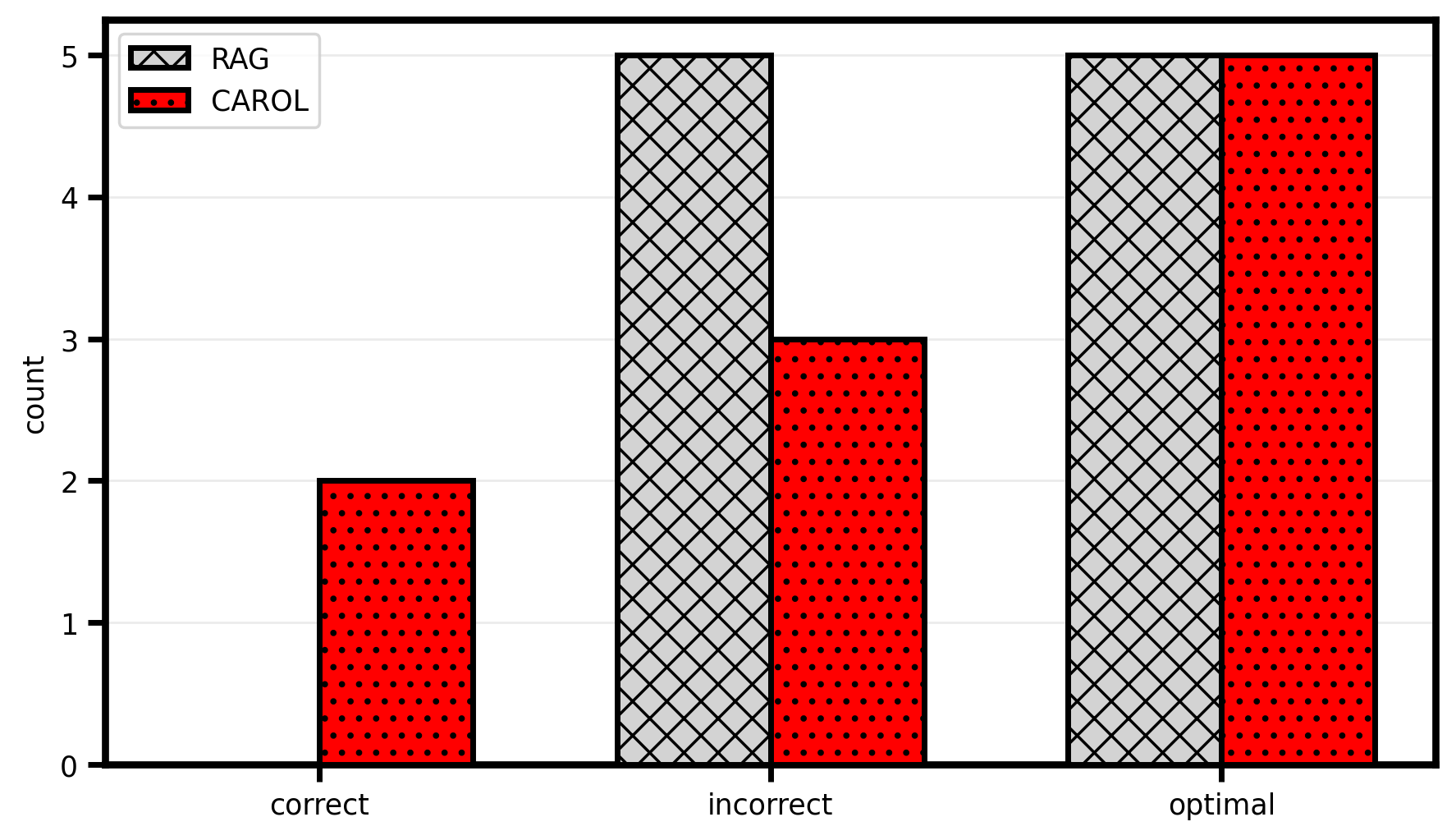}
            \caption{Advertising}
        \end{subfigure}
        \begin{subfigure}{0.3\textwidth}
            \centering
            \includegraphics[width=\linewidth]{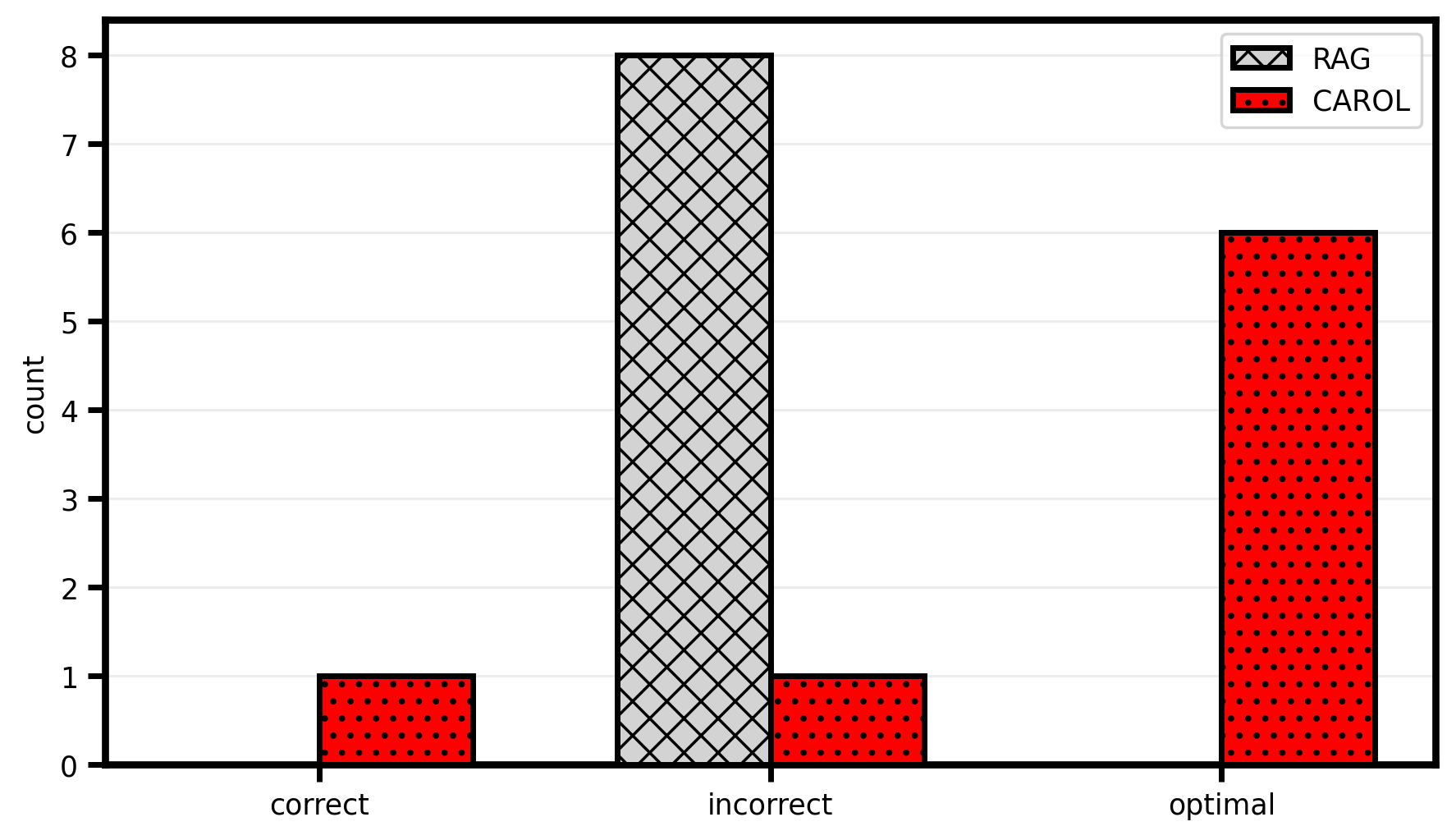}
            \caption{Confusion}
        \end{subfigure}
        \begin{subfigure}{0.3\textwidth}
            \centering
            \includegraphics[width=\linewidth]{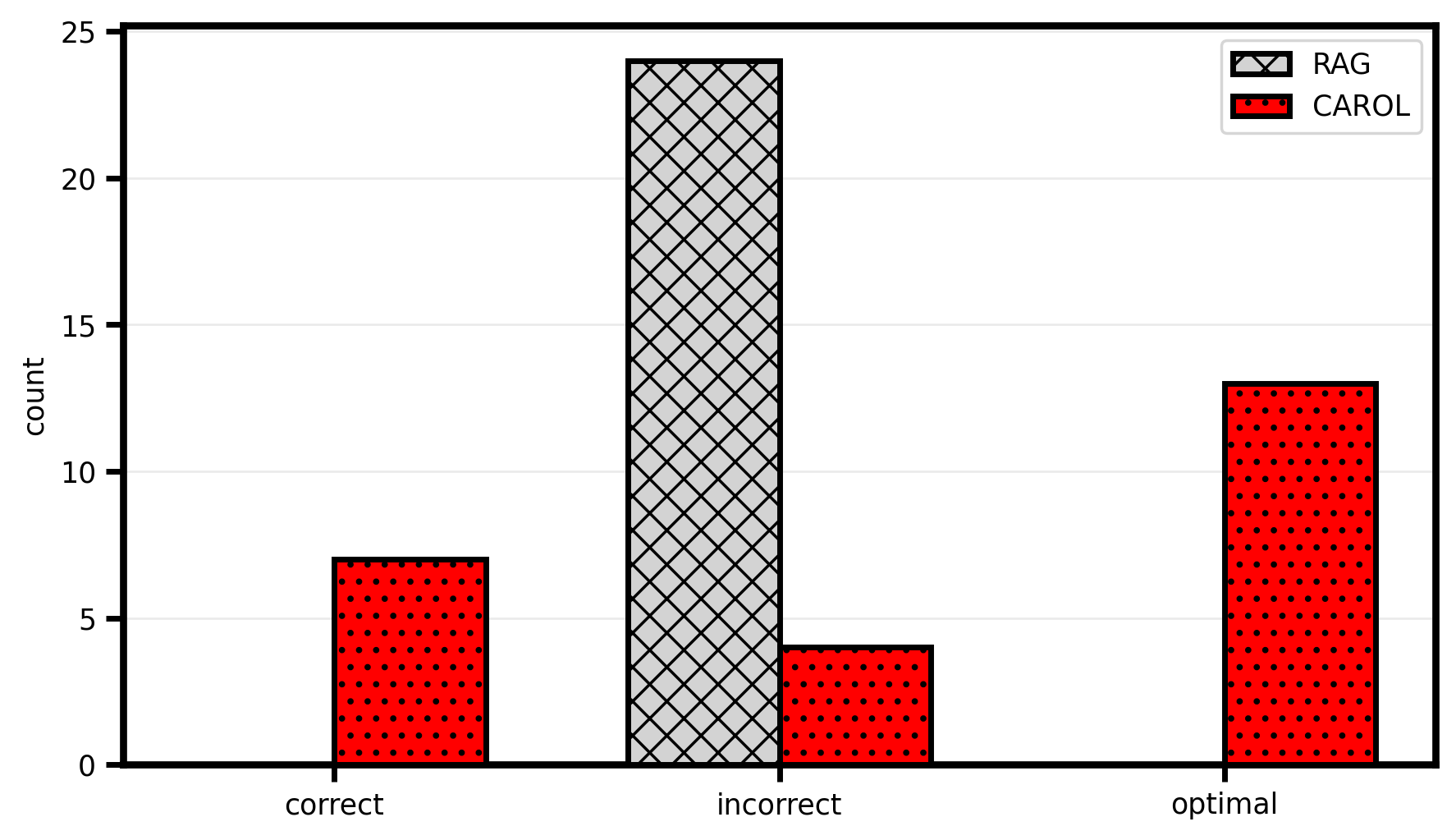}
            \caption{Conspiracies}
        \end{subfigure}
    \end{subcaptiongroup}
    \begin{subcaptiongroup}
        \begin{subfigure}{0.3\textwidth}
            \centering
            \includegraphics[width=\linewidth]{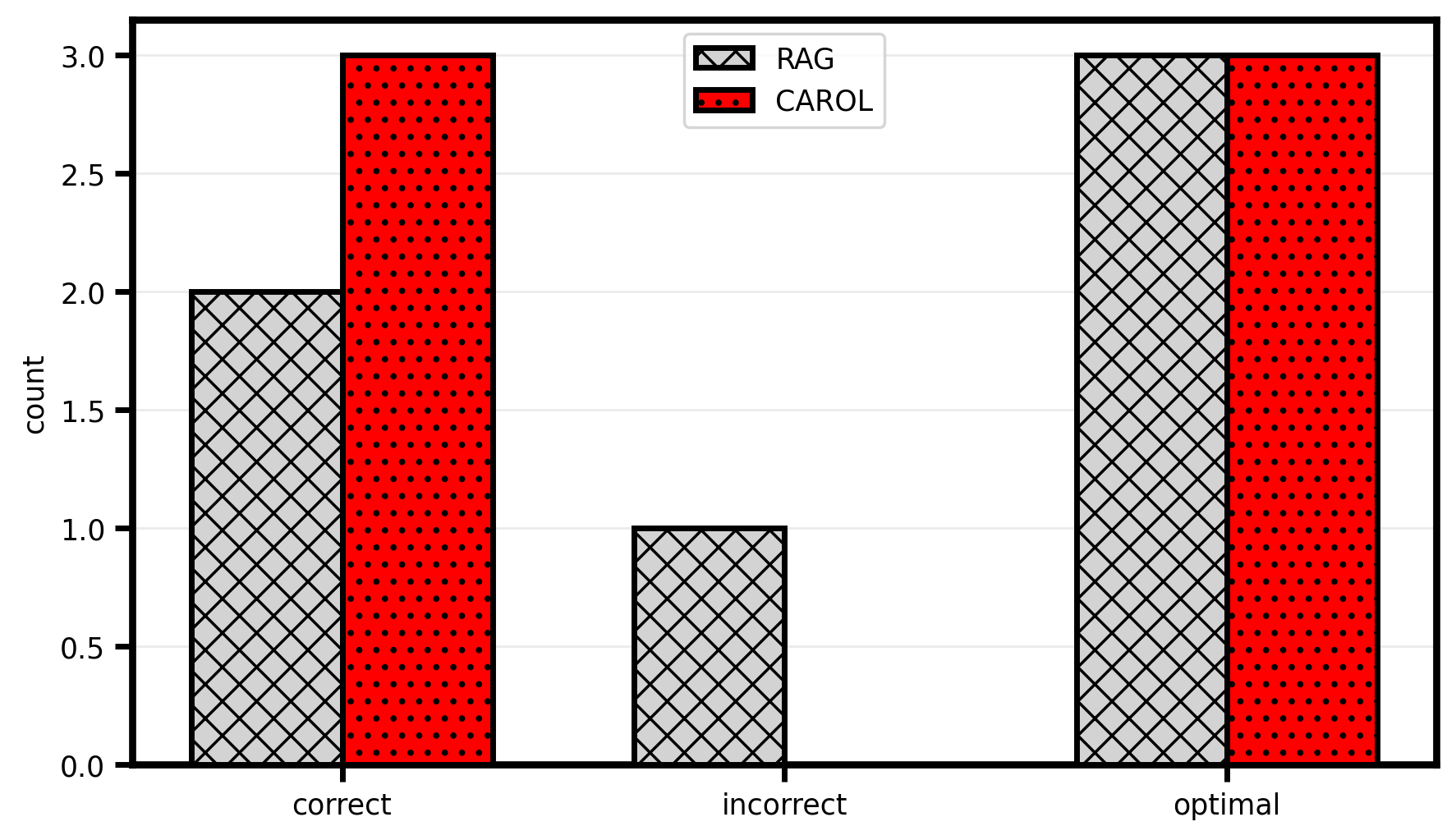}
            \caption{Mandela Effect}
        \end{subfigure}
        \begin{subfigure}{0.3\textwidth}
            \centering
            \includegraphics[width=\linewidth]{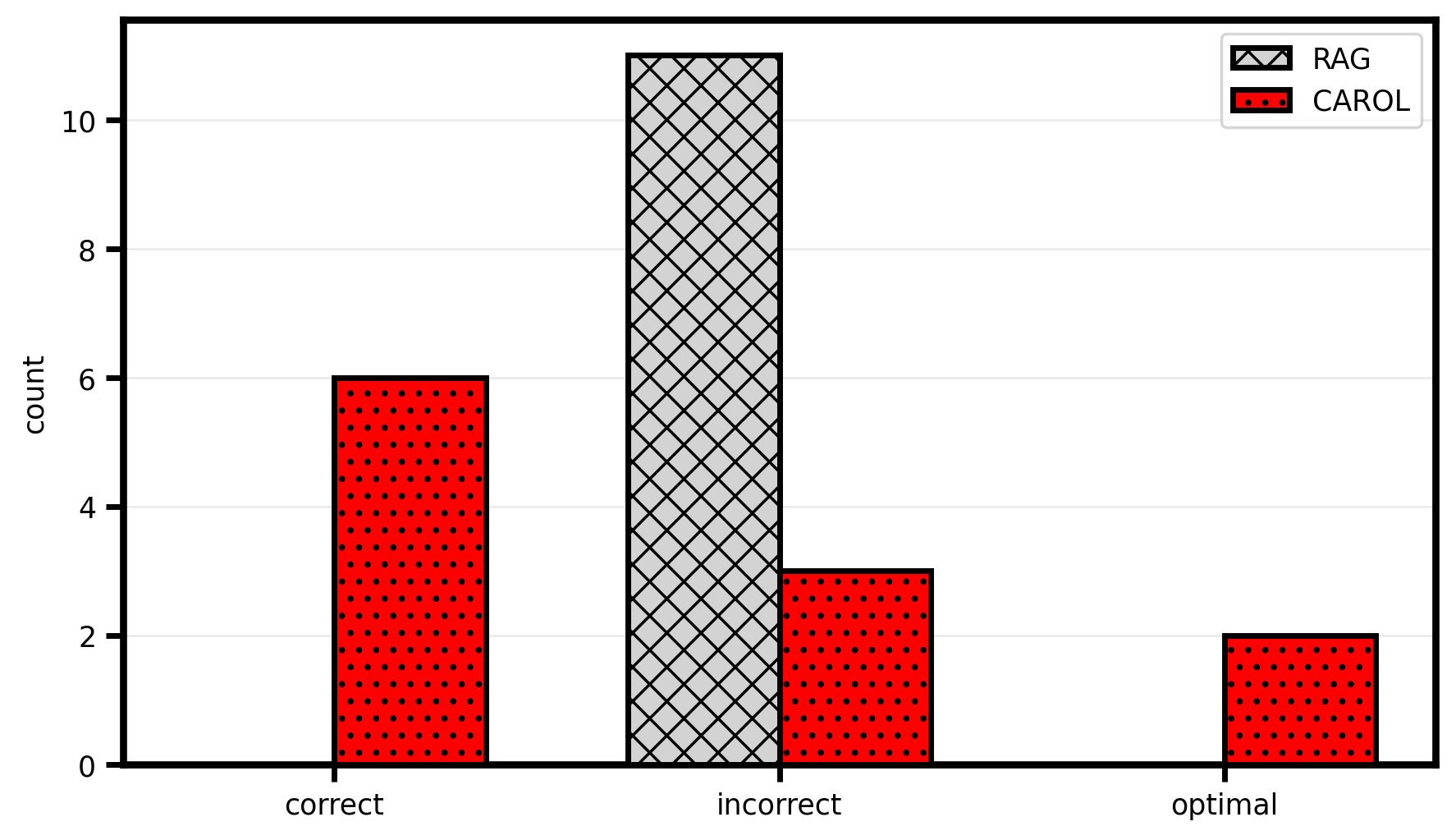}
            \caption{Distraction}
        \end{subfigure}
        \begin{subfigure}{0.3\textwidth}
            \centering
            \includegraphics[width=\linewidth]{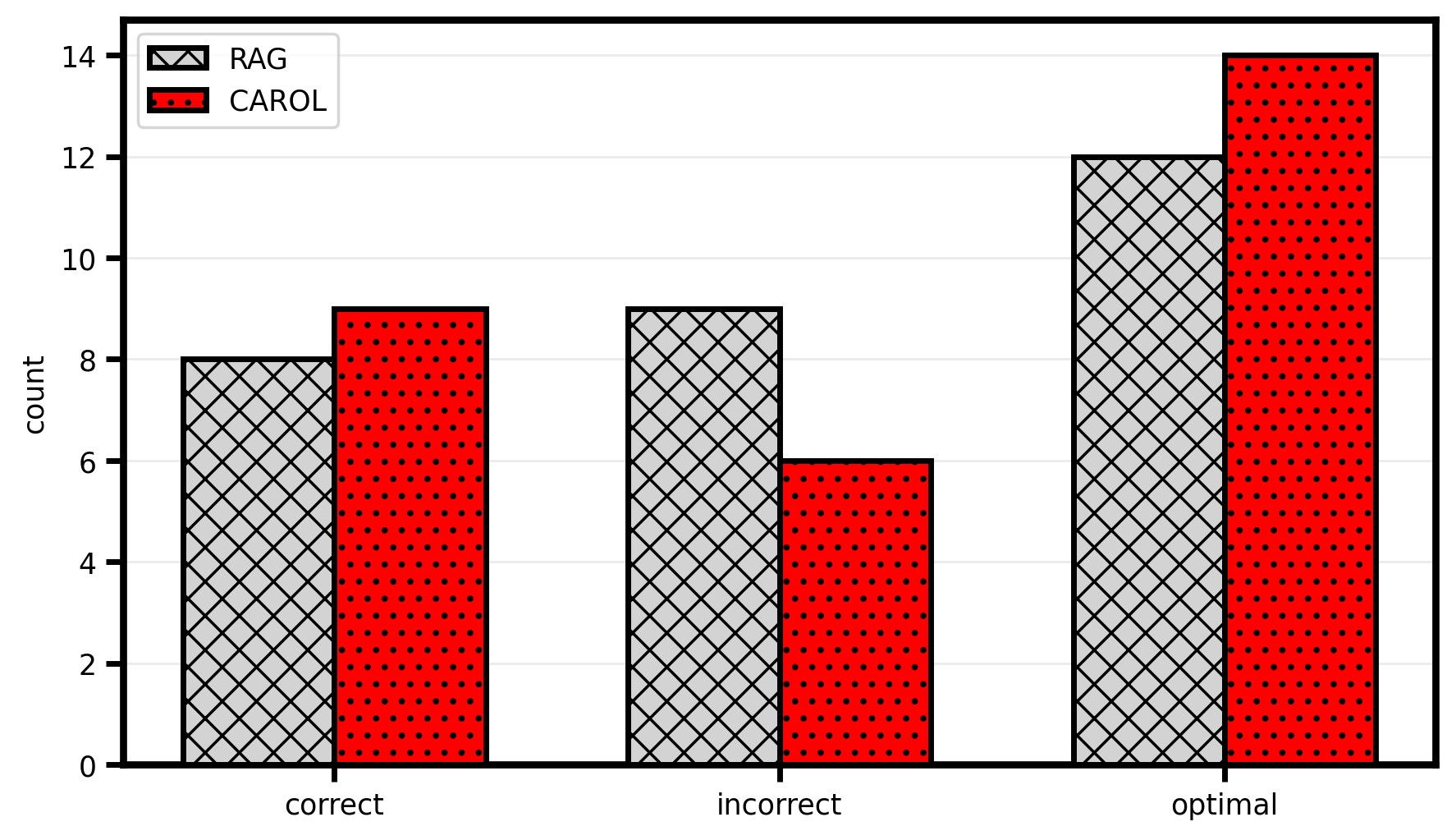}
            \caption{Economics}
        \end{subfigure}
    \end{subcaptiongroup}
    \begin{subcaptiongroup}
        \begin{subfigure}{0.3\textwidth}
            \centering
            \includegraphics[width=\linewidth]{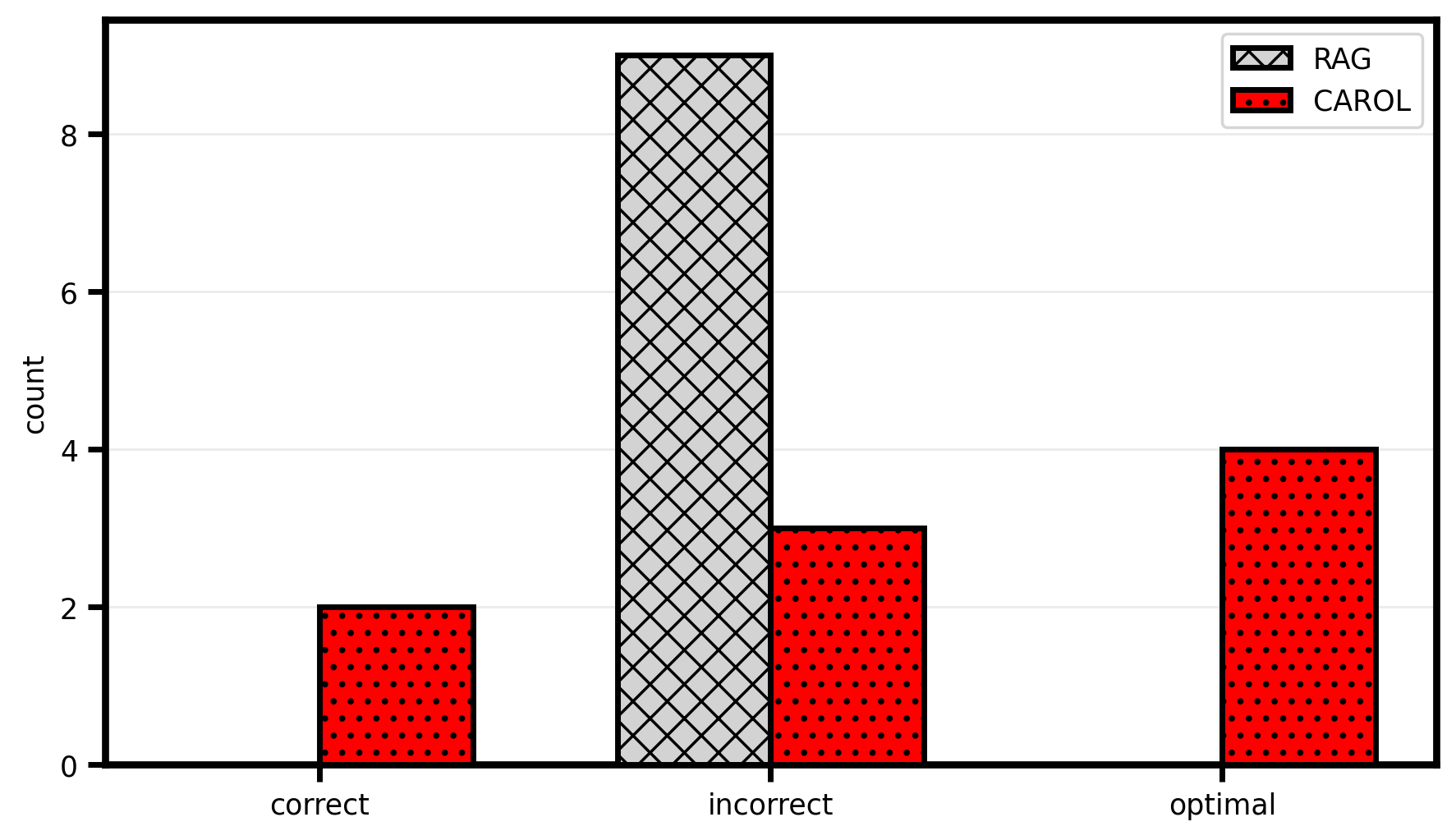}
            \caption{Education}
        \end{subfigure}
        \begin{subfigure}{0.3\textwidth}
            \centering
            \includegraphics[width=\linewidth]{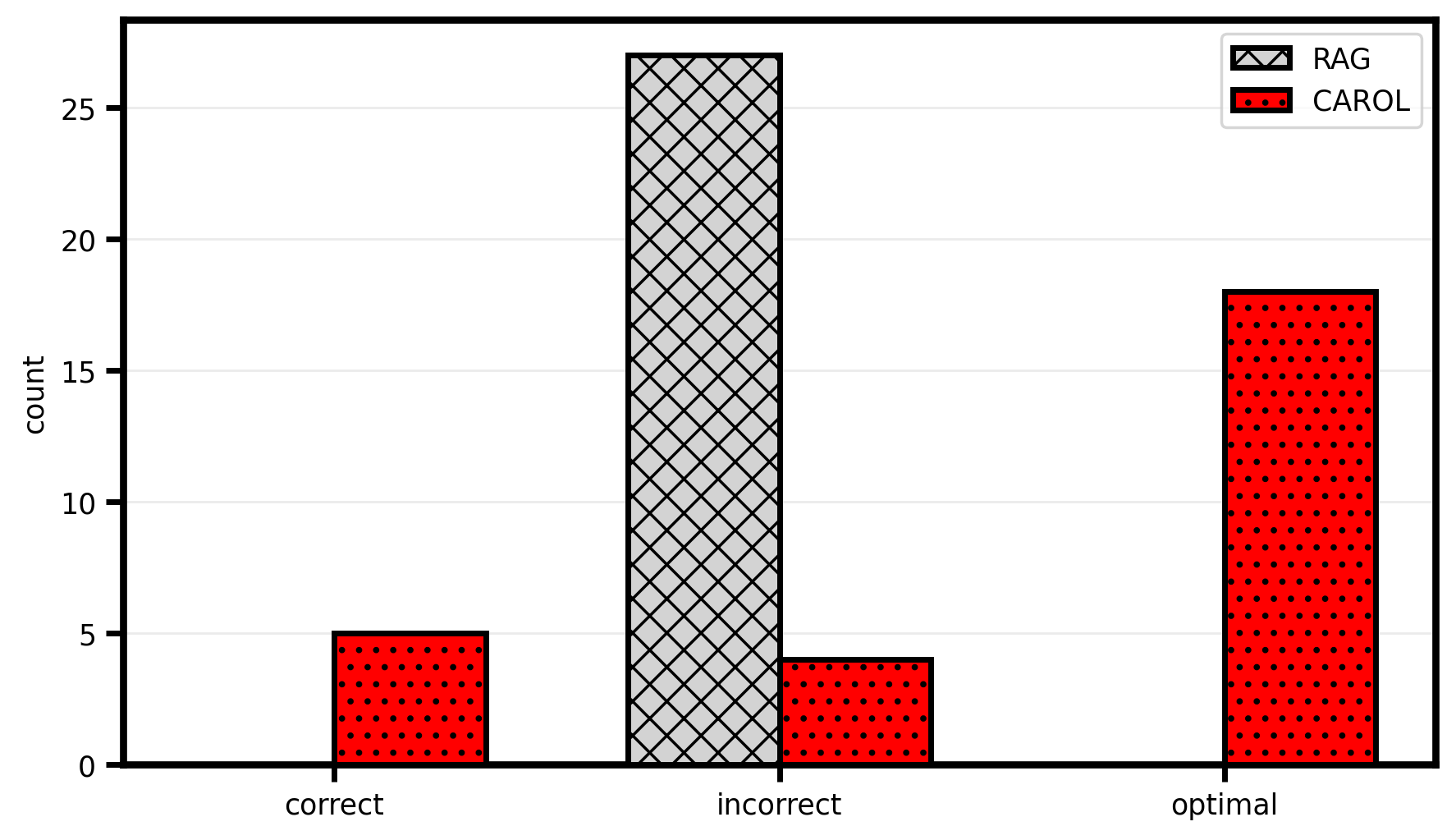}
            \caption{Fiction}
        \end{subfigure}
        \begin{subfigure}{0.3\textwidth}
            \centering
            \includegraphics[width=\linewidth]{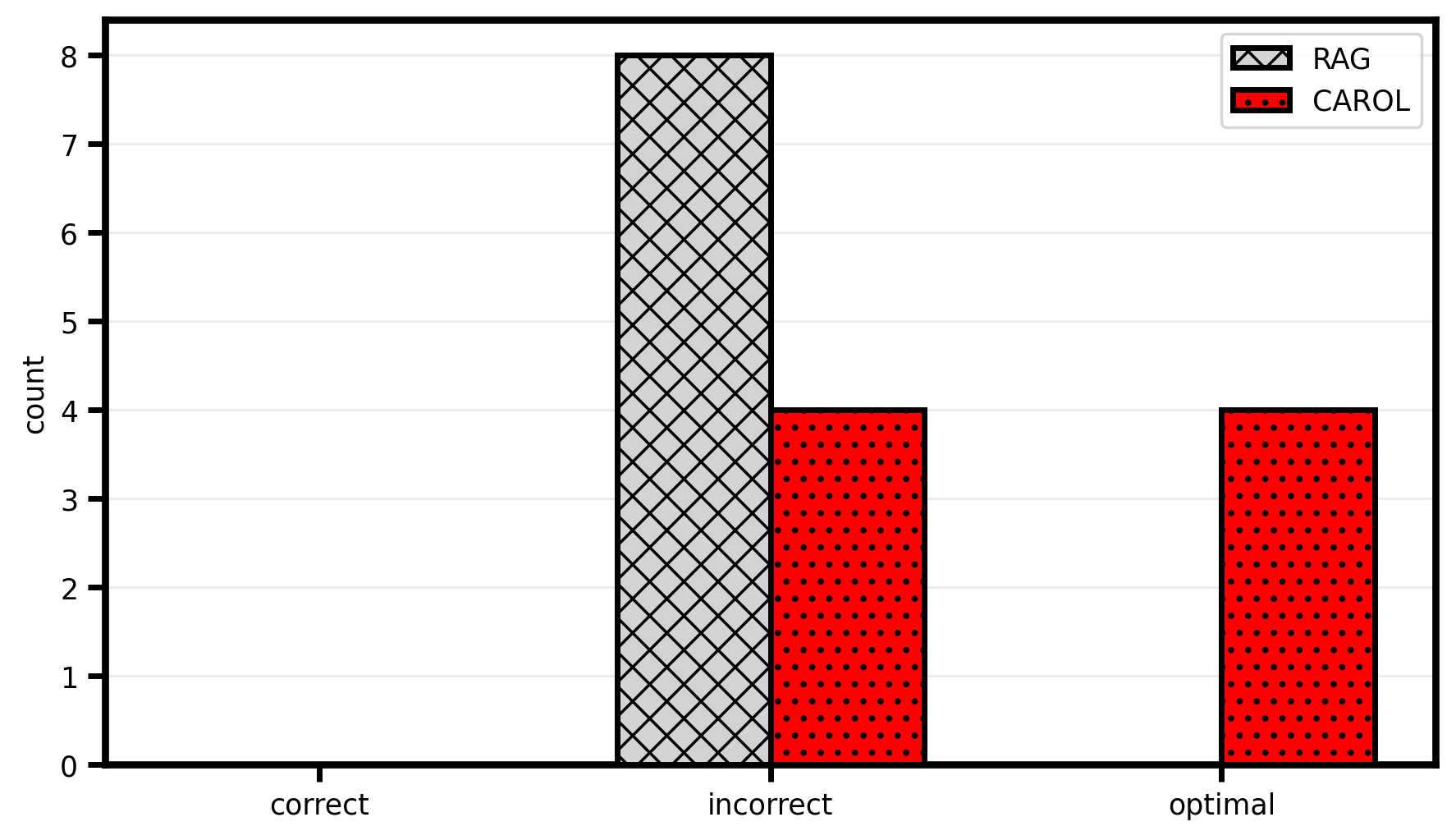}
            \caption{Finance}
        \end{subfigure}
    \end{subcaptiongroup}
    \begin{subcaptiongroup}
        \begin{subfigure}{0.3\textwidth}
            \centering
            \includegraphics[width=\linewidth]{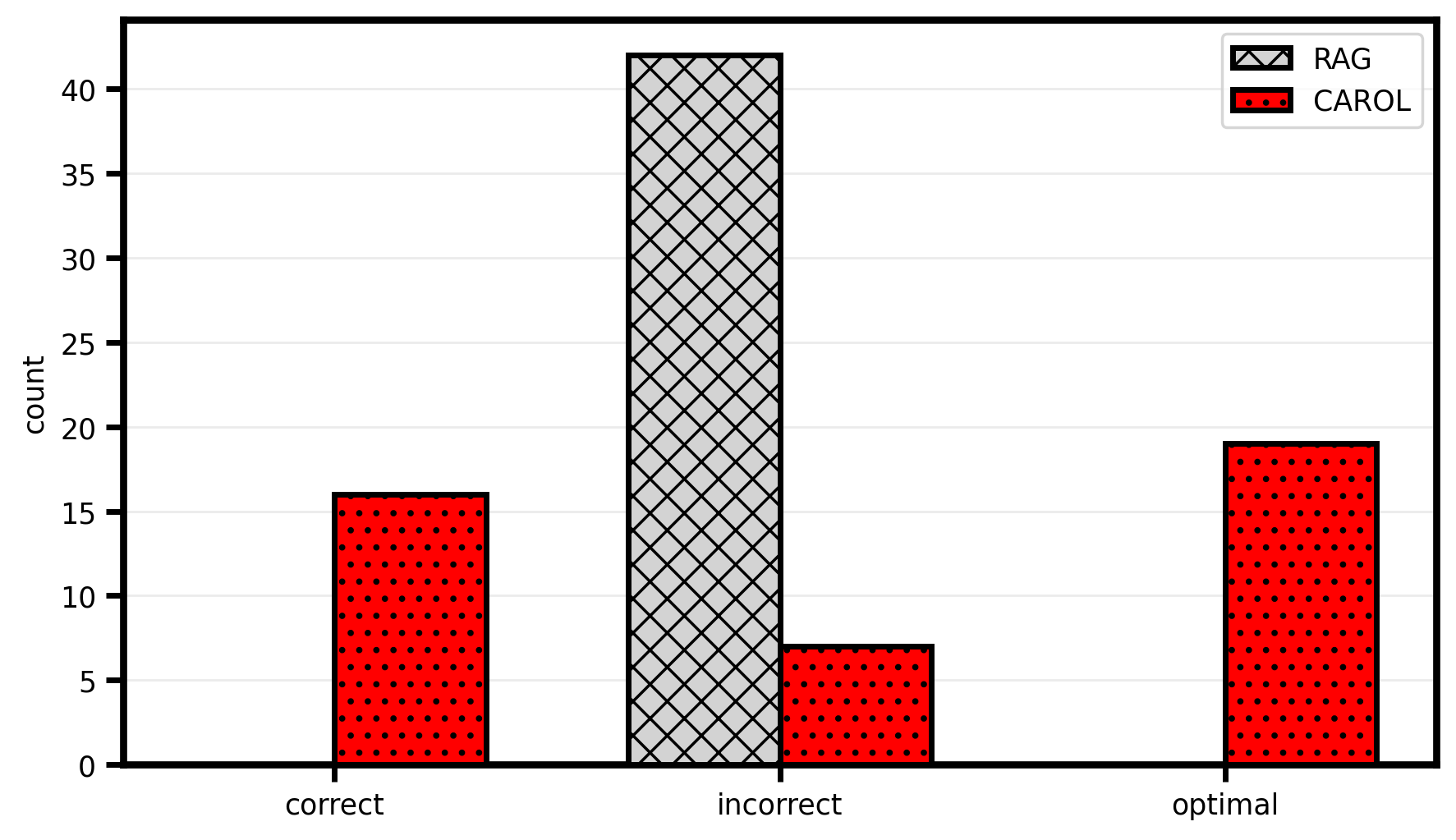}
            \caption{Health}
        \end{subfigure}
        \begin{subfigure}{0.3\textwidth}
            \centering
            \includegraphics[width=\linewidth]{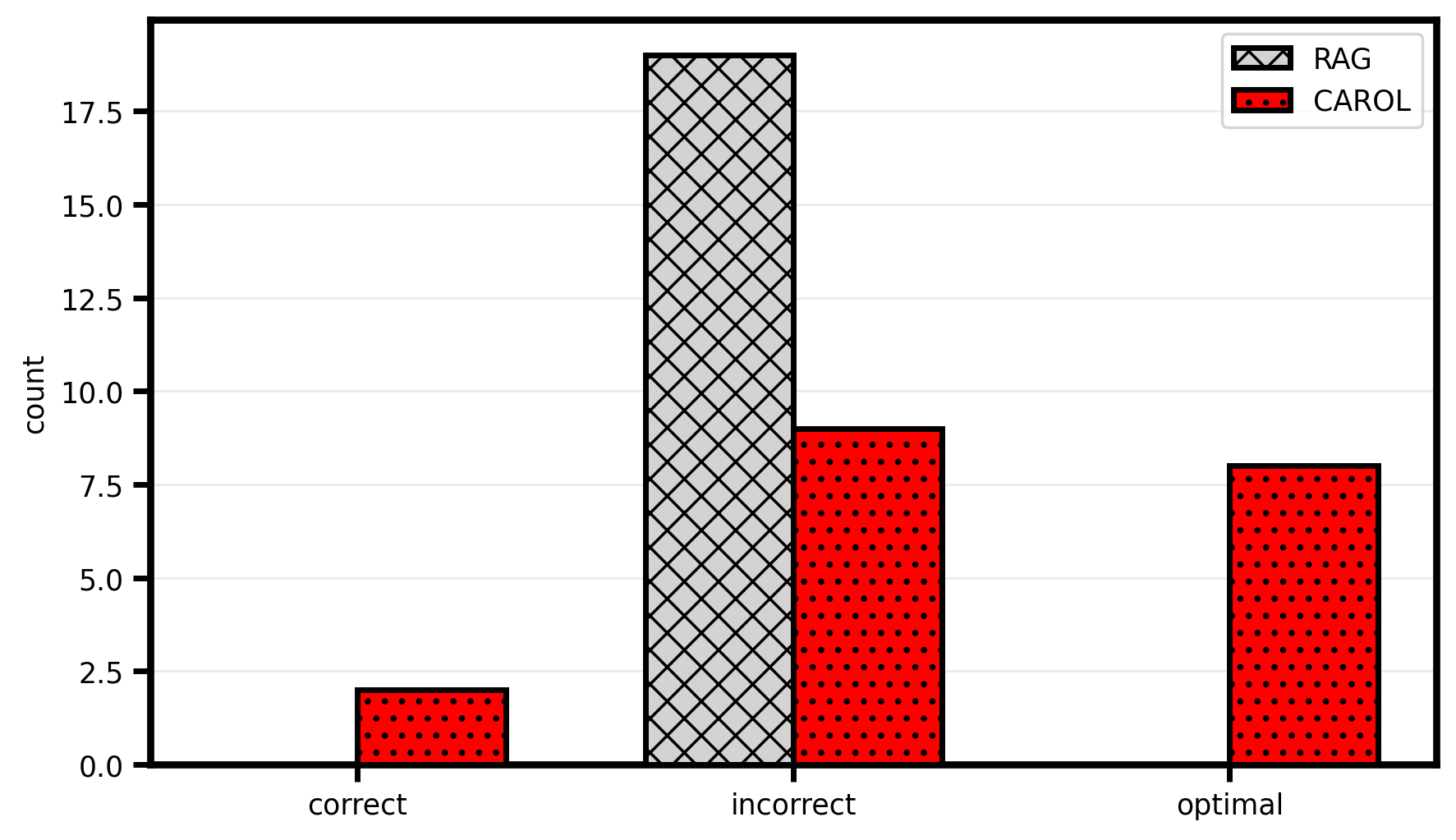}
            \caption{History}
        \end{subfigure}
        \begin{subfigure}{0.3\textwidth}
            \centering
            \includegraphics[width=\linewidth]{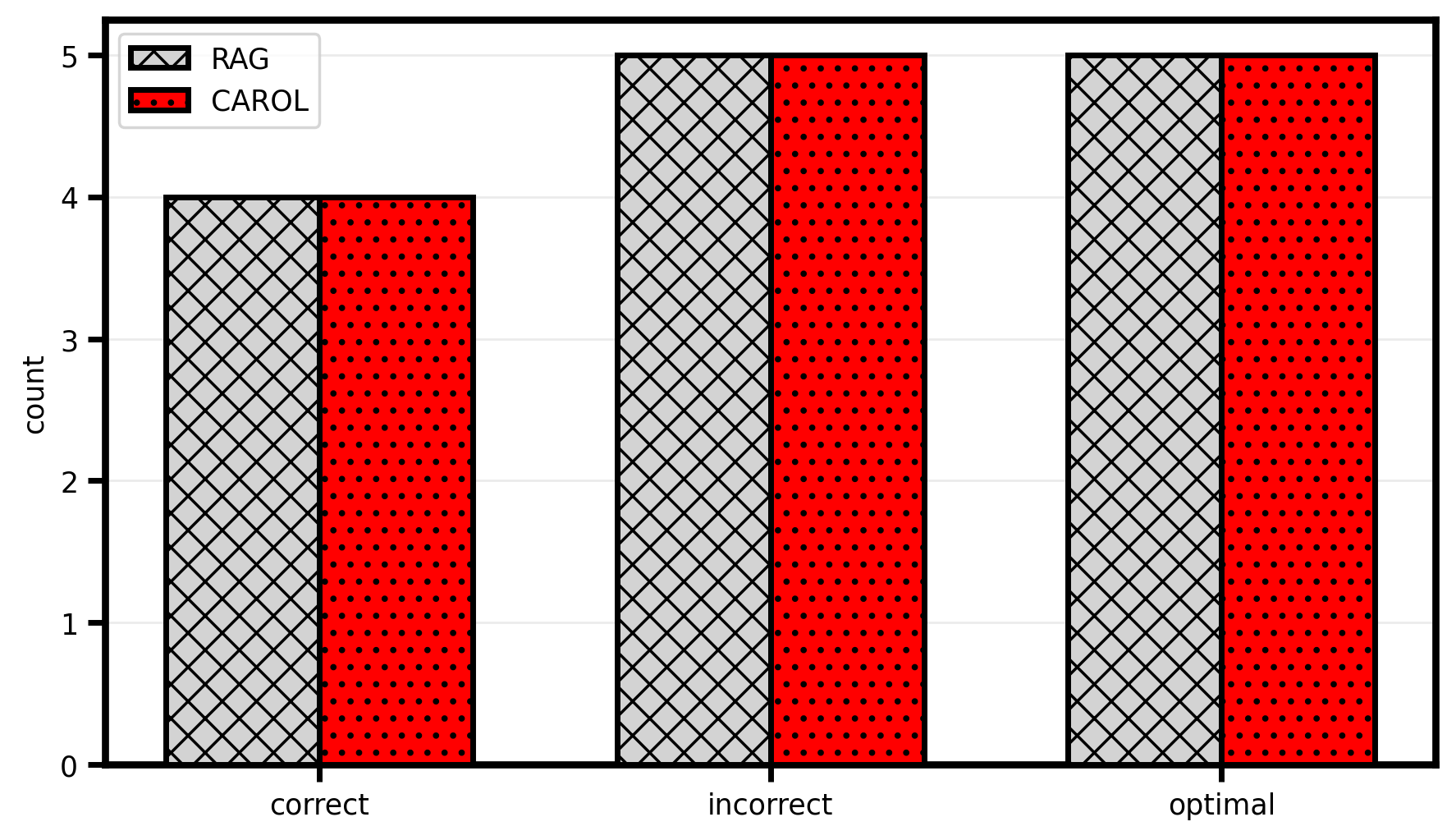}
            \caption{Indexical Error}
        \end{subfigure}
    \end{subcaptiongroup}
\caption{\textbf{TruthfulQA} on Llama-3.1-8B extended results in each category, part $1$.}
\end{figure}

\begin{figure}[ht]
    \centering
    \begin{subcaptiongroup}
        \begin{subfigure}{0.3\textwidth}
            \centering
            \includegraphics[width=\linewidth]{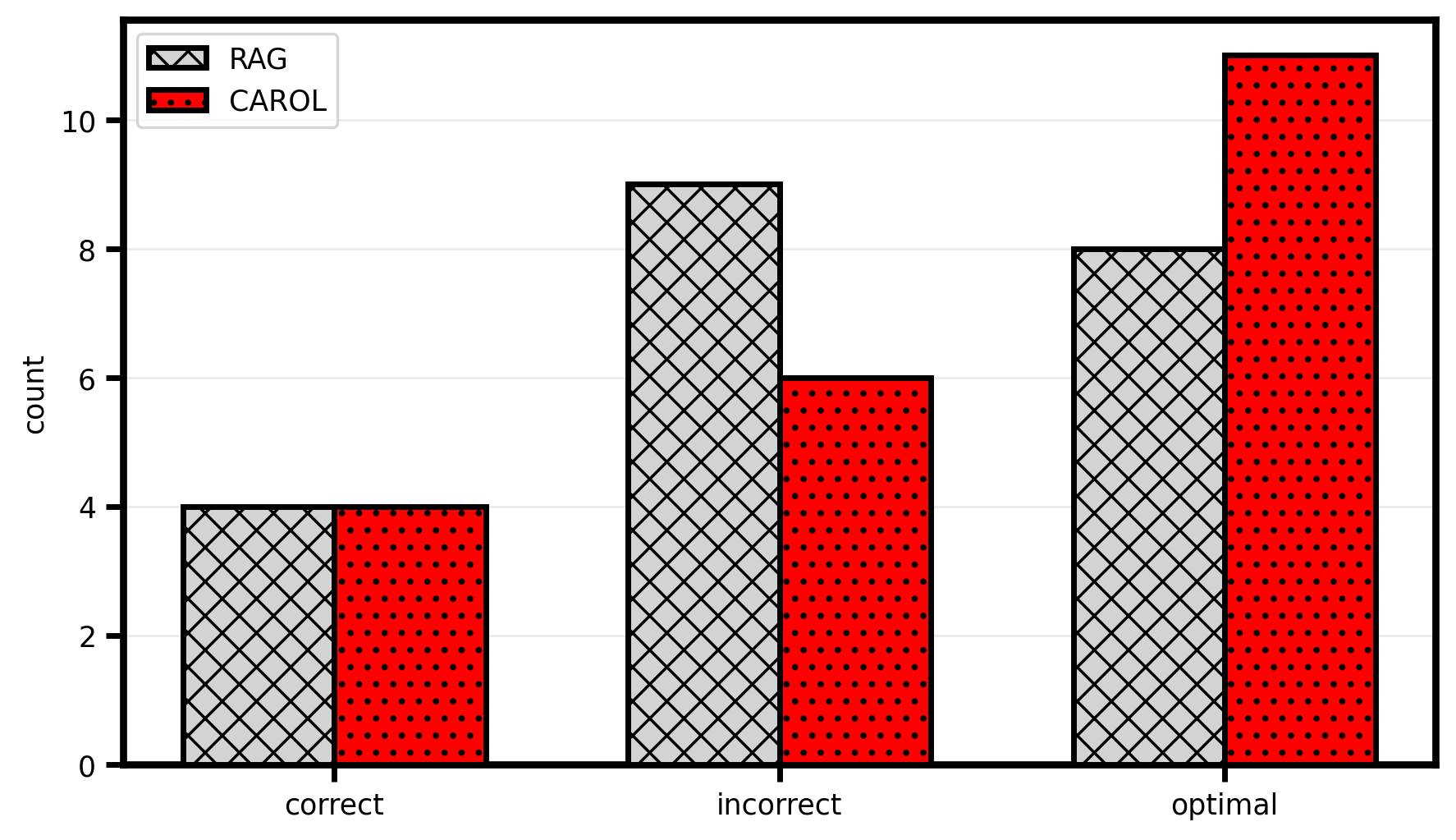}
            \caption{Language}
        \end{subfigure}
        \begin{subfigure}{0.3\textwidth}
            \centering
            \includegraphics[width=\linewidth]{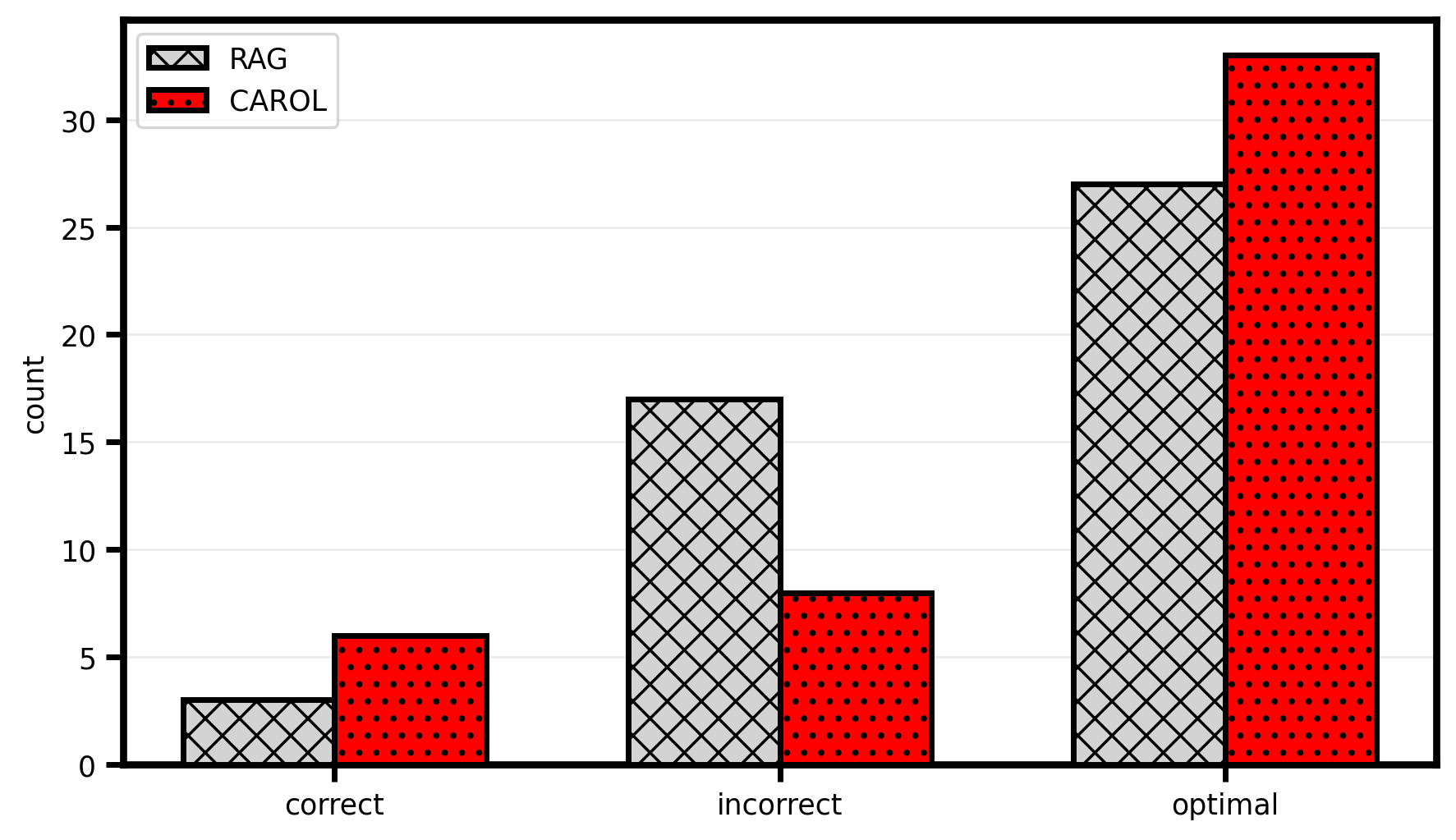}
            \caption{Law}
        \end{subfigure}
        \begin{subfigure}{0.3\textwidth}
            \centering
            \includegraphics[width=\linewidth]{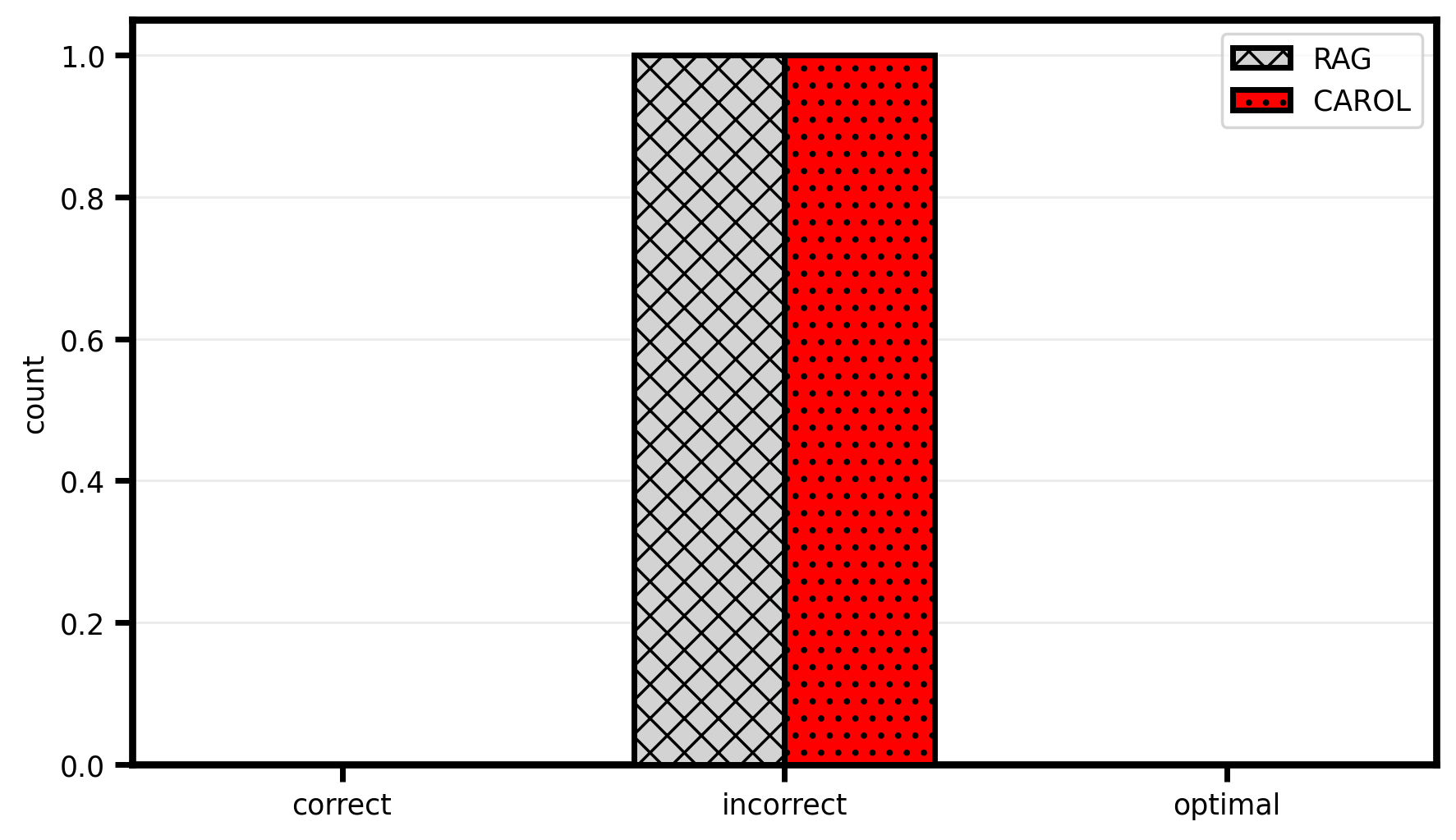}
            \caption{Logical Falsehood}
        \end{subfigure}
    \end{subcaptiongroup}
    \begin{subcaptiongroup}
        \begin{subfigure}{0.3\textwidth}
            \centering
            \includegraphics[width=\linewidth]{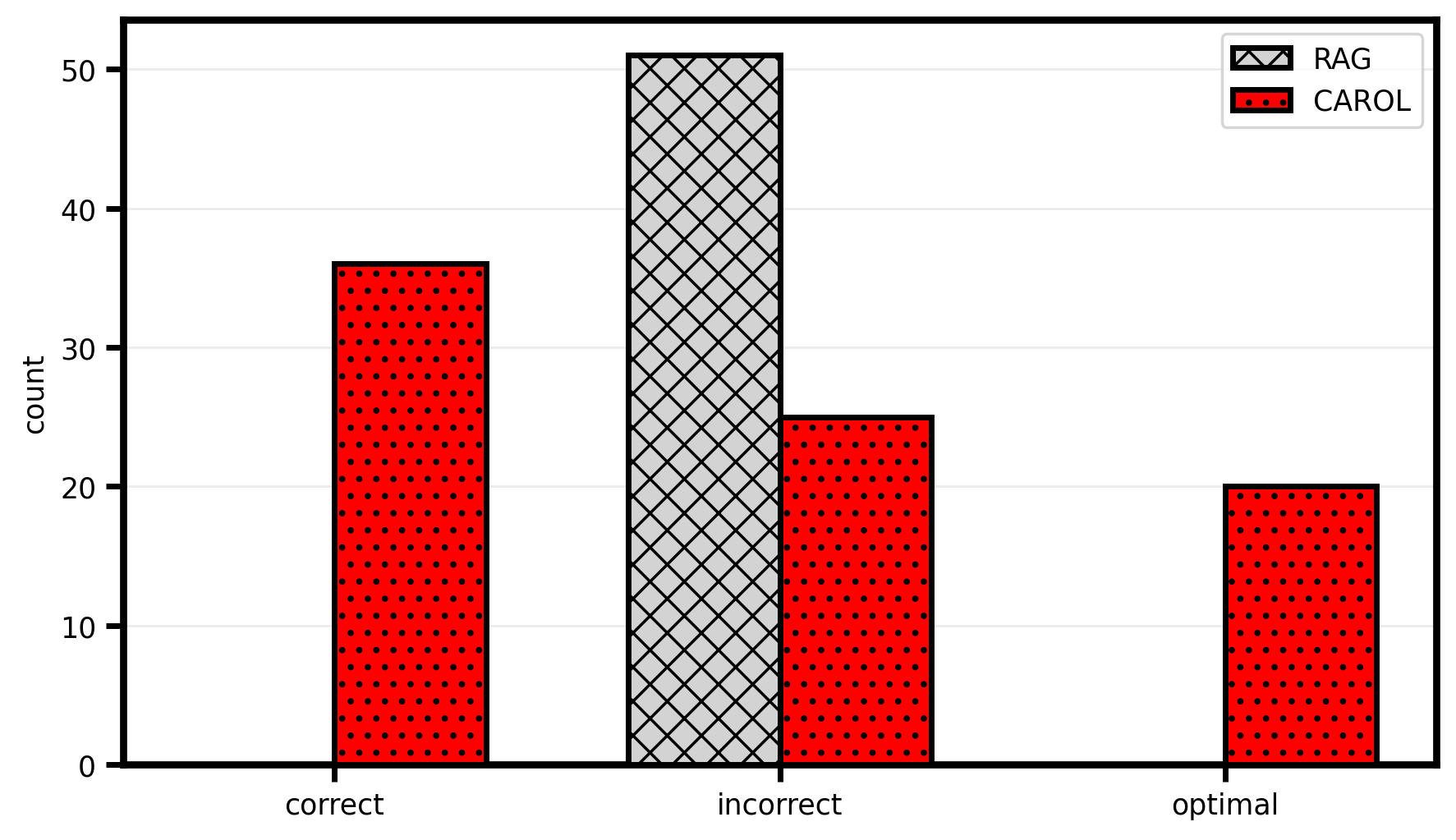}
            \caption{Misconceptions}
        \end{subfigure}
        \begin{subfigure}{0.3\textwidth}
            \centering
            \includegraphics[width=\linewidth]{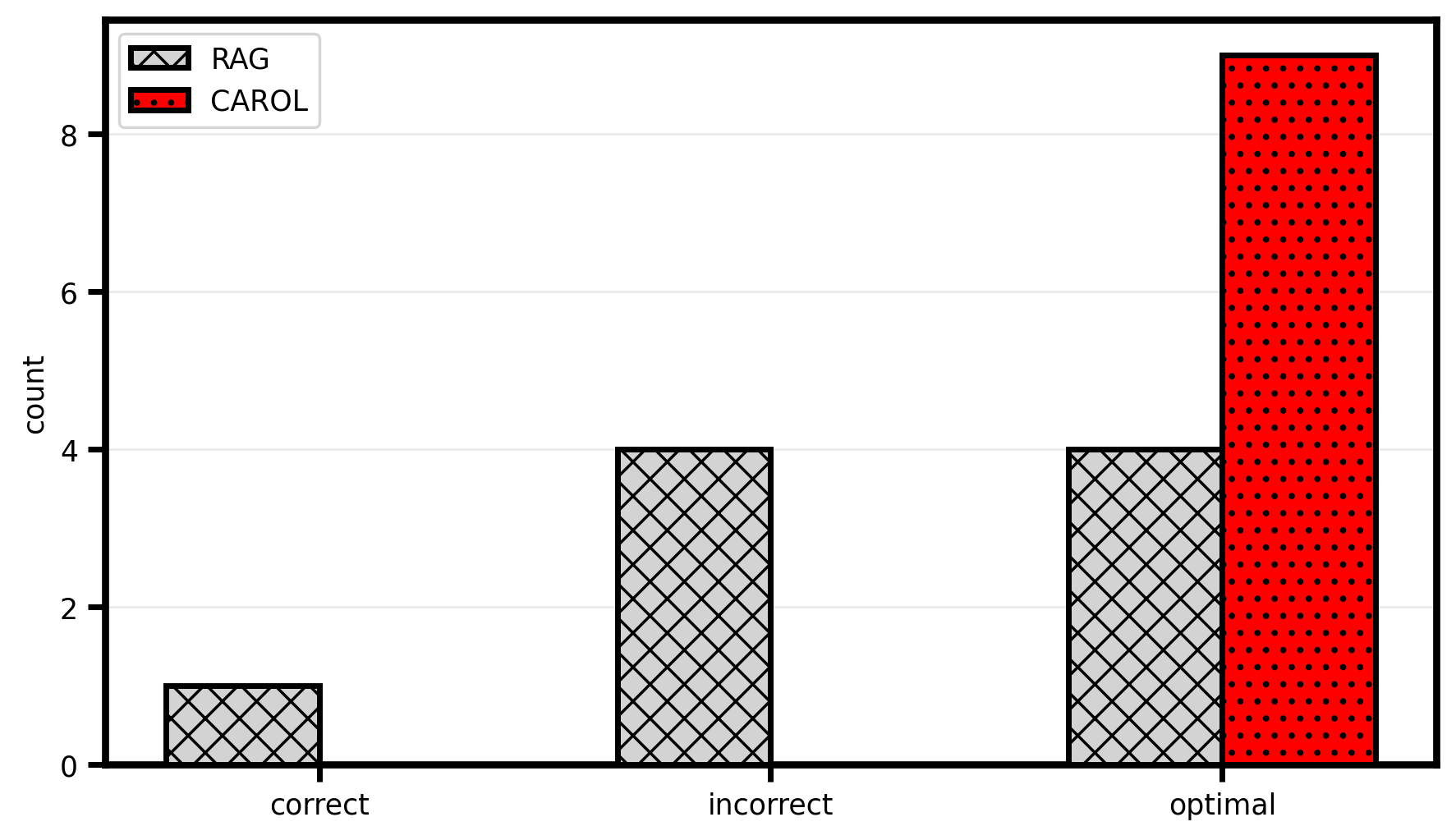}
            \caption{Misinformation}
        \end{subfigure}
        \begin{subfigure}{0.3\textwidth}
            \centering
            \includegraphics[width=\linewidth]{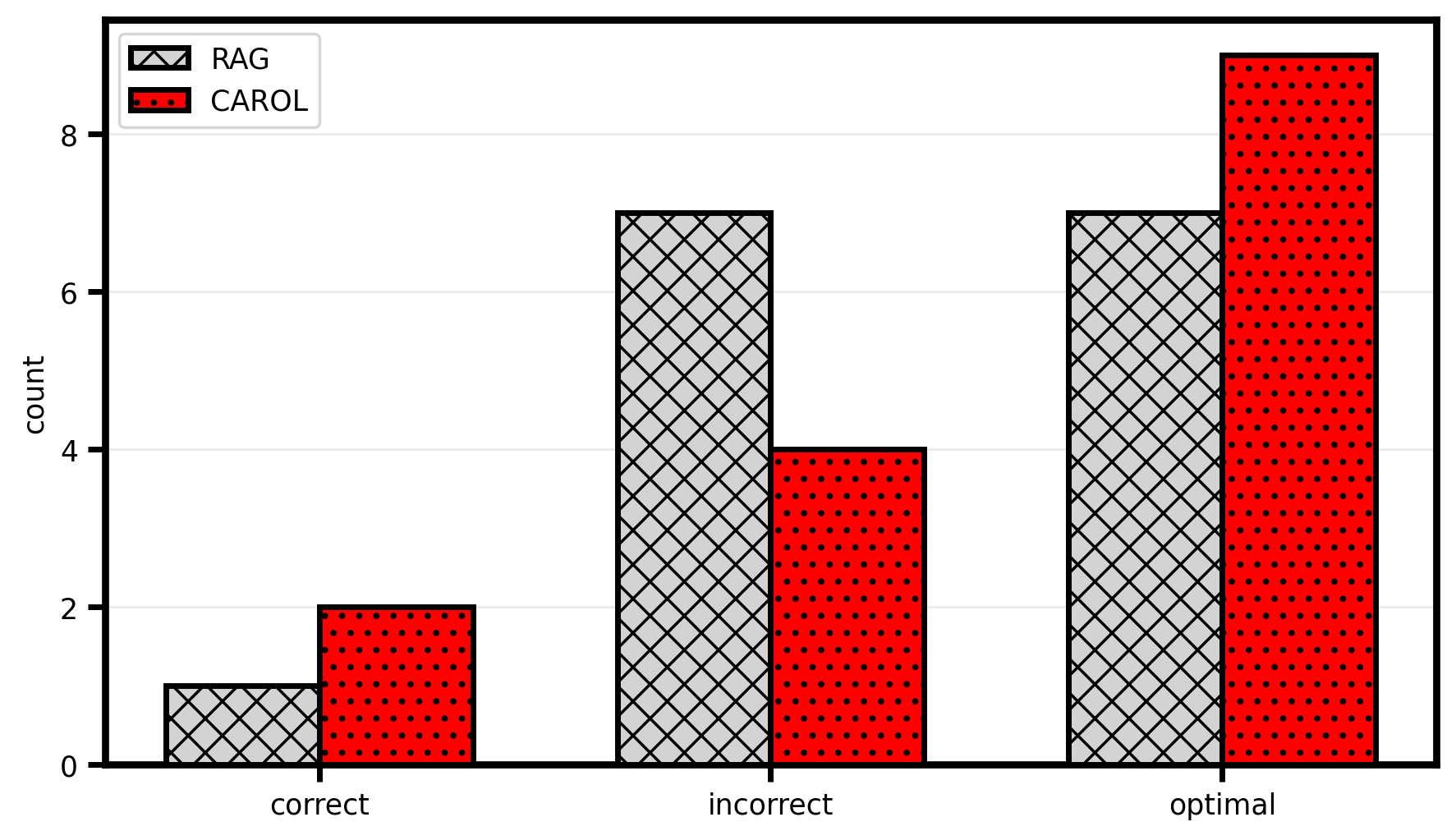}
            \caption{Misquotations}
        \end{subfigure}
    \end{subcaptiongroup}
    \begin{subcaptiongroup}
        \begin{subfigure}{0.3\textwidth}
            \centering
            \includegraphics[width=\linewidth]{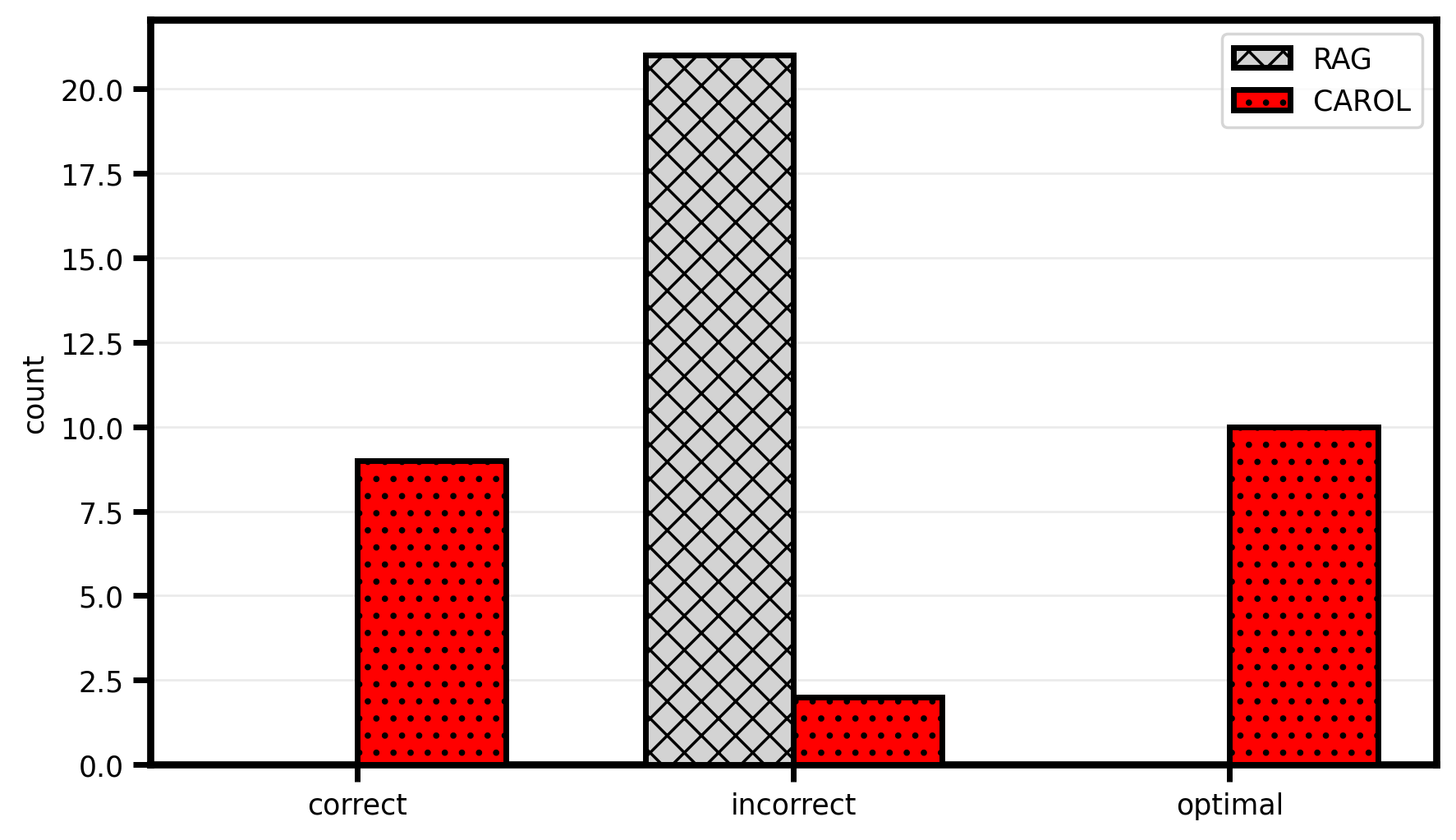}
            \caption{Myths}
        \end{subfigure}
        \begin{subfigure}{0.3\textwidth}
            \centering
            \includegraphics[width=\linewidth]{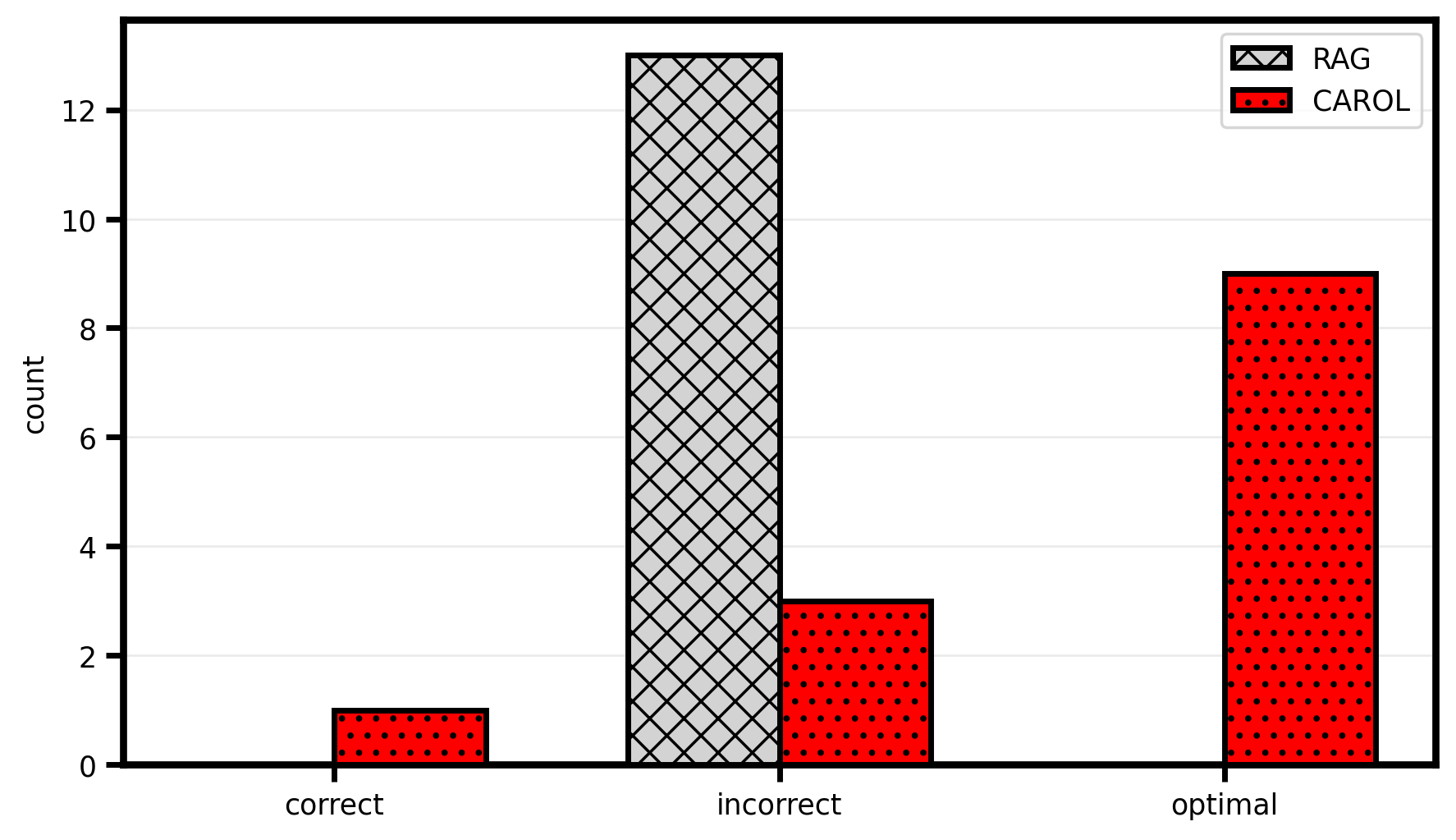}
            \caption{Nutrition}
        \end{subfigure}
        \begin{subfigure}{0.3\textwidth}
            \centering
            \includegraphics[width=\linewidth]{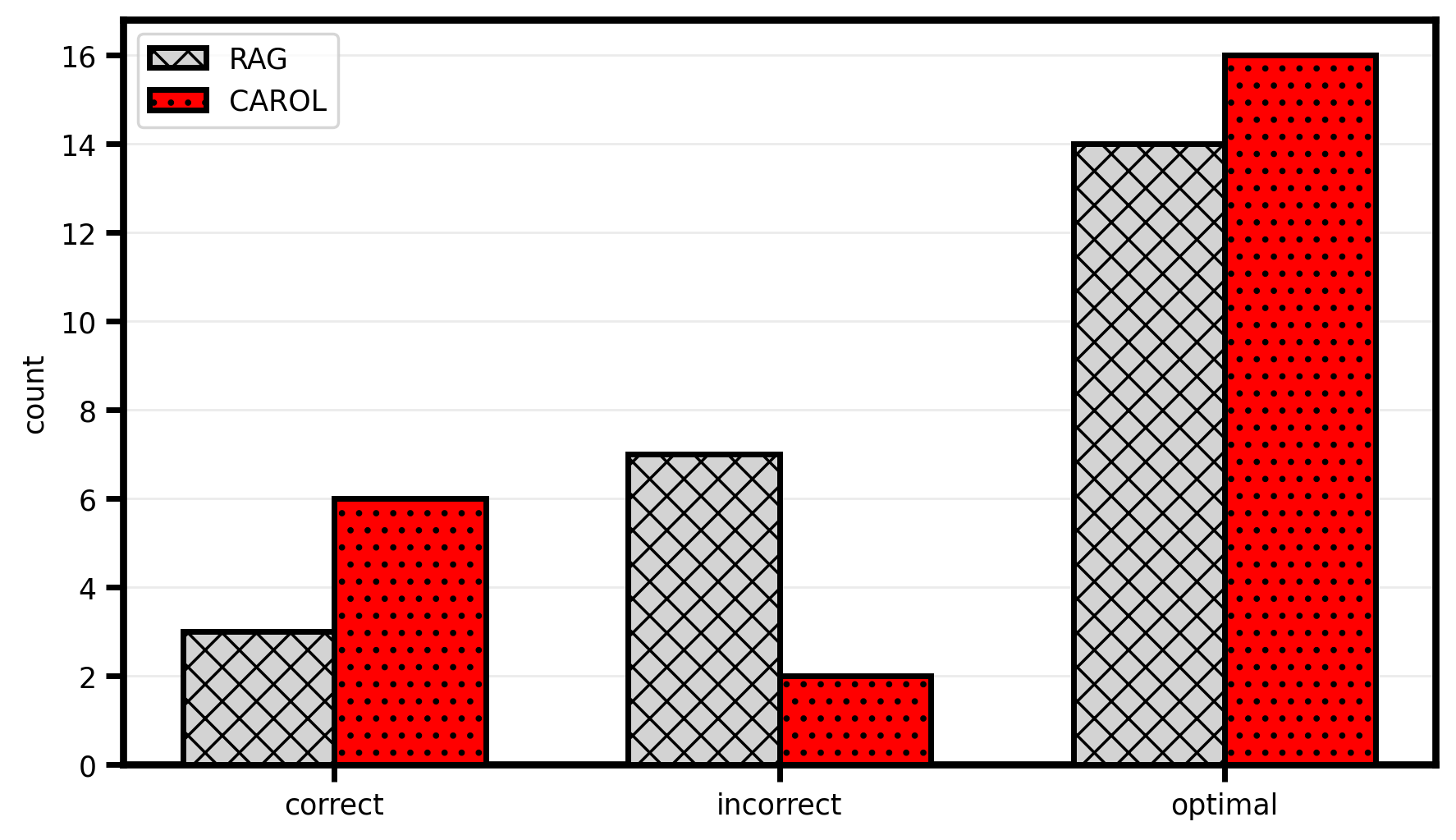}
            \caption{Paranormal}
        \end{subfigure}
    \end{subcaptiongroup}
    \begin{subcaptiongroup}
        \begin{subfigure}{0.3\textwidth}
            \centering
            \includegraphics[width=\linewidth]{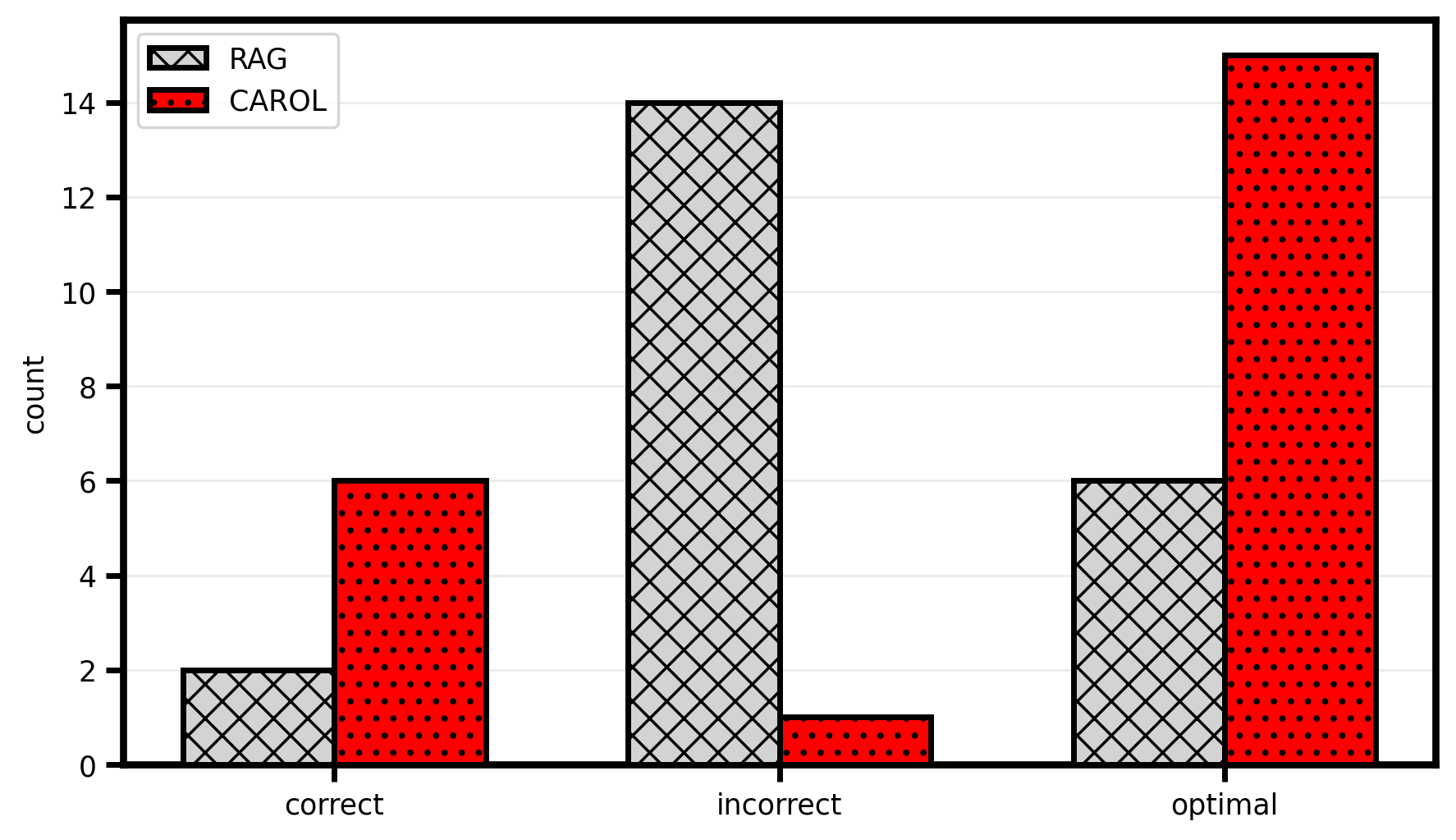}
            \caption{People}
        \end{subfigure}
        \begin{subfigure}{0.3\textwidth}
            \centering
            \includegraphics[width=\linewidth]{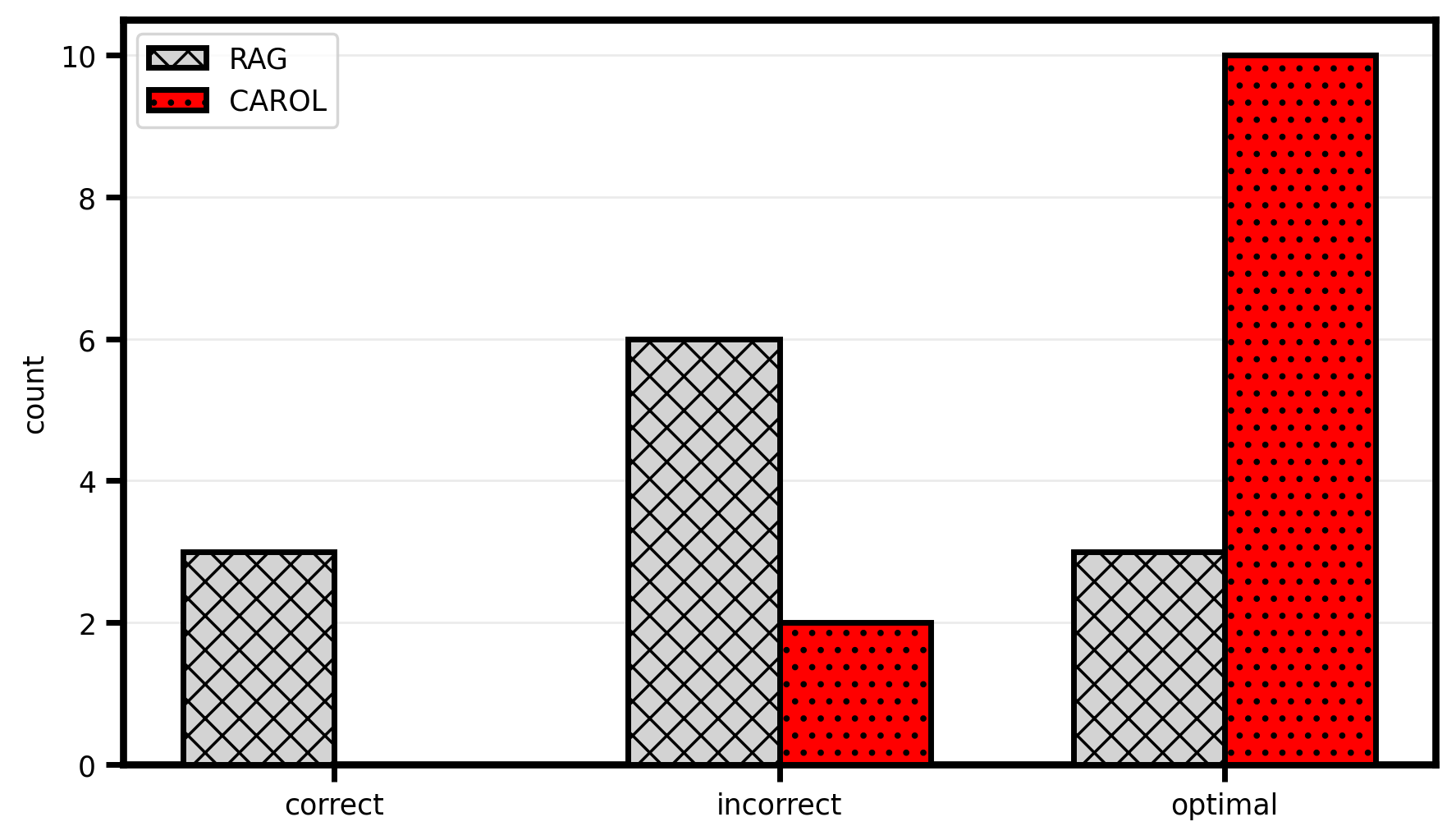}
            \caption{Places}
        \end{subfigure}
        \begin{subfigure}{0.3\textwidth}
            \centering
            \includegraphics[width=\linewidth]{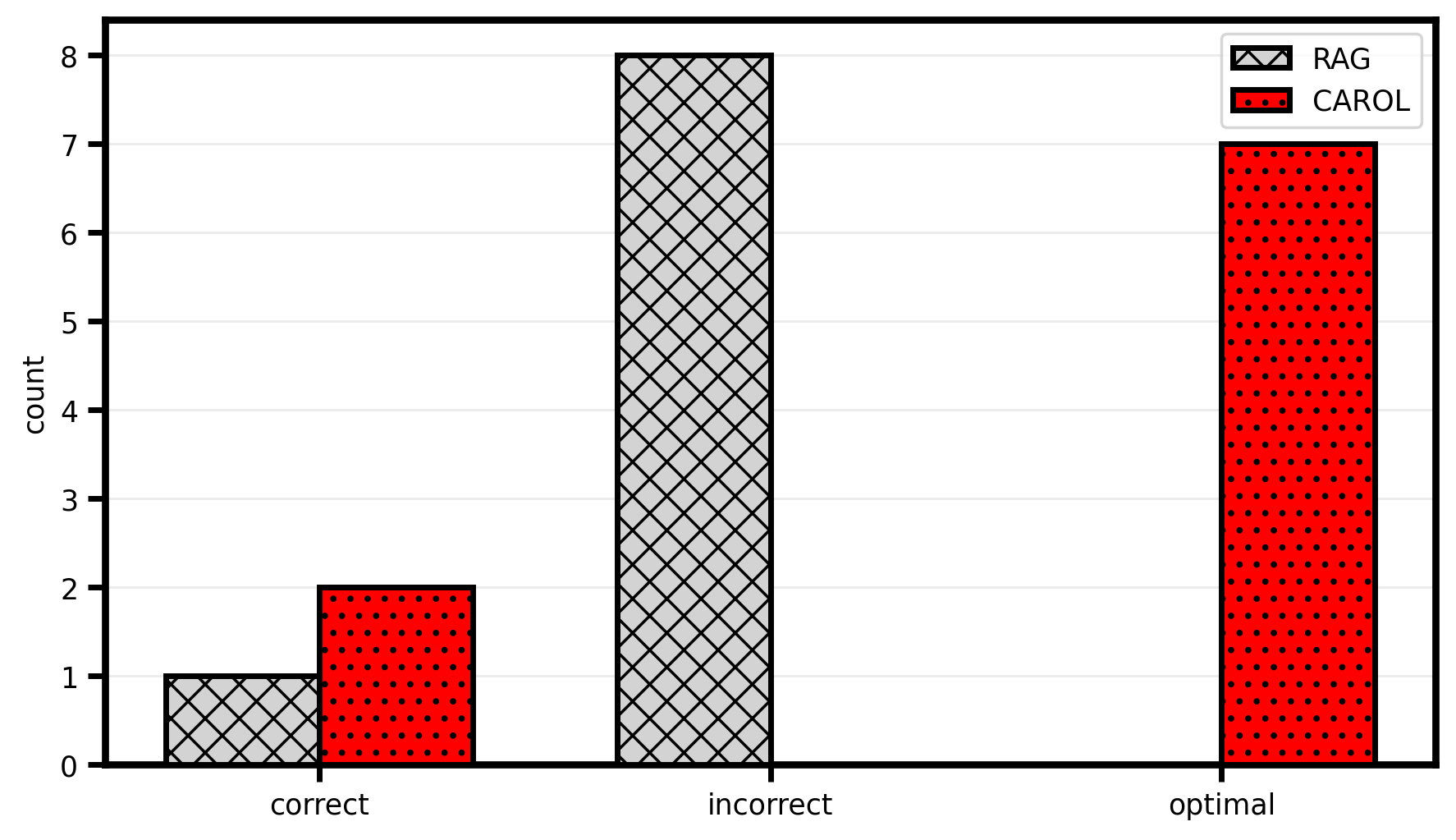}
            \caption{Politics}
        \end{subfigure}
    \end{subcaptiongroup}
\caption{\textbf{TruthfulQA} on Llama-3.1-8B extended results in each category, part $2$.}
\end{figure}

\begin{figure}[ht]
    \centering
    \begin{subcaptiongroup}
        \begin{subfigure}{0.3\textwidth}
            \centering
            \includegraphics[width=\linewidth]{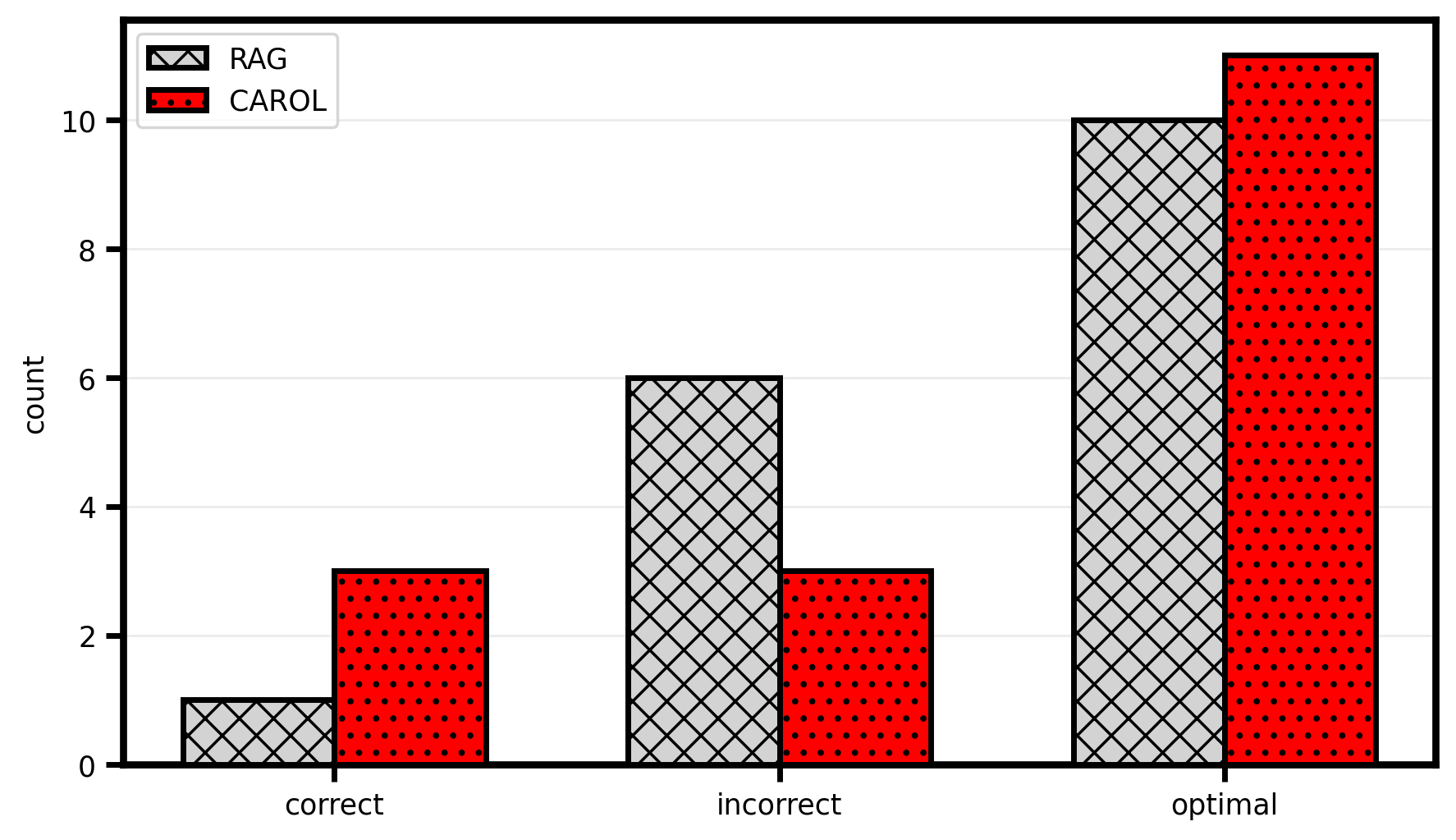}
            \caption{Proverbs}
        \end{subfigure}
        \begin{subfigure}{0.3\textwidth}
            \centering
            \includegraphics[width=\linewidth]{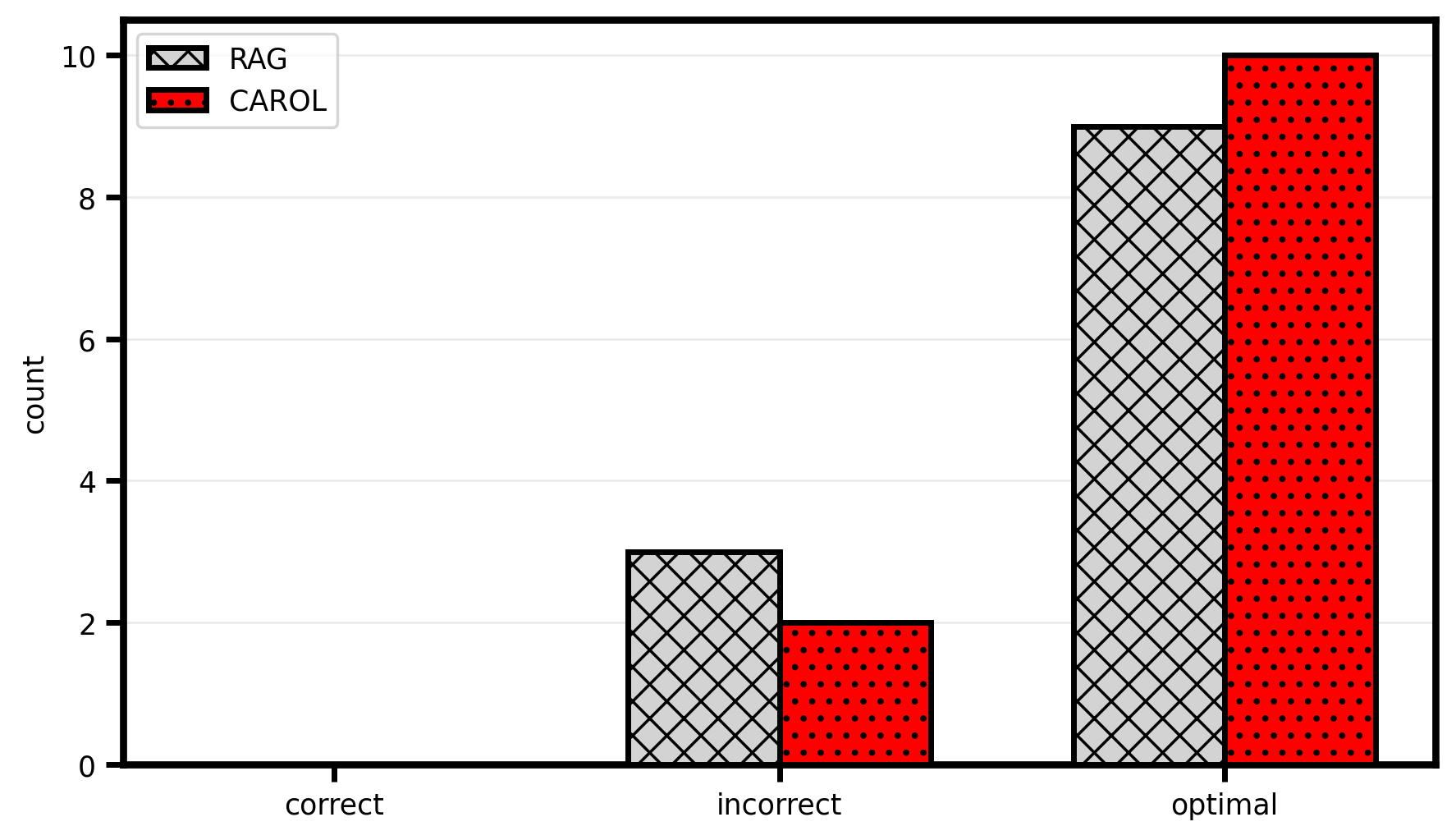}
            \caption{Psychology}
        \end{subfigure}
        \begin{subfigure}{0.3\textwidth}
            \centering
            \includegraphics[width=\linewidth]{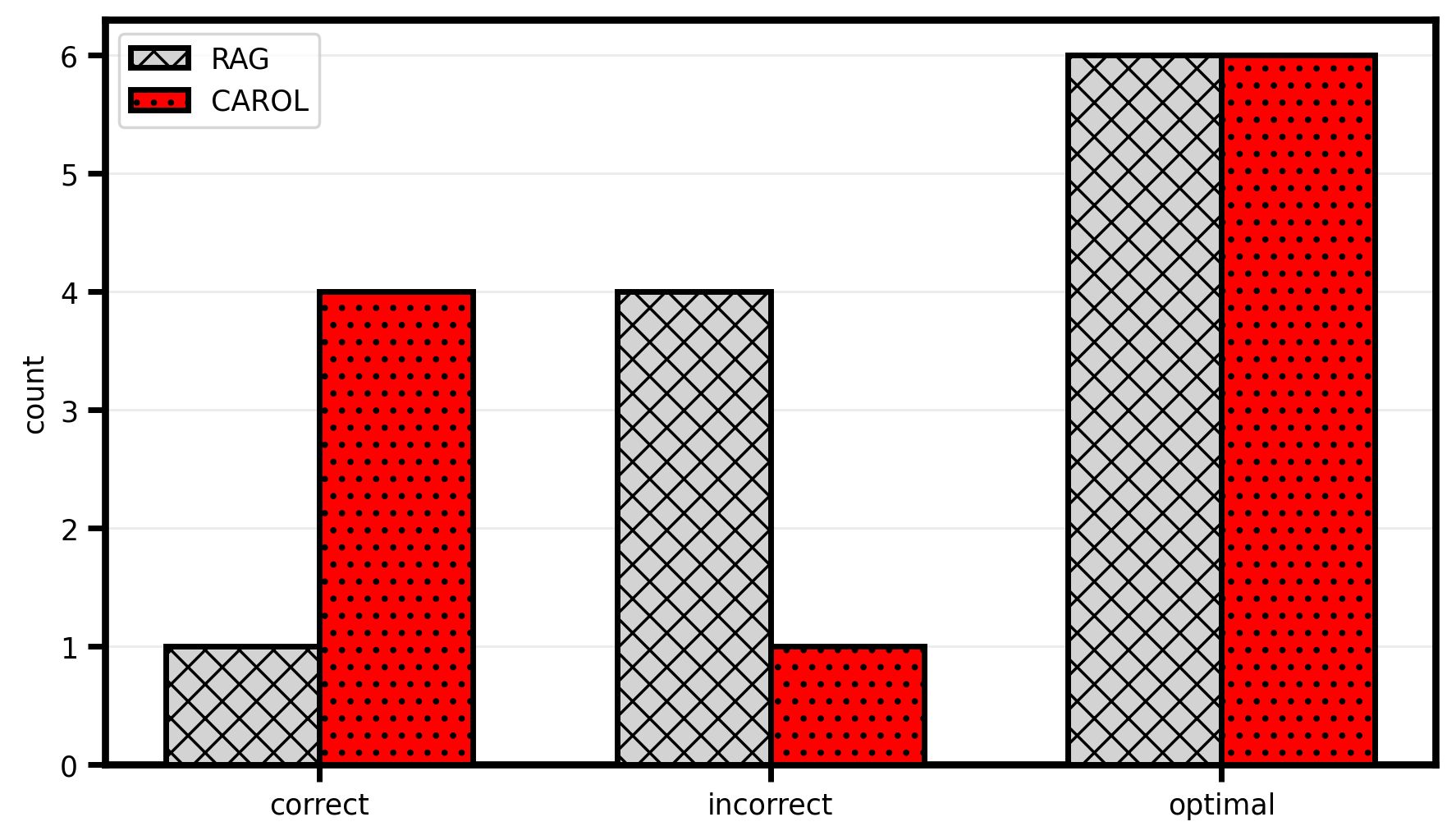}
            \caption{Religion}
        \end{subfigure}
    \end{subcaptiongroup}
    \begin{subcaptiongroup}
        \begin{subfigure}{0.3\textwidth}
            \centering
            \includegraphics[width=\linewidth]{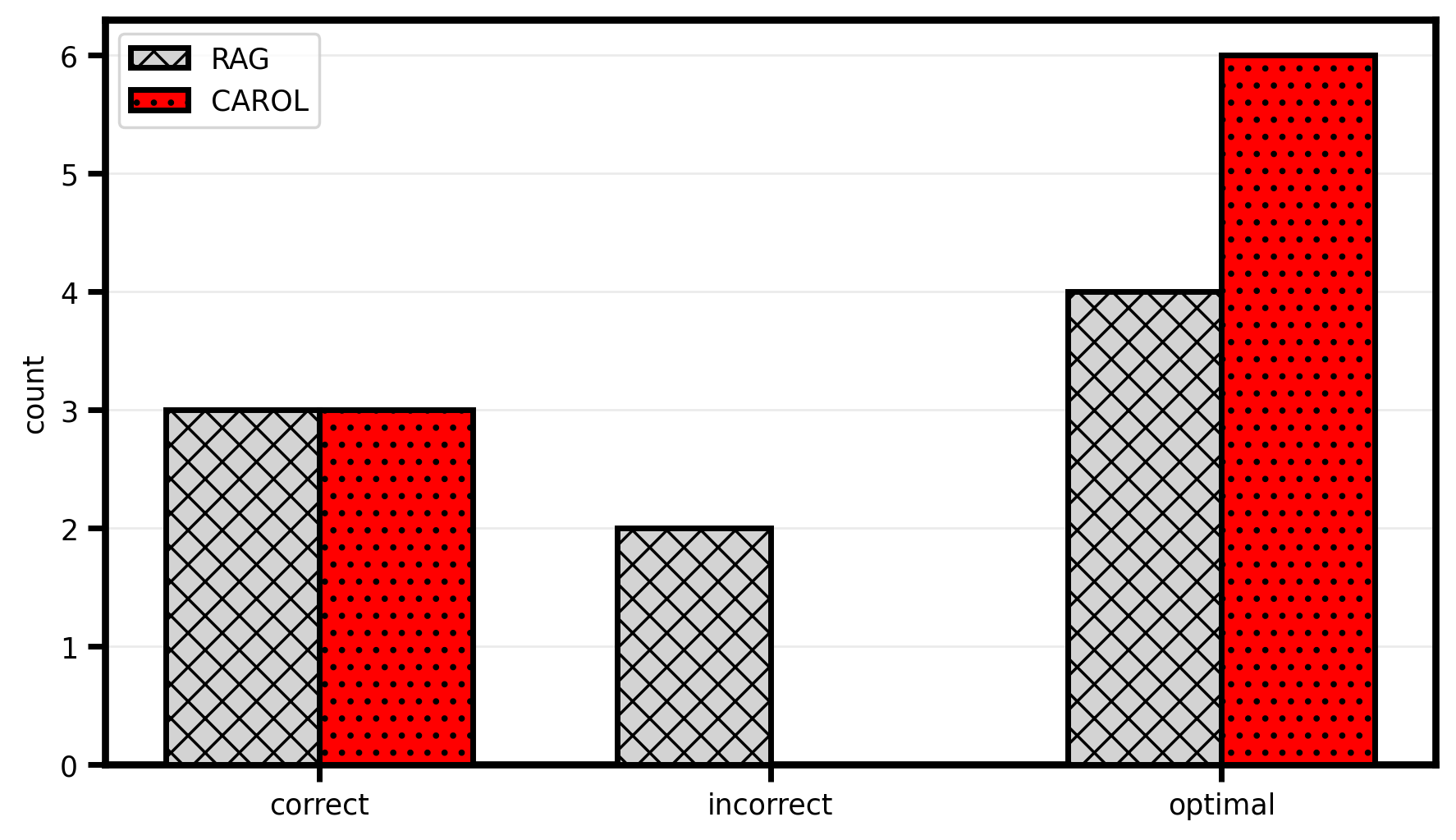}
            \caption{Science}
        \end{subfigure}
        \begin{subfigure}{0.3\textwidth}
            \centering
            \includegraphics[width=\linewidth]{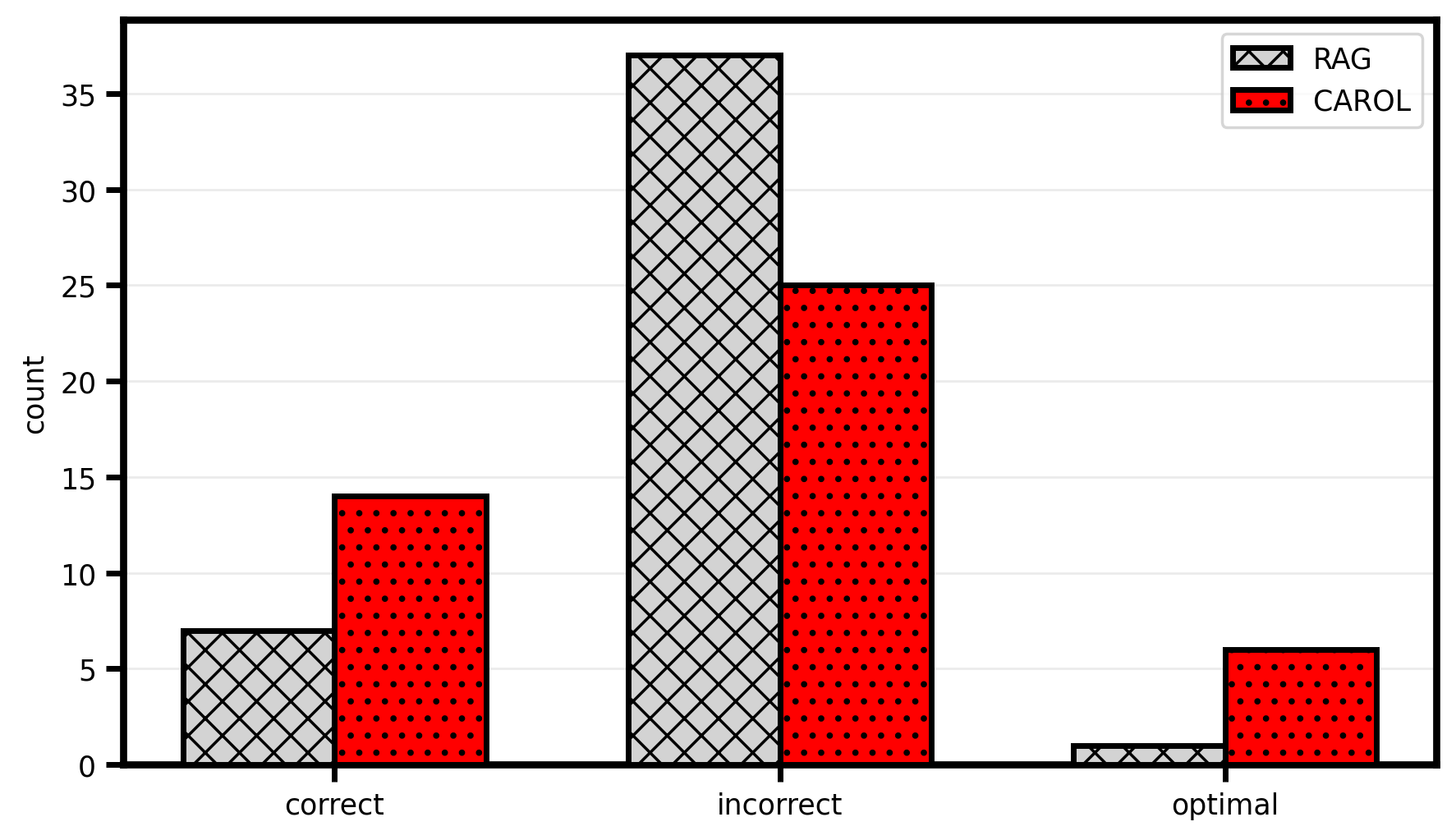}
            \caption{Sociology}
        \end{subfigure}
        \begin{subfigure}{0.3\textwidth}
            \centering
            \includegraphics[width=\linewidth]{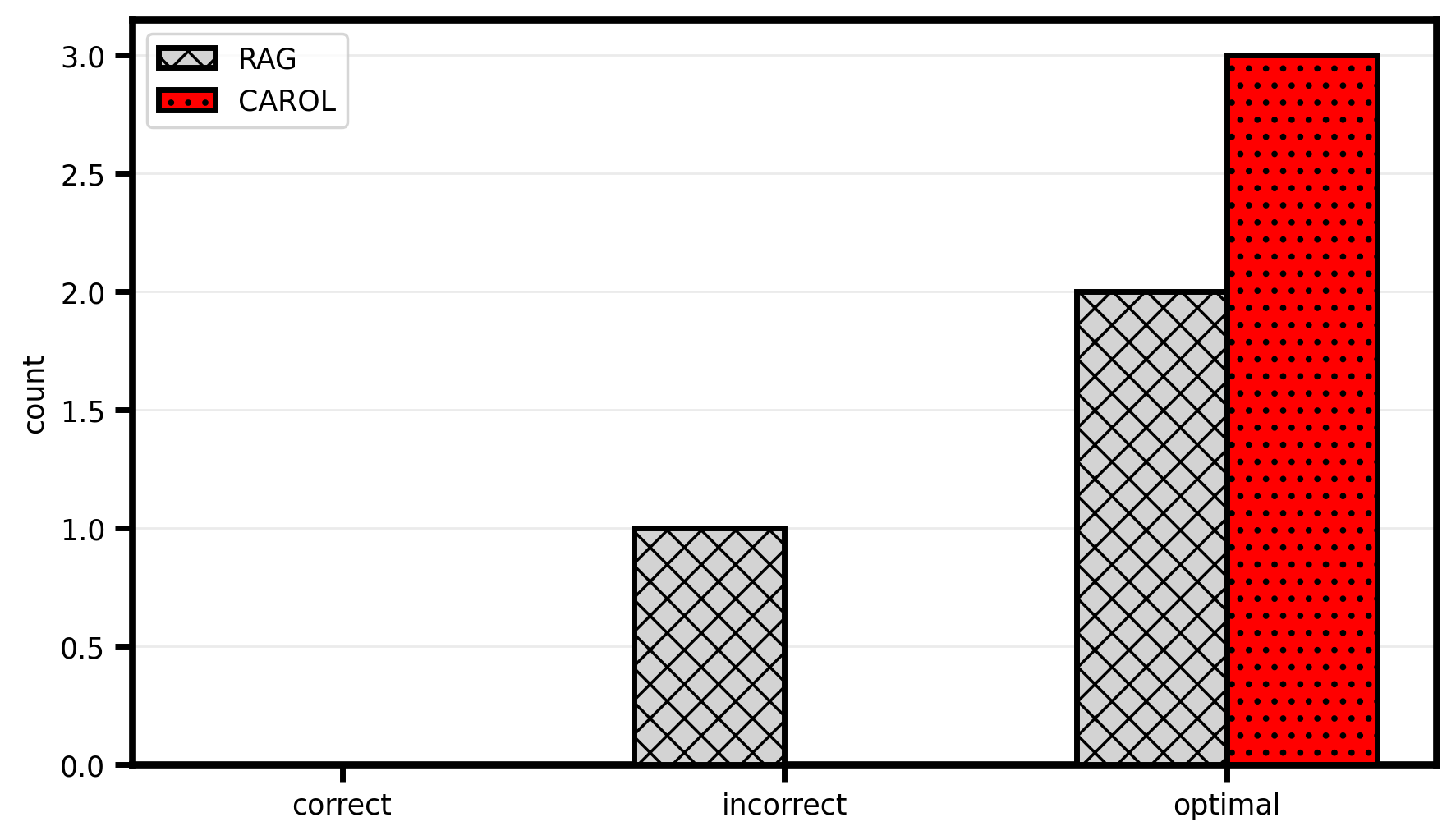}
            \caption{Statistics}
        \end{subfigure}
    \end{subcaptiongroup}
    \begin{subcaptiongroup}
        \begin{subfigure}{0.3\textwidth}
            \centering
            \includegraphics[width=\linewidth]{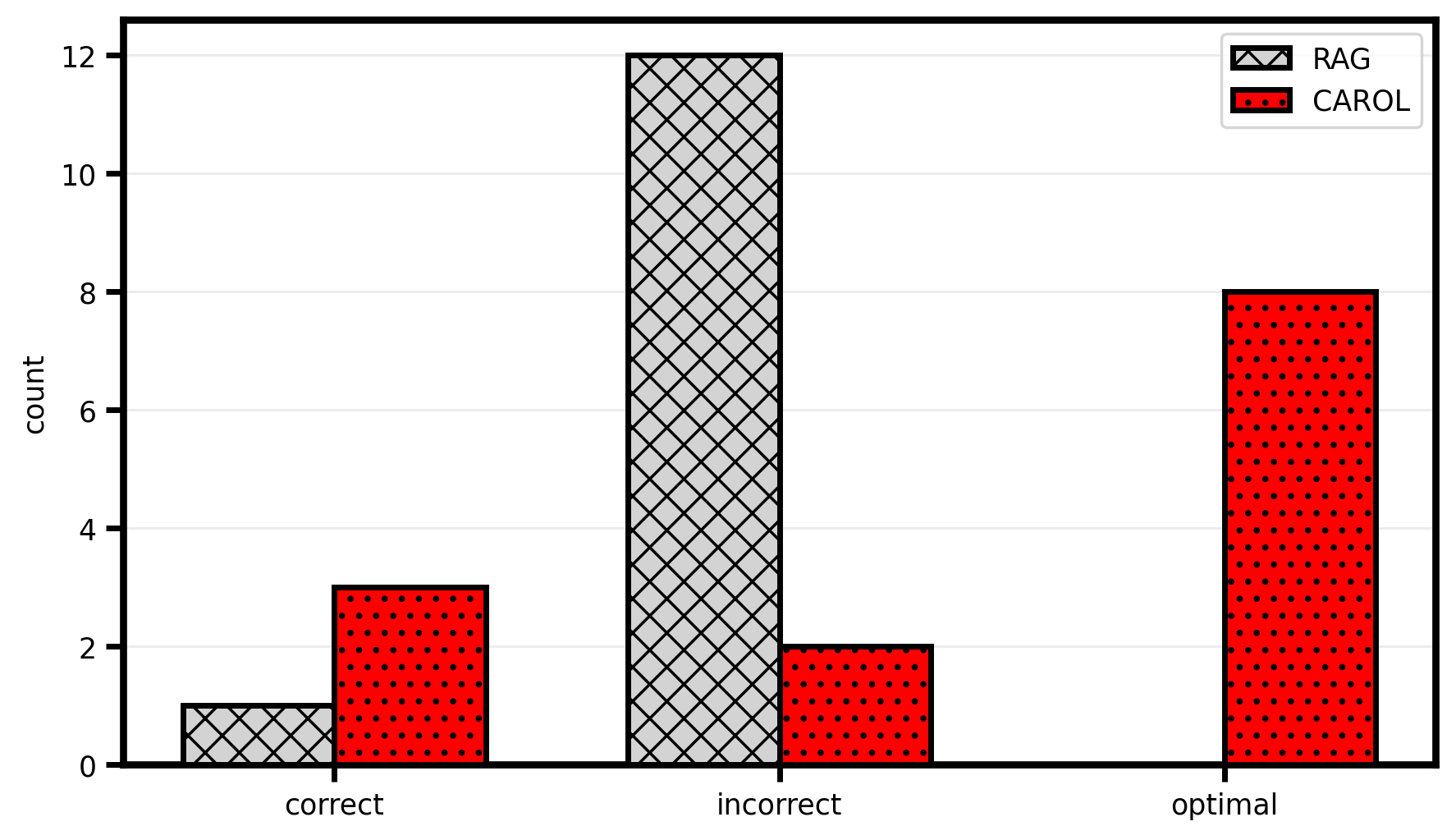}
            \caption{Stereotypes}
        \end{subfigure}
        \begin{subfigure}{0.3\textwidth}
            \centering
            \includegraphics[width=\linewidth]{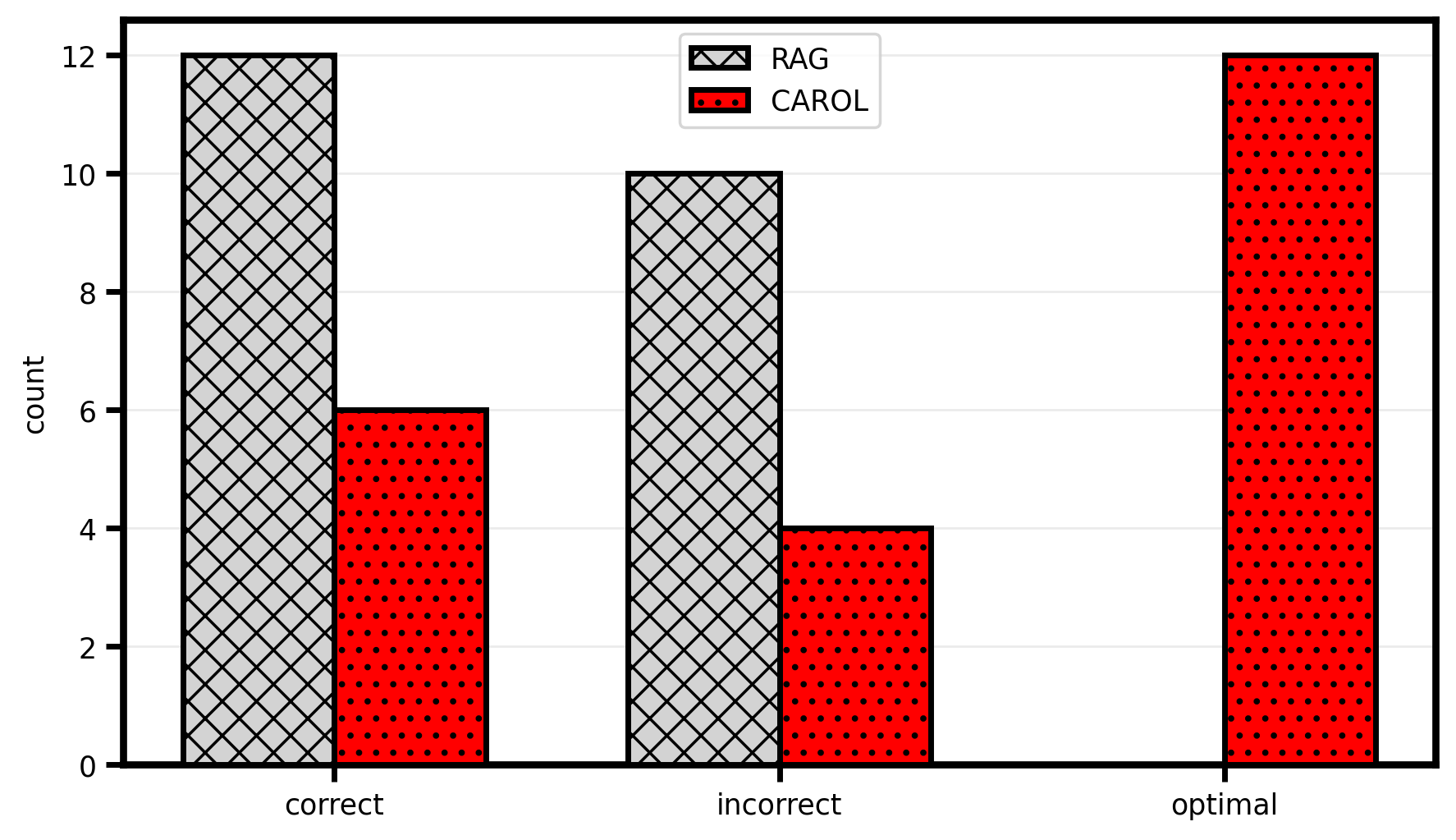}
            \caption{Superstition}
        \end{subfigure}
        \begin{subfigure}{0.3\textwidth}
            \centering
            \includegraphics[width=\linewidth]{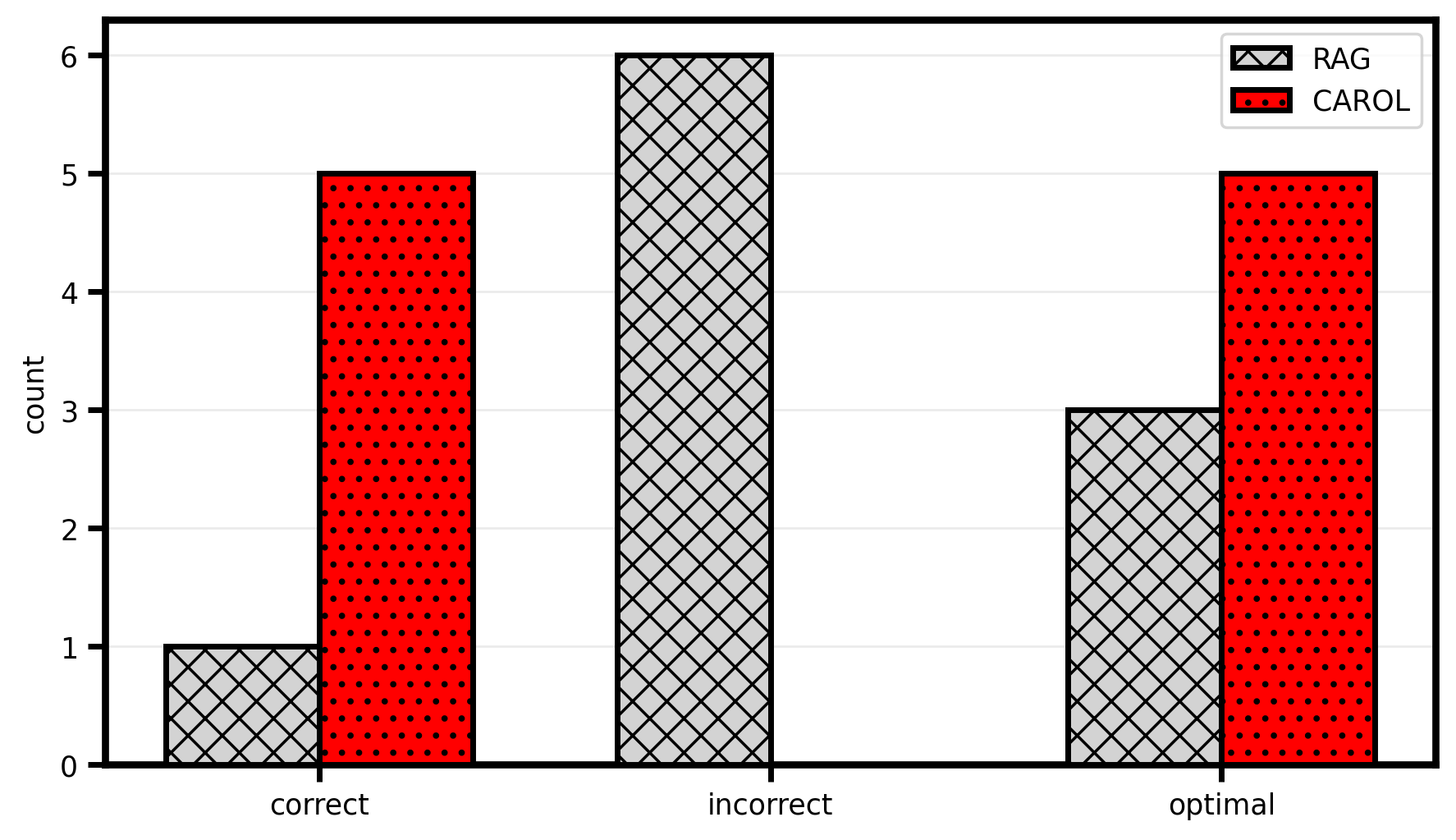}
            \caption{Subjective}
        \end{subfigure}
    \end{subcaptiongroup}
    \begin{subcaptiongroup}
        \begin{subfigure}{0.3\textwidth}
            \centering
            \includegraphics[width=\linewidth]{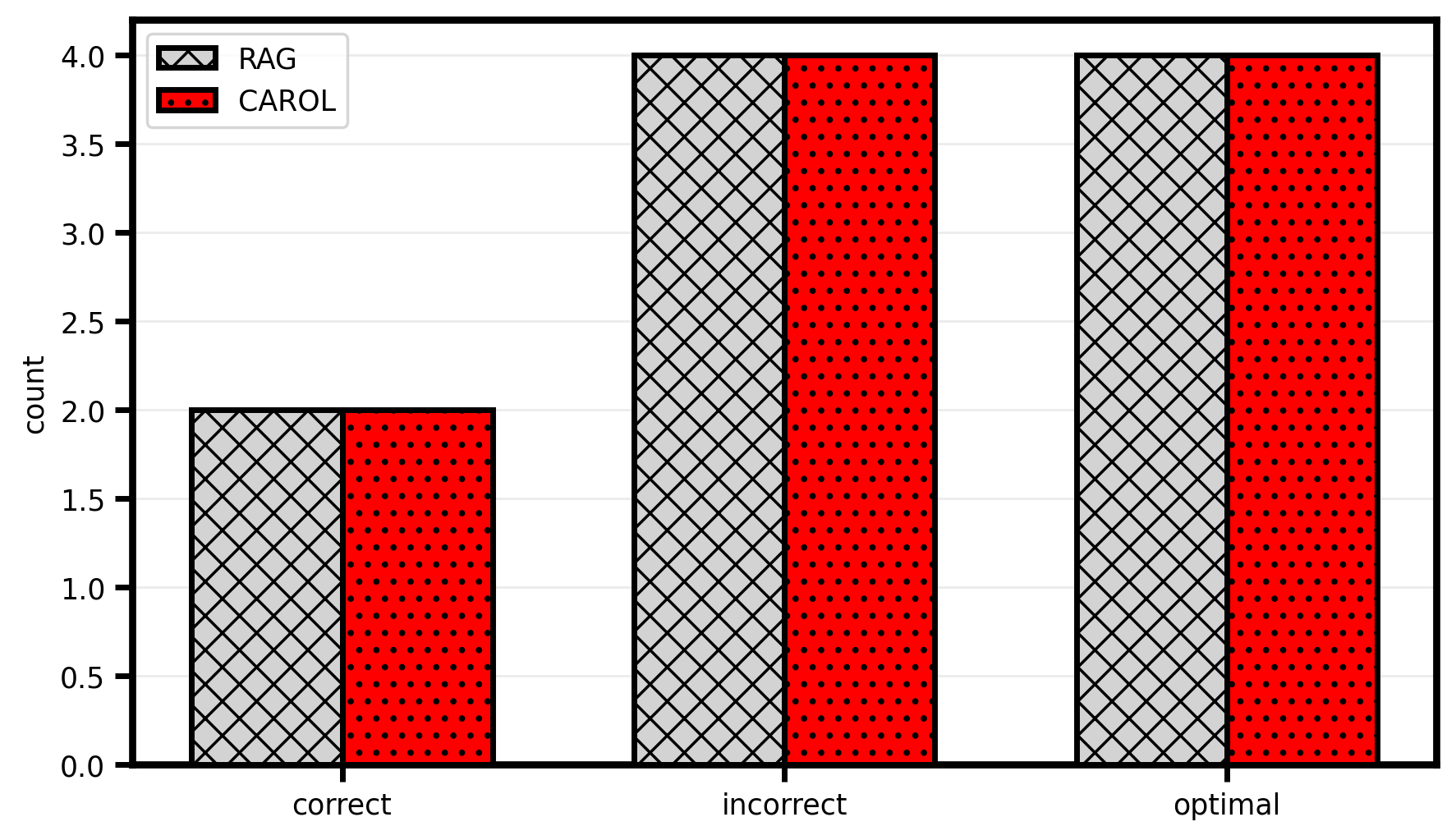}
            \caption{Weather}
        \end{subfigure}
    \end{subcaptiongroup}
\caption{\textbf{TruthfulQA} on Llama-3.1-8B extended results in each category, part $3$.}
\end{figure}


\begin{figure}[ht]
    \centering
    \begin{subcaptiongroup}
        \begin{subfigure}{0.45\textwidth}
            \centering
            \includegraphics[width=\linewidth]{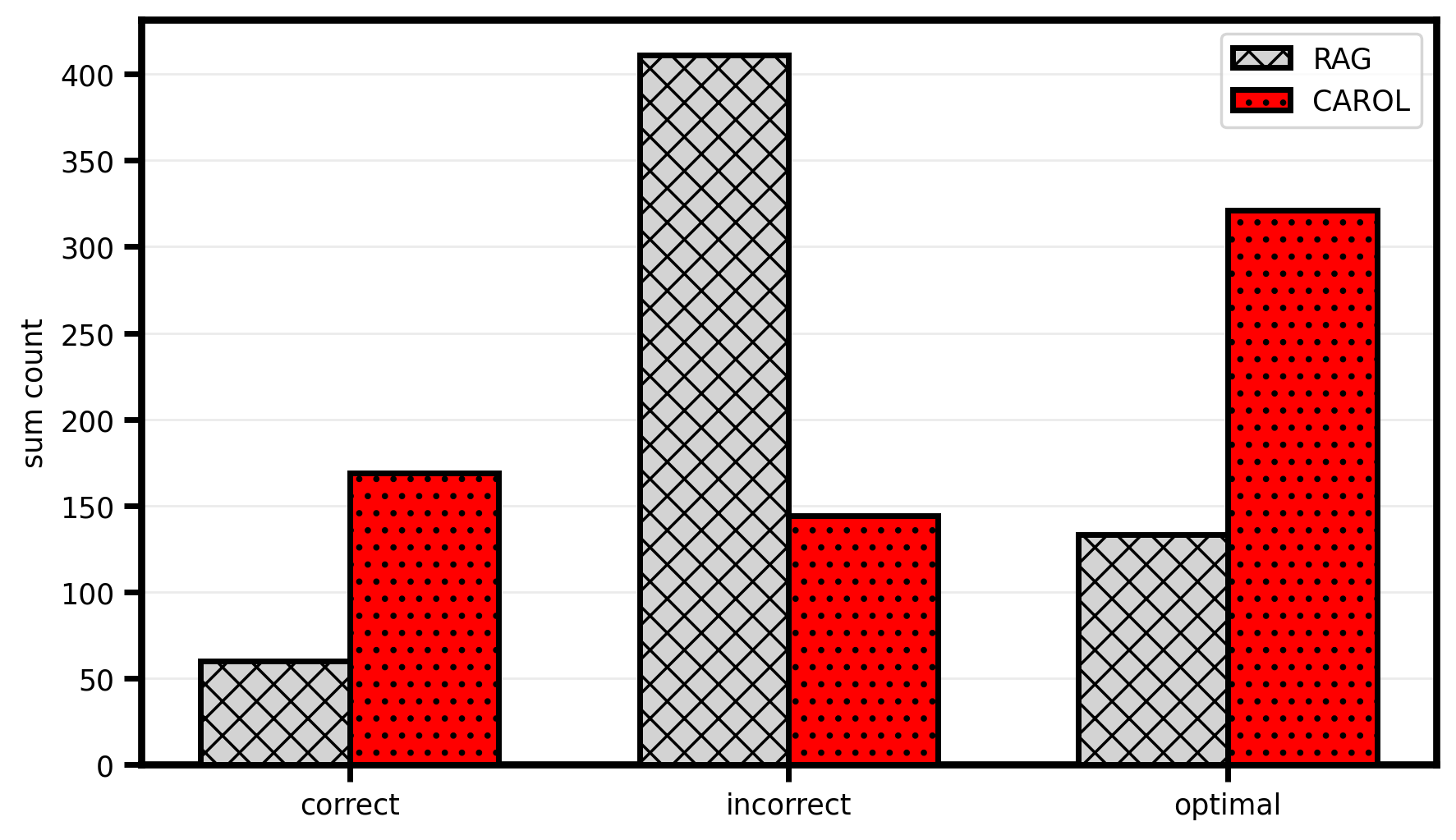}
            \caption{Results}
        \end{subfigure}
        \begin{subfigure}{0.45\textwidth}
            \centering
            \includegraphics[width=\linewidth]{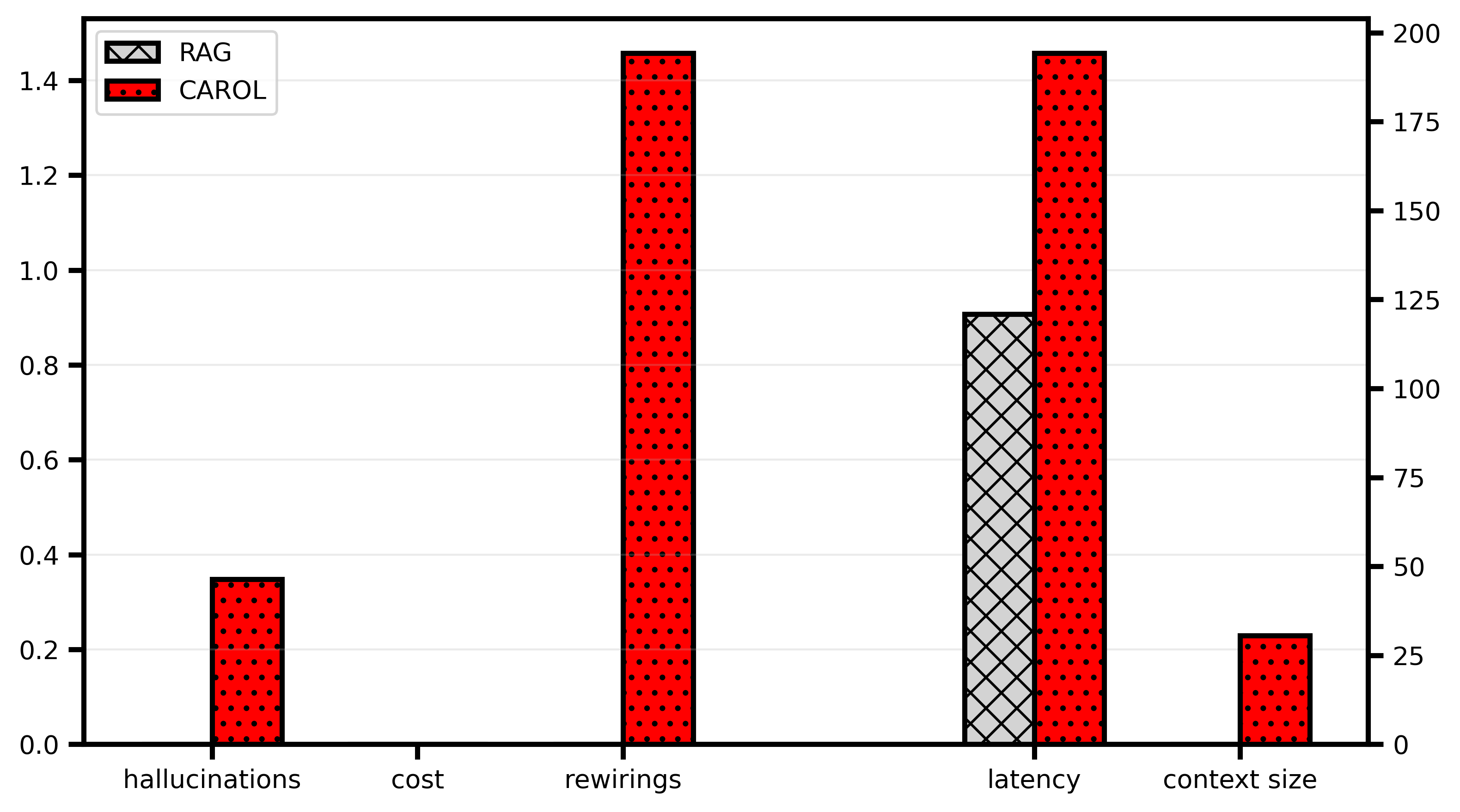}
            \caption{Execution Metrics}
        \end{subfigure}
    \end{subcaptiongroup}
\caption{\textbf{TruthfulQA} on Llama-3.1-8B overall results.}
\end{figure}

\FloatBarrier

\subsection{HaluEval Dataset}
\label{subsec:halueval}

\begin{figure}[ht]
    \centering
    \begin{subcaptiongroup}
        \begin{subfigure}{0.45\textwidth}
            \centering
            \includegraphics[width=\linewidth]{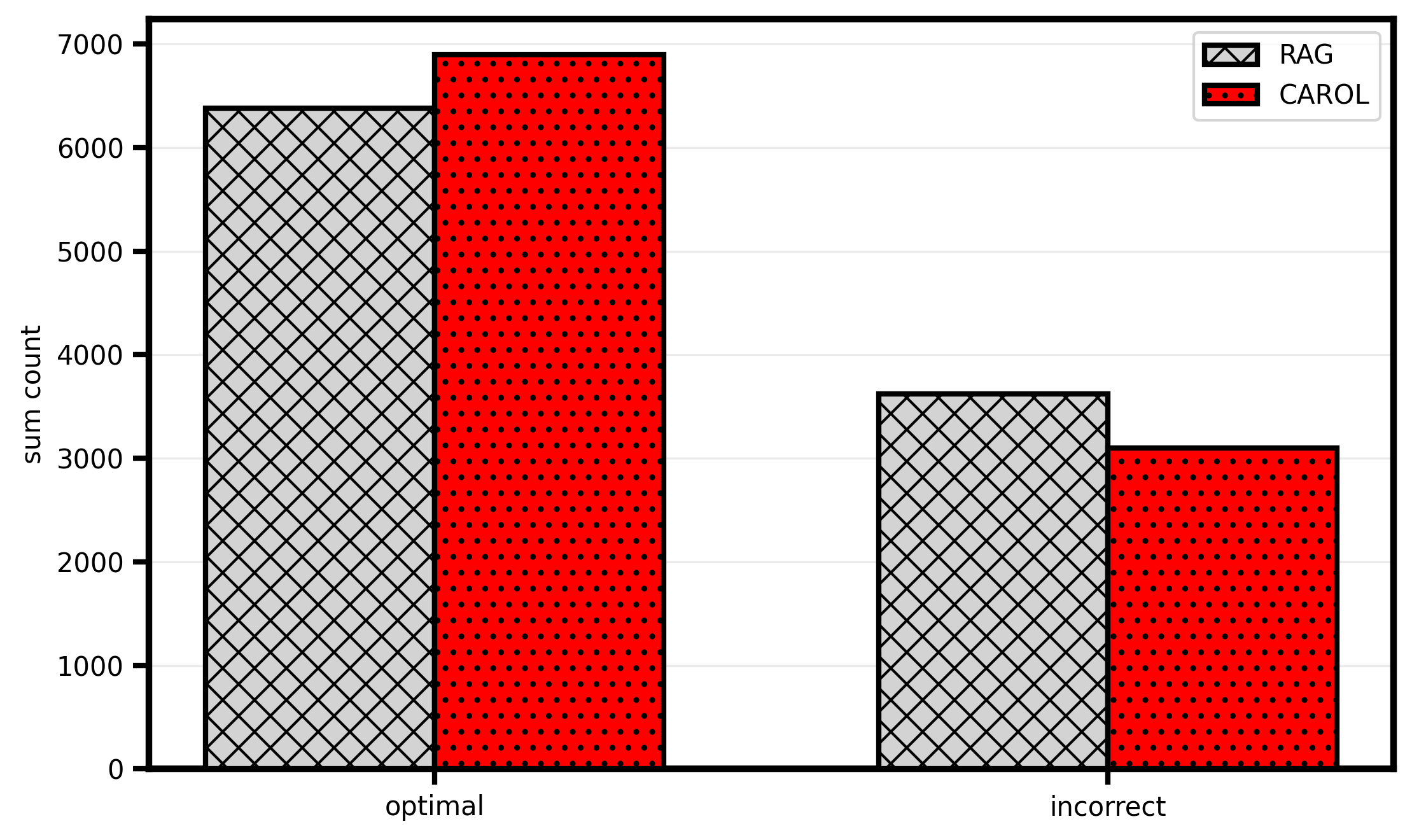}
            \caption{Results}
        \end{subfigure}
        \begin{subfigure}{0.45\textwidth}
            \centering
            \includegraphics[width=1.07\linewidth]{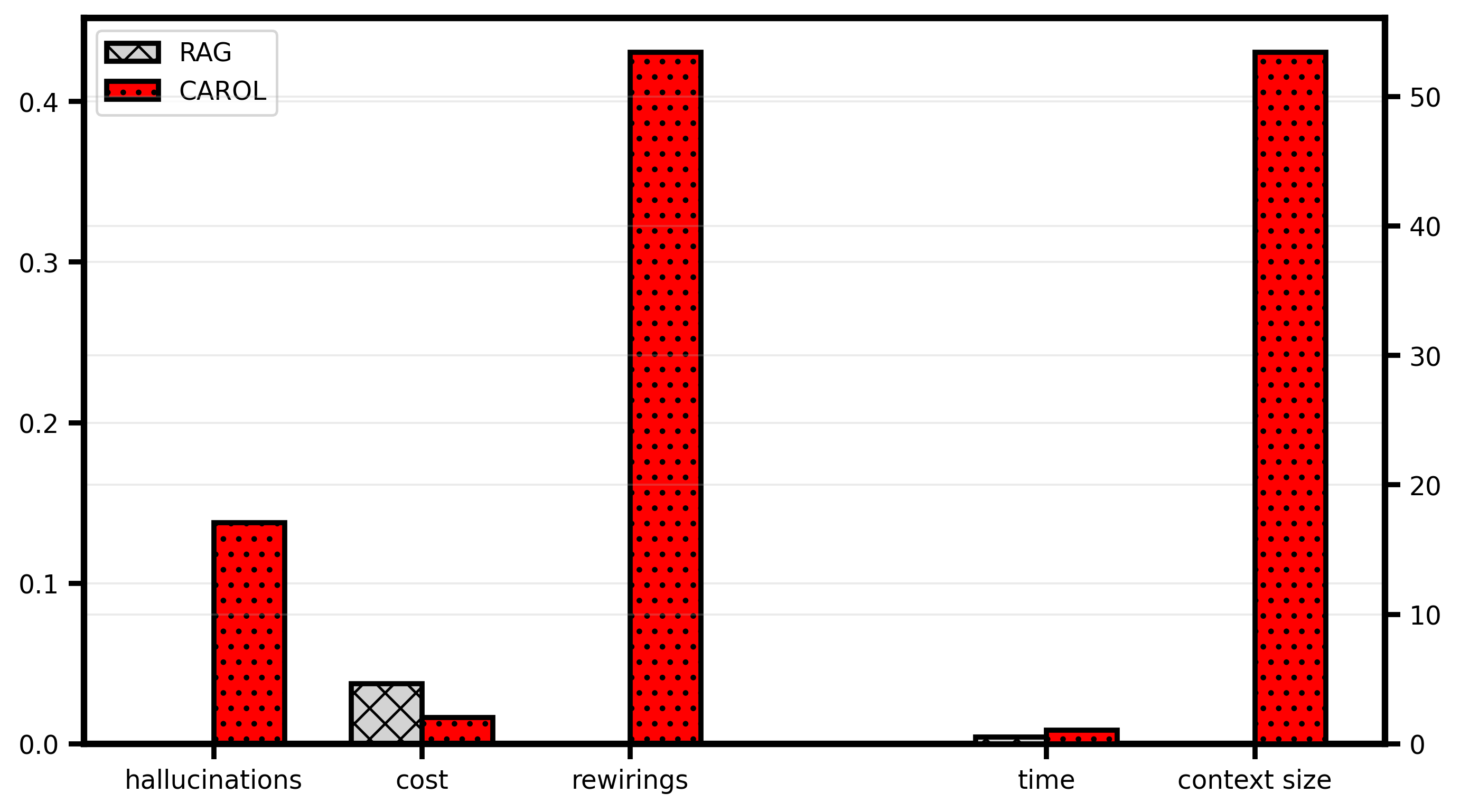}
            \caption{Execution Metrics}
        \end{subfigure}
    \end{subcaptiongroup}
\caption{\textbf{HaluEval} on GPT-5-nano overall results.}
\end{figure}

\begin{figure}[ht]
    \centering
    \begin{subcaptiongroup}
        \begin{subfigure}{0.45\textwidth}
            \centering
            \includegraphics[width=\linewidth]{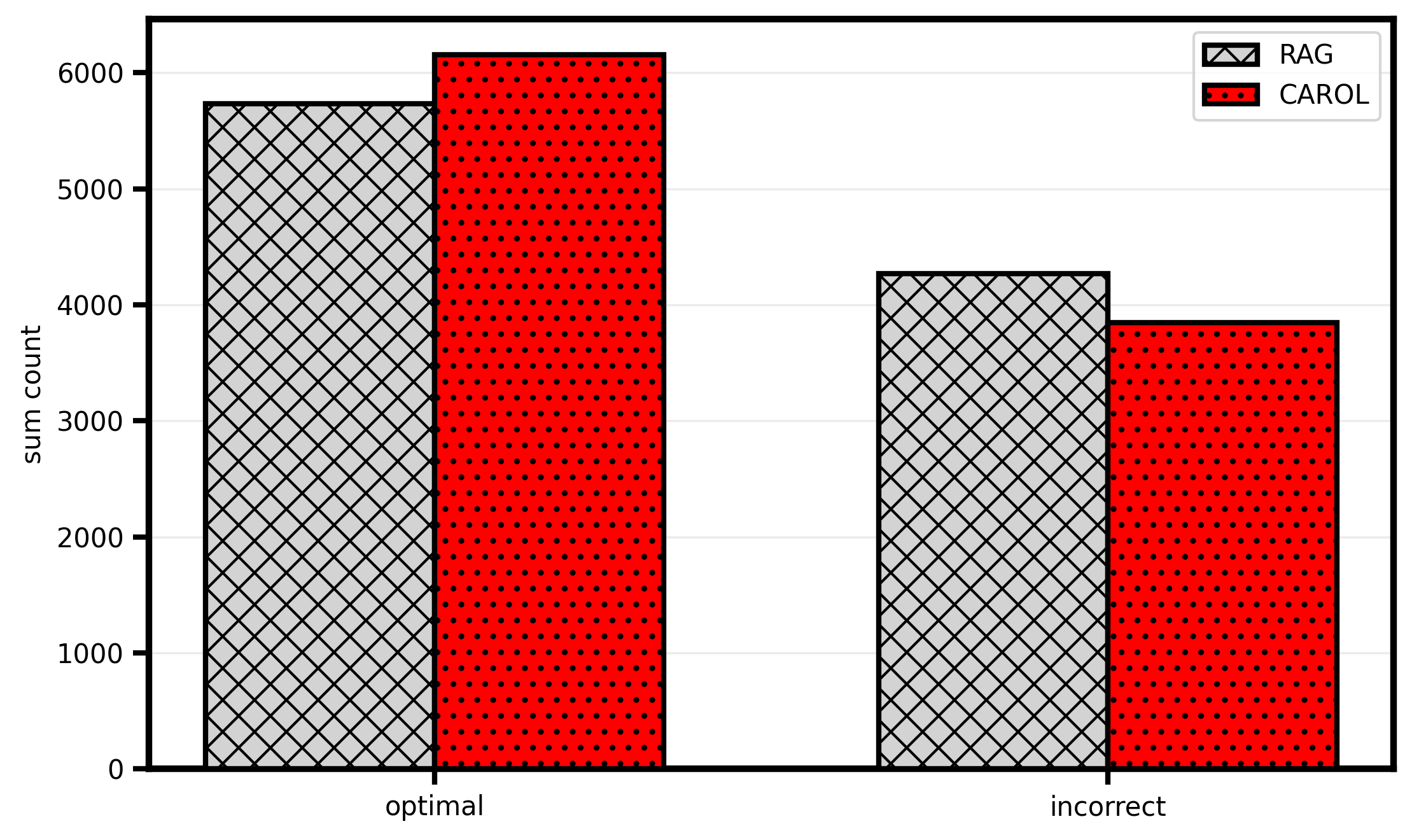}
            \caption{Results}
        \end{subfigure}
        \begin{subfigure}{0.45\textwidth}
            \centering
            \includegraphics[width=1.07\linewidth]{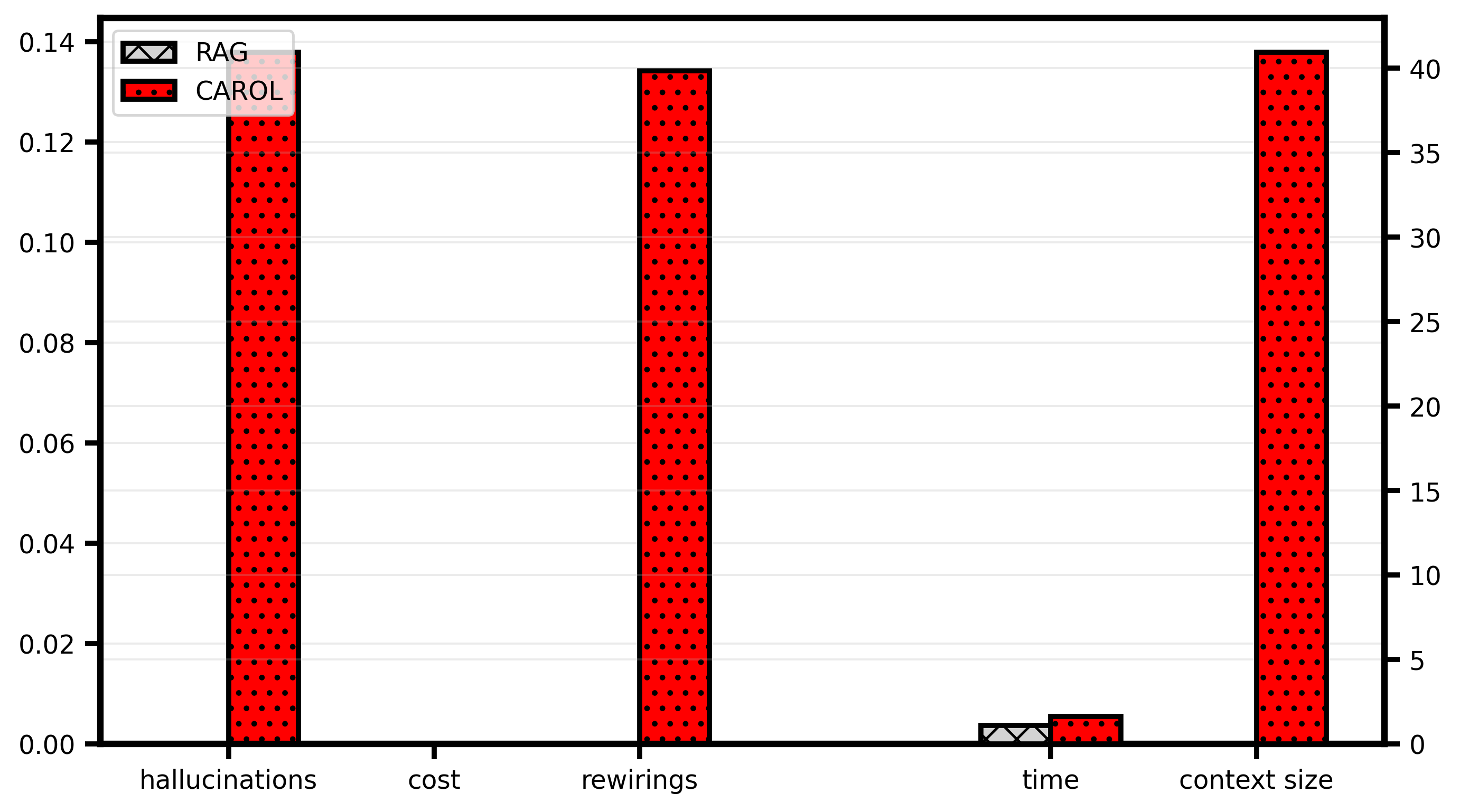}
            \caption{Execution Metrics}
        \end{subfigure}
    \end{subcaptiongroup}
\caption{\textbf{HaluEval} on Llama-3.1-8B overall results.}
\end{figure}

\FloatBarrier

\subsection{HotPotQA Dataset}
\label{subsec:hotpotqa}

\begin{figure}[ht]
    \centering
    \begin{subcaptiongroup}
        \begin{subfigure}{0.45\textwidth}
            \centering
            \includegraphics[width=\linewidth]{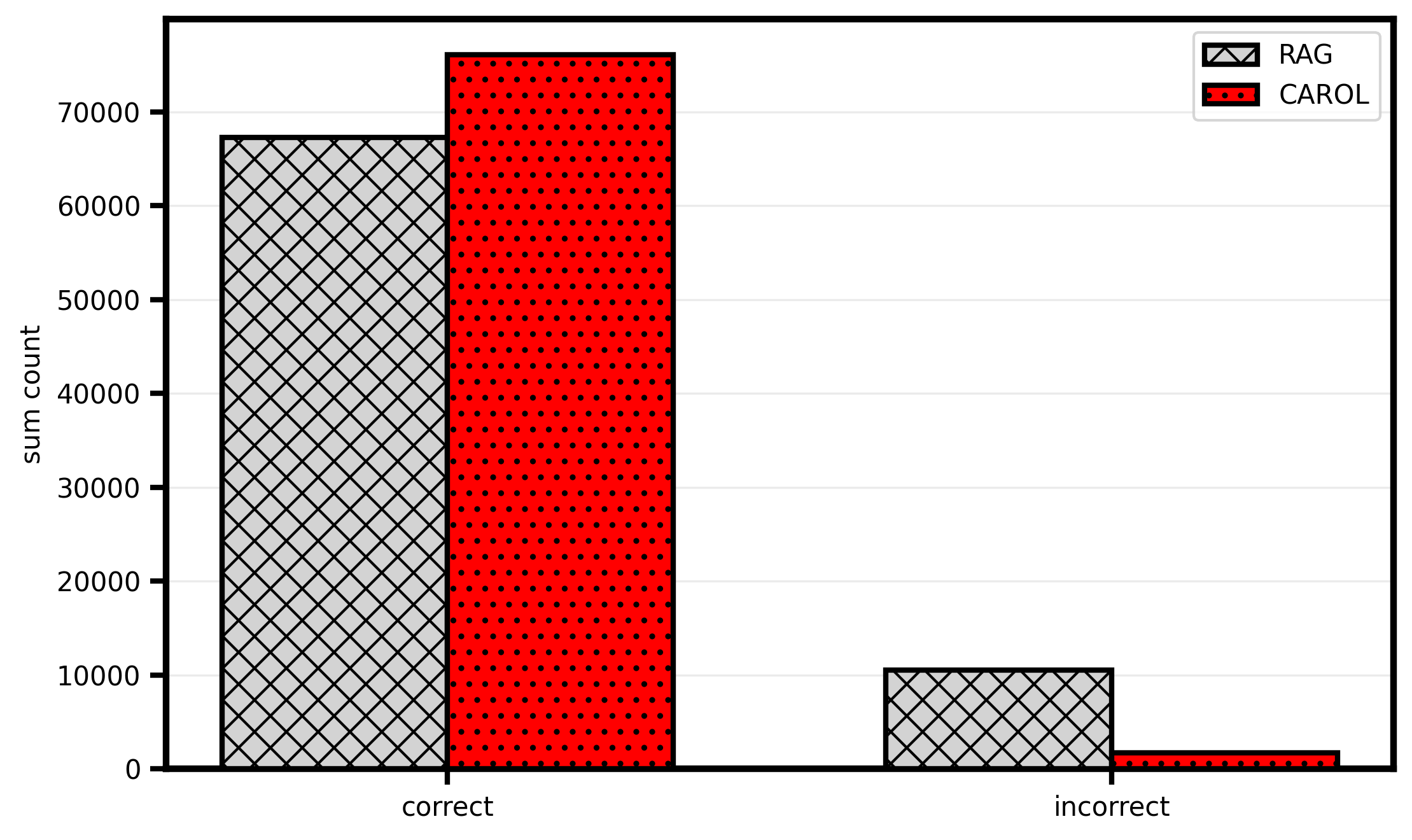}
            \caption{Results}
        \end{subfigure}
        \begin{subfigure}{0.45\textwidth}
            \centering
            \includegraphics[width=\linewidth]{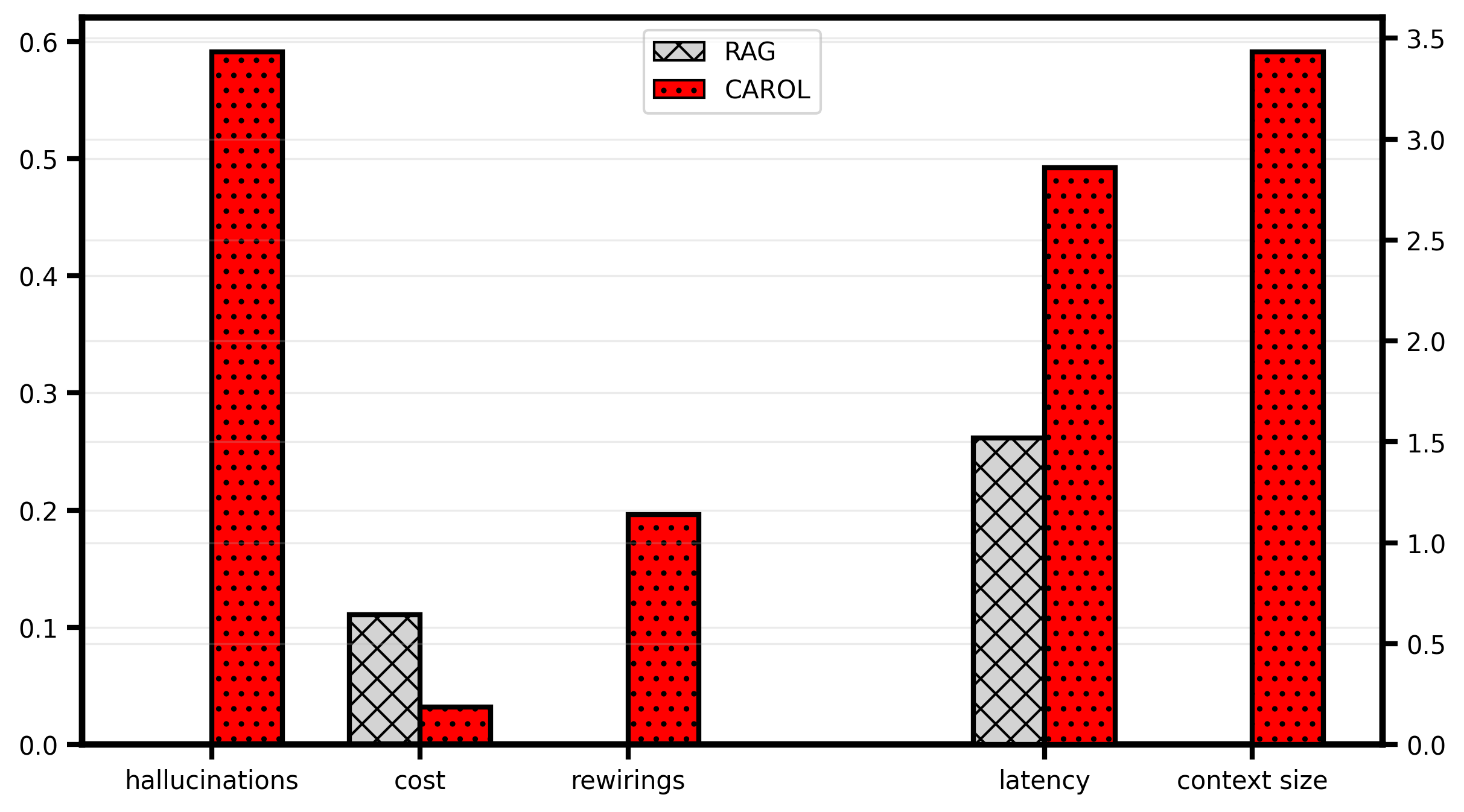}
            \caption{Execution Metrics}
        \end{subfigure}
    \end{subcaptiongroup}
\caption{\textbf{HotPotQA} on GPT-5-nano overall results.}
\end{figure}

\begin{figure}[ht]
    \centering
    \begin{subcaptiongroup}
        \begin{subfigure}{0.45\textwidth}
            \centering
            \includegraphics[width=\linewidth]{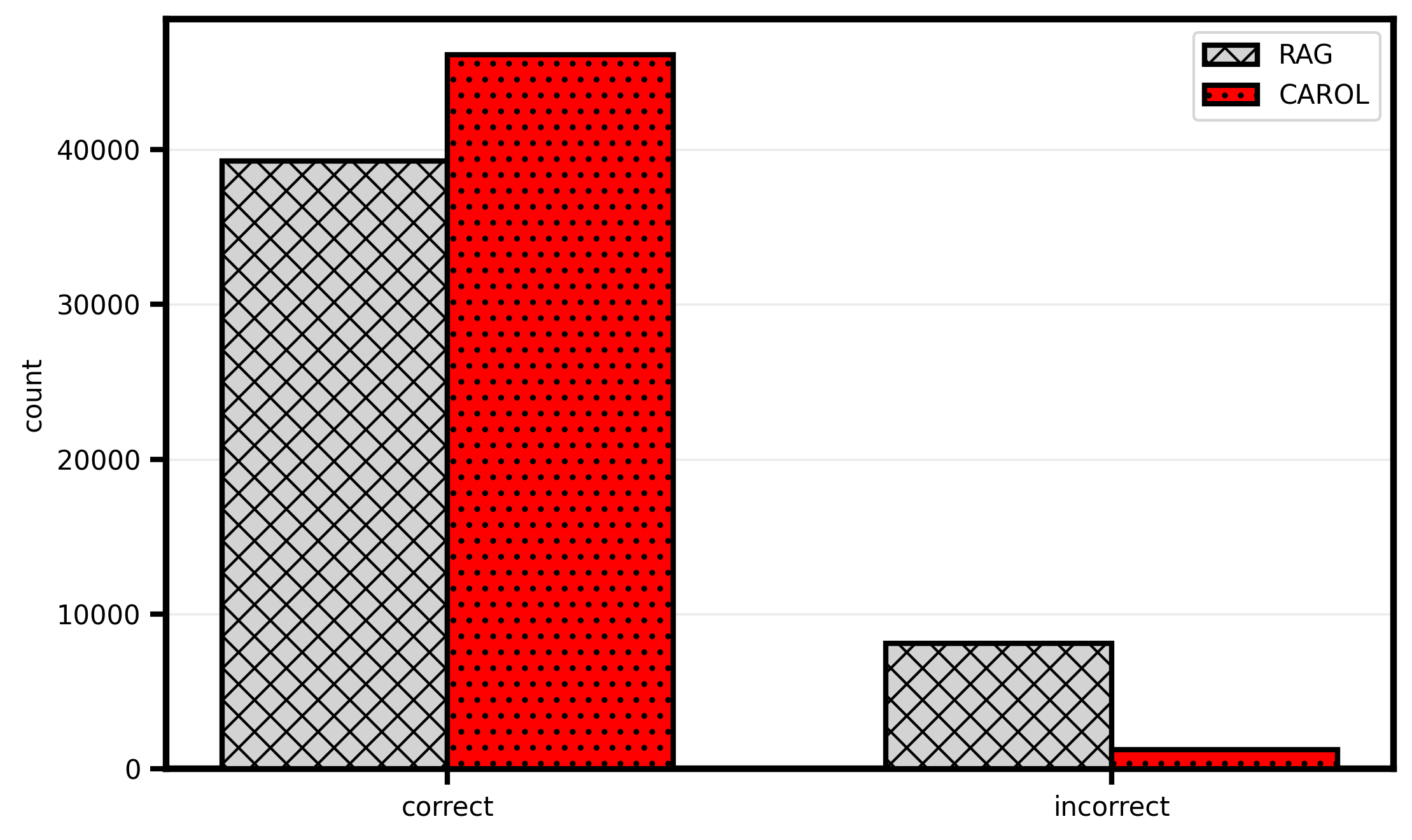}
            \caption{Bridge}
        \end{subfigure}
        \begin{subfigure}{0.45\textwidth}
            \centering
            \includegraphics[width=\linewidth]{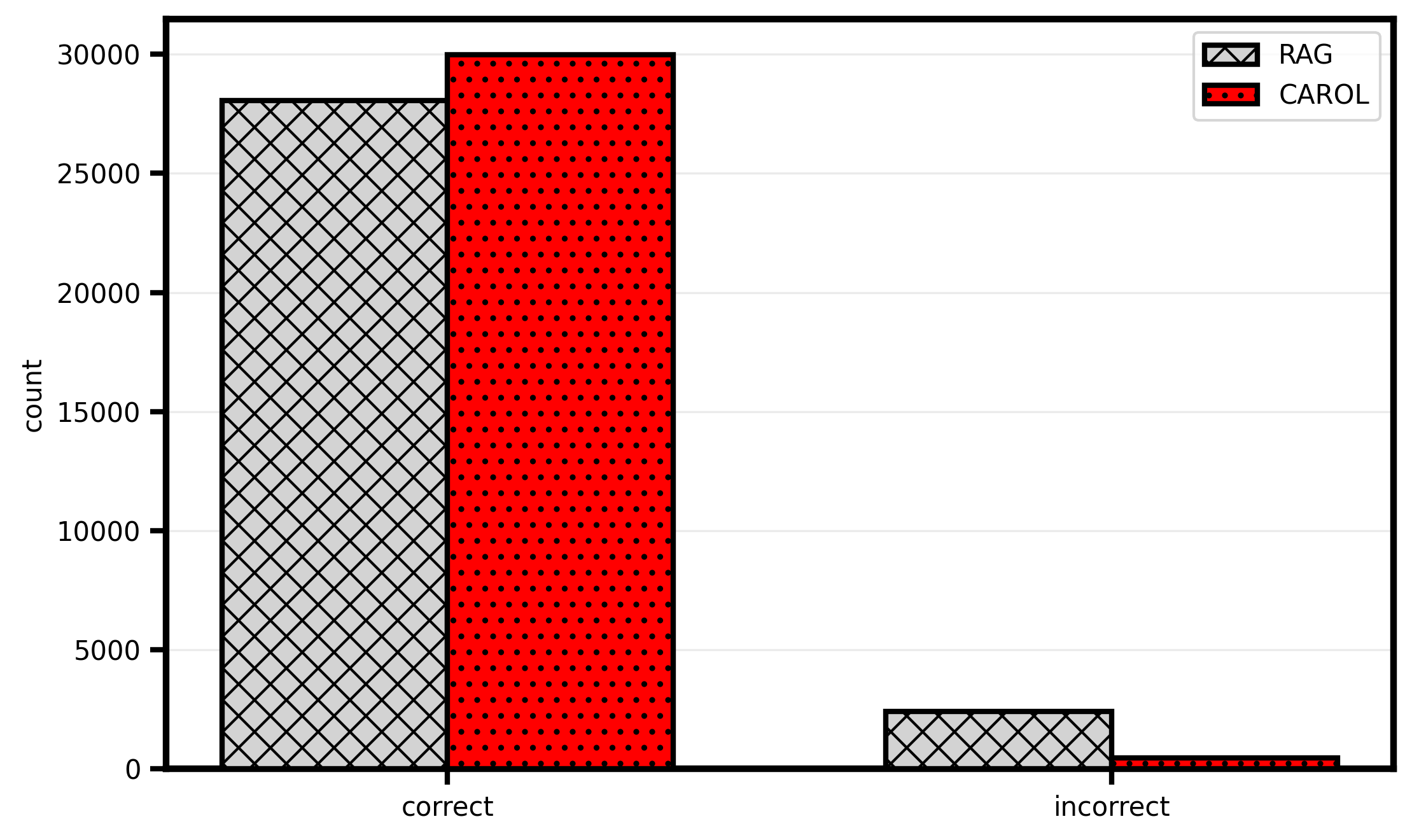}
            \caption{Comparison}
        \end{subfigure}
    \end{subcaptiongroup}
\caption{\textbf{HotPotQA} on GPT-5-nano results for each category.}
\end{figure}

\begin{figure}[ht]
    \centering
    \begin{subcaptiongroup}
        \begin{subfigure}{0.45\textwidth}
            \centering
            \includegraphics[width=\linewidth]{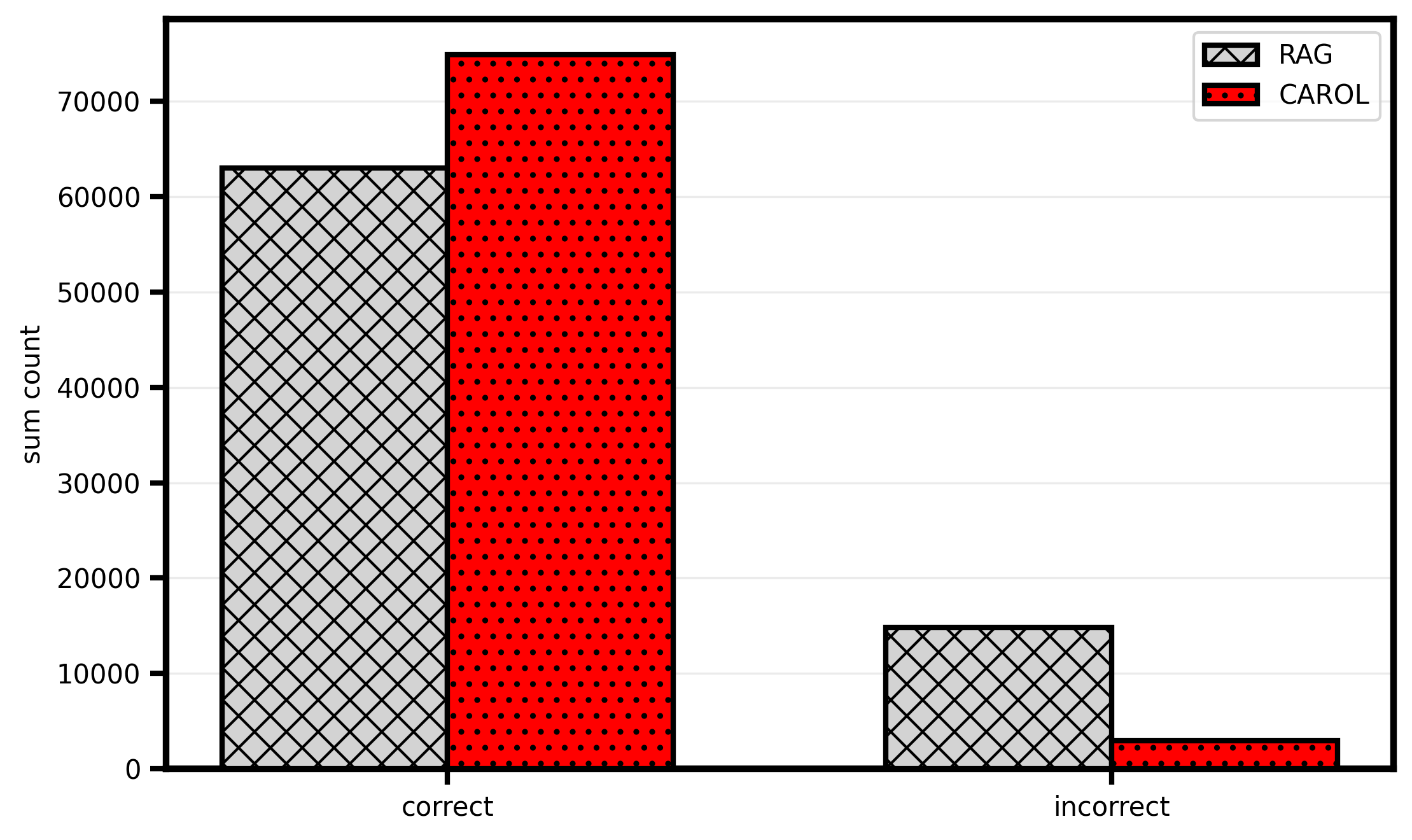}
            \caption{Results}
        \end{subfigure}
        \begin{subfigure}{0.45\textwidth}
            \centering
            \includegraphics[width=\linewidth]{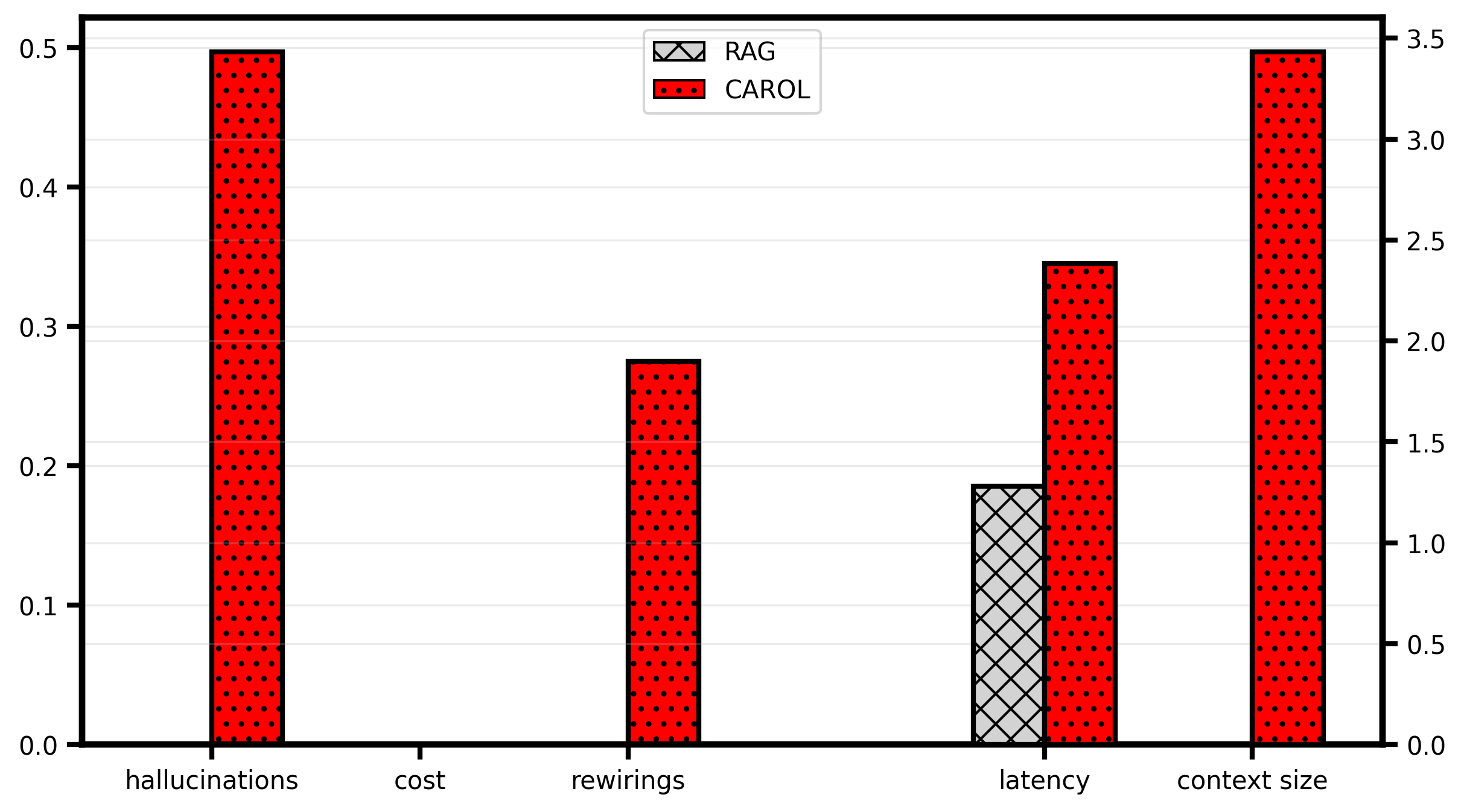}
            \caption{Execution Metrics}
        \end{subfigure}
    \end{subcaptiongroup}
\caption{\textbf{HotPotQA} on Llama-3.1-8B overall results.}
\end{figure}

\begin{figure}[ht]
    \centering
    \begin{subcaptiongroup}
        \begin{subfigure}{0.45\textwidth}
            \centering
            \includegraphics[width=\linewidth]{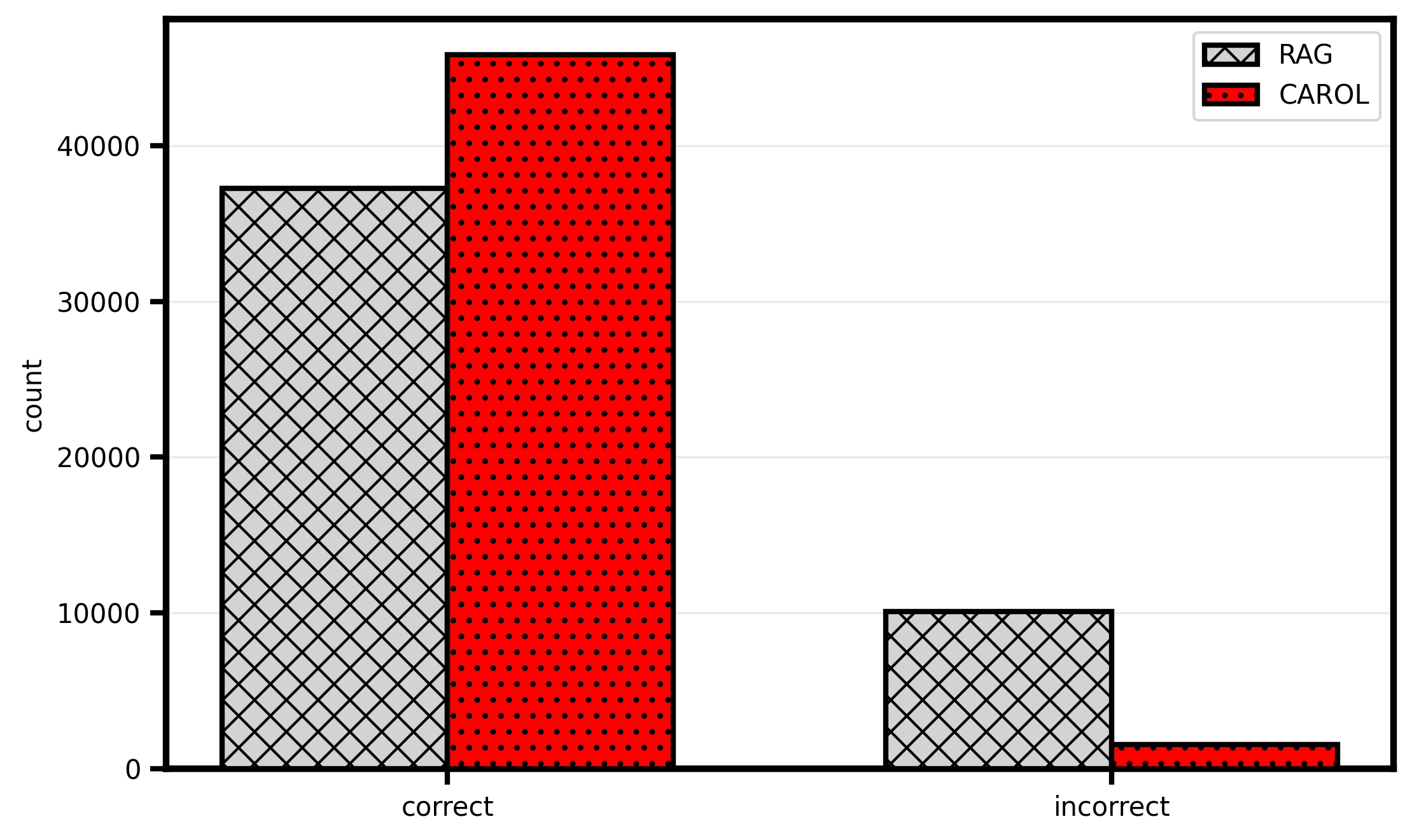}
            \caption{Bridge}
        \end{subfigure}
        \begin{subfigure}{0.45\textwidth}
            \centering
            \includegraphics[width=\linewidth]{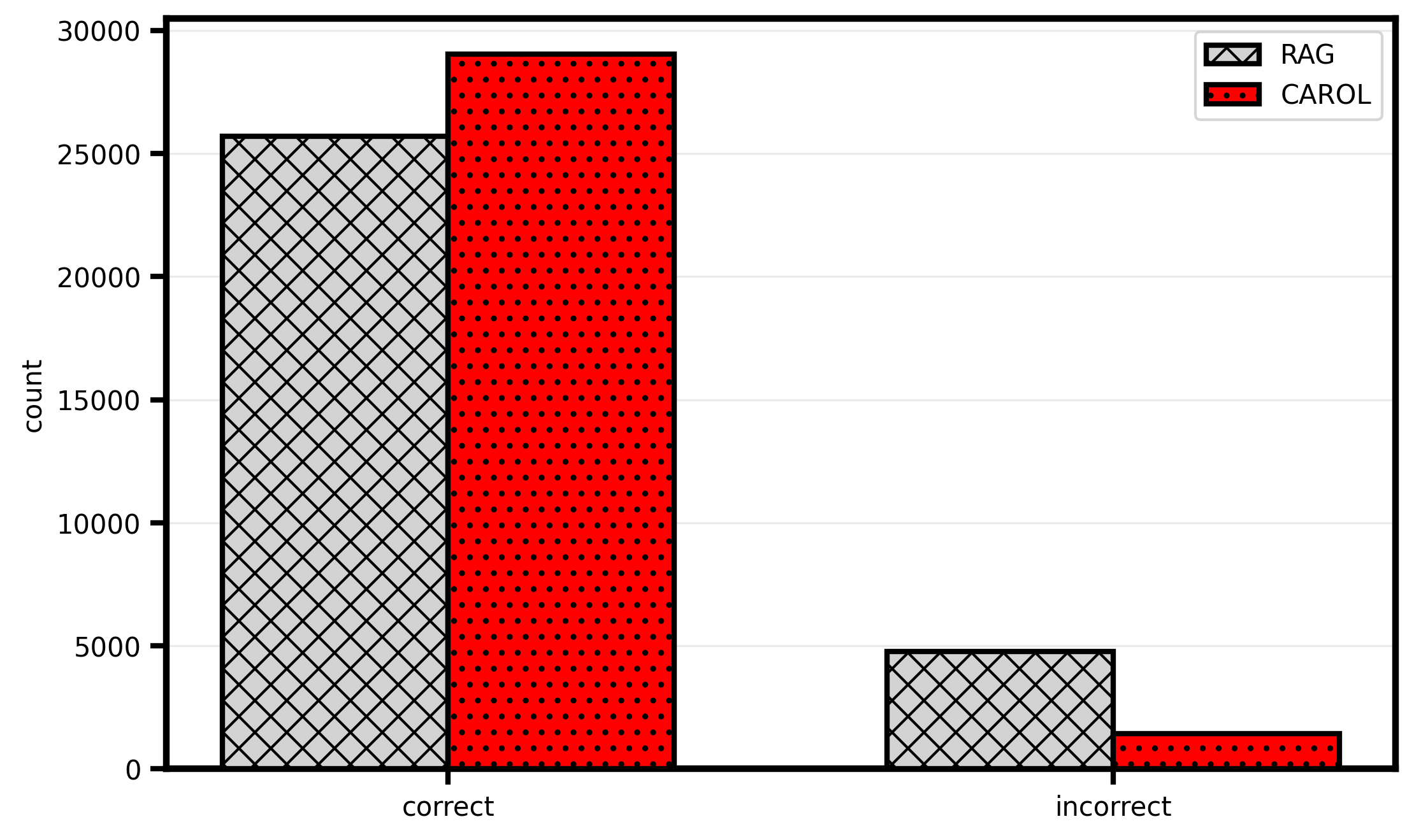}
            \caption{Comparison}
        \end{subfigure}
    \end{subcaptiongroup}
\caption{\textbf{HotPotQA} on Llama-3.1-8B results for each category.}
\end{figure}
\FloatBarrier

\subsection{Clustering Robustness Analysis}
\label{app:clustering_robustness}

This subsection provides an additional empirical analysis supporting the use of exemplar-based clustering in \texttt{CAROL}. The main paper introduces semantic entropy through a clustering-induced partition of semantic units with respect to the trusted context $\Gamma$. Since the resulting partition directly affects the estimated semantic entropy, the clustering mechanism must be stable under changes in sampling or assignment hyperparameters. In particular, temperature-sensitive clustering can produce substantially different partitions for the same semantic content, which may lead to unstable hallucination scores.

We compare two clustering mechanisms: a temperature-scaled soft $k$-means baseline and a temperature-free $k$-medoids method. The experiment uses a curated set of semantically structured sentences grouped into four categories: statements about Paris as the capital of France, general facts about France, Paris landmarks, and incorrect or outlier, see Table~\ref{tab:semantic_density_data}
\footnote{The dataset and code used in this experiment will be publicly available on GitHub after reviews.}. 
Each sentence is embedded with \texttt{BERT} sentence encoder \cite{JD-MC-KL-KT:18} and clustered under a sweep of temperatures $T \in [10^{-2},10]$. For each temperature, we run multiple independent trials and evaluate the resulting partitions using Adjusted Rand Index (ARI) \cite{LH-PA:85}, Normalized Mutual Information (NMI), and run-to-run variance.

\begin{figure}[ht]
    \centering
    \includegraphics[width=0.60\linewidth]{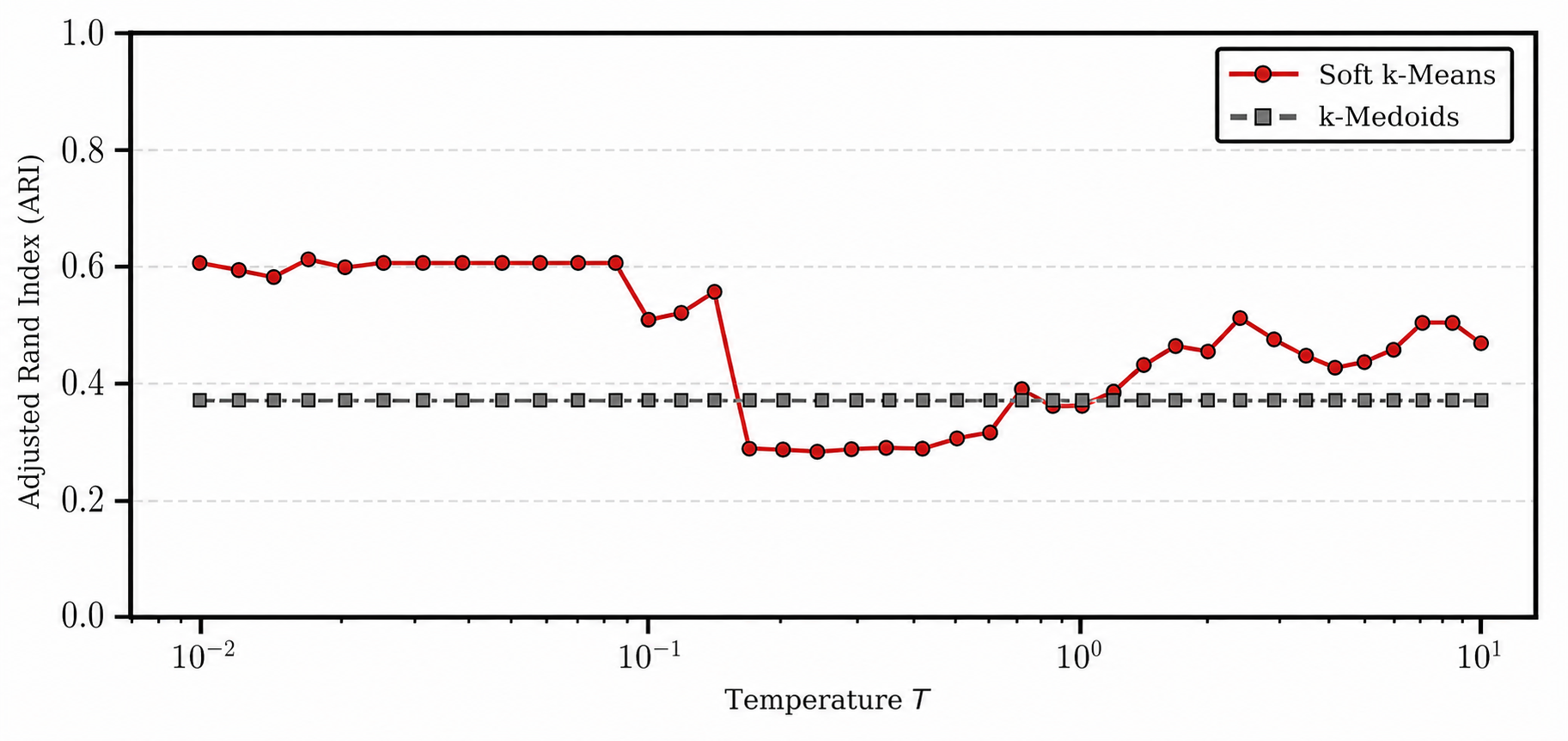}
    \caption{
    Adjusted Rand Index as a function of temperature. Soft $k$-means achieves higher agreement at low temperature but becomes unstable as temperature increases. In contrast, $k$-medoids remains nearly constant because it does not rely on a temperature-scaled soft assignment.
    }
    \label{fig:ari_temperature_appendix}
\end{figure}

\begin{figure}[ht]
    \centering
    \includegraphics[width=0.60\linewidth]{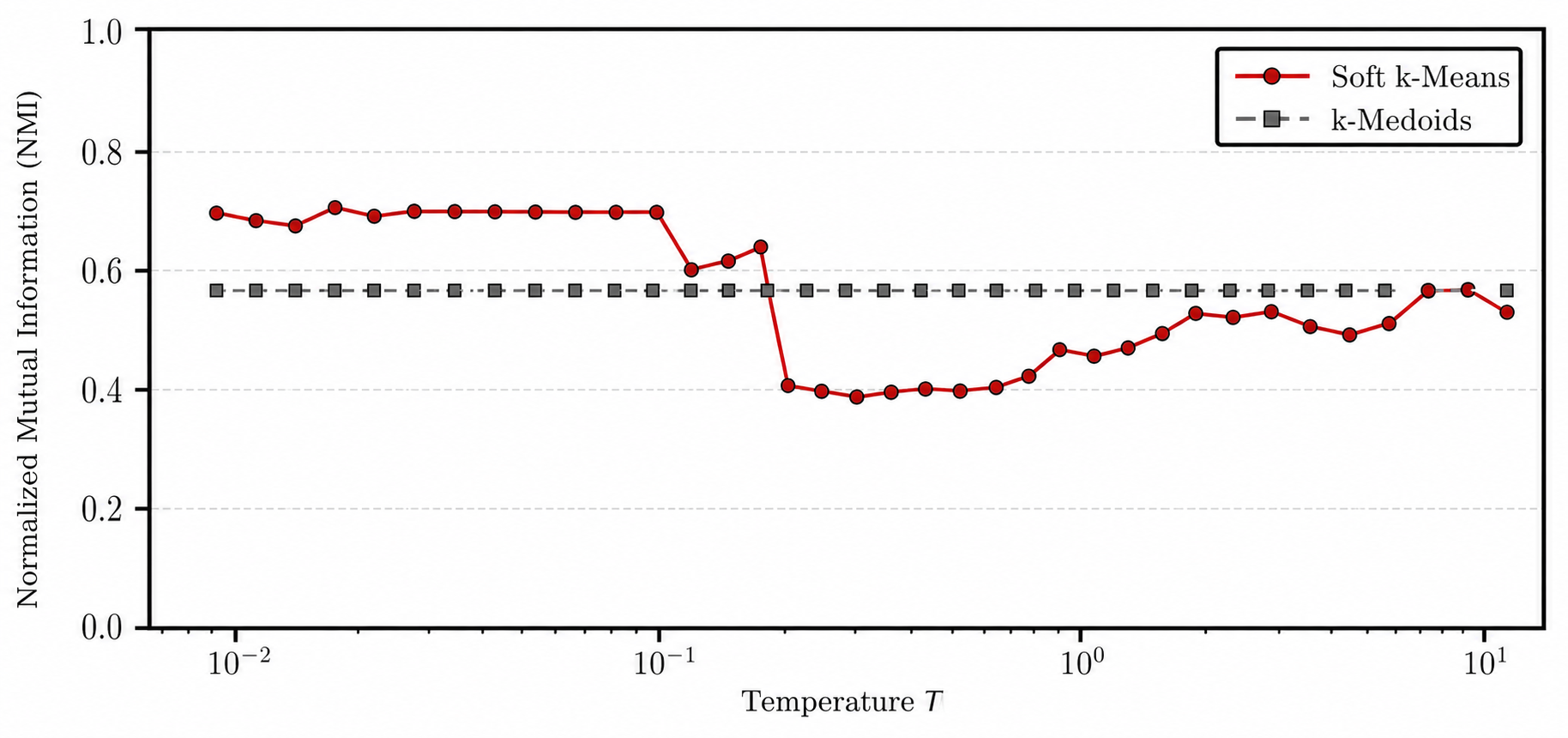}
    \caption{
    Normalized Mutual Information as a function of temperature. The $k$-medoids curve remains flat across the full temperature range, while soft $k$-means exhibits sharp degradation and partial recovery depending on the assignment temperature. This confirms that temperature introduces an additional source of variability into the semantic partition.
    }
    \label{fig:nmi_temperature_appendix}
\end{figure}

The results show that soft $k$-means is highly sensitive to the choice of temperature. At low values of $T$, the assignment is close to hard clustering and can recover meaningful groups. However, as $T$ increases, the assignment distribution becomes more diffuse, causing clusters to merge or split in ways that are not semantically stable. This behavior is undesirable for hallucination detection: two runs on the same semantic content may produce different clusterings, and therefore different entropy estimates, only because of the temperature parameter.

By contrast, $k$-medoids selects actual observed sentences as representatives. This is important for \texttt{CAROL}: medoids preserve semantic interpretability because every cluster center is itself a generated or contextual $\ell$-gram, rather than an averaged embedding that may not correspond to any meaningful sentence. This makes the induced semantic entropy more transparent and easier to inspect. Moreover, because no temperature parameter is required, the method removes an external tuning degree of freedom from the hallucination score.

\begin{figure}[ht]
    \centering
    \includegraphics[width=0.40\linewidth]{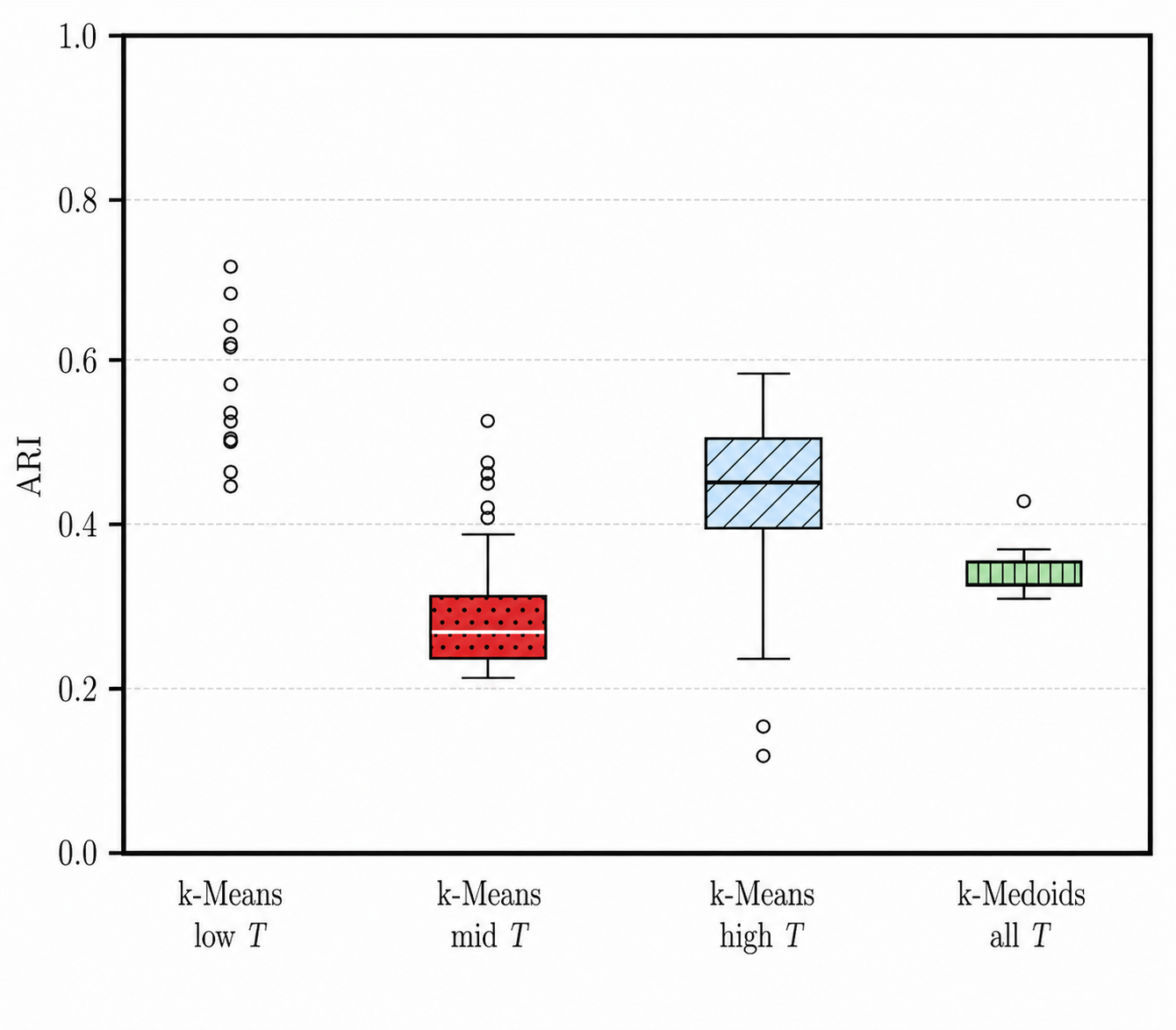}
    \caption{
    Distribution of ARI values across temperature regimes. Soft $k$-means displays strong variability across low-, mid-, and high-temperature regimes. $k$-medoids remains concentrated, indicating greater robustness under repeated trials.
    }
    \label{fig:ari_boxplot_appendix}
\end{figure}

\begin{figure}[ht]
    \centering
    \includegraphics[width=0.40\linewidth]{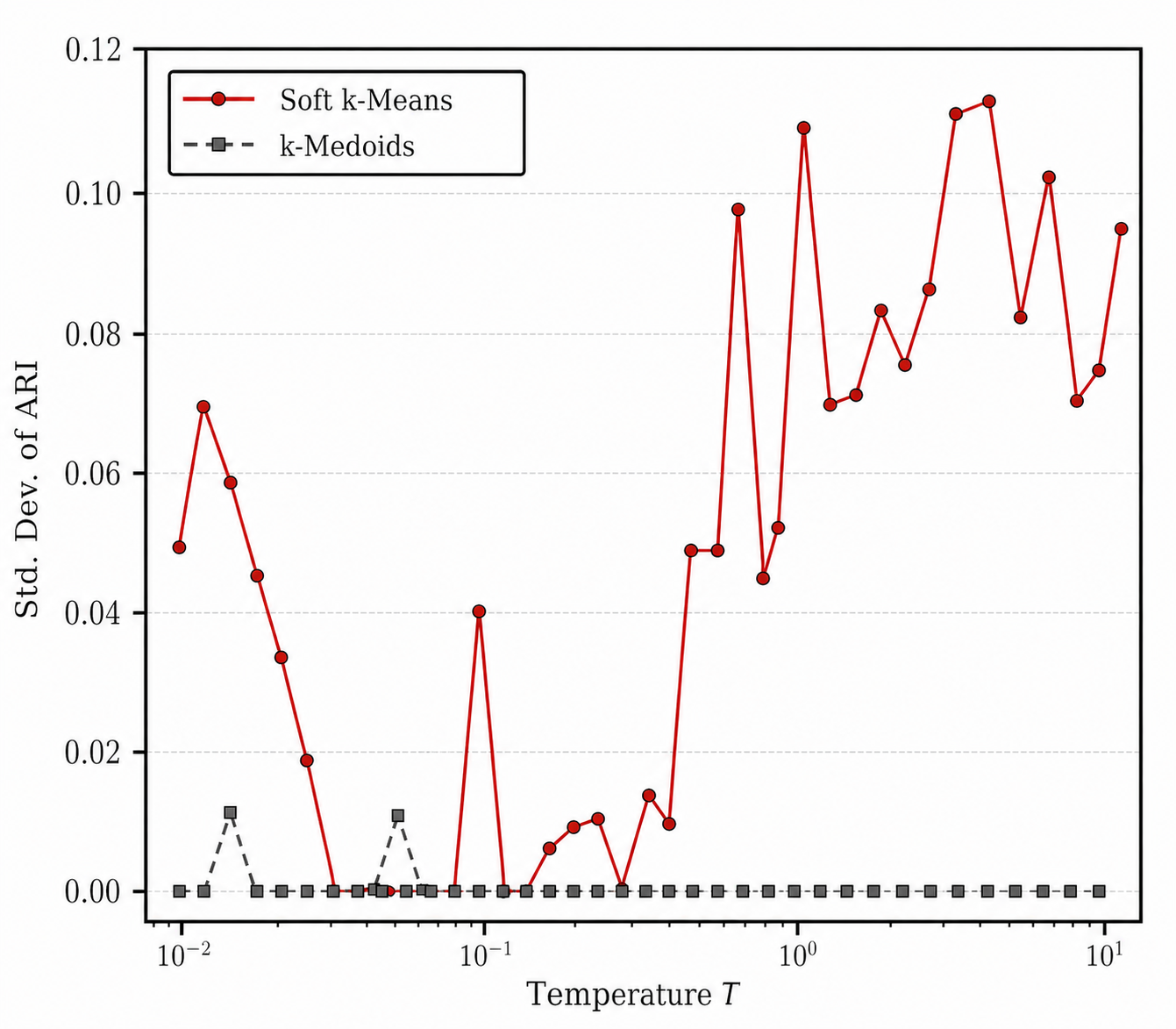}
    \caption{
    Run-to-run standard deviation of ARI across temperatures. The variance of soft $k$-means increases substantially in several temperature regimes, whereas $k$-medoids remains close to zero. This supports the use of exemplar-based clustering as a stable semantic partitioning mechanism.
    }
    \label{fig:ari_variance_appendix}
\end{figure}

Overall, this experiment highlights three advantages of medoid-based clustering over mean-based clustering in the context of semantic entropy. First, medoids are \emph{temperature-free}, avoiding sensitivity to soft assignment hyperparameters. Second, medoids are \emph{semantically interpretable}, since each representative is an actual sentence rather than an artificial centroid in embedding space. Third, medoids are \emph{stable across runs}, producing nearly invariant ARI and NMI values across the temperature sweep. These properties are especially important for hallucination reduction, where the reliability score should depend on semantic agreement with $\Gamma$, not on arbitrary clustering hyperparameters.

This supports the design choice in \texttt{CAROL}: semantic entropy is computed using exemplar-based partitions so that hallucination detection and accept--reject decisions are grounded in stable, interpretable semantic representatives.
\FloatBarrier


\subsection{Cluster Density and Semantic Entropy}
\label{app:cluster_density_entropy}

Following previous section, we designed an experiment centered in Exemplar-based clustering method that illustrates the intuition behind the semantic entropy score used in \texttt{CAROL}. We construct a controlled dataset consisting of a trusted context $\Gamma$ and three response regimes with varying levels of semantic agreement. The trusted context $\Gamma$ contains sentences describing consistent facts about Paris and France, covering geography, governance, and landmarks. We then define three scenarios of generated $\ell$-grams that progressively deviate from this context: (i) dense agreement, where all generated statements are semantically aligned with $\Gamma$; (ii) partially sparse support, where some generated statements are correct while others are inconsistent; and (iii) sparse disagreement, where the generated statements contradict the trusted context.

\begin{table}[t]
\centering
\caption{\footnotesize
Structure of the experimental data. The trusted context $\Gamma$ provides semantically consistent reference information, while the three scenarios introduce increasing levels of inconsistency.}
\small
\begin{tabular}{p{0.95\linewidth}}
\toprule
\textbf{Trusted context $\Gamma$} \\
\midrule
Paris is the capital of France; France's capital city is Paris; The French government is seated in Paris; Paris is the political center of France; France is a country in Western Europe; France belongs to the European Union; Paris is known for the Eiffel Tower; The Louvre Museum is located in Paris; The Seine River flows through Paris; Paris is famous for museums, monuments, and architecture. \\
\midrule
\textbf{Dense agreement (low hallucination)} \\
Paris is the capital of France; The French government is based in Paris; Paris contains the Eiffel Tower and the Louvre. \\
\midrule
\textbf{Partially sparse (mixed support)} \\
Paris is the capital of France; France is located in Europe; Madrid is the capital of France. \\
\midrule
\textbf{Sparse disagreement (high hallucination)} \\
Berlin is the capital of France; The Eiffel Tower is located in Rome; France is a country in South America. \\
\bottomrule
\end{tabular}
\label{tab:semantic_density_data}
\end{table}

For each scenario, we assign elements of $\Gamma$ to the generated $\ell$-gram exemplars and compute three quantities: semantic entropy, assignment confidence, and best cosine similarity. The goal is to separate two effects that are often conflated in semantic-entropy methods: the number of available contextual statements and the density of their semantic support. A larger context $\Gamma$ is useful only when its elements form coherent neighborhoods around the generated statements. If additional context is relevant and mutually consistent, it strengthens the assignment signal and stabilizes the entropy estimate. However, if $\Gamma$ is enlarged with weakly related, noisy, or contradictory statements, the induced partition becomes more diffuse and the entropy score may increase.

Dense semantic agreement should therefore induce concentrated assignments around a small number of well-supported exemplars, producing lower entropy. Conversely, sparse or contradictory responses distribute context elements more diffusely across weak representatives, increasing semantic entropy and signaling higher hallucination risk. In this sense, the sensitivity of the method is not simply to the cardinality $|\Gamma|$, but to the semantic density of $\Gamma$ around the generated $\ell$-grams.

\begin{figure}[ht]
    \centering
    \begin{subfigure}{0.32\linewidth}
        \centering
        \includegraphics[width=\linewidth]{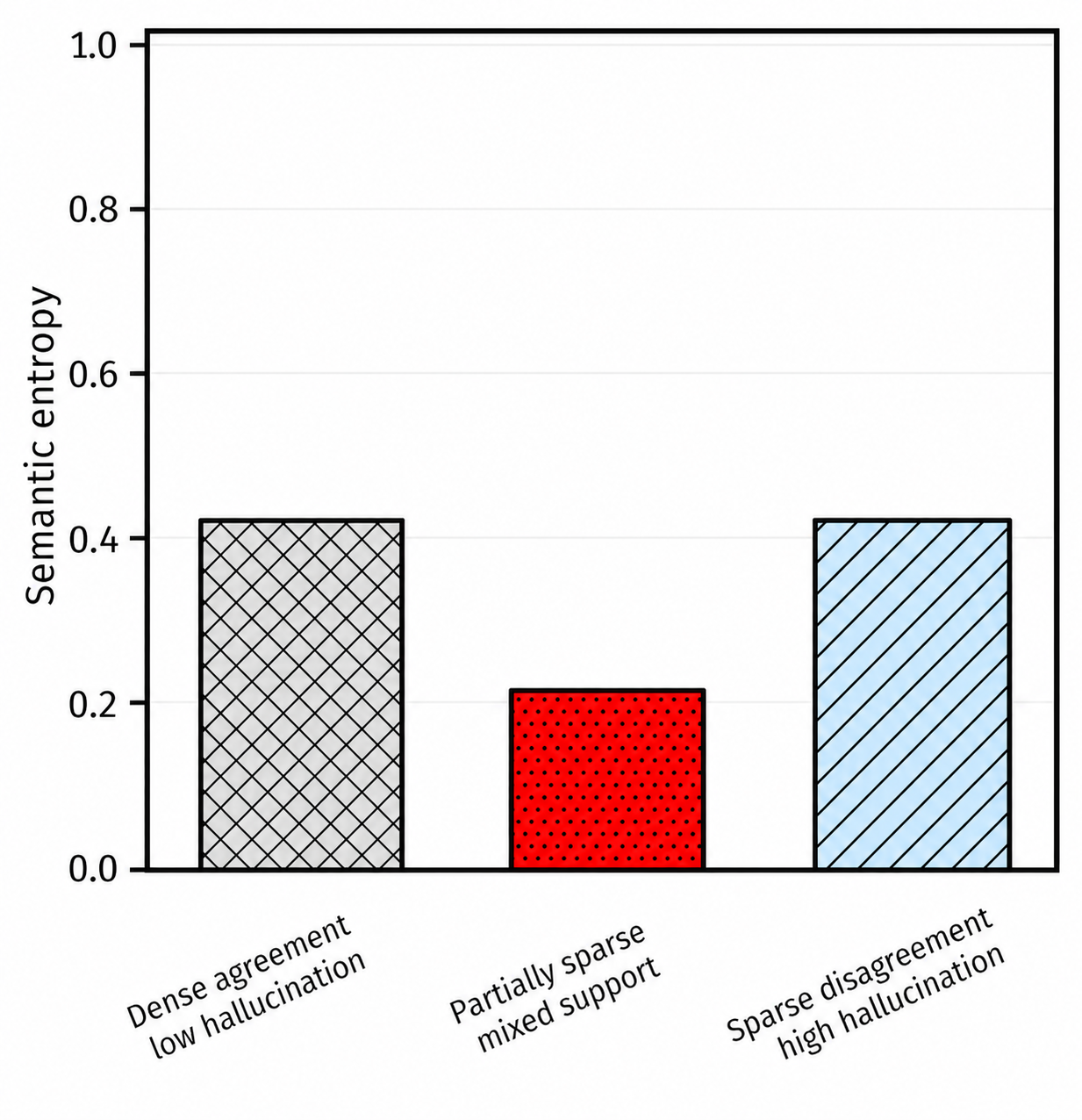}
        \caption{Semantic entropy.}
        \label{fig:density_entropy}
    \end{subfigure}
    \hfill
    \begin{subfigure}{0.32\linewidth}
        \centering
        \includegraphics[width=\linewidth]{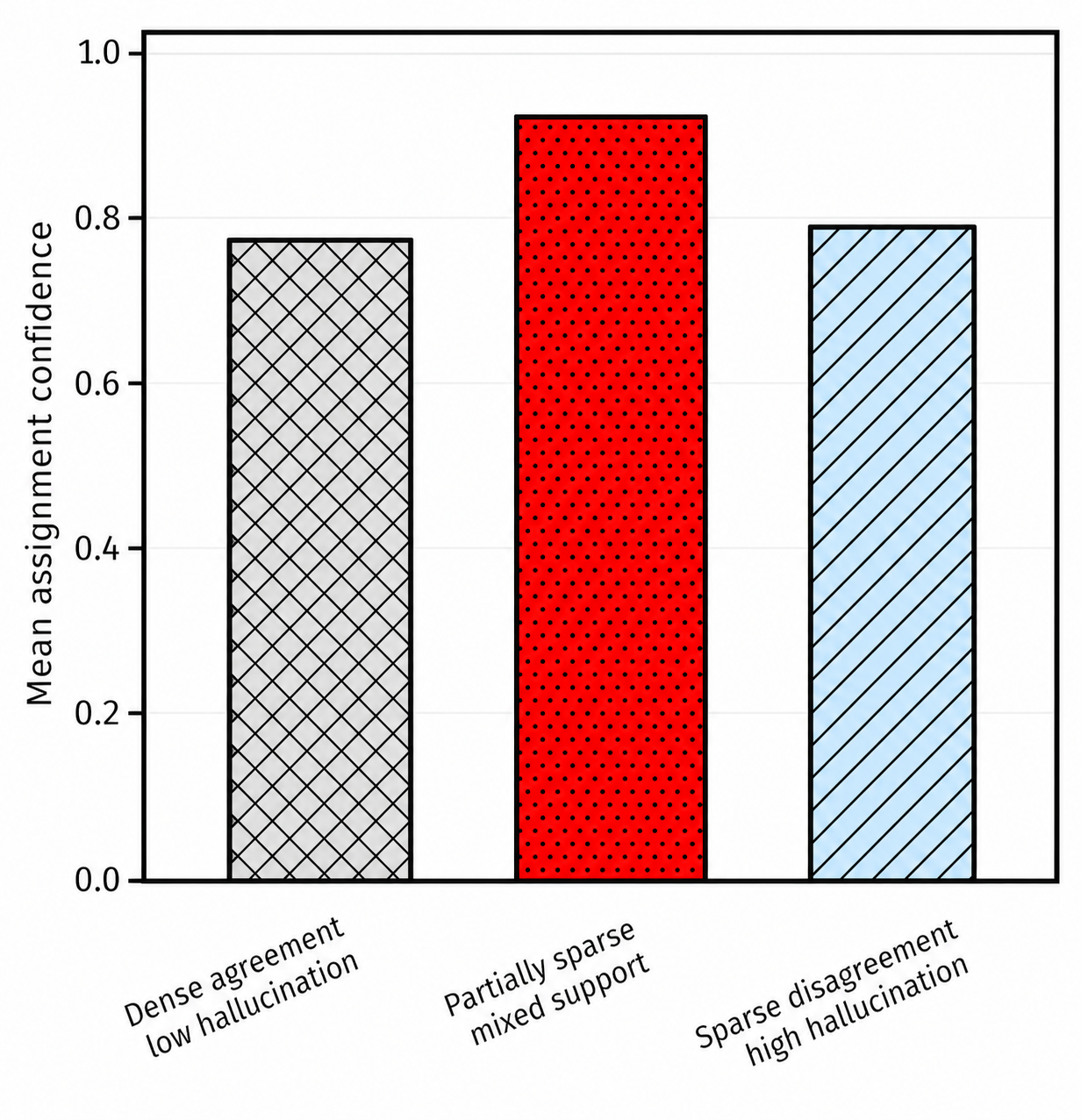}
        \caption{Assignment confidence.}
        \label{fig:density_confidence}
    \end{subfigure}
    \hfill
    \begin{subfigure}{0.32\linewidth}
        \centering
        \includegraphics[width=\linewidth]{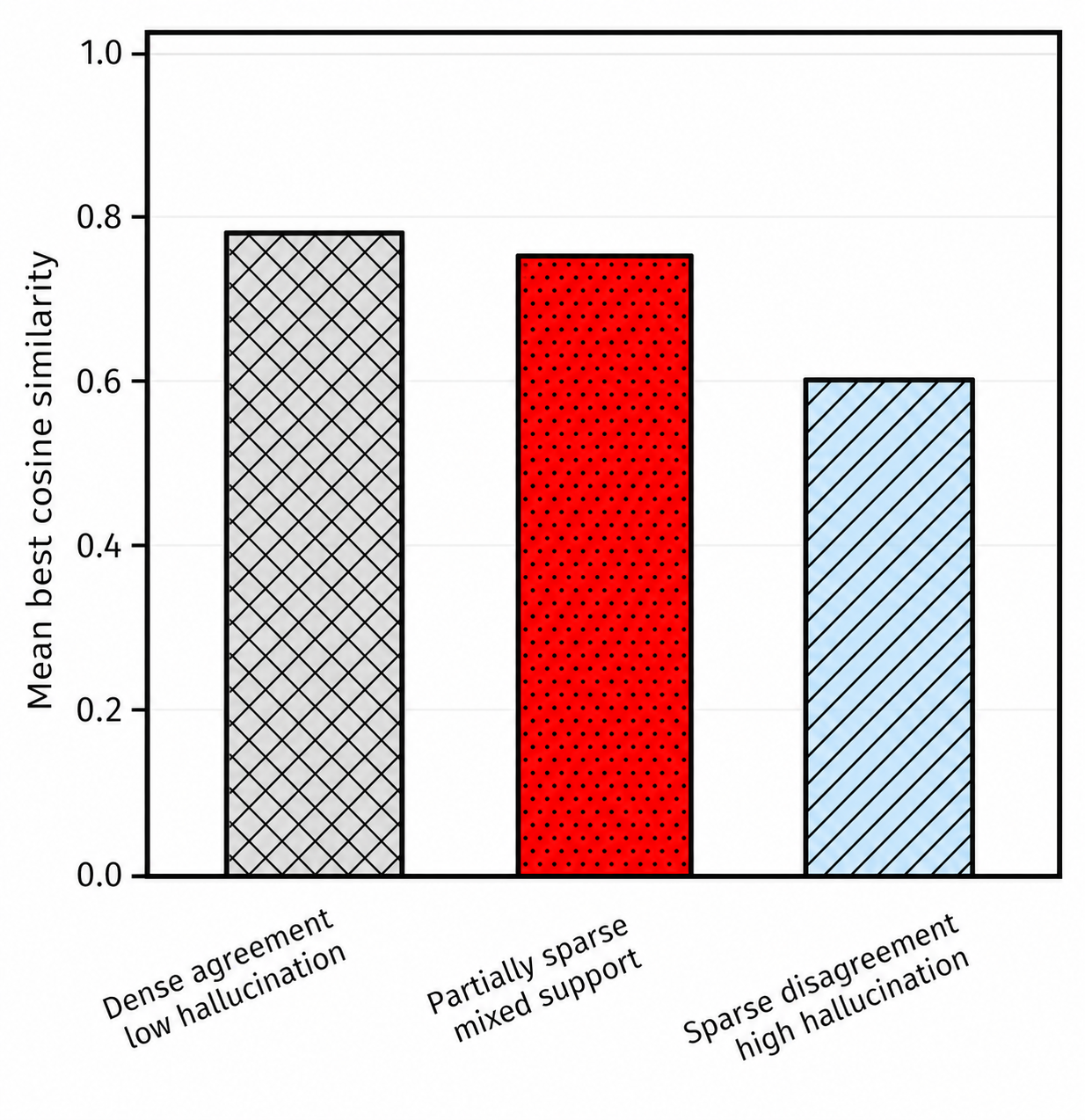}
        \caption{Best cosine similarity.}
        \label{fig:density_similarity}
    \end{subfigure}
    \caption{
    Density-based semantic entropy experiment. Dense agreement with $\Gamma$ yields lower entropy and strong semantic similarity, while sparse disagreement increases entropy and reduces alignment with the trusted context.
    }
    \label{fig:density_barplots}
\end{figure}

Figure~\ref{fig:density_barplots} supports this interpretation. In the dense-agreement regime, the generated $\ell$-grams explain the trusted context well: semantic entropy is low, assignment confidence is high, and the best cosine similarity remains strong. This indicates that the elements of $\Gamma$ are not merely numerous, but concentrated around the generated semantic units. In the partially sparse regime, the presence of both supported and unsupported claims creates a less coherent assignment structure. Some elements of $\Gamma$ remain close to valid exemplars, while others are pulled toward weaker or contradictory representatives. Finally, in the sparse-disagreement regime, the generated exemplars poorly explain the trusted context; assignments become less reliable and semantic entropy increases. Thus, the metric reacts to the quality of semantic coverage rather than to the size of $\Gamma$ alone.

This behavior also clarifies the expected sensitivity to $|\Gamma|$. If $\Gamma$ is too small, the entropy estimate may be unstable because a few statements can dominate the induced partition. In contrast, increasing $|\Gamma|$ with additional relevant and consistent statements should reduce this variance by providing a denser estimate of the trusted semantic support. However, increasing $|\Gamma|$ indiscriminately can have the opposite effect: irrelevant or weakly related statements create diffuse assignments and may artificially increase entropy. Therefore, the practical construction of $\Gamma$ should balance coverage and specificity, for example through top-$k$ retrieval, source filtering, or curation of trusted documents.

\begin{figure}[ht]
    \centering
    \includegraphics[width=0.5\linewidth]{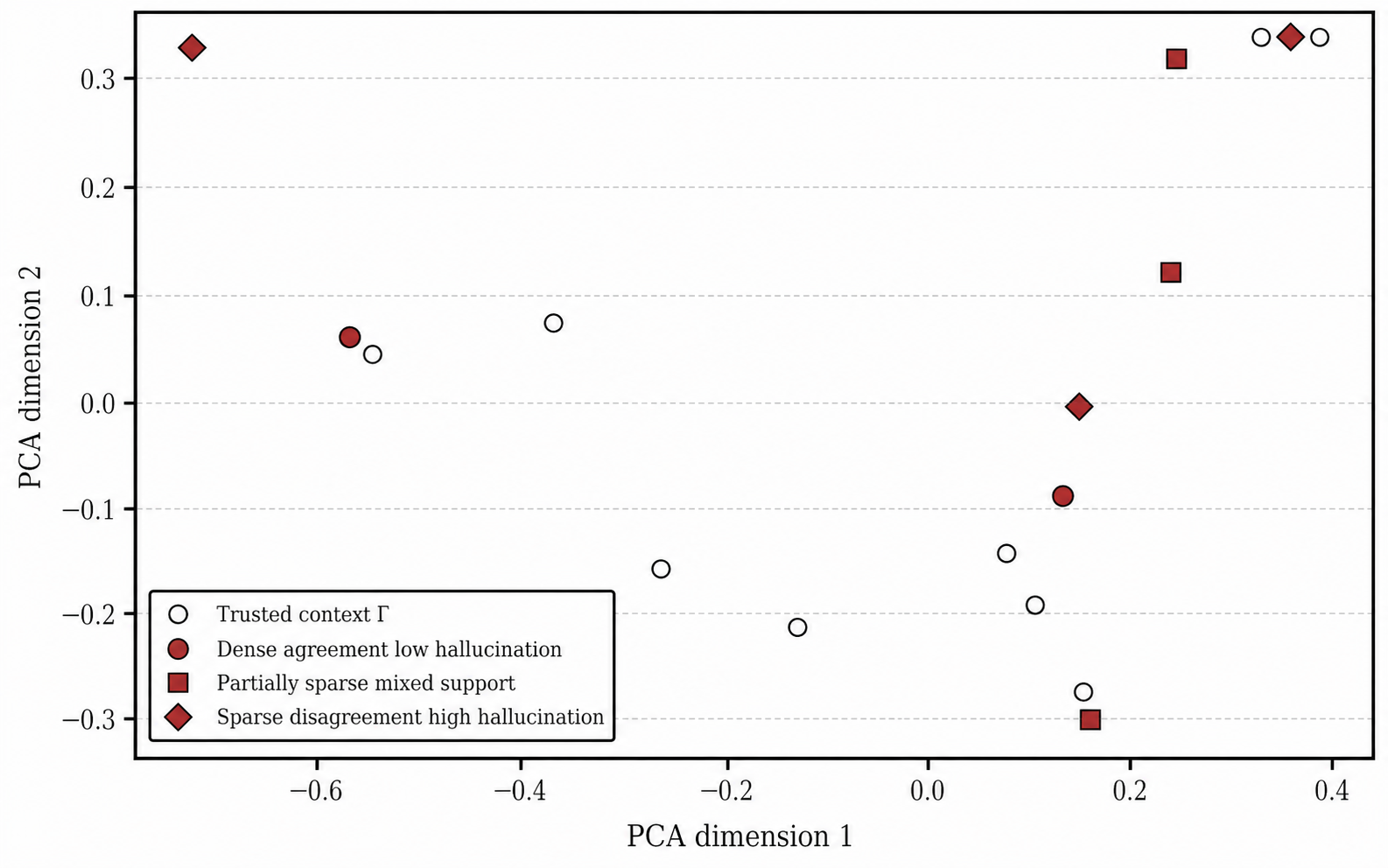}
    \caption{
    PCA visualization of the trusted context $\Gamma$ and the generated $\ell$-gram exemplars under the three regimes. Dense agreement places generated units near supported regions of $\Gamma$, whereas sparse disagreement moves exemplars away from the trusted semantic support, producing a more diffuse partition.
    }
    \label{fig:density_pca}
\end{figure}

Figure~\ref{fig:density_pca} provides a geometric view of the same phenomenon. In the dense-agreement case, the generated exemplars lie close to the regions occupied by the trusted context, so adding more consistent elements to $\Gamma$ would be expected to reinforce the same semantic neighborhood. In the sparse-disagreement case, the exemplars move away from the support of $\Gamma$, so additional trusted statements would not necessarily reduce entropy unless they are relevant to the generated claims. This shows that $\Gamma$ acts as a semantic reference distribution: its usefulness depends on whether it provides dense support near the correct response manifold.

The results show that semantic entropy formulated under \eqref{eqn:cluster_prob} is not merely a function of the number of clusters or the number of context statements, but of how strongly the generated exemplars explain the trusted context. When the generated response is semantically supported by a coherent $\Gamma$, the induced partition is concentrated and the entropy remains low. When the response contains unsupported or contradictory claims, context assignments become less coherent, increasing entropy. This supports the use of exemplar-based semantic entropy as a hallucination signal in \texttt{CAROL}, while also emphasizing that the construction of $\Gamma$ should prioritize reliable, relevant, and semantically coherent support rather than unfiltered context expansion.

\FloatBarrier

\newpage
\section{Axiomatic Context Construction}
\label{app:axioms}

In this section we clarify the nature and construction of the context $\Gamma$ introduced in Assumption~\ref{assump:context}.

\paragraph{Definition and role.}
The context $\Gamma$ is defined as a collection of trusted statements, modeled as a set of axioms $\mathcal{A}$, which serve as the reference for semantic evaluation throughout the generation process. Each element $x \in \Gamma$ represents a unit of information that is assumed to be factually correct within the scope of the task.

Importantly, $\Gamma$ does not aim to be exhaustive; rather, it provides a \emph{reliable grounding substrate} against which generated content is evaluated via entailment (Section~2).

\paragraph{Relation to retrieval-augmented generation (\texttt{RAG}).}
From a systems perspective, the construction of $\Gamma$ does not differ fundamentally from standard retrieval-augmented generation pipelines. In both cases, a query $\textsf{q}$ is used to retrieve relevant documents or pieces of information from external sources, forming an augmented prompt, see Figure~\ref{fig:agentic_pipeline}.

The distinction lies not in \emph{how} the context is retrieved, but in \emph{how it is interpreted}. Classical \texttt{RAG} treats retrieved context as auxiliary information to improve generation quality, without enforcing consistency between the output and the retrieved content. In contrast, our framework treats $\Gamma$ as an \emph{axiomatic reference}: generated statements are explicitly evaluated for semantic consistency with respect to $\Gamma$, and may be accepted or rejected accordingly.

\paragraph{Trusted knowledge requirement.}
A key requirement is that the elements of $\Gamma$ are \emph{trusted} within the application domain. This does not impose additional algorithmic complexity, but rather introduces a structural constraint on the source of the information. In practice, this can be implemented by restricting retrieval to curated or authoritative sources.

Examples include:
\begin{itemize}
    \item \textbf{Enterprise setting:} internal documents such as company policies, technical specifications, or verified knowledge bases.
    \item \textbf{Medical queries:} authoritative sources such as \cite{nih_website}, or peer-reviewed clinical guidelines.
    \item \textbf{Scientific research:} curated repositories such as arXiv or domain-specific digital libraries.
\end{itemize}

The \texttt{RAG} literature implicitly assumes that the retrieved context is reliable, although this assumption is rarely stated explicitly. In this work, we emphasize a closer connection with the logic literature, where the correctness of the underlying premises is fundamental. In particular, it is not sufficient for retrieved information to appear coherent or plausible; it must be grounded in verifiable truth. While this distinction is often overlooked in practical implementations, it is essential for ensuring that generated responses remain semantically valid and free from hallucination.

\paragraph{No increase in computational complexity.}
We emphasize that this requirement does not increase the computational complexity of the pipeline. The retrieval mechanism remains unchanged; only the \emph{selection criteria} for admissible sources differs. Therefore, $\Gamma$ can be constructed using standard retrieval systems, with an additional filtering or curation step to ensure reliability.

\paragraph{Interpretation.}
Conceptually, $\Gamma$ can be understood as defining a \emph{truth-constrained semantic space}. The role of the generation process is not merely to produce fluent text, but to construct a sequence $\mathsf{S}$ that remains within the entailment closure of $\Gamma$. In this sense, $\Gamma$ induces the feasible region of the optimization problem in (4).

This interpretation is consistent with the formulation in the main text, where hallucination is characterized as a deviation from the semantic support provided by $\Gamma$.

\paragraph{Beyond axiomatic truth.}
Although we model $\Gamma$ as an axiom set in this work, the same formulation can be extended to broader notions of correctness. In many scientific domains, truth is not specified purely by logical axioms, but by empirical validation, experimental evidence, simulation, or formal verification. Accordingly, $\Gamma$ may be generalized from a set of trusted statements to a task-specific correctness oracle. For example, in theorem proving or mathematical reasoning, $\Gamma$ could represent a proof checker that verifies whether each generated step follows from previous ones; in scientific settings, it could represent empirically validated knowledge or a domain simulator. This extension would not change the overall structure of \texttt{CAROL}: candidate strings would still be accepted or rejected according to their marginal contribution to correctness with respect to $\Gamma$. Formally, only the evaluation rule defining the acceptance probability changes, while the Markov chain interpretation, submodular objective structure, and resulting convergence bounds remain applicable whenever the induced correctness score satisfies the same diminishing-returns assumptions.

\newpage
\section{Lattice Representation}
\label{app:lattice}

Given the objective in~\eqref{eqn:main_problem}, a key aspect of our formulation is the interpretation of the generation process as a Markov chain over a semantic lattice. As discussed in Section~\ref{sec::preliminaries}, a language model $F(\rho,\theta)$ induces a probability distribution over tokens $v \in \mathbb{V}$. Consequently, the space of all possible generations can be represented as a lattice $\mathcal{L}(\mathbb{V})$, whose elements correspond to all finite sequences (or combinations) over the vocabulary.

This structure can be visualized as a Hasse diagram, where each node represents a sequence and edges encode prefix expansions. Under this view, the generation of an $\ell$-gram corresponds to a path through the lattice induced by (greedy) sampling from the model, as illustrated in Figure~\ref{fig::lattice_combined2}.

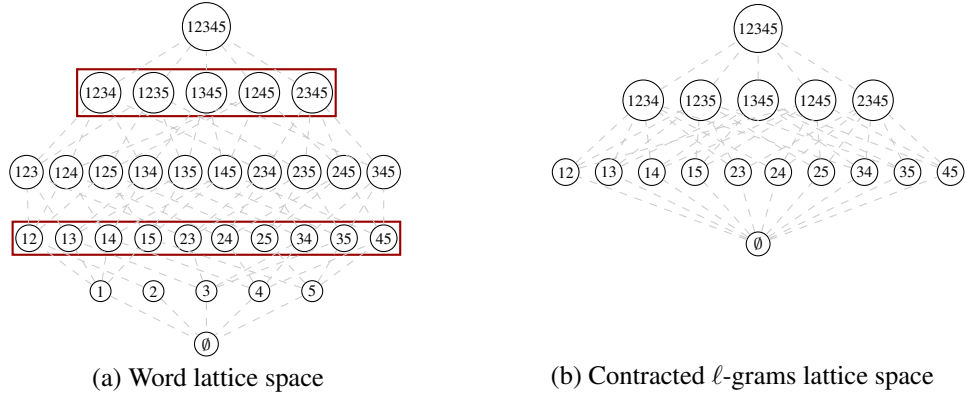
\begin{figure*}[h]
    \centering
    \begin{minipage}{0.45\textwidth}
        \centering
        \begin{tikzpicture}[
            scale=0.35,
            every node/.style={draw, circle, inner sep=1pt, font=\tiny},
            level 1/.style={sibling distance=4cm},
            level 2/.style={sibling distance=1.5cm},
            level 3/.style={sibling distance=0.8cm}
        ]
        \node (12345) at (0,12) {12345};
        \node (1234) at (-4,9.5) {1234};
        \node (1235) at (-2,9.5) {1235};
        \node (1345) at (0,9.5) {1345};
        \node (1245) at (2,9.5) {1245};
        \node (2345) at (4,9.5) {2345};
        \node (123) at (-6.8,6.5) {123};
        \node (124) at (-5.3,6.5) {124};
        \node (125) at (-3.8,6.5) {125};
        \node (134) at (-2.3,6.5) {134};
        \node (135) at (-0.8,6.5) {135};
        \node (145) at (0.7,6.5) {145};
        \node (234) at (2.2,6.5) {234};
        \node (235) at (3.7,6.5) {235};
        \node (245) at (5.2,6.5) {245};
        \node (345) at (6.7,6.5) {345};
        \node (12) at (-6.7,4) {12};
        \node (13) at (-5.2,4) {13};
        \node (14) at (-3.7,4) {14};
        \node (15) at (-2.2,4) {15};
        \node (23) at (-0.7,4) {23};
        \node (24) at (0.7,4) {24};
        \node (25) at (2.2,4) {25};
        \node (34) at (3.7,4) {34};
        \node (35) at (5.2,4) {35};
        \node (45) at (6.7,4) {45};
        \node (1) at (-4,2) {1};
        \node (2) at (-2,2) {2};
        \node (3) at (0,2) {3};
        \node (4) at (2,2) {4};
        \node (5) at (4,2) {5};
        \node (0) at (0,0) {$\emptyset$};
        \draw[black!20!white,dashed] (1) -- (0)
            (2) -- (0)
            (3) -- (0)
            (4) -- (0)
            (5) -- (0);
        \draw[black!20!white,dashed] (1) -- (12)
            (1) -- (13)
            (1) -- (14)
            (1) -- (15)
            (2) -- (12);
        \draw[black!20!white,dashed] (3) -- (13)
            (3) -- (23)
            (3) -- (34)
            (3) -- (35)
            (4) -- (14)
            (4) -- (24)
            (4) -- (34)
            (4) -- (45)
            (5) -- (15)
            (5) -- (25)
            (5) -- (35)
            (5) -- (45)
            (12) -- (123)
            (12) -- (124)
            (12) -- (125);
        \draw[black!20!white,dashed] (13) -- (123)
            (13) -- (134)
            (13) -- (135)
            (14) -- (124)
            (14) -- (134)
            (14) -- (145)
            (23) -- (234)
            (23) -- (123)
            (23) -- (235)
            (24) -- (234)
            (24) -- (124)
            (24) -- (245)
            (34) -- (134)
            (34) -- (234)
            (34) -- (345)
            (15) -- (125)
            (15) -- (145)
            (15) -- (135)
            (25) -- (245)
            (25) -- (235)
            (25) -- (125)
            (35) -- (135)
            (35) -- (235)
            (35) -- (345)
            (45) -- (145)
            (45) -- (245)
            (45) -- (345);
        \draw[black!20!white,dashed] (123) -- (1234)
            (123) -- (1235)
            (125) -- (1235)
            (125) -- (1245);
        \draw[black!20!white,dashed] (134) -- (1234)
            (134) -- (1345)
            (234) -- (1234)
            (234) -- (2345)
            (124) -- (1234)
            (124) -- (1245)
            (135) -- (1345)
            (135) -- (1235)
            (145) -- (1245)
            (145) -- (1345)
            (235) -- (1235)
            (235) -- (2345)
            (245) -- (1245)
            (245) -- (2345)
            (345) -- (2345)
            (345) -- (1345);
        \draw[black!20!white,dashed] (1234) -- (12345)
            (1235) -- (12345)
            (1345) -- (12345)
            (1245) -- (12345)
            (2345) -- (12345);
        \begin{scope}[on background layer]
            \node[draw=red!60!black, thick, rectangle, 
                  inner sep=1pt,
                  fit=(12)(13)(14)(15)(23)(24)(25)(34)(35)(45)] {};
        \end{scope}
        \begin{scope}[on background layer]
            \node[draw=red!60!black, thick, rectangle, 
                  inner sep=1pt,
                  fit=(1234)(1234)(1345)(1245)(2345)] {};
        \end{scope}
        \end{tikzpicture}
        \label{fig:lattice}
        
        \text{(a) Word lattice space}
    \end{minipage}
    \hfill
    \begin{minipage}{0.45\textwidth}
    \begin{tikzpicture}[
            scale=0.38,
            every node/.style={draw, circle, inner sep=1pt, font=\tiny},
            edgebg/.style={black!40!white, dashed}
        ]
            \node (12345) at (0,7.5) {12345};
            \node (1234) at (-4,5) {1234};
            \node (1235) at (-2,5) {1235};
            \node (1345) at (0,5) {1345};
            \node (1245) at (2,5) {1245};
            \node (2345) at (4,5) {2345};
            \node (12) at (-6.7,2.5) {12};
            \node (13) at (-5.2,2.5) {13};
            \node (14) at (-3.7,2.5) {14};
            \node (15) at (-2.2,2.5) {15};
            \node (23) at (-0.7,2.5) {23};
            \node (24) at (0.7,2.5) {24};
            \node (25) at (2.2,2.5) {25};
            \node (34) at (3.7,2.5) {34};
            \node (35) at (5.2,2.5) {35};
            \node (45) at (6.7,2.5) {45};
            \node (0) at (0,0) {$\emptyset$};
            \draw[black!20!white,dashed] (0) -- (12)
            (0) -- (13)
            (0) -- (14)
            (0) -- (15)
            (0) -- (23)
            (0) -- (24)
            (0) -- (25)
            (0) -- (34)
            (0) -- (35)
            (0) -- (45);
            \draw[black!20!white,dashed] (12) -- (1234)
            (12) -- (1235)
            (12) -- (1245)
            (13) -- (1235)
            (13) -- (1234)
            (13) -- (1345)
            (14) -- (1234)
            (14) -- (1345)
            (14) -- (1245)
            (15) -- (1235)
            (15) -- (1345)
            (15) -- (1245)
            (23) -- (1234)
            (23) -- (1235)
            (23) -- (2345)
            (24) -- (1234)
            (24) -- (1245)
            (24) -- (2345)
            (25) -- (1235)
            (25) -- (1245)
            (25) -- (2345)
            (34) -- (1234)
            (34) -- (1345)
            (34) -- (2345)
            (35) -- (1235)
            (35) -- (1345)
            (35) -- (2345)
            (45) -- (1345)
            (45) -- (2345)
            (45) -- (1245);
            \draw[black!20!white,dashed] (1234) -- (12345)
            (1235) -- (12345)
            (1345) -- (12345)
            (1245) -- (12345)
            (2345) -- (12345);
        \end{tikzpicture}
        \vspace{37pt}\\
        \text{(b) Contracted $\ell$-grams lattice space}
        \label{fig::lattice_combined}
    \end{minipage}
    \caption{Hasse diagram for a ground set of five words $\mathbb{V}=\{in, the, mat, cat, is\}$, represented as $\{1,2,3,4,5\}$. In a), all the possible combinations of words are represented in an ordered lattice space where each label comprises combinations of same cardinality. In b) a contracted lattice space is represented, formed by taking $\ell$-grams of cardinality $2$ and $4$ of a). Abstractly, b) represents the combinations of \emph{meaningful} $\ell$-grams from a). For example, $'in~mat~the'$ is an $\ell$-gram with no meaning, so it will be included in a) but not in b).}
    \label{fig::lattice_combined2}
\end{figure*}

It is important to note that the full lattice $\mathcal{L}(\mathbb{V})$ is combinatorially large and, in practice, computationally intractable to explore exhaustively. However, not all elements of $\mathcal{L}(\mathbb{V})$ correspond to semantically meaningful sequences. In particular, we are only interested in $\ell$-grams that form coherent units of meaning (e.g., \emph{``the cat is''}), while sequences such as \emph{``the of in''} are disregarded.

This observation motivates the introduction of a \emph{contracted lattice} $\hat{\mathcal{L}}(\mathbb{V})$, which restricts the search space to semantically valid $\ell$-grams (Figure~\ref{fig::lattice_combined2}). In this reduced space, the generation process can be interpreted as navigating only through meaningful regions of the lattice, effectively pruning infeasible or non-coherent sequences.

From a functional perspective, this restriction induces a \emph{removable discontinuity} in the mapping $F(\rho,\theta)$: the function is well-defined only on the subset of semantically valid sequences, while remaining undefined elsewhere. This reflects the fact that the model implicitly assigns negligible or irrelevant probability mass to non-meaningful combinations.

In practice, Algorithm~\ref{alg:algorithm} implements this restriction explicitly, as lines $5$--$8$ enforce the construction of semantically meaningful $\ell$-grams before evaluation.

\begin{defn}
A \emph{removable discontinuous string function} $f$ is a function defined over $\mathbb{V}^*$ whose value is undefined for a subset of sequences $\mathbb{S} \subset \mathbb{V}^*$ corresponding to semantically invalid or infeasible elements.
\label{defn::removable_function}
\end{defn}

\newpage
\section{Fine-tuning Modeling}
\label{app:fine_tuning}

As discussed in Section~\ref{sec::method}, \texttt{CAROL} incorporates an \emph{update} step to guide the model toward improved responses. One possible implementation, not explored experimentally in this work, is to perform model fine-tuning via a reinforcement learning (RL) procedure. In this section, we outline a mathematical formulation of such an update mechanism and its integration within Algorithm~\ref{alg:algorithm}.

Let $\mathcal{X}$ denote the space of token sequences and $\mathcal{D} \subseteq \mathcal{X}$ a dataset of labeled examples. We assume access to a pretrained auto-regressive language model that maps an input sequence $x \in \mathcal{X}$ to a probability distribution over the next token in a vocabulary $\mathbb{V}$ of size $n$. The probability assigned to a token $y \in \mathbb{V}$ is given by the softmax transformation
\begin{equation}
    s_y(F(x;\theta)) = \frac{\exp(F_{y}(x;\theta))}{\sum_{i=1}^{n} \exp(F_i(x;\theta))},
    \label{eqn::softmax}
\end{equation}
where $F(x;\theta) \in \mathbb{R}^n$ denotes the vector of logits. We use $F(x;\theta^\star)$ to denote the ground-truth logits associated with input $x$.

For illustration, consider the prompt $x = \text{``The cat sat on the''}$ and a vocabulary $\mathbb{V} = [\text{mat, table, chair, car}]$. Suppose the model produces logits $F(x;\theta) = [2.0, 1.0, 0.5, -1.0]$. Then, using~\eqref{eqn::softmax}, the probability assigned to \textit{mat} is
\[
s_{\text{mat}}(F(x;\theta)) = \frac{\exp(2.0)}{\exp(2.0) + \exp(1.0) + \exp(0.5) + \exp(-1.0)} \approx 0.61.
\]
We refer to the model input as a \emph{prompt} $\rho$, which may include both query and contextual information.

To model fine-tuning, we adopt a surrogate linear representation $F(x;\theta) = W_\theta \phi(x)$, where $W_\theta \in \mathbb{R}^{n \times d}$ are learnable parameters and $\phi(x) \in \mathbb{R}^d$ is a feature embedding. This surrogate captures the latent structure induced by the pretrained model and provides a tractable approximation for optimization. Such representations are commonly used to approximate and update large models.

We define the regularized negative log-likelihood loss
\begin{equation}
\mathcal{L}(\theta; \mathcal{D}) = -\sum_{(x,y)\in\mathcal{D}} \log s_y(F(x;\theta)) 
+ \frac{\lambda}{2} \|W_\theta - W_\theta^{\text{pre}}\|_F^2,
\end{equation}
where $\|\cdot\|_F$ denotes the Frobenius norm, $W_\theta^{\text{pre}}$ are the pretrained weights, and $\lambda > 0$ controls the deviation from the original model.

The fine-tuned parameters are obtained as
\begin{equation}
W_\theta^\star = \arg\min_{W_\theta} \mathcal{L}(\theta; \mathcal{D}).
\end{equation}

\paragraph{Construction of the fine-tuning dataset.}
A key practical aspect is the construction of the dataset $\mathcal{D}$ used for fine-tuning. Within the \texttt{CAROL} framework, a natural choice is to collect tuples of the form $(\rho, \mathbb{S}^\star)$, where $\rho$ denotes the prompt (including query and context $\Gamma$) and $\mathbb{S}^\star$ corresponds to a semantically consistent response obtained through the acceptance--rejection process. In this sense, $\mathcal{D}$ can be built online by logging accepted generations, optionally paired with rejected candidates to provide negative feedback signals.

Importantly, the required dataset size is significantly smaller than in standard pretraining or supervised fine-tuning settings. Since \texttt{CAROL} operates at test time and enforces semantic consistency through its objective, the role of fine-tuning is primarily to \emph{bias} the proposal distribution toward more informative regions of the lattice, rather than to learn the task from scratch. Consequently, even a limited number of high-quality prompt–response pairs can be sufficient to improve performance, particularly when combined with regularization toward the pretrained weights. We leave a systematic study of data requirements and online dataset construction strategies for future work.

\newpage
\section{Theoretical Analysis}
\label{app:analysis}

\subsection{Greedoid structure of LLMs generation}

\begin{prop}[Greedoid structure of feasible LLM generations]
\label{prop:llm_greedoid}
Let $\mathbb{V}$ denote the set of semantic units (e.g., tokens, phrases, or $\ell$-grams) and let $\mathcal{I}\subseteq \mathbb{V}^*$ denote the family of
partial generations that satisfy a semantic feasibility oracle
$\mathcal{O}:\mathbb{V}^*\rightarrow\{0,1\}$. Assume that $\mathcal{O}$ satisfies:

\begin{enumerate}[label=\textbf{(A\arabic*)}]
    \item \textbf{Empty feasibility:} $\mathcal{O}(\emptyset)=1$.
    \item \textbf{Prefix accessibility:} if $\mathcal{O}(\mathsf{S})=1$ and
    $\mathsf{S}\neq \emptyset$, then deleting the last semantic unit preserves
    feasibility, i.e.,
    $\mathcal{O}(\mathsf{S}_{1:|\mathsf{S}|-1})=1$.
    \item \textbf{Extendability:} if $\mathsf{S}_1,\mathsf{S}_2\in\mathcal{I}$
    and $|\mathsf{S}_1|<|\mathsf{S}_2|$, then there exists a semantic unit
    $s\in \mathsf{S}_2\ominus \mathsf{S}_1$ such that
    $\mathsf{S}_1\oplus s\in\mathcal{I}$.
\end{enumerate}

Then $(\mathbb{V},\mathcal{I})$ forms a greedoid. \boxend
\end{prop}

\begin{proof}
We verify the three greedoid axioms. First, by \textbf{(A1)}, the empty
generation is feasible, hence $\emptyset\in\mathcal{I}$, which proves
non-emptiness.

Second, let $\mathsf{S}\in\mathcal{I}$ be nonempty. Since $\mathsf{S}$ is a
valid partial generation, $\mathcal{O}(\mathsf{S})=1$. By prefix accessibility
\textbf{(A2)}, removing the last generated semantic unit yields the prefix
$\mathsf{S}_{1:|\mathsf{S}|-1}$, which remains feasible. Therefore, there exists
an element $s\in\mathsf{S}$ whose removal preserves feasibility, establishing
accessibility.

Third, let $\mathsf{S}_1,\mathsf{S}_2\in\mathcal{I}$ with
$|\mathsf{S}_1|<|\mathsf{S}_2|$. By the extendability assumption
\textbf{(A3)}, there exists a semantic unit
$s\in\mathsf{S}_2\ominus\mathsf{S}_1$ such that
$\mathsf{S}_1\oplus s\in\mathcal{I}$. This is exactly the exchange property.
Therefore, $\mathcal{I}$ satisfies non-emptiness, accessibility, and exchange,
and hence $(\mathbb{V},\mathcal{I})$ is a greedoid.
\end{proof}

\subsection{Proof of String Submodularity and Monotonicity in \eqref{eqn:main_problem}}

\begin{lem}[Extension from tokens to semantic strings]
\label{lem:token_to_string_extension}
Let $F:\mathbb{V}^*\rightarrow\mathbb{R}_{\ge 0}$ be string-submodular and prefix-monotone with respect to single-token extensions. Let $\mathsf{U}=(u_1,\ldots,u_m)\in\mathbb{V}^*$ be a finite semantic string, e.g., an $\ell$-gram. Then, for any two prefixes $\mathbf{a}\preceq \mathbf{b}$,
$$F(\mathbf{a}\oplus \mathsf{U})-F(\mathbf{a}) \ge F(\mathbf{b}\oplus \mathsf{U})-F(\mathbf{b}).$$
Moreover,
$$F(\mathsf{S}\oplus \mathsf{U})\ge F(\mathsf{S})$$
for any $\mathsf{S}\in\mathbb{V}^*$. \boxend
\end{lem}

\begin{proof}
Write the semantic string $\mathsf{U}$ as an ordered sequence of atomic units
$\mathsf{U}=(u_1,\ldots,u_m)$. Define the intermediate prefixes
$\mathbf{a}_j := \mathbf{a}\oplus(u_1,\ldots,u_j),
~\mathbf{b}_j := \mathbf{b}\oplus(u_1,\ldots,u_j),$
with $\mathbf{a}_0=\mathbf{a}$ and $\mathbf{b}_0=\mathbf{b}$. Since
$\mathbf{a}\preceq \mathbf{b}$, we also have $\mathbf{a}_j\preceq \mathbf{b}_j$
for every $j=0,\ldots,m-1$.

Using a telescoping decomposition, 
$$F(\mathbf{a}\oplus\mathsf{U})-F(\mathbf{a}) =\sum_{j=1}^{m} \Big(F(\mathbf{a}_{j-1}\oplus u_j)-F(\mathbf{a}_{j-1}) \Big).$$

By string submodularity for single-token extensions and
$\mathbf{a}_{j-1}\preceq \mathbf{b}_{j-1}$,
$$F(\mathbf{a}_{j-1}\oplus u_j)-F(\mathbf{a}_{j-1}) \ge F(\mathbf{b}_{j-1}\oplus u_j)-F(\mathbf{b}_{j-1}).$$
Summing over $j=\{1,\ldots,m\}$ gives
$$F(\mathbf{a}\oplus \mathsf{U})-F(\mathbf{a}) \ge F(\mathbf{b}\oplus \mathsf{U})-F(\mathbf{b}),$$
which proves diminishing returns for appending an entire semantic string.

Finally, prefix-monotonicity for single-token extensions implies
$$F(\mathsf{S}\oplus u_1)\ge F(\mathsf{S}),\quad
F(\mathsf{S}\oplus u_1\oplus u_2)\ge F(\mathsf{S}\oplus u_1),$$
and so on. Chaining these inequalities yields,
$$F(\mathsf{S}\oplus \mathsf{U})\ge F(\mathsf{S}),$$
so monotonicity also extends from token additions to semantic-string additions.
\end{proof}

\begin{proof}
    This proof is based on \cite[Theorem 8.11]{AK-JH:25}. We first define the marginal gain of appending token $v$ to a string $\mathsf{S}$ as
    $$\Delta_F(v \mid \mathsf{S}) := F(\mathsf{S} \oplus v) - F(\mathsf{S}).$$
    Using the chain rule for mutual information,
    $$I(\mathsf{S} \oplus v; \Gamma)=I(\mathsf{S}; \Gamma) + I(v; \Gamma \mid \mathsf{S}),$$
    hence
    $$\Delta_F(v \mid \mathsf{S}) = I(v; \Gamma \mid \mathsf{S}).$$
    Expanding conditional mutual information with semantic entropy gives
    $$\Delta_F(v \mid \mathsf{S}) = \textsf{SE}(v \mid \mathsf{S}) - \textsf{SE}(v \mid \Gamma, \mathsf{S}).$$
    By the conditional-noise assumption,
    $$\textsf{SE}(v \mid \Gamma, \mathsf{S}) = \textsf{SE}(v \mid \Gamma),$$
    so
    $$\Delta_F(v \mid \mathsf{S}) = \textsf{SE}(v \mid \mathsf{S}) - \textsf{SE}(v \mid \Gamma).$$
    Now let $\mathbf{a} \preceq \mathbf{b}$. Since $\mathbf{b}$ contains all information in $\mathbf{a}$ plus possibly more, conditioning on $\mathbf{b}$ cannot increase entropy. Therefore, by the ``information never hurts'' principle,
    $$\textsf{SE}(v \mid \mathbf{a}) \ge \textsf{SE}(v \mid \mathbf{b}).$$
    Subtracting the same quantity \(H(v \mid \Gamma)\) from both sides yields
    $$\Delta_F(v \mid \mathbf{a}) = \textsf{SE}(v \mid \mathbf{a}) - \textsf{SE}(v \mid \Gamma) \ge \textsf{SE}(v \mid \mathbf{b}) - \textsf{SE}(v \mid \Gamma) = \Delta_F(v \mid \mathbf{b}).$$
    Thus $F$ satisfies diminishing returns along prefixes, so it is string-submodular.

    To show monotonicity, observe that for any string $\mathsf{S}$ and token $v$,
    $$F(\mathsf{S} \oplus v) - F(\mathsf{S}) = I(v; \Gamma \mid \mathsf{S}) \ge 0,$$
    since conditional mutual information is always nonnegative. Hence
    $$F(\mathsf{S}) \le F(\mathsf{S} \oplus v),$$
    and by repeated concatenation, $F$ is monotone over $\mathbb{V}^*$. Then, by Lemma~\ref{lem:token_to_string_extension}, the proof is concluded.
\end{proof}

\subsection{Proof of Proposition~\ref{prop:greedy+clustering}}

\begin{proof}
    This proof is based on \cite[Definition 5]{HP-AS-SB-KM:21}. For each context element $x \in \Gamma$, define the point-wise utility
    \begin{equation}
    f_x(\mathbb{S}):=d(x,e_0) - \min_{\mathsf{S} \in \mathbb{S} \cup \{e_0\}} d(x,\mathsf{S}).
    \label{eq:pointwise_utility}
    \end{equation}
    Then
    \begin{equation}
    f(\mathbb{S}) =\frac{1}{|\Gamma|}\sum_{x \in \Gamma} f_x(\mathbb{S}).
    \label{eq:sum_pointwise}
    \end{equation}
Hence it is enough to prove that each $f_x$ is normalized, monotone, and
submodular, since nonnegative linear combinations preserve these properties.

We first rewrite $f_x$. For any $x \in \Gamma$,
$$\min_{\mathsf{S} \in \mathbb{S} \cup \{e_0\}} d(x,\mathsf{S}) = d(x,e_0) - \max_{\mathsf{S} \in \mathbb{S}} \bigl(d(x,e_0)-d(x,\mathsf{S})\bigr)_+,$$
where $(a)_+ := \max\{a,0\}$. Therefore,
\begin{equation}
f_x(\mathbb{S}) = \max_{\mathsf{S} \in \mathbb{S}} w_x(\mathsf{S}), \qquad w_x(\mathsf{S}):=\bigl(d(x,e_0)-d(x,\mathsf{S})\bigr)_+.
\label{eq:max_form}
\end{equation}

\paragraph{Normalization.}
For the empty set, no exemplar is selected, so
$$f_x(\emptyset)=d(x,e_0)-d(x,e_0)=0.$$
Thus $f_x(\emptyset)=0$, and by \eqref{eq:sum_pointwise}, $f(\emptyset)=0$.

\paragraph{Monotonicity.}
If $A \subseteq B$, then the maximization in \eqref{eq:max_form} is taken over a
larger set for $B$, so
$$f_x(A)=\max_{\mathsf{S}\in A} w_x(\mathsf{S}) \le \max_{\mathsf{S}\in B} w_x(\mathsf{S})=f_x(B).$$
Hence each $f_x$ is monotone, and therefore so is $f$.

\paragraph{Submodularity.}
Let $A \subseteq B \subseteq \mathcal{U}$ and let $u \in \mathcal{U}\ominus B$.
We must show
$$f_x(A \cup \{u\}) - f_x(A) \;\ge\; f_x(B \cup \{u\}) - f_x(B).$$
Write, $\qquad a := \max_{\mathsf{S}\in A} w_x(\mathsf{S}), \qquad b := \max_{\mathsf{S}\in B} w_x(\mathsf{S}), \qquad c := w_x(u).$

Since $A \subseteq B$, we have $a \le b$. Using \eqref{eq:max_form},
$$f_x(A \cup \{u\}) - f_x(A) = \max\{a,c\} - a,$$
and
$$f_x(B \cup \{u\}) - f_x(B) = \max\{b,c\} - b.$$
Now consider two cases.

\emph{Case 1:} $c \le b$. Then
$$\max\{b,c\} - b = 0,$$
while
$$\max\{a,c\} - a \ge 0.$$
Hence
$$\max\{a,c\} - a \ge \max\{b,c\} - b.$$

\emph{Case 2:} $c > b$. Since $a \le b < c$, we have
$$\max\{a,c\} - a = c-a, \qquad \max\{b,c\} - b = c-b.$$
Because $a \le b$, it follows that
$$c-a \ge c-b.$$
Thus again
$$f_x(A \cup \{u\}) - f_x(A)\;\ge\; f_x(B \cup \{u\}) - f_x(B).$$

Therefore $f_x$ is submodular. Since $f$ is an average of the functions $f_x$, it is also submodular.

Having established that $f$ is normalized, monotone, and submodular, the
classical \cite{GN-LW-MF:78} result for greedy maximization under the cardinality
constraint $|\mathsf{S}|\le \ell$ yields
$$f(\mathbb{S}) \ge (1-e^{-1}) f(\mathbb{S}^\star),$$
which proves the claim.
\end{proof}

\subsection{Equivalence of Algorithm~\ref{alg:algorithm} to Gibbs Sampling}

\begin{rem}[Gibbs accept--reject step]
\label{rem:gibbs}
Following~\cite{AG-HH-AK:15}, \texttt{CAROL} induces a distribution over a finite candidate set $\mathbb{V}$ of $\ell$-grams,
\[
p(\mathsf{S}) \propto \exp(\beta F(\mathsf{S})), \quad \mathsf{S}\subseteq\mathbb{V}^*,
\]
with $F(\mathsf{S})=\mathsf{I}(\mathsf{S};\Gamma)$. Given a candidate $s_i\in\mathbb{V}$, the add/remove marginal
\[
\Delta F(s_i\mid \mathsf{S}_t)
=
F(\mathsf{S}_t\cup\{s_i\})-
F(\mathsf{S}_t\ominus\{s_i\})
\]
induces the update
\[
p_{\mathrm{add}}=
\frac{\exp(\beta \Delta F(s_i\mid \mathsf{S}_t))}
{1+\exp(\beta \Delta F(s_i\mid \mathsf{S}_t))},
\]
which corresponds to a Gibbs step for $p(\mathsf{S})$ without requiring explicit normalization.
\end{rem}


\subsection{Proof of Theorem~\ref{thm:fast_mixing}}

\begin{proof}
This proof is developed by \emph{path coupling} technique where, following Definition~\ref{defn::mixing_times}, mixing time is defined such as
\begin{subequations}
 \begin{align*}
     \Delta(t) &= \max_{s} || P_s^t-\pi || = \max_s || P_s^t -\sum_w\pi(w)P_w^t || \leq \max_{s,s^\prime} || P_s^t-P_{s,s^\prime}^t || \leq 2\Delta(t)
 \end{align*}
\end{subequations}
being $\pi$ the stationary distribution and $P_s^t$ the probability distribution at time $t$ starting from $s$. Then, let $J_t$ be a joint distribution over two marginal distributions $X_t\sim P_s^t$ and $Y_t\sim P_{s^\prime}^t$ where $|| P_s^t-P_{s^\prime}^t||\leq Pr[X_t\neq Y_t]$. Therefore, if we can find a transition function $P$ such that both processes join so that $X_t = Y_t$, this bound bound the total mixing time such that
$$\Delta(t)\leq\max_{s,s^\prime} Pr[T_{s,s^\prime}>t]$$
where $T_{s,s^\prime}=\min\{t:X_t=Y_t|X_0=s,Y_0=s^\prime\}$.

\begin{thm}[\cite{RB-MD:97}]
    For any Markov Chain $J_t$, let $(X_t,Y_t)$ be a coupling such that for some $a\geq 0$ and $x,y\in\mathcal{P}$ with $x\sim y$, it holds
    $$\mathbb{E}[d(X_{t+1},Y_{t+1})|X_t=x,Y_t=y]\leq e^{-\alpha} d(x,y)$$
    then
    $$t_{mix}(\epsilon)\leq\frac{1}{\alpha}(\log(\text{diam}(\mathcal{P}))+\log(\frac{1}{\epsilon})).$$
    \label{thm:bubley}
\end{thm}

Considering Theorem~\ref{thm:bubley}, note that given two corresponding sequences $\mathsf{S},\mathsf{R}\subseteq\mathbb{V}^*$, representing two distinct states, such that they differ only by one string $\mathsf{R}=\mathsf{S}\oplus\{r\}$, Algorithm~\ref{alg:algorithm} samples one string $i\in\mathbb{V}^*$, see Appendix~\ref{app:lattice}, following a distribution probability $q$, and then adds or removes the element by a conditional probability $p$. Then the expected distance between the new states $\mathsf{S}^\prime$ and $\mathsf{R}^\prime$ can change in two distinct ways:

\underline{Case 1:} If $i=r$ then the conditionals for both chains are identical, and we can couple both chains with probability
$$p_{add} := \frac{p(\mathsf{S}\oplus\{r\})}{p(\mathsf{S}) + p(\mathsf{S}\oplus\{r\}) }$$
or remove it by $1-p_{add}$.

\underline{Case 2:} If $i\neq r$ we cannot always couple the updates of the chains, because the conditional probabilities of the updates are different. In fact, we are forced
to have different updates (one chain adding i, the other chain removing i)
with probability equal to the difference of the corresponding conditionals,
\begin{subequations}
 \begin{align*}
     p_{dif} (i) &:= |\frac{p(\mathsf{S}\oplus\{i\})}{p(\mathsf{S}\ominus\{i\}) + p(\mathsf{S}\oplus\{i\}) } - \frac{p(\mathsf{R}\oplus\{i\})}{p(\mathsf{R}\ominus\{i\}) + p(\mathsf{R}\oplus\{i\}) } | \\
     &= | \frac{\exp(\beta F(\mathsf{S}\oplus\{i\}))}{\exp(\beta F(\mathsf{S}\ominus\{i\})) + \exp(\beta F(\mathsf{S}\oplus\{i\}))}- \frac{\exp(\beta F(\mathsf{R}\oplus\{i\}))}{\exp(\beta F(\mathsf{R}\ominus\{i\})) + \exp(\beta F(\mathsf{R}\oplus\{i\}))}| \\
     &= | \frac{\exp(\beta \Delta_F(i|\mathsf{S}))}{1 + \exp(\beta \Delta_F(i|\mathsf{S}))} - \frac{\exp(\beta \Delta_F(i|\mathsf{R}))}{1 + \exp(\beta \Delta_F(i|\mathsf{R}))}| \\
     &= | \frac{\exp(\beta \Delta_F(i|\mathcal{S})) - \exp(\beta \Delta_F(i|\mathsf{R}))}{(1 + \exp(\beta \Delta_F(i|\mathsf{S})))(1 + \exp(\beta \Delta_F(i|\mathsf{R})))}|\\
     &\leq | \frac{\exp(\beta \Delta_F(i|\mathsf{S})) - \exp(\beta \Delta_F(i|\mathsf{R}))}{\exp(\beta \Delta_F(i|\mathsf{S})) + \exp(\beta \Delta_F(i|\mathsf{R}))}| \\
     &= |\frac{\exp(\beta \Delta_F(i|\mathsf{S})- \Delta_F(i|\mathsf{R}))-1}{\exp(\beta \Delta_F(i|\mathsf{S}) - \Delta_F(i|\mathsf{R}))+1}| \\
     &=\tanh(\frac{\beta}{2}|\Delta_F(i|\mathsf{S}) - \Delta_F(i|\mathsf{R})|).
 \end{align*}
\end{subequations}

Hence, the expectation can be written as 
\begin{subequations}
 \begin{align*}
     \mathbb{E}[d(\mathsf{S}^\prime,\mathsf{R}^\prime)] &= d(\mathsf{S}^\prime,\mathsf{R}^\prime)\cdot( 1 - q(r)p_{add}(r) - q(r)(1-p_{add}(r)) + \sum_{i\neq r}q(i)p_{dif}(i))\\
     &= d(\mathsf{S}^\prime,\mathsf{R}^\prime)\cdot (1 - q(r) + \sum_{i\neq r}q(i)p_{dif}(i)) \\
     &\leq d(\mathsf{S}^\prime,\mathsf{R}^\prime)\cdot (1 - q_{min} + q_{max}\sum_{i\neq r}p_{dif}(i)) \\
     &\leq d(\mathsf{S}^\prime,\mathsf{R}^\prime)\cdot (1 - q_{min} + q_{max}\cdot(n-1)\cdot\gamma) \\
     &\leq d(\mathsf{S}^\prime,\mathsf{R}^\prime)\cdot e^{-(q_{min}-q_{max}\cdot(n-1)\cdot\gamma)}
 \end{align*}
\end{subequations}
Then,
$$t_{mix}(\epsilon)\leq \frac{1}{q_{min}-q_{max}\cdot\bar{\gamma}}(\log{(n)}+\log{(\frac{1}{\epsilon})})$$
which concludes the proof.
\end{proof}

\subsection{Proof of Non-vacuous Theorem~\ref{thm:fast_mixing}}

\begin{lem}[Positivity of the mixing denominator]
\label{lem:positive_denominator}
Let $q$ be the auxiliary proposal distribution over candidate $\ell$-grams, and let $q_{\min}(\mathsf{S}^{\ell}) := \min_{s\in\mathsf{S}^{\ell}} q(s),~q_{\max}(\mathsf{S}^{\ell}) := \max_{s\in\mathsf{S}^{\ell}} q(s).$ Assume $q_{\min}(\mathsf{S}^{\ell})>0$ and
$$\bar{\gamma}_{\mathsf{I},\beta}<\frac{q_{\min}(\mathsf{S}^{\ell})}{q_{\max}(\mathsf{S}^{\ell})}.$$
Then, $q_{\min}(\mathsf{S}^{\ell})-q_{\max}(\mathsf{S}^{\ell})\bar{\gamma}_{\mathsf{I},\beta}>0.$ \boxend
\end{lem}

\begin{proof}
By assumption,
$$\bar{\gamma}_{\mathsf{I},\beta}<\frac{q_{\min}(\mathsf{S}^{\ell})}{q_{\max}(\mathsf{S}^{\ell})},$$
see Remark~\ref{rem:condition}, multiplying both sides by $q_{\max}(\mathsf{S}^{\ell})>0$ gives $q_{\max}(\mathsf{S}^{\ell})\bar{\gamma}_{\mathsf{I},\beta}
<q_{\min}(\mathsf{S}^{\ell}).$
Rearranging yields $q_{\min}(\mathsf{S}^{\ell})-q_{\max}(\mathsf{S}^{\ell})\bar{\gamma}_{\mathsf{I},\beta}>0,$ as desired.
\end{proof}

\begin{rem}
  \label{rem:condition}
    The condition $\bar{\gamma}_{\mathsf{I},\beta}<\frac{q_{\min}(\mathsf{S}^{\ell})}{q_{\max}(\mathsf{S}^{\ell})}$ is mild in practice. Indeed, all quantities lie in $[0,1]$: the ratio $q_{\min}/q_{\max}\in(0,1]$ captures the imbalance of the proposal distribution, while $\bar{\gamma}_{\mathsf{I},\beta}\in[0,1)$ measures the curvature of the objective through differences of marginal gains across distinct sequences. When the proposal is not overly skewed (i.e., $q_{\min}$ is not negligible compared to $q_{\max}$) and the objective exhibits diminishing returns with moderate curvature, the inequality is naturally satisfied. Intuitively, $\bar{\gamma}_{\mathsf{I},\beta}$ reflects how distinguishable two candidate extensions are in terms of their contribution to $F$: when marginal gains do not vary sharply across sequences, the sampler does not concentrate mass on a few configurations, and the proposal distribution remains sufficiently well-balanced. In the extreme case of uniform proposals, $q_{\min}=q_{\max}$, and the condition reduces to $\bar{\gamma}_{\mathsf{I},\beta}<1$, which holds for submodular objectives.
\end{rem}

\subsection{Effect of uncertainty on $\Gamma$ consistency}

The preceding analysis assumes that the axiom set $\Gamma$ is fixed and reliable. 
We now discuss how uncertainty in $\Gamma$ propagates through the mixing bound. 
Let $\Gamma^\star$ denote the ideal trusted context and let $\widehat{\Gamma}$ denote 
the context available to the algorithm. Suppose that the induced mutual-information 
objective is uniformly perturbed, i.e.,
$$\sup_{\mathsf{S}\in\mathcal{I}}\left|\mathsf{I}(\mathsf{S};\widehat{\Gamma})-\mathsf{I}(\mathsf{S};\Gamma^\star)
\right|\leq \eta_{\Gamma}.$$
Then, for any candidate string $i$ and state $\mathsf{S}$,
$$\left|\widehat{\Delta}_F(i\mid \mathsf{S})-\Delta_F^\star(i\mid \mathsf{S})\right|\leq 2\eta_{\Gamma}.$$

Consequently, for two neighboring states $\mathsf{S}$ and $\mathsf{R}$,
$$\left|\left[\widehat{\Delta}_F(i\mid \mathsf{S})-\widehat{\Delta}_F(i\mid\mathsf{R})
\right]-\left[\Delta_F^\star(i\mid \mathsf{S})-\Delta_F^\star(i\mid \mathsf{R})\right]\right|\leq 4\eta_{\Gamma}.$$

Since the function $x\mapsto \tanh(x)$ is $1$-Lipschitz on $\mathbb{R}_{\geq 0}$, the effective curvature term appearing in the path-coupling argument satisfies
$$\widehat{\gamma}_{\mathsf{I},\beta}\leq\gamma_{\mathsf{I},\beta}^{\star}+2\beta m\eta_{\Gamma},$$
where $m$ denotes the number of candidate $\ell$-grams considered in the summation defining  $\gamma_{\mathsf{I},\beta}$. Therefore, the robust mixing-time bound becomes
$$t_{\mathrm{mix}}(\epsilon)\leq\frac{1}{q_{\min}(\mathsf{S}^{\ell})-q_{\max}(\mathsf{S}^{\ell})
\left(\bar{\gamma}_{\mathsf{I},\beta}^{\star}+2\beta m\eta_{\Gamma}\right)
}\cdot \left(\log n+\log\frac{1}{\epsilon}\right).$$
Thus, uncertainty in $\Gamma$ slows convergence by increasing the effective curvature of the semantic objective. The bound remains informative whenever
$q_{\min}(\mathsf{S}^{\ell})>q_{\max}(\mathsf{S}^{\ell})\left(\bar{\gamma}_{\mathsf{I},\beta}^{\star}+2\beta m\eta_{\Gamma}\right)$. If this condition fails, the path-coupling argument becomes vacuous: the algorithm may still mix empirically, but the present proof no longer certifies contraction.

\subsection{Relation to perplexity and model confidence}

The quantities $q_{\min}(\mathsf{S}^\ell)$ and $q_{\max}(\mathsf{S}^\ell)$ are induced by the proposal distribution $q$, which in practice corresponds to the normalized logits of the language model. Therefore, the gap $q_{\max} - q_{\min}$ reflects how concentrated the model distribution is over candidate $\ell$-grams. This concentration is closely related to the notion of \emph{perplexity}, a standard measure of uncertainty in language models. Formally, given a distribution $q$ over a discrete space $\mathcal{X}$, the perplexity is defined as
\begin{equation}
\mathrm{PPL}(q) := \exp\big(H(q)\big)
= \exp\left(-\sum_{x \in \mathcal{X}} q(x)\log q(x)\right),
\end{equation}
where $H(q)$ denotes the Shannon entropy. Low perplexity corresponds to low entropy and thus a highly concentrated (confident) distribution, whereas high perplexity indicates a diffuse and uncertain distribution.

In our setting, a low-perplexity proposal $q$ implies that probability mass is concentrated on a small subset of $\ell$-grams, yielding a large separation between $q_{\max}$ and $q_{\min}$. Inspecting Theorem~\ref{thm:fast_mixing}, we observe that the mixing time scales inversely with the term $q_{\min}(\mathsf{S}^\ell) - q_{\max}(\mathsf{S}^\ell)\bar{\gamma}_{\mathsf{I},\beta}$. Consequently, when the model is confident (low perplexity), this gap becomes more pronounced, increasing the denominator and leading to faster mixing. Intuitively, a confident proposal distribution guides the Markov chain toward high-value regions of the lattice more consistently, reducing exploratory randomness and accelerating convergence to the stationary distribution. In contrast, high-perplexity regimes (uncertain models) flatten $q$, shrinking the gap between $q_{\min}$ and $q_{\max}$ and thus slowing down the mixing process.

In the extreme case of a uniform distribution (maximum perplexity), $q_{\max}(\mathsf{S}^\ell)\bar{\gamma}_{\mathsf{I},\beta}\approx q_{\min}(\mathsf{S}^\ell)$, and the bound deteriorates, reflecting slow mixing due to lack of directional guidance.

\newpage
\section{Limitations and Broader Impact}
\label{app:limitations_broader_impact}

\paragraph{\textbf{Dependence on Trusted Context $\Gamma$.}}
\texttt{CAROL} relies on the availability of a trusted contextual reference set $\Gamma$. This is both a strength and a limitation. On the one hand, separating generation from semantic evaluation allows the method to avoid treating retrieved context merely as prompt augmentation. Instead, $\Gamma$ acts as an axiomatic reference against which generated $\ell$-grams are accepted or rejected. On the other hand, the guarantees of the framework are only meaningful relative to the reliability of $\Gamma$. If $\Gamma$ is incomplete, biased, outdated, or factually incorrect, then \texttt{CAROL} may enforce consistency with an invalid reference and thereby preserve or even amplify errors.

This limitation is inherent to any grounding-based hallucination mitigation approach: factuality cannot be certified without assumptions on the truthfulness of the grounding source. In the sense of Tarskian semantics, truth is defined relative to a model or interpretation. Therefore, while \texttt{CAROL} can test whether a generated statement is semantically supported by $\Gamma$, it cannot independently certify that $\Gamma$ itself is universally true. Ultimately, trust must be grounded externally, and it is in general impossible for the system to verify it in isolation. For this reason, practical deployments should construct $\Gamma$ from curated, auditable, and domain-authoritative sources, especially in high-stakes applications such as medicine, law, finance, or cyber--physical systems.

\paragraph{\textbf{Clustering as a Computational Bottleneck.}}
A second limitation concerns the clustering step used to compute semantic entropy and marginal semantic gain. Since \texttt{CAROL} evaluates candidate $\ell$-grams by comparing them against the trusted context, the clustering operation may become a computational bottleneck when $\Gamma$ is large. However, this bottleneck is not fundamental. Because $\Gamma$ is fixed for a given query or task instance, its embeddings can be precomputed and cached. Distance computations can be accelerated using approximate nearest-neighbor search, low-dimensional indexing, or relative-distance schemes that compare candidates only against representative key points rather than all elements of $\Gamma$ \cite{PJ-BK-ID-KG:10}. 

In addition, stochastic variants of clustering \cite{DH-SL-TP-AS-KT:18} can estimate the marginal gain using random subsets of $\Gamma$, reducing computational cost while preserving the accept--reject structure in expectation. These approximations provide natural directions for scaling \texttt{CAROL} to large contexts without altering the conceptual framework and are left for future work.

\paragraph{\textbf{Dependence on Sufficient $\ell$-gram Generation.}}
The clustering mechanism requires a sufficient number of generated $\ell$-grams to form meaningful semantic partitions. If the model produces only a very short response, or if the task admits only a single atomic answer, the entropy estimate may become less informative due to the lack of diversity in semantic units. Similarly, highly creative or open-ended tasks may require a richer context $\Gamma$ to distinguish unsupported novelty from valid creative extrapolation.

However, this requirement is not a significant practical limitation in the agentic settings targeted by \texttt{CAROL}. Modern LLM-based systems typically generate intermediate reasoning steps, tool outputs, and candidate responses across multiple iterations. These intermediate outputs naturally provide enough $\ell$-grams for clustering to be meaningful. Thus, while \texttt{CAROL} may be less informative in extremely short single-step settings, it is well aligned with multi-step agentic workflows where sufficient generation is inherently available. A deeper analysis on the  effect of clustering size is left for future work. 

\paragraph{\textbf{Broader Impact and Responsible Use.}}
\texttt{CAROL} is designed to reduce hallucinations and improve interpretability by explicitly conditioning generation on semantic consistency with trusted information. This can enhance reliability in applications where unsupported outputs may lead to downstream harm. However, the same mechanism may also induce a false sense of certainty if consistency with $\Gamma$ is interpreted as absolute truth. 

It is therefore important to emphasize that \texttt{CAROL} enforces grounded consistency, not universal correctness. Responsible deployment requires transparency in how $\Gamma$ is constructed, awareness of its potential incompleteness or bias, and human oversight in high-stakes domains. Specifically, that is why Assumption~\ref{assump:context} is grounded on axioms, which are the basic information unit in any scientific approach. When used appropriately, \texttt{CAROL} provides a principled tool for improving reliability in large language model systems without modifying their internal structure.

\paragraph{\textbf{Integration in verifiers pipelines.}} We acknowledge that recent verifier-based approaches such as \texttt{SelfCheckGPT} \cite{PM-AL-MG:23}, \texttt{RARR} \cite{JR-EK-AV-BK-GZ:25}, and attribution-based metrics provide complementary perspectives. Our focus in this work is on a unified objective-based formulation rather than verifier pipelines. Integrating \texttt{CAROL} with such methods is an important direction for future work. While ablations on $\ell$-gram size, clustering choices, and $\beta$ are valuable, our goal in this work is to establish a principled framework. We leave systematic empirical analysis of these design choices for future work.

In addition, we include comparisons with the recent \texttt{HalluciNot} framework, which evaluates hallucination detection across multiple benchmarks, thereby providing a complementary and up-to-date baseline that partially covers the verifier-based evaluation space within our experiments.

\newpage
\section{Experimental Setting: \texttt{URSA} Framework and Reproducibility}
\label{app:ursa}

All experiments are implemented using \texttt{URSA} (Universal Research and Scientific Agent), an open-source framework for building modular agentic systems. \texttt{URSA} provides a standardized interface for defining, composing, and executing agents, enabling reproducible multi-agent experimentation.

\paragraph{Installation.}
To ensure reproducibility, we rely on a fixed version of \texttt{URSA}. The framework can be installed directly from the official repository:
\begin{verbatim}
git clone https://github.com/lanl/ursa.git
cd ursa
pip install -e .
\end{verbatim}
We recommend using a virtual environment with a fixed Python version (e.g., Python 3.10) and freezing dependencies via \texttt{requirements.txt}.

\paragraph{BaseAgent abstraction.}
\texttt{URSA} defines a core abstraction, \texttt{BaseAgent}, which standardizes agent behavior through a unified interface. Each agent implements:
\begin{itemize}
    \item an internal model (e.g., LLM or tool),
    \item a \texttt{forward()} method that maps inputs to outputs,
    \item optional memory and state management.
\end{itemize}

This abstraction allows all agents in our pipeline to be defined as modular and interchangeable components, facilitating controlled experimentation.

\paragraph{Agent definitions.}
Following the architecture described in Section~\ref{sec::numerical} (see also Fig.~\ref{fig:agentic_pipeline}), we instantiate three agents:

\begin{itemize}
    \item \textbf{Planner Agent:} generates the initial query $q$ and constructs the context $\Gamma$.
    \item \textbf{Researcher Agent:} produces candidate responses (treated as $\ell$-grams in \texttt{CAROL}).
    \item \textbf{Reasoner Agent:} aggregates and refines responses across iterations.
\end{itemize}

Each agent is implemented by subclassing \texttt{BaseAgent}. A minimal example is shown below:

\begin{verbatim}
from ursa.agents.base_agent import BaseAgent

class ResearcherAgent(BaseAgent):
    def forward(self, input_text):
        response = self.model.generate(input_text)
        return response
\end{verbatim}

\paragraph{Integration with \texttt{CAROL}.}
\texttt{URSA} is used as the execution backbone for \texttt{CAROL}. At each iteration:
\begin{enumerate}
    \item The \textbf{Planner} provides $(q, \Gamma)$.
    \item The \textbf{Researcher} generates a candidate string $\mathsf{S}^t$.
    \item \texttt{CAROL} evaluates $\mathsf{S}^t$ via the submodular objective, Algorithm~\ref{alg:algorithm}.
    \item The \textbf{Reasoner} aggregates accepted responses.
\end{enumerate}

Importantly, \texttt{CAROL} treats each Researcher output as a semantic unit (i.e., an $\ell$-gram), aligning with the agent-level abstraction rather than token-level generation.

\paragraph{Determinism and reproducibility.}
To ensure reproducibility, we enforce:
\begin{itemize}
    \item Fixed random seeds for all stochastic components (LLM sampling, \texttt{CAROL} acceptance).
    \item Fixed temperature and decoding parameters.
    \item Deterministic execution order of agents.
    \item Logged intermediate states $(\mathcal{S}_t)$ and acceptance decisions.
\end{itemize}

All experiments are run on a controlled hardware environment (NVIDIA RTX 2080 Ti GPU), matching the setup described in Section~\ref{sec::numerical}.

\paragraph{Why \texttt{URSA}.}
\texttt{URSA} enables:
\begin{itemize}
    \item \textbf{Modularity:} agents can be swapped without affecting the pipeline.
    \item \textbf{Traceability:} each agent’s output is explicitly recorded.
    \item \textbf{Reproducibility:} standardized interfaces reduce implementation variance.
\end{itemize}

This makes it particularly suitable for evaluating agentic algorithms such as \texttt{CAROL}, where the interaction between components is central to performance.

\begin{figure*}[t]
\centering
\begin{tikzpicture}[
    scale=0.75,
    transform shape,
    font=\small,
    >=Latex,
    node distance=8mm and 10mm,
    box/.style={draw, rounded corners, thick, align=center, minimum height=10mm, minimum width=24mm},
    agent/.style={draw, rounded corners, thick, align=center, minimum height=11mm, minimum width=26mm, fill=gray!8},
    proc/.style={draw, rounded corners, thick, align=center, minimum height=11mm, minimum width=30mm, fill=gray!4},
    data/.style={draw, rounded corners, dashed, align=center, minimum height=9mm, minimum width=22mm},
    arrow/.style={->, thick},
    feedback/.style={->, thick, dashed}
]

\node[box, minimum width=40mm] (base) {\texttt{URSA BaseAgent}\\standardized agent interface};

\node[agent, below left=14mm and 20mm of base] (planner) {Planner Agent:\\builds query $\mathsf{q}$ and context $\Gamma$};
\node[agent, below=14mm of base] (researcher) {Researcher Agent:\\generates candidate response $\mathsf{S}^t$};
\node[agent, below right=14mm and 20mm of base] (reasoner) {Reasoner Agent:\\aggregates accepted outputs};

\node[proc, below=30mm of researcher, minimum width=42mm] (carol) {\texttt{CAROL} acceptance--rejection step\\evaluate $\Delta F(\mathsf{S}^t \mid \mathcal{S}_{t-1},\Gamma)$};
\node[proc, below left=10mm and 18mm of carol, minimum width=30mm] (accept) {Accept\\append to state};
\node[proc, below right=10mm and 18mm of carol, minimum width=30mm] (reject) {Reject\\rewire / revise};

\node[data, left=14mm of carol, minimum width=24mm] (state) {Current state\\$(\mathsf{q},\Gamma,\mathsf{S}_{t-1})$};
\node[data, right=14mm of carol, minimum width=24mm] (cand) {Candidate\\$\mathsf{S}^t$};
\node[data, below=17mm of accept, minimum width=36mm] (accstate) {Updated trajectory\\$\mathcal{S}_t \leftarrow \mathcal{S}_{t-1} \oplus \mathsf{S}^t$};
\node[data, below=10mm of reject, minimum width=36mm] (revstate) {Revised prompt / state\\CoT or policy update};

\node[draw, rounded corners, ultra thick, inner sep=8mm,
      fit=(base)(planner)(researcher)(reasoner),
      label={[yshift=1mm]\textbf{\texttt{URSA} execution layer}}] (ursafit) {};

\node[draw, rounded corners, ultra thick, inner sep=8mm,
      fit=(carol)(accept)(reject)(state)(cand)(accstate)(revstate),
      label={[yshift=1mm]\textbf{\texttt{CAROL} control layer}}] (carolfit) {};

\draw[arrow] (base) -- (planner);
\draw[arrow] (base) -- (researcher);
\draw[arrow] (base) -- (reasoner);

\draw[arrow] (planner) -- node[above left, pos=0.55] {$\mathsf{q},\Gamma$} (researcher);
\draw[arrow] (researcher) -- node[right] {$\mathsf{S}^t$} (cand);
\draw[arrow] (cand) -- (carol);
\draw[arrow] (state) -- (carol);

\draw[arrow] (carol) -- node[above left, pos=0.55] {accept} (accept);
\draw[arrow] (carol) -- node[above right, pos=0.55] {reject} (reject);

\draw[arrow] (accept) -- (accstate);
\draw[arrow] (reject) -- (revstate);

\draw[arrow]
  (accstate.east) -- ++(118mm,0) |-([xshift=-0.2mm]reasoner.east);

\coordinate (accaux) at ($(accstate.south)+(0,-6mm)$);
\draw[arrow]
  (reasoner.east)
  |- ++(8mm,0)
  |- 
  (accaux)
  -- node[pos=0.25,left] {final response} (accstate.south);

\coordinate (aux1) at ([xshift=-28mm]researcher.south);
\coordinate (aux2) at ($(aux1)+(0,6mm)$);
\draw[feedback]
  (revstate.west)
  -| (aux1)
  -- (aux2)
  |- (researcher.west);

\draw[feedback] (accstate.west) -| ([xshift=-16mm]state.west) |- (state.west);
\draw[feedback] ($(planner.south)+(-12mm,0)$) |- ([xshift=-10mm]state.west) -- (state.west);

\node[align=center, above=1mm of state, font=\scriptsize] {semantic memory\\and running context};
\node[align=center, above=1mm of cand, font=\scriptsize] {Researcher output\\treated as $\ell$-gram};

\end{tikzpicture}
\caption{\texttt{URSA}--\texttt{CAROL} experimental pipeline. \texttt{URSA} provides the reproducible multi-agent execution layer through a shared \texttt{BaseAgent} abstraction, while \texttt{CAROL} acts as a control layer that evaluates each Researcher output $\mathcal{S}_t$ against the current state $(q,\Gamma,\mathsf{S})$ and either accepts it into the trajectory or rejects it and rewires the prompt/state. This representation reflects the experimental setting in which Planner, Researcher, and Reasoner agents are orchestrated in \texttt{URSA} and \texttt{CAROL} operates on each intermediate Researcher response as the semantic unit of refinement.}
\label{fig:ursa_carol_pipeline}
\end{figure*}
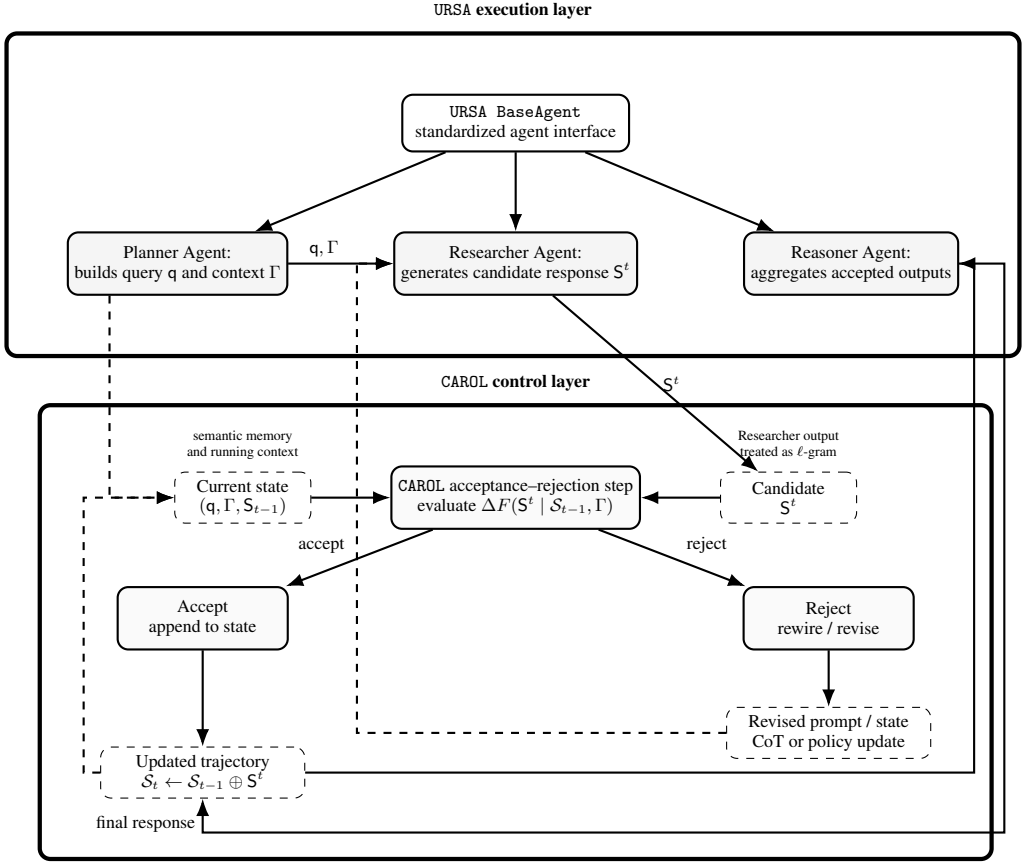

\paragraph{\textbf{Statistical Significance and Experimental Variability.}}
To ensure the reliability and reproducibility of our empirical findings, all reported results are evaluated over multiple independent runs with different random seeds controlling (i) stochastic decoding in the language model, (ii) the probabilistic acceptance--rejection mechanism of \texttt{CAROL}, and (iii) dataset shuffling. We perform $k=3$ independent runs per configuration using fixed seeds $\{0,1,2\}$. For each metric, we report the mean performance together with the standard deviation. In addition, we compute $95\%$ confidence intervals via non-parametric bootstrap resampling (with $1000$ resamples) over the evaluation set.

To assess statistical significance of performance differences between \texttt{CAROL} and baseline methods (e.g., \texttt{RAG} and \texttt{HalluciNot}), we perform paired hypothesis tests over identical evaluation samples. Specifically, we use two-sided paired $t$-tests, and confirm results with Wilcoxon signed-rank tests to account for potential non-normality. Improvements reported in Tables~\ref{tab:hallucination_comparison} and~\ref{tab:carol_vs_rag_hierarchy} are statistically significant at the $p < 0.05$ level.

We observe consistently low variance across runs (see reported standard deviations), indicating that the proposed method is robust to stochastic effects in both generation and inference-time control. This suggests that performance gains of \texttt{CAROL} are not driven by favorable sampling effects, but reflect consistent improvements across stochastic realizations.

\vspace{4pt}
\begin{center}
\scriptsize
\begin{tabular}{l l}
\toprule
\textbf{Component} & \textbf{Configuration} \\
\midrule
Number of runs ($k$) & $3$ \\
Random seeds & $\{0,1,2\}$ \\
Decoding & temperature sampling (fixed across runs) \\
\texttt{CAROL} stochasticity & acceptance--rejection sampling \\
Dataset shuffling & enabled per run \\
Bootstrap resamples & $1000$ \\
Confidence intervals & $95\%$ (non-parametric bootstrap) \\
Significance threshold & $p < 0.05$ \\
Evaluation protocol & identical samples across methods \\
\bottomrule
\end{tabular}
\end{center}

\paragraph{\textbf{Compute Resources and Runtime.}}
All experiments are conducted on a workstation equipped with an \emph{NVIDIA GeForce RTX 2080 Ti GPU} ($11$GB VRAM), an Intel Xeon-class CPU, and $64$GB of system memory. For open-weight models (e.g., \emph{Llama-3.1-8B}), inference is performed locally on the GPU, while proprietary models (e.g., \emph{OpenAI}-family) are accessed via API calls, with compute handled externally by the provider. 

End-to-end evaluation runtime varies depending on the dataset and model. On average, a full evaluation pass over a benchmark dataset (e.g., \textbf{TruthfulQA} or \textbf{HaluEval}) requires approximately a few GPU-hours for local models, with per-sample latency ranging from sub-second (baseline \texttt{RAG}) to a few seconds (\texttt{CAROL}), primarily due to the clustering step and iterative acceptance--rejection procedure. Experiments involving API-based models incur additional wall-clock time due to network latency but negligible local compute cost.

We report per-sample latency and relative computational cost in Tables~\ref{tab:hallucination_comparison} and~\ref{tab:carol_vs_rag_hierarchy}. All experiments are executed with fixed hyperparameters and deterministic seeds to ensure reproducibility. Overall, \texttt{CAROL} introduces moderate computational overhead compared to standard generation pipelines, but remains tractable within commonly available research hardware. Future optimizations, such as precomputing embeddings of $\Gamma$ and approximate distance evaluations, can further reduce runtime.

\newpage
\section{Datasets and External Resources}
\label{app:datasets}

We evaluate \texttt{CAROL} on four widely used benchmarks covering factual verification, hallucination detection, and multi-hop reasoning: \textbf{FEVER}, \textbf{TruthfulQA}, \textbf{HaluEval}, and \textbf{HotPotQA}. Each dataset provides a query (or claim), optional supporting context, and a ground-truth label or reference answer. In our framework, the query is used to construct the prompt $\rho$, while the retrieved or provided context is mapped into the trusted axiom set $\Gamma$. Below we summarize the structure of each dataset in a standardized format dataset cards.

\begin{tcolorbox}[breakable,colback=gray!5,colframe=gray!50,title=\textbf{FEVER} Dataset Structure]
\textbf{Fields:}
\begin{itemize}
    \item \texttt{claim}: textual statement to be verified
    \item \texttt{evidence}: list of supporting/refuting sentences (Wikipedia)
    \item \texttt{label}: \{\texttt{SUPPORTS}, \texttt{REFUTES}, \texttt{NOT ENOUGH INFO}\}
\end{itemize}
\textbf{Usage in \texttt{CAROL}:}
\begin{itemize}
    \item Query $q$: \texttt{claim}
    \item Context $\Gamma$: retrieved evidence sentences
    \item Target: detect and reduce hallucinated statements
\end{itemize}
\end{tcolorbox}

\begin{tcolorbox}[breakable,colback=gray!5,colframe=gray!50,title=\textbf{TruthfulQA} Dataset Structure]
\textbf{Fields:}
\begin{itemize}
    \item \texttt{question}: natural language query
    \item \texttt{best\_answer}: ground-truth correct answer
    \item \texttt{incorrect\_answers}: list of common misconceptions
\end{itemize}
\textbf{Usage in \texttt{CAROL}:}
\begin{itemize}
    \item Query $q$: \texttt{question}
    \item Context $\Gamma$: retrieved factual sources (external or curated)
    \item Target: minimize alignment with incorrect answers while maximizing semantic consistency
\end{itemize}
\end{tcolorbox}

\begin{tcolorbox}[breakable,colback=gray!5,colframe=gray!50,title=\textbf{HaluEval} Dataset Structure]
\textbf{Fields:}
\begin{itemize}
    \item \texttt{query}: input prompt (QA, summarization, or dialogue)
    \item \texttt{response}: generated output
    \item \texttt{label}: \{\texttt{hallucinated}, \texttt{factual}\}
\end{itemize}
\textbf{Usage in \texttt{CAROL}:}
\begin{itemize}
    \item Query $q$: \texttt{query}
    \item Context $\Gamma$: task-specific retrieved documents
    \item Target: classify and mitigate hallucinated outputs
\end{itemize}
\end{tcolorbox}

\begin{tcolorbox}[breakable,colback=gray!5,colframe=gray!50,title=\textbf{HotPotQA} Dataset Structure]
\textbf{Fields:}
\begin{itemize}
    \item \texttt{question}: multi-hop query
    \item \texttt{context}: supporting paragraphs (Wikipedia)
    \item \texttt{answer}: ground-truth answer
\end{itemize}
\textbf{Usage in \texttt{CAROL}:}
\begin{itemize}
    \item Query $q$: \texttt{question}
    \item Context $\Gamma$: supporting facts and retrieved passages
    \item Target: evaluate hallucination in multi-hop reasoning
\end{itemize}
\end{tcolorbox}

\paragraph{\textbf{Access and Loading.}}
All datasets are publicly available and can be accessed through standard repositories such as \emph{HuggingFace} Datasets. A typical loading pattern is:
\begin{verbatim}
from datasets import load_dataset

fever = load_dataset("fever", split="validation")
truthfulqa = load_dataset("truthful_qa", "generation")
halueval = load_dataset("halueval")
hotpotqa = load_dataset("hotpot_qa", "fullwiki")
\end{verbatim}

\paragraph{\textbf{\texttt{HalluciNot} Baseline.}}
We compare against the \texttt{HalluciNot} hallucination detection framework, which combines contextual verification with external knowledge sources. The method is available as an open-source implementation accompanying the original publication. It can be obtained from public repositories (e.g., GitHub) or via the authors’ release page associated with the paper \cite{BP-AL-PJ-PA:25}.

In practice, \texttt{HalluciNot} is used as a post-hoc detector:
\begin{verbatim}
from hallucination_detector import HalluciNot

detector = HalluciNot()
score = detector.predict(response, context)
\end{verbatim}

The model evaluates whether a generated response is supported by the provided context and common knowledge. In our experiments, we use \texttt{HalluciNot} as a baseline for hallucination detection, comparing its predictions with the semantic entropy metric induced by \texttt{CAROL}. While \texttt{HalluciNot} relies on external verification signals, \texttt{CAROL} operates through an intrinsic consistency mechanism with respect to $\Gamma$, enabling a unified detection and mitigation framework.

\paragraph{\textbf{Licenses, Data Usage, and External Assets.}}
All datasets and models used in this work are publicly available and employed in accordance with their respective licenses and usage terms. Specifically, \textbf{FEVER}, \textbf{TruthfulQA}, \textbf{HaluEval}, and \textbf{HotPotQA} are distributed for research purposes and are commonly accessed through open repositories such as \emph{HuggingFace} Datasets. These datasets are used strictly for evaluation and benchmarking, without modification of their original content beyond standard preprocessing.

Regarding models, experiments involving proprietary APIs (e.g., \emph{OpenAI}-family models) are conducted under the terms of service of the respective providers, while open-weight models (e.g., \emph{Llama-3.1-8B}) are used in compliance with their research and non-commercial licenses. The \texttt{URSA} framework is open-source and utilized under its corresponding license, ensuring reproducibility of the agentic pipeline. 

We do not introduce new datasets, collect personal data, or involve human subjects in this work. All experiments rely exclusively on publicly available resources. Users of \texttt{CAROL} should ensure that downstream applications respect the licensing constraints of both the underlying models and the contextual data used to construct $\Gamma$, particularly in commercial or sensitive deployments.

\end{document}